\documentclass[10pt]{article}
\usepackage[letterpaper,margin=1in]{geometry}

\usepackage[utf8]{inputenc} %
\usepackage[T1]{fontenc}    %
\usepackage{url}            %
\usepackage{booktabs}       %
\usepackage{amsfonts}       %
\usepackage{nicefrac}       %
\usepackage{microtype}      %
\usepackage{xcolor}         %

\usepackage{microtype}
\usepackage{graphicx}
\usepackage{booktabs} %

\usepackage{graphicx,amsmath,amsfonts,amscd,amssymb,bm,url,color,latexsym,amsthm}
\usepackage{mathtools}
\usepackage[T1]{fontenc}
\usepackage{tgpagella}
\usepackage{fullpage}
\usepackage[small,bf]{caption}
\usepackage{subcaption}
\usepackage{bbm}
\usepackage{microtype}
\usepackage{algorithm}
\usepackage{algorithmicx}
\usepackage{algpseudocode} 
\usepackage{verbatim}
\usepackage{framed}
\usepackage{ulem}
\usepackage{tikz}
\usepackage{fge}
\usepackage{mathdots} 
\usepackage[toc, page]{appendix}
\usepackage{etoc}
\usepackage{scalerel,stackengine} %
\usepackage{bbding}
\usepackage{overpic}
\usepackage{caption}
\usepackage{subcaption}
\stackMath

\usepackage{hyperref}
\numberwithin{equation}{section}
\usepackage[capitalize]{cleveref}
\crefname{certprob}{Problem}{Problems}

\crefname{equation}{}{}

\allowdisplaybreaks
\usepackage[toc,page]{appendix}

\usepackage[
backend=biber,
style=numeric-comp,
sorting=ynt,
doi=false,
hyperref=true,
isbn=false,
url=false,
maxbibnames=99
]{biblatex}
\renewbibmacro{in:}{}
\AtEveryBibitem{%
	\clearlist{language}%
}

\allowdisplaybreaks  
\hypersetup{
    colorlinks=true,%
    citecolor=blue,%
    filecolor=blue,%
    linkcolor=blue,%
    urlcolor=blue,
    linktocpage = true
}

\newtheorem{theorem}{Theorem}[section]
\newtheorem{lemma}[theorem]{Lemma}

\newtheorem{definition}[theorem]{Definition}

\newtheorem{remark}[theorem]{Remark}

\newtheorem*{certprob}{Certificate Problem}

\renewenvironment{proof}{\noindent {\bf Proof.} }{\endprf\par}
\def \endprf{\hfill {\vrule height6pt width6pt depth0pt}\medskip}

\renewcommand{\mathbf}{\boldsymbol} 
\newcommand{\mb}{\mathbf}
\newcommand{\mf}{\mathfrak}

\newcommand{\bb}{\mathbb}

\newcommand{\R}{\bb R}
\newcommand{\Z}{\bb Z}
\newcommand{\N}{\bb N}

\newcommand{\indicator}[1]{\mathbbm 1_{#1}}

\newcommand{\paren}[1]{ \left( #1 \right) }
\newcommand{\ceil}[1]{\left\lceil #1 \right\rceil}
\newcommand{\floor}[1]{\left\lfloor #1 \right\rfloor}

\makeatletter
\newcommand{\subalign}[1]{%
  \vcenter{%
    \Let@ \restore@math@cr \default@tag
    \baselineskip\fontdimen10 \scriptfont\tw@
    \advance\baselineskip\fontdimen12 \scriptfont\tw@
    \lineskip\thr@@\fontdimen8 \scriptfont\thr@@
    \lineskiplimit\lineskip
    \ialign{\hfil$\m@th\scriptstyle##$&$\m@th\scriptstyle{}##$\crcr
      #1\crcr
    }%
  }
}
\makeatother

\newcommand\widecheck[1]{%
\savestack{\tmpbox}{\stretchto{%
  \scaleto{%
    \scalerel*[\widthof{\ensuremath{#1}}]{\kern-.6pt\bigwedge\kern-.6pt}%
    {\rule[-\textheight/2]{1ex}{\textheight}}%
  }{\textheight}%
}{0.5ex}}%
\stackon[1pt]{#1}{\scalebox{-1}{\tmpbox}}%
}

\newcommand{\fco}{\text{\FourClowerOpen}}

\input{macros/sam_macros.def}
\input{macros/macros_notation.def}

\author{%
  Tingran Wang\thanks{Corresponding author: \href{mailto:tw2579@columbia.edu}{\color{black}tw2579@columbia.edu}}\,\,\,\thanks{Department of Electrical Engineering, Columbia
  University}\,\,\,\thanks{Data Science Institute, Columbia University} \qquad %
  Sam Buchanan\footnotemark[2]\,\,\,\footnotemark[3]\,\,\, \qquad %
  Dar Gilboa\thanks{{Center for Brain Science, Harvard University}} \qquad %
  John Wright\footnotemark[2]\,\,\,\footnotemark[3]\,\,\,\thanks{Department of Applied Physics and Applied Mathematics, Columbia University}
}

\begin{document}
\title{Deep Networks Provably Classify Data on Curves}
\maketitle

\begin{abstract} 
	Data with low-dimensional nonlinear structure are ubiquitous in engineering and scientific problems. %
We study a model problem with such structure---%
a binary classification task that uses a deep fully-connected neural network to classify data drawn from two disjoint smooth curves on the unit sphere. %
Aside from mild regularity conditions, we place no restrictions on the configuration of the curves. %
We prove that when (i) the network depth is large relative to certain geometric %
properties that set the difficulty of the problem %
and (ii) the network width and number of samples are polynomial in the depth, randomly-initialized gradient descent quickly learns to correctly classify all points on the two curves with high probability. %
To our knowledge, this is the first generalization guarantee for deep networks with nonlinear data that depends only on intrinsic data properties.
Our analysis proceeds by a reduction to dynamics in the neural tangent kernel (NTK) regime, where the network depth plays the role of a fitting resource in solving the classification problem. In particular, via fine-grained control of the decay properties of the NTK, we demonstrate that when the network is sufficiently deep,
the NTK can be locally approximated by a translationally invariant operator on the manifolds and stably inverted over smooth functions, which guarantees convergence and generalization.

\end{abstract}

\etocdepthtag.toc{mtchapter}
\etocsettagdepth{mtchapter}{subsection}
\etocsettagdepth{mtappendix}{none}

\section{Introduction}
\label{sec:intro}

In applied machine learning, engineering, and the sciences, we are frequently confronted with the
problem of identifying low-dimensional structure in high-dimensional data. In certain
well-structured data sets, identifying a good low-dimensional model is the principal task:
examples include convolutional sparse models in microscopy \cite{Cheung2020-ft} and neuroscience
\cite{Ekanadham2011-aa,Chung2017-qc}, and low-rank models in collaborative filtering
\cite{Bell2007-va,Koren2009-jh}. Even more complicated datasets from problems such as image
classification exhibit some form of low-dimensionality: recent experiments estimate the effective
dimension of CIFAR-10 as 26 and the effective dimension of ImageNet as 43 \cite{pope-2021-the}.
The variability in these datasets can be thought of as comprising two parts: a ``probabilistic''
variability induced by the distribution of geometries associated with a given class, and a
``geometric'' variability associated with physical nuisances such as pose and illumination. The
former is challenging to model analytically; virtually all progress on this issue has come through
the introduction of large datasets and high-capacity learning machines. The latter induces a much
cleaner analytical structure: transformations of a given image lie near a low-dimensional
submanifold of the image space (\Cref{fig:two_man_two_curve}). The celebrated successes of convolutional
neural networks in image classification seem to derive from their ability to simultaneously handle
both types of variability. Studying how neural networks compute with data lying near a low-dimensional manifold is an
essential step towards understanding how neural networks achieve invariance to continuous
transformations of the image domain, and towards the longer term goal of developing a more
comprehensive mathematical understanding of how neural networks compute with real data. At the
same time, in some scientific and engineering problems, classifying manifold-structured data {\em
is} the goal---one example is in gravitational wave astronomy \cite{gabbard2018matching,
gebhard2019convolutional}, where the goal is to distinguish true events from noise, and the events
are generated by relatively simple physical systems with only a few degrees of freedom.

\newcommand*{\subfigsize}{0.20}
\begin{figure}[!htb]
  \begin{subfigure}[b]{0.5\textwidth}
			\centering
      \includegraphics[width=\textwidth]{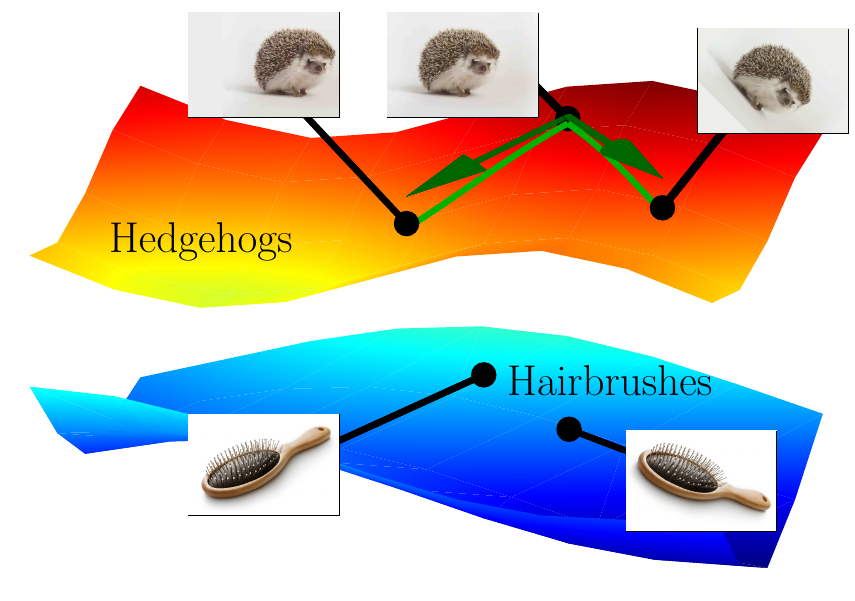}
\label{fig:two_man}
	\end{subfigure}%
  \begin{subfigure}[b]{0.5\textwidth}
			\centering
      \includegraphics[width=\textwidth]{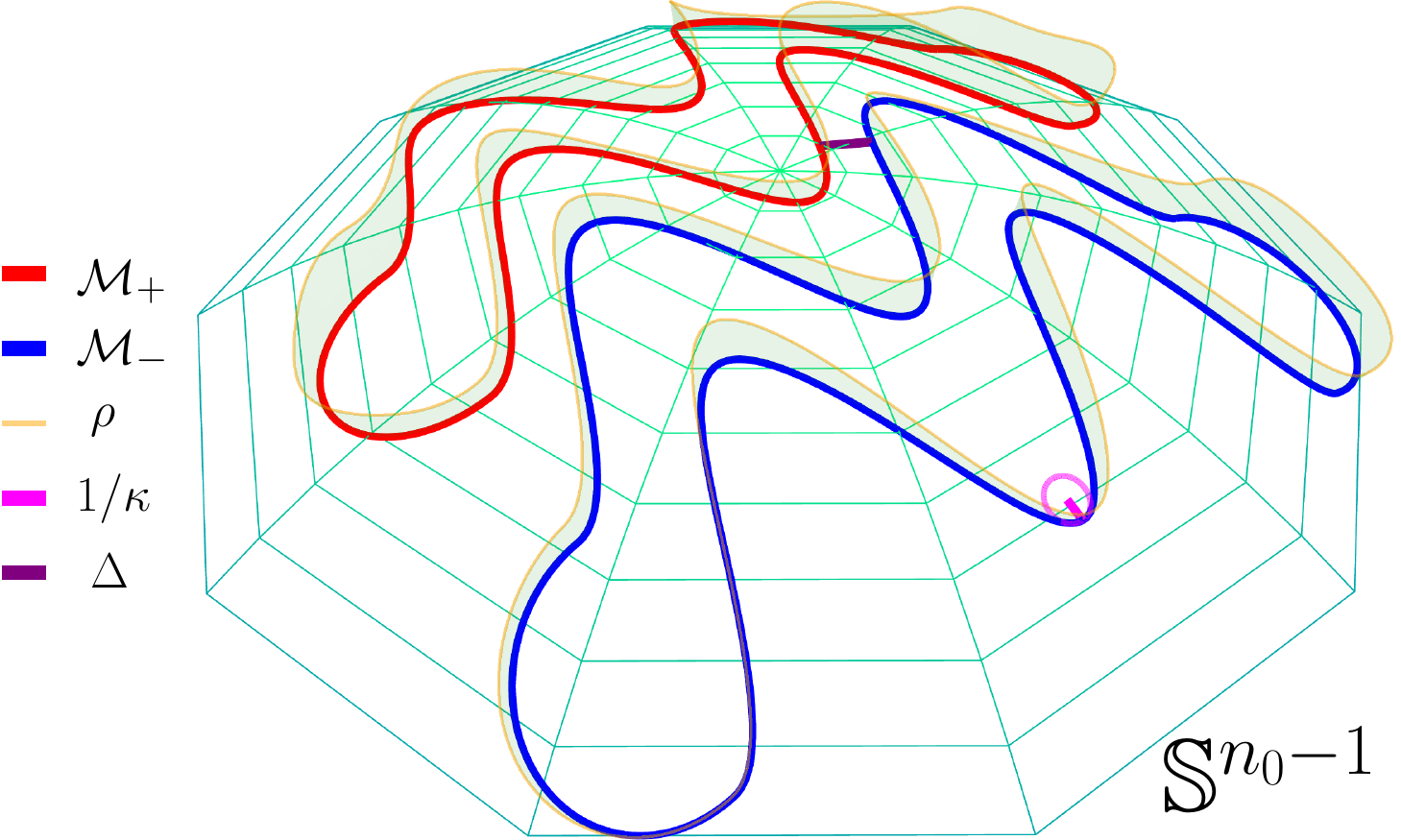}
		\label{fig:two_curve}
	\end{subfigure}
  \caption{\textbf{Low-dimensional structure in image data and the two curves problem.
    }\textit{Left:} Manifold structure in natural images arises due to invariance of the label to
    continuous domain transformations such as translations and rotations. \textit{Right:} The two
    curve problem. We train a neural network to classify points sampled from a density $\rho$ on
    the submanifolds $\mc M_+, \mc M_-$ of the unit sphere. We illustrate the angle injectivity radius $\Delta$ and curvature $1/\kappa$. These parameters help to control the difficulty
    of the problem: problems with smaller separation and larger curvature are more readily
    separated with deeper networks.
  }
	\label{fig:two_man_two_curve}
\end{figure}

Motivated by these long term goals, in this paper we study the {\em multiple manifold problem} (\Cref{fig:two_man_two_curve}), a mathematical model problem in which we are presented with a finite set of labeled samples lying on disjoint low-dimensional submanifolds of a high-dimensional space, and the goal is to correctly classify {\em every} point on each of the submanifolds---a strong form of generalization. 
The central mathematical question is how the {\em structure} of the data (properties of the manifolds such as dimension, curvature, and separation) influences the {\em resources} (data samples, and network depth and width) required to guarantee generalization. 
Our main contribution is the first end-to-end analysis of this problem for a nontrivial class of manifolds: one-dimensional smooth curves that are non-intersecting, cusp-free, and without antipodal pairs of points.
Subject to these constraints, the curves can be oriented essentially arbitrarily (say, non-linearly-separably, as in \Cref{fig:two_man_two_curve}), and the hypotheses of our results depend only on architectural resources and intrinsic geometric properties of the data.
To our knowledge, this is the first generalization result for training a deep nonlinear network to classify structured data that makes no a-priori assumptions about the representation capacity of the network or about properties of the network after training. 

Our analysis proceeds in the neural tangent kernel (NTK) regime of training, where the network is wide enough to guarantee that gradient descent can make large changes in the network output while making relatively small changes to the network weights. 
This approach is inspired by the recent work \cite{Buchanan2020-er}, which reduces the analysis of generalization in the one-dimensional multiple manifold problem to an auxiliary problem called the \textit{certificate problem}. 
Solving the certificate problem amounts to proving that the target label function lies near the stable range of the NTK. %
The existence of certificates (and more generally, the conditions under which practically-trained neural networks can fit structured data) is open, except for a few very simple geometries which we will review below---in particular, \cite{Buchanan2020-er} leaves this question completely open. 
Our technical contribution is to show that setting the network depth sufficiently large relative to intrinsic properties of the data guarantees the existence of a certificate (\Cref{thm:certificate_main}), resolving the one-dimensional case of the multiple manifold problem for a broad class of curves (\Cref{thm:main_formal}). 
This leads in turn to a novel perspective on the role of the network depth as a \textit{fitting resource} in the classification problem, which is inaccessible to shallow networks.

\subsection{Related Work}
\label{sec:related_work}

\paragraph{Deep networks and low dimensional structure.}
Modern applications of deep neural networks include numerous examples of low-dimensional manifold structure, including pose and illumination variations in image classification \cite{belhumeur1998set, donoho2005image},
as well as detection of structured signals such as electrocardiograms \cite{xiong2016stacked, Antczak2018-sv}, gravitational waves \cite{gabbard2018matching, gebhard2019convolutional}, audio signals \cite{laufer2016semi}, and solutions to the diffusion equation \cite{kadeethum2020physics}. 
Conventionally, to compute with such data one might begin by extracting a low-dimensional representation using nonlinear dimensionality reduction (``manifold learning'') algorithms  \cite{Saul2020-ny,Tenenbaum2000-ct,Agrawal2021-ob,Belkin2003-xz,Roweis2000-sh,Cayton2006-lg,Wang2015-hh}. For supervised tasks, there is also theoretical work on kernel regression over manifolds \cite{Fuselier2012-ak,Altschuler2018-qf,Hangelbroek2010-ls,McRae2020-nh}. 
These results rely on very general Sobolev embedding theorems, which are not precise enough to specify the interplay between regularity of the kernel and properties of the data need to obtain concrete resource tradeoffs in the two curve problem. 
There is also a literature which studies the resource requirements associated with approximating functions over low-dimensional manifolds  \cite{Chen2019-eq,Cloninger2020-ov,Schmidt-Hieber2019-jt,Basri2017-vn}: a typical result is that for a sufficiently smooth function there {\em exists} an approximating network whose complexity is controlled by intrinsic properties such as the dimension. 
In contrast, we seek \textit{algorithmic} guarantees that prove that we can efficiently train deep neural networks for tasks with low-dimensional structure. This requires us to grapple with how the geometry of the data influences the dynamics of optimization methods.

\paragraph{Neural networks and structured data---theory?}
Spurred by insights in asymptotic {infinite width} \cite{Jacot2018-dv, Mei2018-ck} and
non-asymptotic \cite{Allen-Zhu2018-wz, Du2018-pt} settings, there has been a surge of recent
theoretical work aimed at establishing guarantees for neural network training and generalization
\cite{Allen-Zhu2018-rf,arora2019fine,Ji2019-cl,Oymak2019-pn,Cao2019-cm,Suzuki2020-zq,Li2020-oj,Allen-Zhu2020-jn}.
Here, our interest is in end-to-end generalization guarantees, %
which are scarce in the literature: those that exist pertain to unstructured data with general targets, in the regression setting \cite{Liang2019-mo, Mei2021-qg, Ghorbani2020-ef, ghorbani2019linearized}, and those that involve low-dimensional structure consider only linear structure (i.e., spheres) \cite{Ghorbani2020-ef}. 
For less general targets, there exist numerous works that pertain to the teacher-student setting,
where the target is implemented by a neural network of suitable architecture with unstructured
inputs \cite{Allen-Zhu2018-id, Allen-Zhu2020-jn, Li2020-oj, Goldt2019-rp, Zhou2021-jf}.  Although
adding this extra structure to the target function allows one to establish interesting separations
in terms of e.g.\ sample complexity
\cite{Li2020-oj,Yehudai2019-ro,Ghorbani2019-ji,Refinetti2021-rh} relative to the preceding
analyses, which proceed in the ``kernel regime'', we leverage kernel regime techniques in our
present work because they allow us to study the interactions between deep networks and data with
nonlinear low-dimensional structure, which is not possible with existing teacher-student tools. %
Relaxing slightly from results with end-to-end guarantees, there exist `conditional' guarantees which require the existence of an efficient representation of the target mapping in terms of a certain RKHS associated to the neural network \cite{Ji2019-cl,nitanda2020optimal,Buchanan2020-er,Chen2021-mv}. In contrast, our present work obtains unconditional, end-to-end generalization guarantees for a nontrivial class of low-dimensional data geometries.

\section{Problem Formulation}

\paragraph{Notation.}
We use bold notation $\vx$, $\vA$ for vectors and matrices/operators (respectively).  
We write $\norm{\vx}_p = (\sum_{i=1}^n \abs{x_i}^p)^{1/p}$
for the $\ell^p$ norm of $\vx$, $\ip{\vx}{\vy} = \sum_{i=1}^n x_i y_i$ for the euclidean
inner product, and 
for a measure space $(X,\mu)$, $\norm{g}_{L^p_\mu}
= (\int_{X} \abs{g(x)}^p\diff \mu(x) )^{1/p}$ denotes the $L^p_\mu$ norm of a function $g: X \to
\bbR$. 
The unit sphere in $\bbR^n$ is denoted $\Sph^{n-1}$,
and $\angle(\vx,\vy) = \acos(\ip{\vx}{\vy})$
denotes the angle between unit vectors.
For a kernel $K : X \times X \to \bbR$, we write $\vK_\mu[g](x) = \int_{X} K(x, x')
g(x')\diff \mu(x') $ for the action of the associated Fredholm integral operator; an omitted
subscript denotes Lebesgue measure. 
We write $\vP_S$ to denote the orthogonal projection operator onto a (closed)
subspace $S$.
Full notation is provided in \Cref{sec:notation_general}.

\subsection{\texorpdfstring{The Two Curve Problem\protect\footnote{%
The content of this section follows the presentation of \cite{Buchanan2020-er}; we reproduce it
here for self-containedness. We omit some nonessential definitions and derivations for concision;
see \Cref{sec:prob_form_extended} for these details.  }}{The Two Curve Problem}}
\label{sec:two_curves}

A natural model problem for the tasks discussed in \Cref{sec:intro} is the classification of low-dimensional submanifolds using a neural network. 
In this work, we study the one-dimensional, two-class case of this problem, which we refer to as the {\em two curve problem}. 
To fix ideas, let $\adim \geq 3$ denote the ambient dimension, and let $\mplus$ and $\mminus$ be two disjoint smooth regular simple closed curves taking values in $\Sph^{\adim -1}$, which represent the two classes (\Cref{fig:two_man_two_curve}). 
In addition, we require that the curves lie in a spherical cap of radius $\pi/2$: %
for example, the intersection of the sphere and the nonnegative orthant $\set{\vx \in \bbR^{\adim} \given \vx \geq 0}$.%
\footnote{The specific value $\pi/2$ is immaterial to our arguments: this constraint is only to avoid technical issues that arise when antipodal points are present in $\manifold$, so any constant less than $\pi$ would work just as well. This choice allows for some extra technical expediency, and connects with natural modeling assumptions (e.g.\ data corresponding to image manifolds, with nonnegative pixel intensities).} %
Given $N$ i.i.d.\ samples $\set{\vx_i}_{i=1}^N$ from a density $\rho$ supported on $\manifold = \mplus \cup \mminus$, which is bounded above and below by positive constants $\rho_{\max}$ and $\rho_{\min}$ and has associated measure $\mu$, as well as their corresponding $\pm 1$ labels, we train a feedforward neural network $f_{\vtheta} : \bbR^{\adim} \to \bbR$ with ReLU nonlinearities, uniform width $n$, and depth $L$ (and parameters $\vtheta$) by minimizing the empirical mean squared error using randomly-initialized gradient descent. 
Our goal is to prove that this procedure yields a separator for the geometry given sufficient resources $n$, $L$, and $N$---i.e., that $\sign(f_{\vtheta_k}) = 1$ on $\mplus$ and $-1$ on $\mminus$ at some iteration $k$ of gradient descent.

To achieve this, we need an understanding of the progress of gradient descent. 
Let $f_{\star} : \manifold \to \set{\pm 1}$ denote the classification function for $\mplus$ and $\mminus$ that generates our labels, write $\zeta_{\vtheta}(\vx) = f_{\vtheta}(\vx) - f_{\star}(\vx)$ for the network's prediction error, and let $\vtheta_{k+1} = \vtheta_{k} - (\tau/N) \sum_{i=1}^N \zeta_{\vtheta_k}(\vx_i) \nabla_{\vtheta}f_{\vtheta_k}(\vx_i) $ denote the gradient descent parameter sequence, where $\tau>0$ is the step size and $\vtheta_0$ represents our Gaussian initialization.
Elementary calculus then implies the error dynamics equation
$\zeta_{\vtheta_{k+1}} = \zeta_{\vtheta_k} - (\tau /N )
\sum_{i=1}^N \Theta_k^N(\spcdot, \vx_i) \zeta_{\vtheta_k}(\vx_i)$ for $k = 0, 1, \dots$,
where $\Theta_k^N : \manifold \times \manifold \to \bbR$ is a certain kernel. %
The precise expression for this kernel is not important for our purposes: what matters is that (i) making the width $n$ large relative to the depth $L$ guarantees that $\Theta_k^N$ remains close throughout training to its `initial value' %
$\Theta^{\mathrm{NTK}}(\vx,\vx') = \ip{\nabla_{\vtheta} f_{\vtheta_0}(\vx)}{\nabla_{\vtheta} f_{\vtheta_0}(\vx')}$, the {\em neural tangent kernel};
and (ii) taking the sample size $N$ to be sufficiently large relative to the depth $L$ implies that a nominal error evolution defined as $\zeta_{k+1} = \zeta_k - \tau \vTheta^{\mathrm{NTK}}_{\mu}[\zeta_k]$ with $\zeta_0 = \zeta_{\vtheta_0}$ uniformly approximates the actual error $\zeta_{\vtheta_k}$ throughout training.
In other words: to prove that gradient descent yields a neural network classifier that separates the two manifolds, it suffices to overparameterize, sample densely, and show that the norm of $\zeta_k$ decays sufficiently rapidly with $k$. This constitutes the ``NTK regime'' approach to gradient descent dynamics for neural network training \cite{Jacot2018-dv}.

The evolution of $\zeta_k$ is relatively straightforward: we have $\zeta_{k+1} = ( \Id - \tau \vTheta^{\mathrm{NTK}}_{\mu})^k[\zeta_0]$, and $\vTheta^{\mathrm{NTK}}_{\mu}$ is a positive, compact operator, so there exist an orthonormal basis of $L^2_{\mu}$ 
functions $v_i$ and eigenvalues $\lambda_1 \geq \lambda_2 \geq \dots \geq 0$ such that $\zeta_{k+1} = \sum_{i=1}^{\infty} (1 - \tau \lambda_i)^{k} \ip{\zeta_0}{v_i}_{L^2_{\mu}} v_i$.
In particular, with bounded step size $\tau < \lambda_1^{-1}$, gradient descent leads to rapid decrease of the error if and only if the initial error $\zeta_0$ is well-aligned with the eigenvectors of $\vTheta^{\mathrm{NTK}}_{\mu}$ corresponding to large eigenvalues. Arguing about this alignment explicitly is a challenging problem in geometry: although closed-form expressions for the functions $v_i$ exist in cases where $\manifold$ and $\mu$ are particularly well-structured, \textit{no such expression is available for general nonlinear geometries}, even in the one-dimensional case we study here. 
However, this alignment can be guaranteed \textit{implicitly} if one can show there exists a function $g : \manifold \to \bbR$ of small $L^2_{\mu}$ norm such that $\vTheta^{\mathrm{NTK}}_{\mu}[g] \approx \zeta_0$---in this situation, most of the energy of $\zeta_0$ must be concentrated on directions corresponding to large eigenvalues. We call the construction of such a function the {certificate problem} \cite[Eqn. (2.3)]{Buchanan2020-er}:
\begin{certprob}
Given a two curves problem instance $(\manifold, \rho)$, 
find conditions on the architectural hyperparameters $(n, L)$ so that there exists $g : \manifold \to \bbR$ satisfying
$ \norm{\vTheta^{\mathrm{NTK}}_{\mu}[g] - \zeta_0}_{L^2_{\mu}} \ltsim 1 / L$ and $\norm{g}_{L^2_{\mu}} \ltsim 1 / n$,
with constants depending on the density $\rho$ and logarithmic factors suppressed.
\label{prob:certprob}
\end{certprob}
The construction of certificates demands a fine-grained understanding of the integral operator $\vTheta_{\mu}^{\mathrm{NTK}}$ and its interactions with the geometry $\manifold$. 
We therefore proceed by identifying those intrinsic properties of $\manifold$ that will play a role in our analysis and results.

\subsection{Key Geometric Properties}\label{sec:geometry_def}

In the NTK regime described in \Cref{sec:two_curves}, gradient descent makes rapid progress if there exists a small certificate $g$ satisfying $\vTheta^{\mathrm{NTK}}_{\mu}[g] \approx \zeta_0$. 
The NTK is a function of the network width $n$ and depth $L$---in particular, we will see that the depth $L$ serves as a fitting resource, enabling the network to accommodate more complicated geometries. 
Our main analytical task is to establish relationships between these architectural resources and the intrinsic geometric  properties of the manifolds that guarantee existence of a certificate.

Intuitively, one would expect it to be
harder to separate curves that are close together or
oscillate wildly. In this section, we formalize these intuitions in terms of the curves' {\em curvature}, and quantities which we term the {\em angle injectivity radius} and {\em $\fco$-number}, which control the separation between the curves and their tendency to self-intersect. 
Given that the curves are regular, we may parameterize the two curves at unit speed with respect to arc length: for
$\sigma \in \{\pm\}$, we write $\len(\manifold_{\sigma})$ to denote the length of each curve, 
and use $\mb x_{\sigma}(s): [0, \len(\sM_{\sigma})] \to \mathbb{S}^{n_0-1}$ to represent these parameterizations.
We let $\mb x_{\sigma}^{(i)}(s)$ denote the $i$-th derivative of
$\mb x_{\sigma}$ with respect to arc length. Because our parameterization is unit speed, $\| \mb
x_{\sigma}^{(1)}(s)\|_2 = 1$ for all $\mb x_{\sigma}(s) \in \manifold$.
We provide full details regarding this parameterization in \Cref{sec:geo_defs}.

\paragraph{Curvature and Manifold Derivatives.} 

Our curves $\mc M_\sigma$ are submanifolds of the sphere $\bb S^{n_0 - 1}$. The curvature of $\mc M_\sigma$ at a point $\mb x_\sigma(s)$ is the norm $\| \mb P_{\mb x_\sigma(s)^\perp} \mb x_\sigma^{(2)}(s) \|_2$ of the component $\mb P_{\mb x_\sigma(s)^\perp} \mb x_\sigma^{(2)}(s)$ of the second derivative of $\mb x_\sigma(s)$ that lies tangent to the sphere $\bb S^{n_0-1}$ at $\mb x_\sigma(s)$. Geometrically, this measures the extent to which the curve $\mb x_\sigma(s)$ deviates from a geodesic (great circle) on the sphere.
Our technical results are phrased in terms of the {\em maximum curvature}  %
  $\kappa = \sup_{\sigma, s} \{ \| \mb P_{\mb x_{\sigma}(s)^{\perp}} \mb x_{\sigma}^{(2)}(s)\|_2 \} $.
In stating results, we also use $\hat{\kappa} = \max\{ \kappa, \tfrac{2}{\pi} \}$ to simplify various dependencies on $\kappa$. When $\kappa$ is large, $\mc M_\sigma$ is highly curved, and we will require a larger network depth $L$. 
In addition to the maximum curvature $\kappa$, our technical arguments require $\mb x_\sigma(s)$ to be five times continuously differentiable, and use bounds 
$M_i = \sup_{\sigma, s} \{ \| \mb x_{\sigma}^{(i)}(s) \|_{2} \}$
on their higher order derivatives.

\paragraph{Angle Injectivity Radius.}
Another key geometric quantity that determines the hardness of the problem is the
separation between manifolds: the problem is more difficult when $\mc M_+$ and $\mc M_-$ are close
together. We measure closeness through the extrinsic distance (angle) $\angle(\mb x, \mb x') =
\cos^{-1}\innerprod{\mb x}{\mb x'}$ between $\mb x$ and $\mb x'$ over the sphere. In contrast, we
use $d_{\manifold}(\mb x, \mb x')$ to denote the intrinsic distance between $\mb x$ and $\mb x'$
on $\manifold$, setting $d_{\manifold}(\mb x,\mb x') = \infty$ if $\mb x$ and $\mb x'$ reside on
different components $\mc M_+$ and $\mc M_-$. We set
\begin{equation}
  \Delta = \inf_{\mb x, \mb x' \in \manifold} \set{ \angle(\mb x, \mb x') \mid d_{\mc M}( \mb x, \mb
  x') \ge \tau_1 }, \label{eq:injective_radius_def}
\end{equation}
where $\tau_1 = \frac{1}{\sqrt{20}\hat{\kappa}}$, and call this quantity the {\em angle injectivity radius}. 
In words, the angle injectivity radius is the minimum angle between two points whose intrinsic distance exceeds $\tau_1$. 
The angle injectivity radius $\Delta$ (i) lower bounds the distance between different components $\mc M_+$ and $\mc M_-$, {\em and} (ii) accounts for the possibility that a component will ``loop back,'' exhibiting points with large intrinsic distance but small angle. 
This phenomenon is important to account for: the certificate problem is harder when one or both components of $\mc M$ nearly self-intersect. 
At an intuitive level, this increases the difficulty of the certificate problem because it
introduces nonlocal correlations across the operator $\vTheta^{\mathrm{NTK}}_\mu$, hurting its
conditioning. As we will see in \Cref{sec:proof_sketch}, increasing depth $L$ makes
$\Theta^{\mathrm{NTK}}$ better localized; setting $L$ sufficiently large relative to $\Delta^{-1}$
compensates for these correlations. 

\paragraph{\FourClowerOpen-number}

The conditioning of $\vTheta^{\mathrm{NTK}}_\mu$ depends not only on how near $\mc M$ comes to intersecting itself, which is captured by $\Delta$, but also on {\em the number of times} that $\mc M$ can ``loop back'' to a particular point. If $\mc M$ ``loops back'' many times, $\vTheta^{\mathrm{NTK}}_\mu$ can be highly correlated, leading to a hard certificate problem. The {\em \FourClowerOpen-number} (verbally, ``clover number'') reflects the number of near self-intersections:
\begin{equation}
  \text{\FourClowerOpen}(\mc M) = \sup_{\mb x\in\mc M} \left\{ N_{\mc M}\left(
    \set*{ \mb x' \given
      d_{\mc M}(\mb x,\mb x')\ge\tau_1,
      \angle(\mb x,\mb x')\le\tau_2
    },
    \frac{1}{\sqrt{1+\kappa^{2}}}
  \right) \right\}
  \label{eq:winding_def}
\end{equation}
with $\tau_2 = \frac{19}{20\sqrt{20}\hat{\kappa}}$. The set $\set*{ \mb x' \given d_{\mc M}(\mb x,\mb x')\ge \tau_1, \angle(\mb x,\mb x')\le \tau_2}$ is the union of looping pieces, namely points that are close to $\mb x$ in extrinsic distance but far in intrinsic distance. 
$N_{\mc M}( T, \delta )$ is the cardinality of a minimal $\delta$ covering of $T \subset \manifold$ in the intrinsic distance on the manifold, serving as a way to count the number of disjoint looping pieces. 
The \FourClowerOpen-number accounts for the maximal volume of the curve where the angle injectivity
radius $\Delta$ is active. 
It will generally be large if the manifolds nearly
intersect multiple times, as illustrated in \cref{fig:winding}. 
The \fco-number is typically small, but can be large when the data are generated in a way that induces certain near symmetries, as in the right panel of \cref{fig:winding}. 

\begin{figure}[!htb]
  \begin{subfigure}[b]{0.63\textwidth}
    \centering
    \begin{subfigure}[b]{0.48\textwidth}
    \begin{overpic}[width=\textwidth]{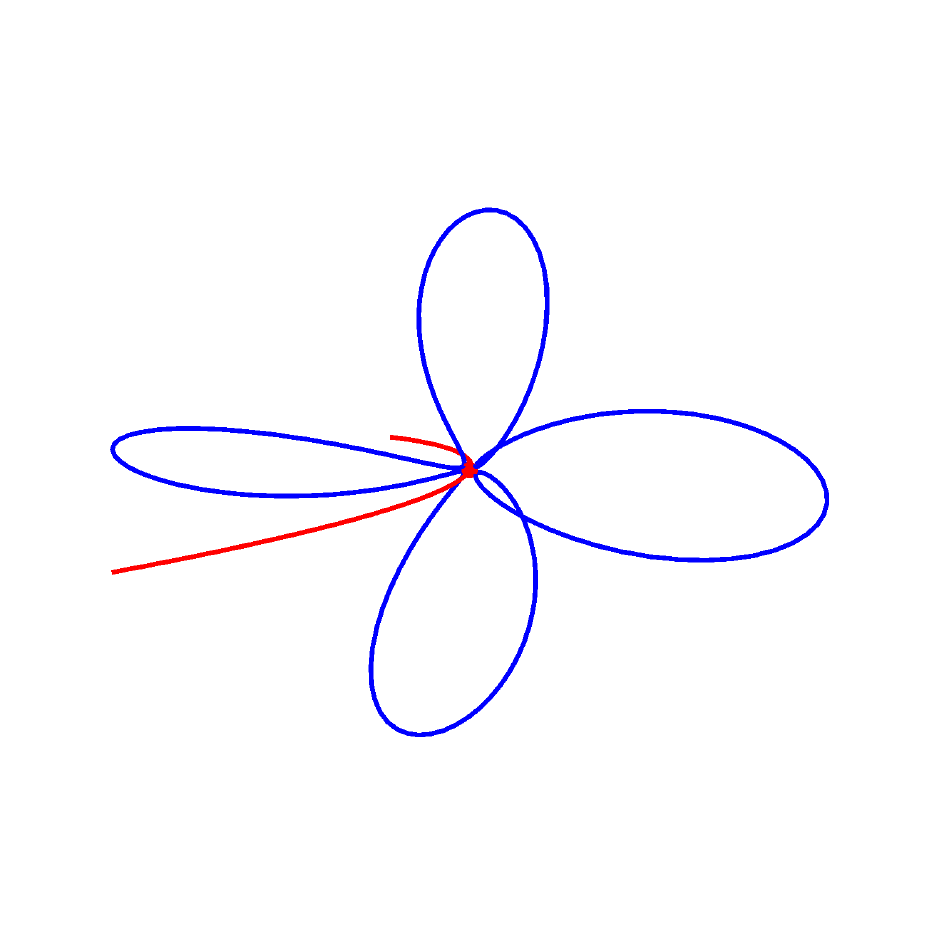}
        \put(2,0){$\fco = 4$}
         \put(60,0){\frame{\includegraphics[width=.62in]{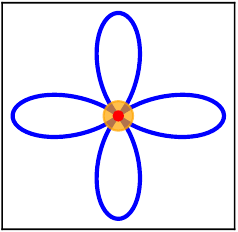}}}
      \end{overpic}
  \vspace{-.5in}
    \end{subfigure}
    \begin{subfigure}[b]{0.48\textwidth}
    \centering
    \begin{overpic}[width=\textwidth]{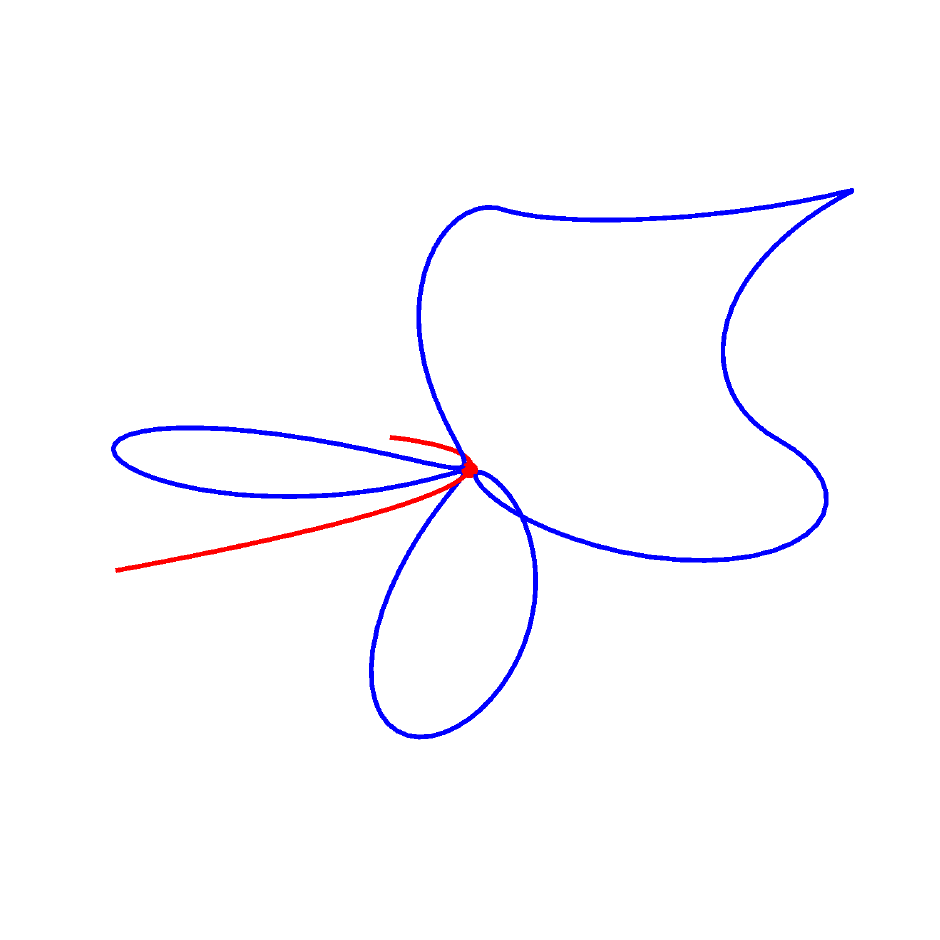}
	\put(2,0){$\fco = 3$}
	\put(60,0){\frame{\includegraphics[width=.62in]{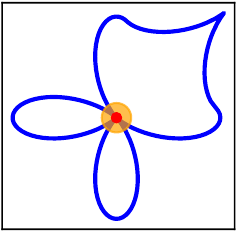}}}
\end{overpic}
\vspace{-.5in}
    \end{subfigure}
    \centering
        \begin{subfigure}[b]{0.48\textwidth}
      \begin{overpic}[width=\textwidth]{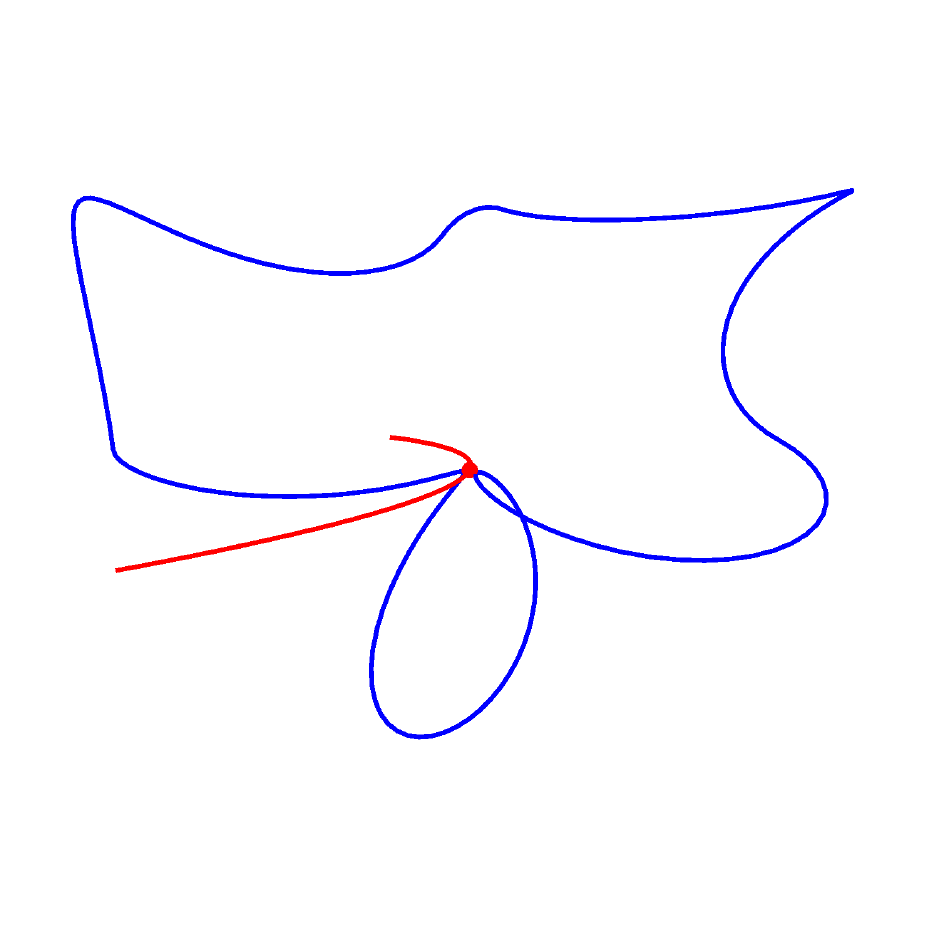}
  	\put(2,0){$\fco = 2$}
  	\put(60,0){\frame{\includegraphics[width=.62in]{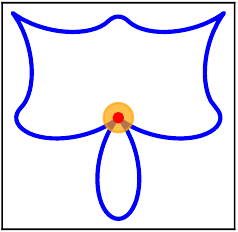}}}
  \end{overpic}
    \end{subfigure}
    \centering
    \begin{subfigure}[b]{0.48\textwidth}
    \begin{overpic}[width=\textwidth]{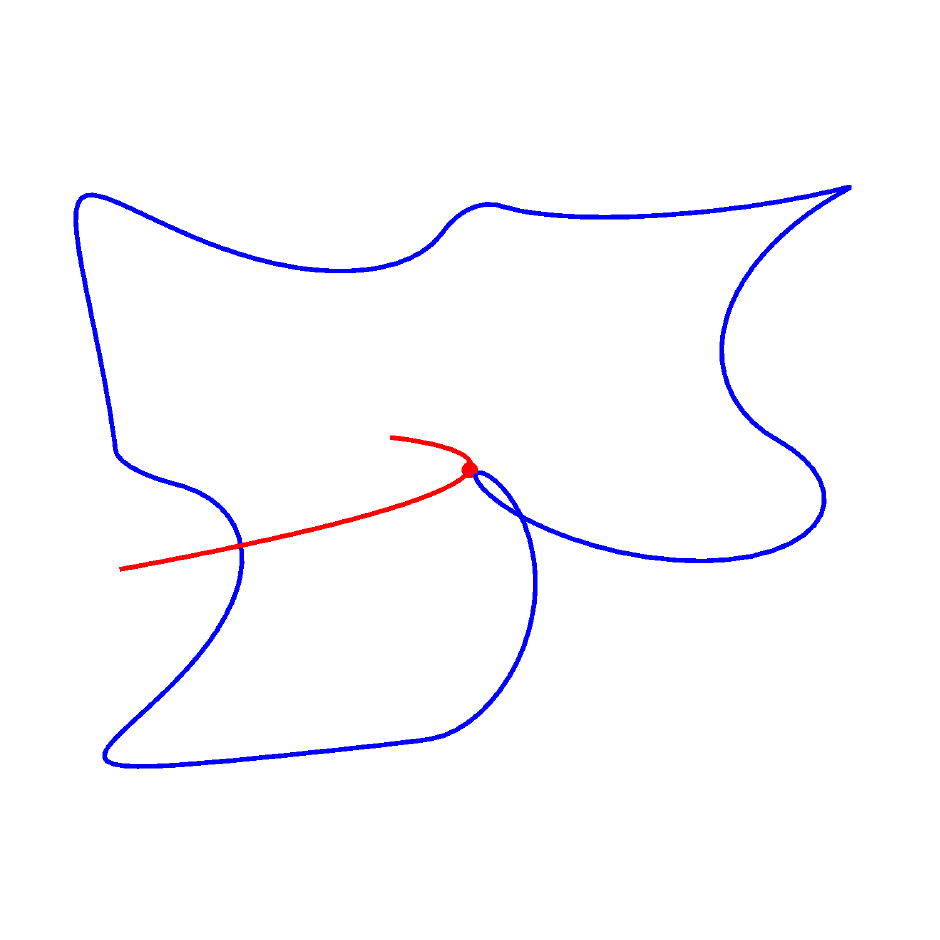}
	\put(2,0){$\fco = 1$}
	\put(60,0){\frame{\includegraphics[width=.62in]{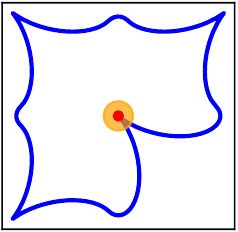}}}
\end{overpic}
    \end{subfigure}
  \end{subfigure}
  \centering
  \begin{subfigure}[b]{0.35\textwidth}
  \centering
  		\begin{overpic}[width=0.9\textwidth]{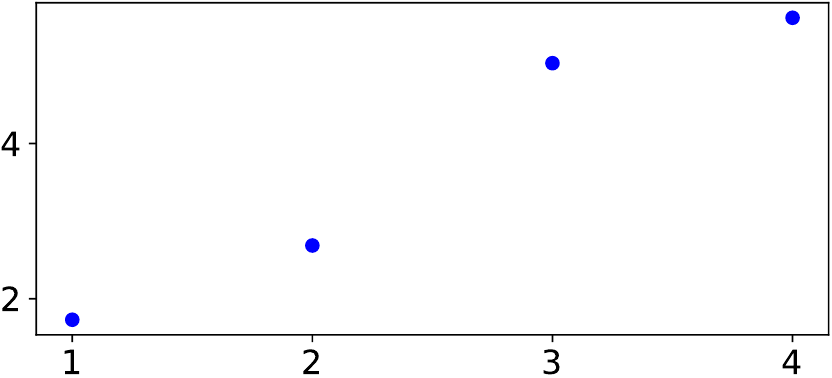}
  			\put(-7,5){\rotatebox{90}{Certificate Norm}}
  				\put(40,-7){\FourClowerOpen-Number }
  		\end{overpic}
\vspace{0.2in}
\par\medskip
  \begin{minipage}{.5\linewidth}
	\centering
	      \begin{overpic}[width=\textwidth, height=1in]{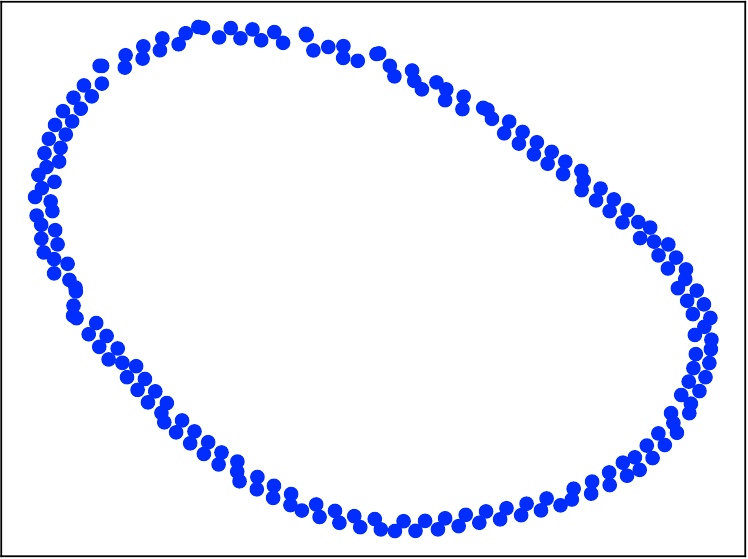}
	\put(,0){\frame{\includegraphics[width=.25in]{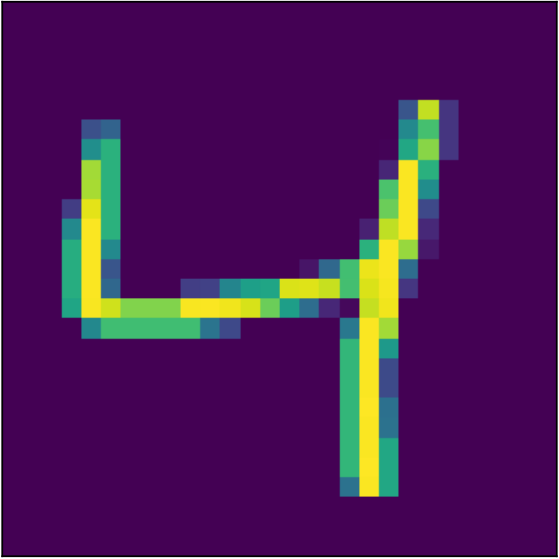}}}
\end{overpic}
 \vspace{1.8in}
\end{minipage}%
\begin{minipage}{.5\linewidth}
	\centering
      \begin{overpic}[width=\textwidth, height=1in]{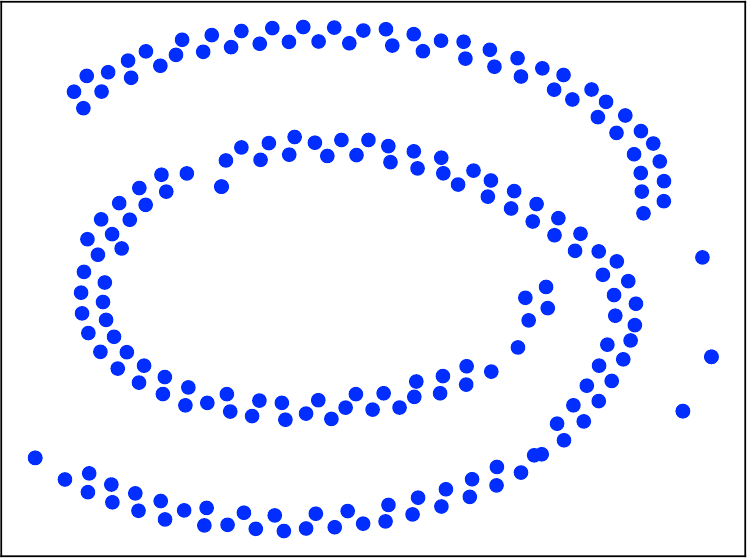}
        \put(,0){\frame{\includegraphics[width=.25in]{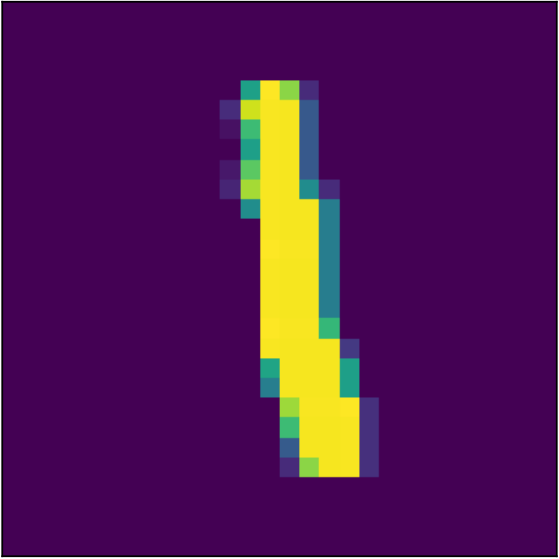}}}
      \end{overpic}
  \vspace{1.8in}
\end{minipage}
\end{subfigure}
\vspace{-1.4in}

  \caption{\textbf{The \FourClowerOpen-number---theory and practice.} \textit{Left:} 
  We generate a parametric family of space curves with fixed maximum curvature and length, but decreasing \fco-number, by reflecting `petals' of a clover %
  about a circumscribing square. %
  We set $\mplus$ to be a fixed circle with large radius that crosses the center of the
  configurations, then rescale and project the entire geometry onto the sphere to create a two curve
  problem instance. In the insets, we show a two-dimensional
  projection of each of the blue $\mminus$ curves as well as a base point $\vx \in \mplus$ at the
  center (also highlighed in the three-dimensional plots). The intersection of $\mminus$ with the neighborhood of $\vx$ denoted in orange
  represents the set whose covering number gives the $\fco$-number of the configuration (see
  \cref{eq:winding_def}).  \textit{Top right:} We numerically generate a certificate for each of
  the four geometries at left
  and plot its norm as a function of \fco-number. The trend demonstrates that
  {increasing \fco-number correlates with increasing classification difficulty}, measured through
  the certificate problem: this is in line with the intuition we have discussed.
    \textit{Bottom right:} t-SNE projection of MNIST images (top: a ``four'' digit; bottom: a ``one''
    digit) subject to rotations. Due
    to the approximate symmetry of the one digit under rotation by an angle $\pi$, the projection
    appears to nearly intersect itself. This may lead to a higher \FourClowerOpen-number compared
    to the embedding of the less-symmetric four digit. For experimental details for all panels,
    see \Cref{sec:fig_details}. }
  \label{fig:winding}
\end{figure}

\section{Main Results}
\label{sec:main_result}

Our main theorem establishes a set of sufficient resource requirements for the certificate problem under the class of geometries we consider here---by the reductions detailed in \Cref{sec:two_curves}, this implies that gradient descent rapidly separates the two classes given a neural network of sufficient depth and width. 
First, we note a convenient aspect of the certificate problem, which is its amenability to approximate solutions: that is, if we have a kernel $\Theta$ that approximates $\Theta^{\mathrm{NTK}}$ in the sense that $\norm{\vTheta_{\mu} - \vTheta_{\mu}^{\mathrm{NTK}}}_{L^2_{\mu}\to L^2_{\mu}} \ltsim n/L$,
and a function $\zeta$ such that $\norm{\zeta - \zeta_0}_{L^2_{\mu}} \ltsim 1/L$, then by the triangle inequality and the Schwarz inequality, it suffices to solve the equation $\vTheta_{\mu}[g] \approx \zeta$ instead. In our arguments, we will exploit the fact that the random kernel $\Theta^{\mathrm{NTK}}$ concentrates well for wide networks with $n \gtrsim L$, choosing $\Theta$ as
\begin{equation}
  \Theta(\vx,\vx') = (n/2) \sum_{\ell=0}^{L-1} \prod_{\ell'=\ell}^{L-1}\left(
  1 - (1/\pi) \ivphi{\ell'}(\angle(\vx,\vx')\right),
  \label{eq:body_ntk}
\end{equation} 
where $\varphi(t) = \acos( (1-t/\pi)\cos t + (1/\pi) \sin t)$ and $\ivphi{\ell'}$ denotes
$\ell'$-fold composition of $\varphi$; 
as well as the fact that for wide networks with $n \gtrsim L^5$, depth `smooths out' the initial error $\zeta_0$, choosing $\zeta$ as the piecewise-constant function
$\prederr(\vx) = -f_{\star}(\vx) + \int_{\manifold} f_{\vtheta_0}(\vx') \diff \mu(\vx')$.
We reproduce high-probability concentration guarantees from the literature that justify these approximations in \Cref{sec:aux}.

\begin{theorem}[Approximate Certificates for Curves]\label{thm:certificate_main} 
  Let $\manifold$ be two disjoint smooth, regular, simple closed curves, satisfying $ \angle(\vx, \vx') \leq\pi/2 $ for all $\vx, \vx' \in \manifold$.
  There exist absolute constants $C, C', C'', C'''$ and a polynomial $P = \poly(M_3, M_4, M_5,
  \len(\manifold), \Delta^{-1})$ of degree at most $36$, with degree at most $12$ in $(M_3, M_4,
  M_5, \len(\manifold))$ and degree at most $24$ in $\Delta^{-1}$, such that when
	\begin{eqnarray}
	    L &\ge& \max\Biggl\{ \exp(C'\len(\manifold) \hat{\kappa}), \left(\Delta\sqrt{1 + \kappa^2} \right)^{-C'' \fco(\manifold)}, C'''\hat{\kappa}^{10}, P, \rho_{\max}^{12} \Biggr\},
    \nonumber
	\end{eqnarray}
	there exists a certificate $g$ with 
	$\| g \|_{L^2_\mu} \le \frac{C \norm{\zeta}_{L^2_\mu}}{ \rho_{\min} n \log L} $
	such that
	$\norm{\mb \Theta_{\mu}[g] - \zeta}_{L^2_\mu} \le \frac{\norm{\zeta}_{L^\infty}}{L}$.
\end{theorem}

\Cref{thm:certificate_main} is our main technical contribution: it provides a sufficient condition
on the network depth $L$ to resolve the approximate certificate problem for the class of geometries we consider,
with the required resources depending only on the geometric properties we introduce in
\Cref{sec:geometry_def}.  
Given the connection between certificates and gradient descent, \Cref{thm:certificate_main} demonstrates that \textit{deeper networks fit more complex geometries}, which shows that the network depth plays the role of a fitting resource in classifying the two curves.
We provide a numerical corroboration of the interaction between the network depth, the
geometry, and the size of the certificate in \Cref{fig:certificates_stuff}.
For any family of geometries with bounded \fco-number,
\Cref{thm:certificate_main} implies a polynomial dependence of the depth on the angle injectivity radius
$\Delta$, whereas we are unable to avoid an exponential dependence of the depth on the curvature
$\kappa$.  
Nevertheless, these dependences may seem overly pessimistic in light of the existence of `easy' two curve problem instances---say, linearly-separable classes, each of which is a highly nonlinear manifold---for which one would expect gradient descent to succeed without needing an unduly large depth.
In fact, such geometries \textit{will not} admit a small certificate norm in general unless the depth is sufficiently large: 
intuitively, this is a consequence of the operator $\vTheta_{\mu}$ being ill-conditioned for such geometries.\footnote{%
  Again, the equivalence between the difficulty of the certificate problem and the progress of
  gradient descent on decreasing the error is a consequence of our analysis proceeding in the
  kernel regime with the square loss---using alternate techniques to analyze the
  dynamics can allow one to prove that neural networks continue to fit such `easy' classification
problems efficiently (e.g.\ \cite{Ji2019-cl}). }

\begin{figure}[!htb]
  \begin{subfigure}[b]{0.4\textwidth}
    \centering
    \includegraphics[width=\textwidth]{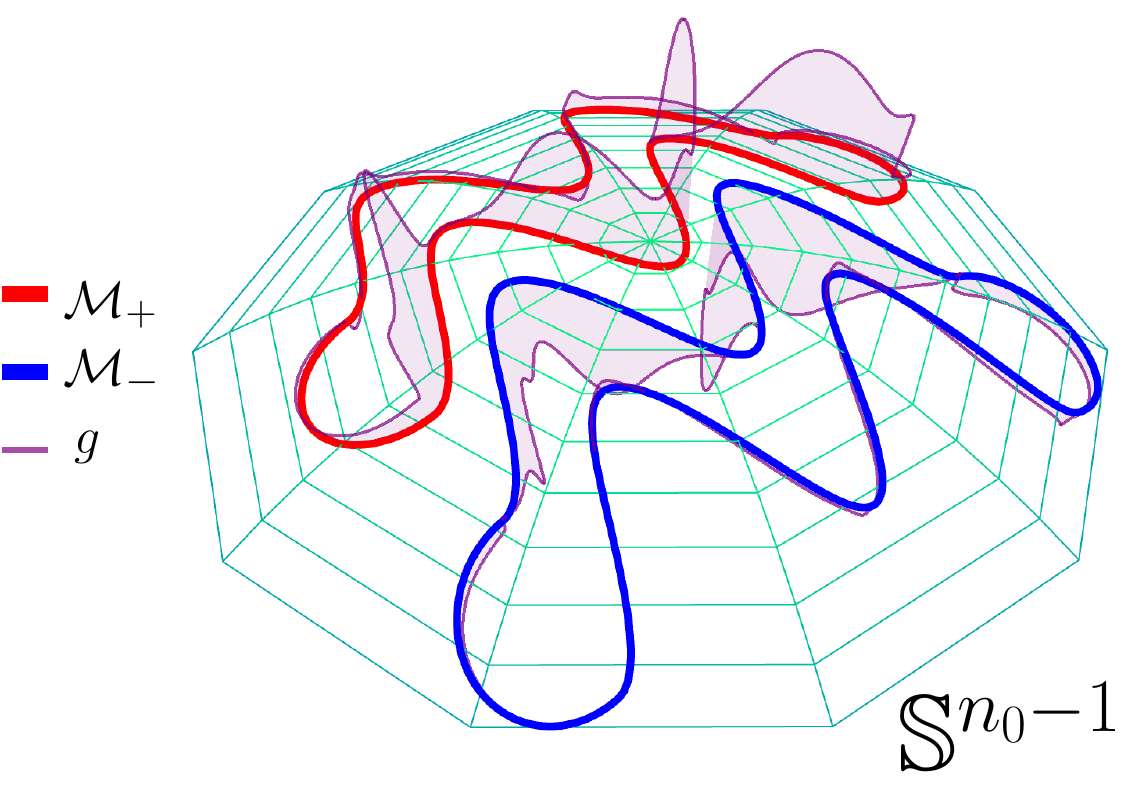}
    \label{fig:cert_on_sph}
  \end{subfigure}%
  \begin{subfigure}[b]{0.6\textwidth}
    \centering
    \includegraphics[width=\textwidth]{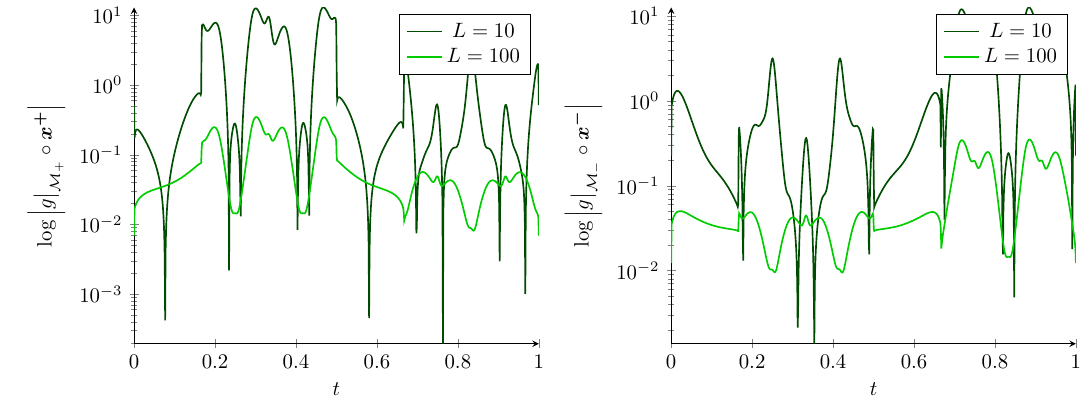}
    \label{fig:cert_chart}
  \end{subfigure}
  \caption{\textbf{The effect of geometry and depth on the certificate.} \textit{Left:} The
    certificate $g$ computed numerically from the kernel $\ntk$ for depth $L=50$ (defined in
    \Cref{eq:body_ntk}) and the geometry from \Cref{fig:two_man_two_curve} with a uniform density,
    graphed over the manifolds. Control of the norm of the certificate implies rapid progress of
    gradient descent, as reflected in \Cref{thm:main_formal}. Comparing to \Cref{fig:two_curve},
    we note that the certificate has large magnitude near the point of minimum distance between
    the two curves---this is suggestive of the way the geometry sets the difficulty of the fitting
    problem.
    \textit{Right:} To visualize the certificate norm more precisely, we graph the log-magnitude
    of the certificate for kernels $\Theta$ of varying depth $L$, viewing them through the
    arc-length parameterizations $\vx_{\sigma}$ for the curves (left: $\mplus$; right: $\mminus$).
    At a coarse scale, the maximum magnitude decreases as the depth increases; at a finer scale,
    curvature-associated defects are `smoothed out'. This indicates the role of depth as a fitting
  resource. See \Cref{sec:fig_details} for further experimental details. } 
  \label{fig:certificates_stuff}
\end{figure}

The proof of \Cref{thm:certificate_main} is novel, both in the context of kernel regression on manifolds and in the context of NTK-regime neural network training. 
We detail the key intuitions for the proof in \Cref{sec:proof_sketch}.
As suggested above, applying \Cref{thm:certificate_main} to construct a certificate is
straightforward: given a suitable setting of $L$ for a two curve problem instance, we obtain an
approximate certificate $g$ via \Cref{thm:certificate_main}.  Then with the triangle inequality
and the Schwarz inequality, we can bound 
\begin{equation*}
\norm{\vTheta^{\mathrm{NTK}}_{\mu}[g] - \zeta_0}_{L^2_{\mu}} 
\leq \norm{\vTheta_{\mu}^{\mathrm{NTK}} - \vTheta_{\mu}}_{L^2_\mu \to L^2_\mu} \norm{g}_{L^2_{\mu}}
+ \norm{\zeta_0 - \zeta}_{L^2_{\mu}} + \norm{\vTheta_{\mu}[g] - \zeta}_{L^2_{\mu}},
\end{equation*}
and leveraging suitable probabilistic control (see \Cref{sec:aux}) of the approximation errors in
the previous expression, as well as on $\norm{\zeta}_{L^2_{\mu}}$, then yields bounds for the
certificate problem.  Applying the reductions from gradient descent dynamics in the NTK regime to
certificates discussed in \Cref{sec:two_curves}, we then obtain an end-to-end guarantee for the
two curve problem.

\begin{theorem}[Generalization]\label{thm:main_formal}
	  Let $\manifold$ be two disjoint smooth, regular, simple closed curves, satisfying $ \angle(\vx, \vx') \leq\pi/2 $ for all $\vx, \vx' \in \manifold$.
	  For any $0 < \delta \leq 1/e$, choose $L$ so that
	    \begin{align*}
	    L &\geq K \max\set*{
	    \frac{1}{\left( \Delta\sqrt{1 + \kappa^2} \right)^{C \fco(\manifold)}},
	      C_{\mu} \log^9(\tfrac{1}{\delta}) 
	      \log^{24} ( C_{\mu} \adim \log(\tfrac{1}{\delta}) ),
	      e^{C' \max \set{\len(\sM)\hat{\kappa}, \log(\hat{\kappa})}},
	    P } \\
	 n &= K'L^{99}\log^{9}(1/\delta)\log^{18}(L\adim) \\
	 N &\geq L^{10},
	  \end{align*}
	  and fix $\tau > 0$ such that $\frac{C''}{nL^2} \leq \tau \leq \frac{c}{nL}$.
	  Then with probability at least $1 - \delta$, the parameters obtained at iteration
	  $\floor{L^{39/44} / (n\tau)}$ of gradient descent on the finite sample loss %
	  yield a classifier that separates the two manifolds.
	
	  The constants $c, C, C', C'', K, K' > 0$ are absolute, 
	  and the constant $C_{\mu}$ is equal to
	  $ \frac{\max \set{\rho_{\min}^{19}, \rho_{\min}^{-19}} (1 + \rho_{\max})^{12}  }
	  {\left(\min\,\set*{\mu(\mplus), \mu(\mminus)} \right)^{11/2}}$. 
	  $P$ is a polynomial $\poly\{M_3, M_4, M_5, \len(\manifold), \Delta^{-1}\}$ of degree at most $36$, with degree at most $12$ when viewed as a polynomial in $M_3, M_4, M_5$ and $\len(\manifold)$, and of degree at most $24$ as a polynomial in $\Delta^{-1}$.
\end{theorem}	
\Cref{thm:main_formal} represents the first end-to-end guarantee for training a deep neural network to classify a nontrivial class of low-dimensional nonlinear manifolds.
We call attention to the fact that the hypotheses of \Cref{thm:main_formal} are completely
self-contained, making reference only to intrinsic properties of the data and the architectural
hyperparameters of the neural network (as well as $\poly(\log \adim)$), 
and that the result is algorithmic, as it applies
to training the network via constant-stepping gradient descent on the empirical square loss and
guarantees generalization within $L^2$ iterations.  
Furthermore, \Cref{thm:main_formal} can be readily extended to the more general setting of regression on curves, given that we have focused on training with the square loss.

\section{Proof Sketch}
\label{sec:proof_sketch}

In this section, we provide an overview of the key elements of the proof of \Cref{thm:certificate_main}, where we show that the equation $\ntkoper_\mu[g] \approx \zeta$ admits a solution $g$ (the certificate) of small norm. 
To solve the certificate problem for $\manifold$, we require a fine-grained understanding of the kernel $\Theta$. The most natural approach is to formally set $g = \sum_{i=1}^\infty \lambda_i^{-1} \ip{\zeta}{v_i}_{L^2_\mu} v_i$ using the eigendecomposition of $\vTheta_{\mu}$ (just as constructed in \Cref{sec:two_curves} for $\vTheta_{\mu}^{\mathrm{NTK}}$), and then argue that this formal expression converges by studying the rate of decay of $\lambda_i$ and the alignment of $\zeta$ with eigenvectors of $\vTheta_{\mu}$; this is the standard approach in the literature \cite{Ghorbani2020-ef,nitanda2020optimal}.
However, as discussed in \Cref{sec:two_curves},
the nonlinear structure of $\manifold$ makes obtaining a full diagonalization for $\vTheta_{\mu}$ intractable, and simple asymptotic characterizations of its spectrum are insufficient to prove that the solution $g$ has small norm. 
Our approach will therefore be more direct: we will study the `spatial' properties of the kernel $\Theta$ itself, in particular its rate of decay away from $\vx = \vx'$, and thereby use the network depth $L$ as a resource to reduce the study of the operator $\vTheta_{\mu}$ to a simpler, localized operator whose invertibility can be proved using harmonic analysis. 
We will then use differentiability properties of $\Theta$ to transfer the solution obtained by inverting this auxiliary operator back to the operator $\vTheta_{\mu}$.
We refer readers to \Cref{sec:proof_certificate_theorem} for the full proof. 

We simplify the proceedings using two basic reductions. First, with a small amount of  auxiliary
argumentation, we can reduce from the study of the operator-with-density $\vTheta_{\mu}$ to the
density-free operator $\vTheta$. 
Second, the kernel $\Theta(\mb x,\mb x')$ is a function of the
angle $\angle(\mb x,\mb x')$, and hence is rotationally invariant. 
This kernel is maximized at
$\angle(\mb x,\mb x') = 0$ and decreases monotonically as the angle increases, reaching its
minimum value at $\angle(\mb x, \mb x') = \pi$. 
If we subtract this minimum value, it should not affect our ability to fit functions, and we obtain a rotationally invariant kernel $\Theta^\circ(\mb x, \mb x') = \psi^\circ(\angle(\mb x,\mb x'))$ that is concentrated around angle $0$. 
In the following, we focus on certificate construction for the kernel $\Theta^\circ$. Both simplifications are justified in \Cref{sec:certificate_dc_density}.

\subsection{The Importance of Depth: Localization of the Neural Tangent Kernel}
\label{sec:pf_sketch_loc}
The first problem one encounters when attempting to directly establish (a property like) invertibility of the operator $\vTheta^\circ$ is its action across connected components of $\manifold$: the operator $\vTheta^\circ$ acts by integrating against functions defined on $\manifold = \mplus \cup \mminus$, and although it is intuitive that most of its image's values on each component will be due to integration of the input over the same component, there will always be some `cross-talk' corresponding to integration over the opposite component that interferes with our ability to apply harmonic analysis tools. 
To work around this basic issue (as well as others we will see below), our argument proceeds via a localization approach: we will exploit the fact that as the depth $L$ increases, the kernel $\Theta^\circ$ sharpens and concentrates around its value at $\vx = \vx'$, to the extent that we can neglect its action across components of $\manifold$ and even pass to the analysis of an auxiliary localized operator.
This reduction is enabled by new sharp estimates for the decay of the angle function $\psi^\circ$ that we establish in \Cref{sec:approx_upper_bound}. Moreover, the perspective of using the network depth as a resource to localize the kernel $\Theta^\circ$ and exploiting this to solve the classification problem appears to be new: this localization is typically presented as a deficiency in the literature (e.g.\ \cite{Huang2020-pn}).

At a more formal level, when the network is deep enough compared to geometric properties of the
curves, for each point $\mb x$, the majority of the mass of the kernel $\Theta^\circ(\mb x, \mb
x')$ is taken within a small neighborhood $d_{\manifold}(\mb x, \mb x') \le r$ of $\mb x$. When
$d_{\manifold}(\mb x, \mb x')$ is small relative to $\kappa$, we have $d_{\manifold}(\mb x, \mb
x') \approx \angle(\mb x, \mb x')$. This allows us to approximate the local component by the
following invariant operator:
\begin{equation} \label{eqn:M-def}
  \wh{\mb M} [ f ]( \mb x_{\sigma}(s) ) = \int_{s' = s-r}^{s+r} \psi^\circ(\abs{s-s'}) f(\mb x_{\sigma}(s')) ds'.
\end{equation}
This approximation has two main benefits: (i) the operator $\wh{\mb M}$ is defined by intrinsic distance $s' - s$, and (ii) it is highly localized. 
In fact, \eqref{eqn:M-def} takes the form of a convolution over the arc length parameter $s$. 
This implies that $\wh{\mb M}$ diagonalizes in the Fourier basis, giving an explicit characterization of its eigenvalues and eigenvectors. 
Moreover, because $\wh{\mb M}$ is localized, the eigenvalues corresponding to slowly oscillating Fourier basis functions are large, and  $\wh{\mb M}$ is stably invertible over such functions. 
Both of these benefits can be seen as consequences of depth: depth leads to localization, which facilitates approximation by $\wh{\mb M}$, {\em and} renders that approximation invertible over low-frequency functions. 
In our proofs, we will work with a subspace $S$ spanned by low-frequency basis functions that are nearly constant over a length $2r$ interval (this subspace ends up having dimension proportional to $1/r$; see \Cref{sec:subspace_def} for a formal definition), and use Fourier arguments to prove invertibility of $\wh{\vM}$ over $S$ (see \Cref{lem:main_term_pd}).

\subsection{Stable Inversion over Smooth Functions}
\label{sec:pf_sketch_invert}

Our remaining task is to leverage the invertibility of $\wh{\mb M}$ over $S$ to argue that $\mb \Theta$ is also invertible. In doing so, we need to account for the residual $\mb \Theta - \wh{\mb M}$. We accomplish this directly, using a Neumann series argument: when  setting $r \lesssim L^{-1/2}$ and the dimension of the subspace $S$ proportional to $1/r$, the minimum eigenvalue of $\wh{\mb M}$ over $S$ exceeds the norm of the residual operator $\mb \Theta^\circ - \wh{\mb M}$ (\Cref{lem:main_res_norm_compare}). This argument leverages a decomposition of the domain into ``near'', ``far'' and ``winding'' pieces, whose contribution to $\mb \Theta^\circ$ is controlled using the curvature, angle injectivity radius and $\fco$-number (\Cref{lem:near_ub}, \Cref{lem:winding_ub}, \Cref{lem:far_ub}). This guarantees the strict invertibility of $\mb \Theta^\circ$ over the subspace $S$, and yields a unique solution $g_S$ to the  {\em restricted} equation $\mb P_S \mb \Theta^\circ [ g_S ] = \zeta$ (\Cref{thm:certificate_over_S}). 

This does not yet solve the certificate problem, which demands near solutions to the {\em
unrestricted} equation $\mb \Theta^\circ [ g ] = \zeta$. To complete the argument, we set $g =
g_S$ and use harmonic analysis considerations to show that $\mb \Theta^\circ [ g ]$ is very close
to $S$. The subspace $S$ contains functions that do not oscillate rapidly, and hence whose
derivatives are small relative to their norm (\Cref{lem:fourier_deriv_lb}). We prove that $\mb
\Theta^\circ [ g ]$ is close to $S$ by controlling the first three derivatives of $\mb
\Theta^\circ[g]$, which introduces dependencies on $M_1, \cdots, M_5$ in the final statement of our
results (\Cref{lem:Theta_Sperp_proj}). In controlling these derivatives, we leverage the assumption that $\sup_{\vx,\vx' \in \manifold} \angle(\vx, \vx') \leq \pi/2$ to avoid issues that arise at antipodal points---we believe the removal of this constraint is purely technical, given our sharp characterization of the decay of $\psi^\circ$ and its derivatives. %
Finally, we move from $\mb \Theta^\circ$ back to $\mb
\Theta$ by combining near solutions to $\mb \Theta^\circ[ g ] = \zeta$ and $\mb \Theta^\circ [g_1
] = 1$, and iterating the construction to reduce the approximation error to an acceptable level
(\Cref{sec:certificate_dc_density}).

\section{Discussion}
\label{sec:discussion}

\paragraph{A role for depth.} 
In the setting of fitting functions on the sphere $\Sph^{\adim -1}$ in the NTK regime with
unstructured (e.g., uniformly random) data, it is well-known that there is very little marginal
benefit to using a deeper network: for example,
\cite{Mei2021-qg,ghorbani2019linearized,Ghorbani2020-ef} show that the risk lower bound for RKHS
methods is nearly met by kernel regression with a $2$-layer network's NTK in an asymptotic
($\adim \to \infty$) setting, and results for fitting degree-$1$ functions in the nonasymptotic
setting \cite{Montanari2020-qh} are suggestive of a similar phenomenon. 
In a similar vein, fitting in the NTK regime with a deeper network does not change the kernel's
RKHS \cite{Chen2020-pm, Geifman2020-fd, Bietti2020-hr}, 
and in a certain ``infinite-depth'' limit,
the corresponding NTK for networks with ReLU activations, as we consider here, is a spike,
guaranteeing that it fails to generalize \cite{Liang2020-nr, Huang2020-pn}.
Our results are certainly not in contradiction to these facts---we consider a setting where the
data are highly structured, and our proofs only show that an appropriate choice of the depth
relative to this structure is \textit{sufficient} to guarantee generalization, not necessary---but
they nonetheless highlight an important role for the network depth in the NTK regime that has not
been explored in the existing literature.
In particular, the localization phenomenon exhibited by the deep NTK is completely inaccessible by
fixed-depth networks, and simultaneously essential to our arguments to proving
\Cref{thm:main_formal}, as we have described in \Cref{sec:proof_sketch}.
It is an interesting open problem to determine whether there exist low-dimensional geometries that
cannot be efficiently separated without a deep NTK, or whether the essential sufficiency of the
depth-two NTK persists.

\paragraph{Closing the gap to real networks and data.} \Cref{thm:main_formal} represents an
initial step towards understanding the interaction between neural networks and data with
low-dimensional structure, and identifying network resource requirements sufficient to guarantee
generalization. There are several important avenues for future work. First, although the resource
requirements in \Cref{thm:certificate_main}, and by extension \Cref{thm:main_formal}, reflect only
intrinsic properties of the data, the rates are far from optimal---improvements here will demand a more refined harmonic analysis argument beyond the localization approach we take in \Cref{sec:pf_sketch_loc}.
A more fundamental advance would consist of extending the analysis to the setting of a model for
image data, such as cartoon articulation manifolds, and the NTK of a convolutional neural network
with architectural settings that impose translation invariance \cite{Novak2018-sv, Li2019-bw}---%
recent results show asymptotic statistical efficiency guarantees with the NTK of a simple convolutional architecture, but only in the context of generic data \cite{Mei2021-tv}. %
The approach to certificate construction we develop in \Cref{thm:certificate_main} will be of use
in establishing guarantees analogous to \Cref{thm:main_formal} here, as our
approach does not require an explicit diagonalization of the NTK. 
In addition, extending our certificate construction approach to smooth manifolds of dimension
larger than one is a natural next step. 
We believe our localization argument generalizes to this setting: as our bounds for the kernel
$\psi$ are sharp with respect to depth and independent of the manifold dimension, one could
seek to prove guarantees analogous to \Cref{thm:certificate_main} with a similar
subspace-restriction argument for sufficiently regular manifolds, such as manifolds diffeomorphic
to spheres, where the geometric parameters of \Cref{sec:geometry_def} have natural extensions.
Such a generalization would incur at best an exponential dependence of the network on the manifold
dimension for localization in high dimensions. 

More broadly, the localization phenomena at the core of our argument appear to be relevant beyond
the regime in which the hypotheses of \Cref{thm:main_formal} hold: we provide a preliminary
numerical experiment to this end in \Cref{sec:exp_depth}. Training fully-connected networks with gradient descent on a simple manifold classification task, low training error appears to be easily achievable only when the decay scale of the kernel is small relative to the inter-manifold distance even at moderate depth and width, and this decay scale is controlled by the depth of the network.

\section*{Acknowledgements}
This work was supported by a Swartz fellowship (DG), by a fellowship award
(SB) through the National Defense Science and Engineering Graduate (NDSEG) Fellowship Program,
sponsored by the Air Force Research Laboratory (AFRL), the Office of Naval Research (ONR) and the Army Research Office (ARO), and by the National Science Foundation through grants NSF 1733857, NSF 1838061, NSF 1740833, and NSF 174039. We thank Alberto Bietti for bringing to our attention relevant prior art on kernel regression on manifolds.

\printbibliography

\newpage

\appendix
\etocdepthtag.toc{mtappendix}
\etocsettagdepth{mtchapter}{none}
\etocsettagdepth{mtappendix}{subsection}
\tableofcontents

\section{Details of Figures}
\label{sec:fig_details}

\subsection{{\Cref{fig:winding}}}

\paragraph{$\fco$-number experiment.} \label{sec:clover_exp} In each panel, the two curves are projection of curves $\vx_+ :
[0, 2 \pi] \to \Sph^{3}$ and $\vx_- : [0, 2 \pi] \to \Sph^{3}$. We actually generate the curves as
shown in the figure (i.e., in a three-dimensional space), then map them to the sphere using the map $(u, v, w)
\mapsto (u, v, w,  \sqrt{1 - u^2 - v^2 - w^2})$. In this three-dimensional space, the top left
panel's blue curve (denoted $\vx_{-}$ henceforth) and each panel's red curve (denoted $\vx_{+}$
henceforth, and which is the same for all panels) are defined by the parametric equations 
\begin{equation*}
\begin{aligned} & \left(\begin{array}{c}
		x_{-,1}(t)\\
		x_{-,2}(t)\\
		x_{-,3}(t)
	\end{array}\right) & =\,\, & \left(\begin{array}{c}
		\cos(4t)\\
		\cos\left(\frac{\pi}{8}\right)\cos(t)\left(\sin(4t)+1+\delta\right)+\sin\left(\frac{\pi}{8}\right)\sin(t)\left(\sin(4t)+1+\delta\right)\\
		-\sin\left(\frac{\pi}{8}\right)\cos(t)\left(\sin(4t)+1+\delta\right)+\cos\left(\frac{\pi}{8}\right)\sin(t)\left(\sin(4t)+1+\delta\right)
	\end{array}\right)\\
	& \left(\begin{array}{c}
		x_{+,1}(t)\\
		x_{+,2}(t)\\
		x_{+,3}(t)
	\end{array}\right) & =\,\, & \left(\begin{array}{c}
		4\sin(t)\\
		4\left(\cos(t)-1\right)\\
		0
	\end{array}\right),
\end{aligned}
\end{equation*}
where $\delta$ sets the separation between the manifolds and is set here to $\delta=0.05$. We then
rescale both curves by a factor $.01$: %
the scale of the curves is chosen such that the curvature of the sphere has a
negligible effect on the curvature of the manifolds (since the chart mapping we use here distorts
the curves more nearer to the boundary of the unit disk $\set{(u, v, w) \given u^2 + v^2 + w^2 \leq
  1}$).\footnote{Although this adds a minor confounding effect to our experiments with certificate
  norm in the top-right panel, it is suppressed by setting the scale sufficiently small, and it
  can be removed in principle by using an isometric chart for the upper hemisphere instead of the
map given above. }

From here, we use an ``unfolding'' process to obtain the blue curves in the other three panels
from $\vx_{-}$.
To do this,
points where $|\frac{\mathrm{d}x_{-,2}}{\mathrm{d}t}|=|\frac{\mathrm{d}x_{-,3}}{\mathrm{d}t}|$ are
found numerically.
There are $8$ such points in total, and parts of the curve between pairs of these points are
reflected across the line defined by such a pair in the $(x_{2},x_{3})$ plane. This can be done for any number of pairs between
$1$ and $4$, generating the curves shown. This procedure ensures that aside from the set of $8$
points, the curvature at every point along the curve is preserved and there is no discontinuity in
the first derivative, while making the geometries loop back to the common center point more.
For an additional visualization of the geometry, see \Cref{fig:slinky_full}.\footnote{For a
  three-dimensional interactive visualization, see
\url{https://colab.research.google.com/drive/1xmpYeLK606DtXOkJEt_apAniEB9fARRv?usp=sharing}.}
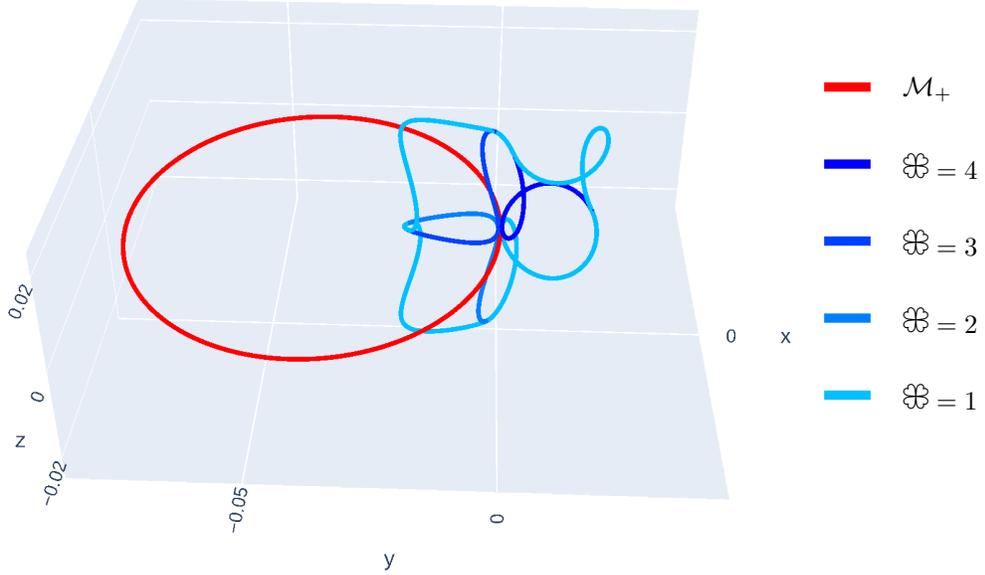
\begin{figure}
	\vspace{-2in}
	    \begin{overpic}[width=\textwidth]{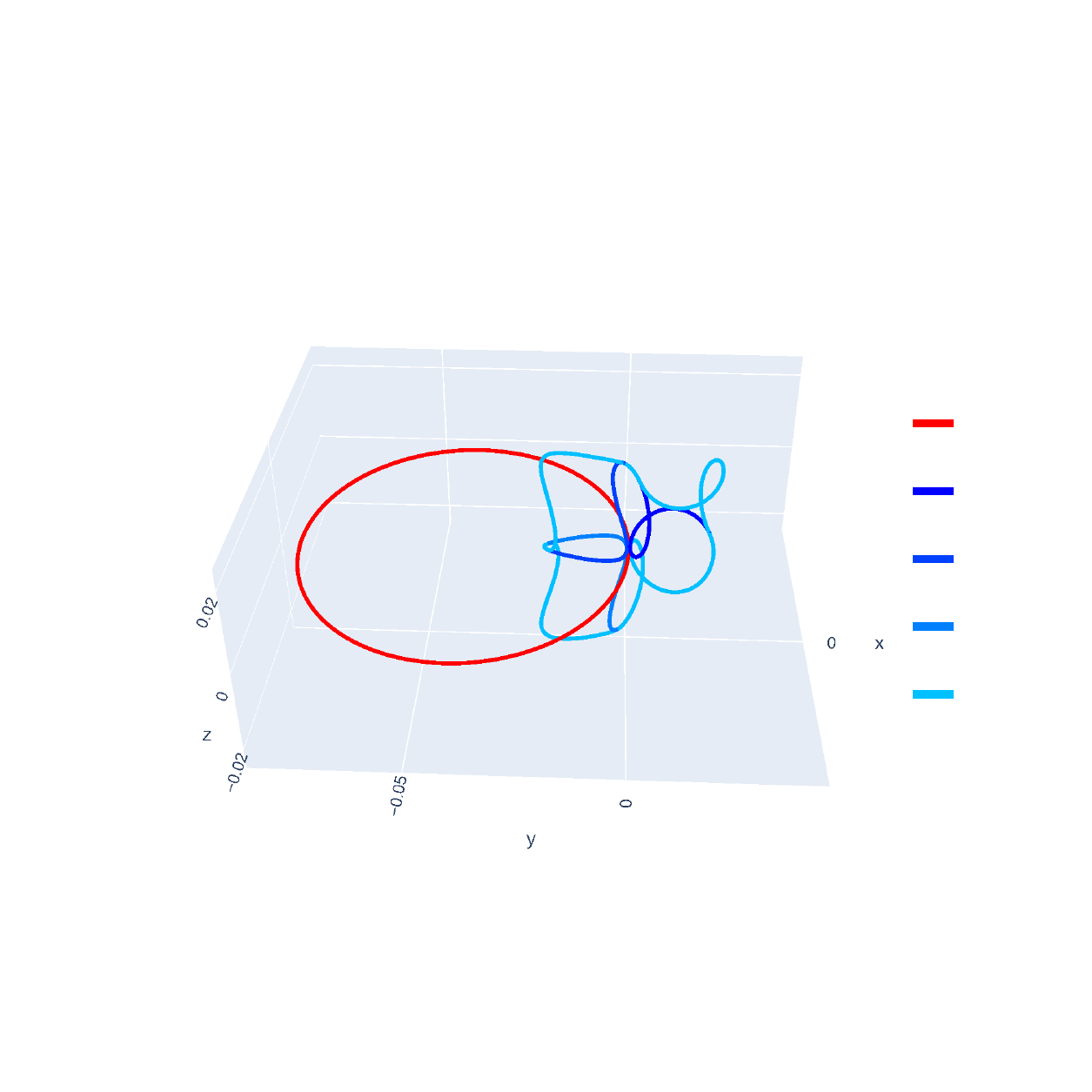}
		\put(90,35.2){$\fco = 1$}
		\put(90,41.4){$\fco = 2$}
		\put(90,47.6){$\fco = 3$}
		\put(90,53.8){$\fco = 4$}
		\put(90,60.5){$\mplus$}
	\end{overpic}
\vspace{-1.5in}
\caption{The two curve geometry described in \Cref{sec:clover_exp}. The different choices of $\mminus$ that lead to different $\fco$-number are overlapping. The legend indicates the $\fco$-number of the two curves problem obtained by considering the same $\mplus$ but a different $\mminus$ as indicated by the color.} 
\label{fig:slinky_full}
\end{figure} 

Given these geometries, in order to compute the certificate norm for the experiment in the
top-right panel, we evaluate the resulting curves at $200$ points each, chosen by picking
equally spaced points in $[0,2\pi]$ and evaluating the parametric equations. The certificate
itself is evaluated numerically as in \Cref{sec:cert_fig}.

\paragraph{Rotated MNIST digits.} We rotate an MNIST image around its center by $i * \pi / 100$ for integer $i$ between $0$ and $199$. We then apply t-SNE \cite{van2008visualizing} using the scikit-learn package with perplexity $20$ to generate the embeddings. 

\subsection{{\Cref{fig:certificates_stuff}}} \label{sec:cert_fig}
We give full implementation details for this figure here, mixed with conceptual ideas that
underlie the implementation. The
manifolds $\mplus$ and $\mminus$ are defined by parametric equations $\vx_+ : [0, 1] \to \Sph^{2}$
and $\vx_- : [0, 1] \to \Sph^{2}$; it is not practical to obtain unit-speed parameterizations of
general curves, so we also have parametric equations for their derivatives $\dot{\vx}_{\sigma} :
[0, 1] \to \bbR^2$.  These are important in our setting since for non-unit-speed curves, the chain
rule gives for the integral of a function (say) $f: \mplus \to \bbR$
\begin{equation*}
  \int_{\mplus} f(\vx) \diff \vx = \int_{[0,1]} (f \circ \vx_+(t)) \norm{ \dot{\vx}_+(t) }_2 \diff t.
\end{equation*}
In particular, in our experiments, we want to work with a uniform density $\rho = (\rho_+,
\rho_-)$ on the manifolds, where the classes are balanced.  To achieve this, use the previous
equation to get that we require
\begin{align*}
  1 &= \int_{\mplus} \rho_+(\vx) \diff \vx + \int_{\mminus} \rho_-(\vx) \diff \vx\\
  &= \int_{\mplus} (\rho_+ \circ \vx_+)(t) \norm{ \dot{\vx}_+(t) }_2\diff t
  + \int_{\mminus} (\rho_- \circ \vx_-)(t) \norm{ \dot{\vx}_-(t) }_2 \diff t.
\end{align*}
A uniform density on $\manifold$ is not a constant value---rather, it is characterized by being
translation-invariant. It follows that $\rho_{\sigma}$ should be defined by
\begin{equation*}
  \rho_\sigma \circ \vx_\sigma(t) = \frac{1}{2 \norm{\dot{\vx}_\sigma(t)}_2} .
\end{equation*}

For the experiment, we solve a discretization of the certificate problem, for which the
above ideas will be useful. Consider $\Theta$ in \Cref{eq:body_ntk} for a fixed depth $L$ (and
$n=2$, since width is essentially irrelevant here). By the above discussion, the certificate
problem in this setting is to solve for the certificate $g = (g_+, g_-)$
\begin{equation*}
  f_{\star} = \frac{1}{2} \left(
    \int_{[0,1]} \Theta(\spcdot, \vx_{+}(t)) g_+ \circ \vx_{+}(t) \diff t
    +\int_{[0,1]} \Theta(\spcdot, \vx_{-}(t)) g_- \circ \vx_{-}(t) \diff t
  \right).
\end{equation*}
Here, we have eliminated the initial random neural network output $f_{\vtheta_0}$ from the RHS.
Aside from making computation easier, this is motivated by fact that the network output is
approximately piecewise constant for large depth $L$, and we therefore expect it not to play much
of a role here. Let $M \in \bbN$ denote the discretization size. Then a finite-dimensional
approximation of the previous integral equation is given by the linear system
\begin{equation}
  f_{\star} \circ \vx_{\sigma}(t_i) = \frac{1}{2M} \left(
    \sum_{j=1}^M \Theta(\vx_{\sigma}(t_i), \vx_{+}(t_j)) g_+ \circ \vx_{+}(t_j)
    + \sum_{j=1}^M \Theta(\vx_{\sigma}(t_i), \vx_{-}(t_j)) g_- \circ \vx_{-}(t_j)
  \right)
  \label{eq:experiment_discretiz}
\end{equation}
for all $i \in [M]$ and $\sigma \in \set{\pm 1}$, and where $t_i = (i-1) / M$. Of course,
$f_{\star} \circ \vx_{\sigma}(t) = \sigma$, so the equation simplifies further, and because the
kernel $\Theta$ and this target $f_{\star}$ are smooth, there is a convergence of the data in this
linear system in a precise sense to the data in the original integral equation as $M \to \infty$.
In particular, define a matrix $\vT^{+}$ by $T^+_{ij} = \Theta(\vx_{+}(t_i), \vx_{+}(t_j))$,
define a matrix $\vT^{-}$ by $T^-_{ij} = \Theta(\vx_{-}(t_i), \vx_{-}(t_j))$, and define a matrix
$\vT^{\pm}$ by $T^{\pm}_{ij} = \Theta(\vx_{+}(t_i), \vx_{-}(t_j))$, all of size $M \times M$. Then
the $2M \times 2M$ linear system
\begin{equation}
  \begin{bmatrix}
    \One \\
    -\One
  \end{bmatrix}
  =
  \frac{1}{2M}
  \begin{bmatrix}
    \vT^+ & \vT^{\pm} \\
    \left(\vT^{\pm}\right)\adj & \vT^{-}
  \end{bmatrix}
  \begin{bmatrix}
    \vg_+ \\
    \vg_-
  \end{bmatrix}
  \label{eq:experiment_discretiz_2}
\end{equation}
is equivalent to the discretization in \Cref{eq:experiment_discretiz}. We implement and solve the
system in \Cref{eq:experiment_discretiz_2} using the definitions we have given above, using the
pseudoinverse of the $2M \times 2M$ matrix appearing in this expression to obtain $[\vg_+,
\vg_-]\adj$, and plot the results in \Cref{fig:certificates_stuff}, in particular interpreting
$(\vg_{\sigma})_{i}$ as the sampled point $g_{\sigma} \circ \vx_{\sigma}(t_i)$ as in
\Cref{eq:experiment_discretiz} when we plot in the left panel of \Cref{fig:certificates_stuff}.
Evidently, it would be immediate to modify the experiment to replace the LHS of
\Cref{eq:experiment_discretiz} by the error $f_{\vtheta_0} - f_{\star}$: the same protocol given
above would work, but there would be an element of randomness added to the experiments.

Specifically, in \Cref{fig:certificates_stuff} we set $M = 900$. When plotting the solution to
\Cref{eq:experiment_discretiz_2}, i.e. the vector $[\vg_+, \vg_-]\adj$, we moreover scale the
vector by a factor of $0.3$ to facilitate visualization.

\subsection{Kernel Decay Scale and Trainability of Realisting Networks: Empirical Evidence}\label{sec:exp_depth}

\definecolor{forestgreen}{rgb}{0,0.5,0}
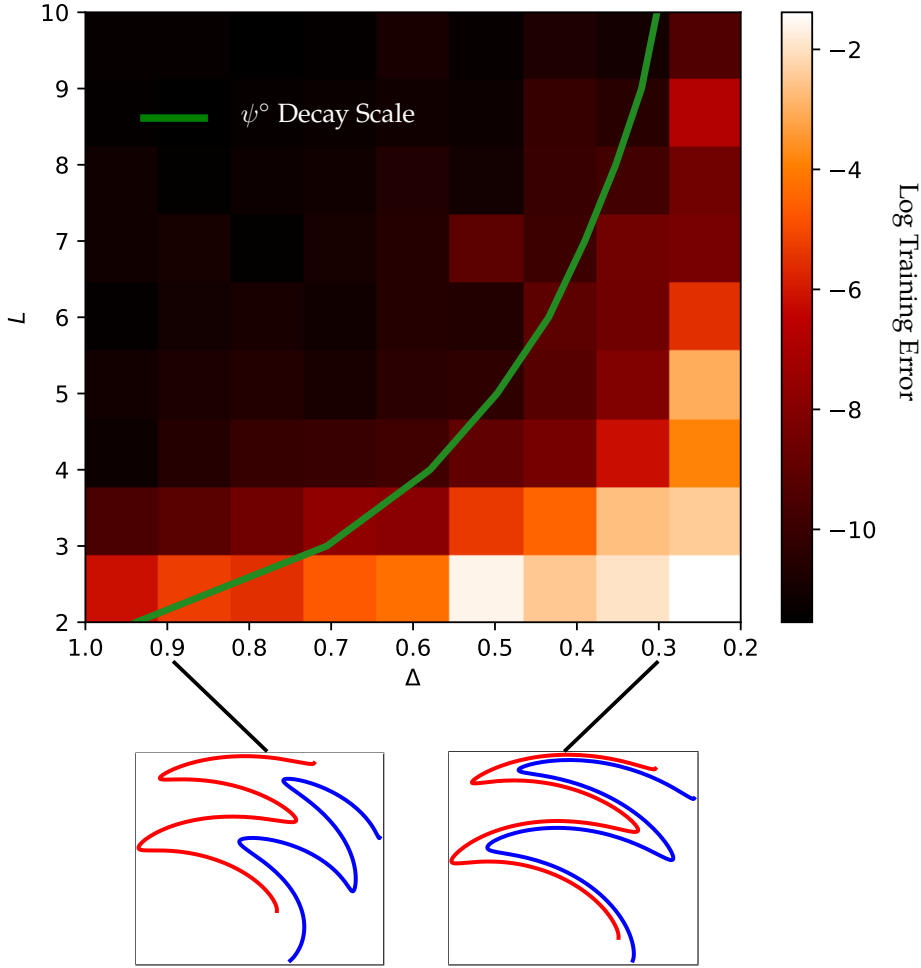
\begin{figure}[!htb]
	\centering{
		\begin{overpic}[width=0.7\textwidth]{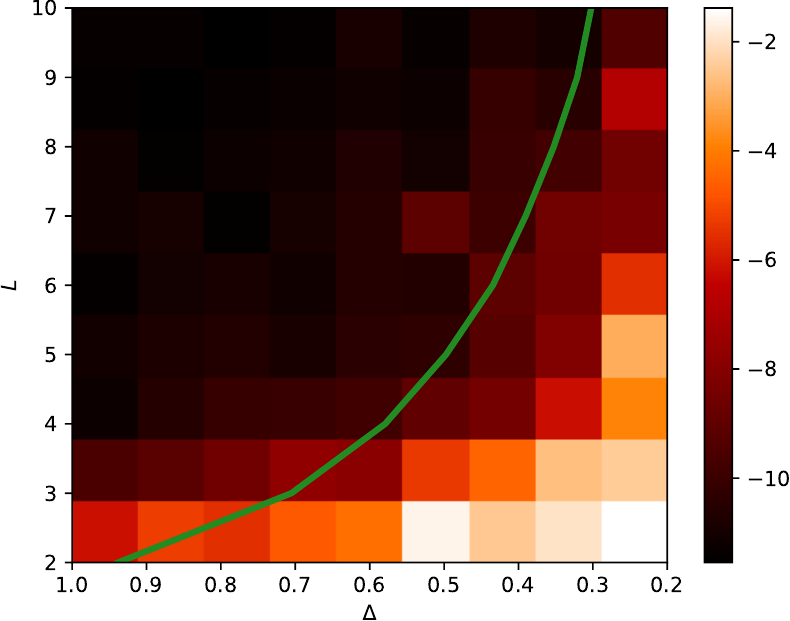}
			\put(102,58){\rotatebox{270}{Log Training Error}}
			\put( 15,  -32) {\frame{\includegraphics[scale=.5]{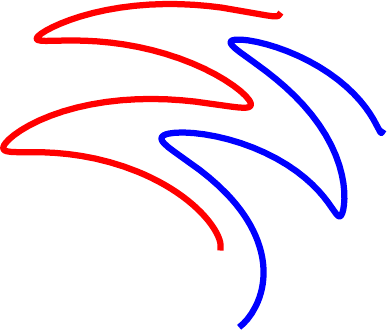}}}
			\put( 51,  -32) {\frame{\includegraphics[scale=.5]{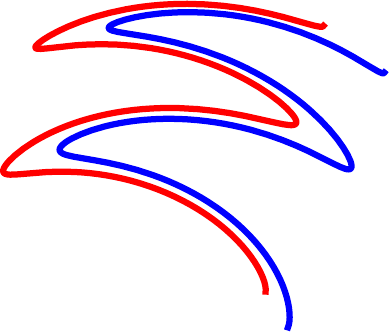}}}
			
			\put( 15.00,  65) {\tikz \draw[forestgreen, line width=1mm] (0,0)--(.9,0);}
			\put( 19,  -7.7) {\tikz \draw[line width=0.5mm] (0,0)--(-1.25,1.2);}
			\put( 64,  -7.7) {\tikz \draw[line width=0.5mm] (0,0)--(1.25,1.2);}
			\put(27,65){{\color{white}{$\psi^{\circ}$ Decay Scale}}}
		\end{overpic}
		\vspace{1.5in}
		\caption{\textbf{The decay properties of the NTK are predictive of trainability on a toy dataset.} We plot the log training error of networks of varying depth that are trained to classify two curves with varying separation. The insets show a projection of the geometry onto the plane for separation values $0.3$ and $0.9$. For each depth $L$, the characteristic decay scale of the DC-subtracted NTK ($\psi^{\circ}$) is computed numerically and plotted in green. We find that small training loss is only achievable if the decay scale of the kernel is small compared to the inter-manifold distance, hence the decay scale is predictive of trainability.}
		\label{fig:depth}
	}
\end{figure}

One of the main insights into the manifold classification problem that is utilized to obtain \cref{thm:main_formal} is that (roughly speaking) the depth of a fully-connected network controls the decay properties of the network's NTK, and that fitting can be guaranteed once the decay occurs on a spatial scale that is small relative to certain geometric properties of the data. Here we provide empirical evidence that this phenomenon holds beyond the regime in which our main theorems hold, and in fact is relevant for networks of moderate width and depth as well. 

We draw $400$ samples each from a uniform distribution over a union of two curves that are related
by a rotation by a geodesic angle that is varied from $0.2$ to $1.0$ in increments of $0.1$. The
curves are not linearly separable even for large angle (see insets in \cref{fig:depth}). These
curves are embedded in $\mathbb{S}^{n_{0}-1}$ for $n_0=128$ and subjected to a rotation drawn
uniformly from the Haar measure. We then train a fully-connected network to classify the curves
using $\ell^2$ loss. The network has width $n=256$ and we vary the depth from $L=2$ to $L=10$, and
train using full-batch gradient descent for $10^5$ iterations with learning rate $\tau = 1 / (4 n
L)$ (so that the total effective "training time" is independent of depth). We plot the log
training error after training as a function of depth and the inter-manifold distance. For each
depth $L$, we estimate an effective ``decay scale'' of the DC-subtracted skeleton $\psi^{\circ}$
by determining the point $s^\star$ such that $\psi^{\circ}(s^{\star})=\frac{\psi^{\circ}(0)}{2}$. 

The results are presented in \cref{fig:depth}. We observe that the network convergences to small
training loss only when the depth is large comparable to the inverse of the manifold separation.
As the depth represents the decay rate of the NTK, this indicates that a deeper network generates
a localized NTK, allowing faster decay of the training error and making the classification problem
easy. Notice that since the geometry of the dataset and network architecture do not satisfy all
the assumptions of \cref{thm:main_formal}, the experiment provides evidence that the underlying
phenomena regarding the role of the depth hold in greater generality. This preliminary result also
suggests that the connection between the network architecture and the data geometry, as expressed
through the decay properties of the NTK, %
can have a dramatic effect on the training process even for fully-connected networks.

\newcommand{\Lp}[1]{L^{#1}}

\section{Notation}
\label{sec:notation_general}

We use bold lowercase $\mb x$ for vectors and uppercase $\mb A$ for matrices and operators. We generally use
non-bold notation to represent scalars and scalar-valued functions.  $\R, \bb C, \Z$ are used
for the real numbers, complex numbers and integers, respectively. $\N_0$ represents non-negative
integers, and $\N$ represents the natural numbers. $\R^{n}$ represents $n$-dimensional
Euclidean space, $\bbC^n$ represents the space of complex $n$-tuples (as a $n$-dimensional vector space over $\bbC$) and $\bb{S}^{n-1} \subset \bbR^n$ represents the $n-1$ dimensional sphere centered at zero with
unit radius. 
For a complex number $z = x + i y$ (or a complex-valued function), $\abs{z} = \sqrt{x^2 +
y^2}$ denotes the complex modulus, and $\conj{z} = x - i y$ denotes the complex conjugate. %
For $\vx, \vy \in \bbC^n$, we denote $\norm{\mb x}_{p} = \paren{\sum_{i =
1}^n \abs{x_i}^p}^{1/p}$ as the $p$-norm and $\ip{\vx}{\vy} = \sum_{i=1}^n \conj{x}_i y_i$ as the standard (second-argument-linear) inner product. We use $\mb x^*$ and $\mb A^*$ to represent the conjugate transpose of vectors or matrices of complex numbers (so e.g.\ $\vx\adj \vy = \ip{\vx}{\vy}$).
We use $\mb P_S$ to represent the orthogonal projection operator onto a closed subspace $S$ of a normed vector space (typically a Hilbert space). 

For a Borel measure space $(X, \mu)$ and any measurable function $f: X \to \bbC$, we use
$\norm{f}_{L^p_\mu} =\linebreak (\int_{\mb x \in X}\abs{f(\mb x)}^p d \mu(\mb x))^{1/p}$ to
represent the $L^p$ norm of $f$ for $0 < p < \infty$. We omit the measure from the notation when
it is clear from context. For $p = \infty$, we use
$\norm{f}_{L^\infty_\mu} = \inf\set{C \ge 0 \given\abs{f(\mb x)} \le C \text{ for $\mu$-almost every
$\mb x$}}$ to represent its essential supremum. We denote the $L^p$ space of $(X, \mu)$ by
$L^p_{\mu}(X)$ (or simply $L^p_\mu$ when the space is clear from context), which is formed by all
complex-valued measurable functions with finite $L^p_{\mu}$ norm. For another space $(Y, \nu)$ and a
(linear) operator $\vT: L^p_{\mu}(X) \to L^q_{\nu}(Y)$, we represent its $L_\mu^p \to L_\nu^q$
operator norm as $\norm{\vT}_{L_\mu^p \to L_\nu^q} = \sup_{\norm{f}_{L^p_\mu} = 1}
\norm{\vT[f]}_{L^q_\nu}$. %
When $X= Y$, $\mu = \nu$, and $p = q = 2$ (and $(X, \mu)$ is sufficiently regular), we have a Hilbert space; we write $\ip{f}{g}_{L^2_{\mu}} = \int_{X} \bar{f}(\vx) g(\vx) \diff \mu(\vx)$ for the inner product, and $\vT \adj$ to denote the associated adjoint of an operator $\vT$ (so e.g.\ $f\adj = \ip{f}{}$ denotes the corresponding dual element of a function $f$).
We use $\Id: L_\mu^p(X) \to L_\mu^p(X)$ to denote the identity operator,
i.e.\ $\Id[f] = f$ for every $f \in L^p_{\mu}$. For $S \subset X$, we use $\indicator{S}$ to
represent the indicator function $\indicator{S}(\mb x) = 1, \forall \mb x \in S$ and $0$
otherwise; we will write $\indicator{}$ to denote $\indicator{X}$.   For a map $\varphi: X \to X$
and $i \in N$, we use $\varphi^{[i]}$ to denote its $i$-th fold iterated composition of itself,
i.e. $\varphi^{[i]}(\mb x) = \varphi\paren{\varphi^{(i-1)}(\mb x)}$. For $i \in N$, $f^{(i)}$ is
normally used to represent a function of a real variable $f$'s $i$-th order derivatives. For example, when the space is
a two curve problem instance $\manifold$, if $\mb h: \manifold \to \bbC^n$, we define its
derivatives $\mb h^{(i)}$ in \Cref{eq:dh}; for a kernel $\Theta: \manifold \times
\manifold \to \R$, we define its derivatives along the curve in \Cref{def:deriv-R}. 

For a Borel measure space $(X, \mu)$, a kernel $K$ is a mapping $K : X \times
X \to \R$. We use $\mb K$ for its associated Fredholm integral operator. In other words,
for measurable function $f$ we have $\mb K_{\mu}[g](\mb x) = \int_{\mb x' \in X}
K(\mb x, \mb x')f(\mb x')\diff \mu(\mb x')$. %
When $X$ is a Riemannian manifold, an omitted subscript/measure will always denote the Riemannian
measure. 

We use both lowercase and uppercase letters $c, C$ for absolute constants whose value are
independent of all parameters and $c_\tau$, $C_\tau$ for numbers whose value only depend on some
parameter $\tau$. Throughout the text, $c$ is used to represent numbers whose value
should be small while $C$ is for those whose value should be large. We use $C_1, C_2, \dots$ for
constants whose values are fixed within a proof while values of $C, C', C'', \dots$ may change
from line to line.

\section{Key Definitions}
\subsection{Problem Formulation}
\label{sec:prob_form_extended}

The contents of this section will mirror \Cref{sec:two_curves}, but provide additional technical
details that were omitted there for the sake of concision and clarity of exposition. In this
sense, we will focus on a rigorous formulation of the problem here, rather than on intuition: we
encourage the reader to consult \Cref{sec:two_curves} for a more conceptually-oriented problem
formulation. As in \Cref{sec:two_curves}, we acknowledge that much of this material follows the
technical exposition of \cite{Buchanan2020-er}.

Adopting the model proposed in \cite{Buchanan2020-er}, we let $\manifold_+, \manifold_-$, denote
two class manifolds, each a smooth, regular, simple closed curve in $\mathbb{S}^{n_0-1}$, with
ambient dimension $n_0 \geq 3$. 
We further assume $\manifold$ precludes antipodal points by asking
\begin{equation}\label{eq:max_angle}
    \angle(\mb x, \mb x') \le \pi/2, \quad \forall \mb x, \mb x' \in \manifold. 
\end{equation}
We denote $\manifold = \manifold_+ \cup \manifold_-$, and the data measure supported on
$\manifold$ as $\mu$. We assume that $\mu$ admits a density $\rho$ with respect to the Riemannian
measure on $\manifold$, and that this density is bounded from below by some $\rho_{\min} > 0$. We
will also write $\rho_{\max} = \sup_{\vx \in \manifold} \rho(\vx)$. For background on curves and
manifolds, we refer the reader to to \cite{Absil2009-nc,Lee2018-xb}.

Given $N$ i.i.d.\ samples $(\mb x_1, \cdots, \mb x_N)$ from $\mu$ and their labels, given by the
labeling function $f_{\star} : \manifold \to \set{\pm 1}$ defined by
\begin{equation*}
  f_{\star}(\vx) = \begin{cases}
    +1 & \vx \in \mplus \\
    -1 & \vx \in \mminus,
  \end{cases}
\end{equation*}
we train a fully-connected network with ReLU activations and $L$ hidden layers of width $n$ and
scalar output. We will write $\vtheta = (\weight{1}, \dots, \weight{L+1})$ to denote an abstract
set of admissible parameters for such a network; concretely,
the features at layer $\ell \in \set{1, 2, \dots, L}$ with parameters $\vtheta$ and input $\vx$
are written as
$\bm{\alpha}_{\vtheta}^{\ell}(\bm{x})=\left[\mathbf{W}^{\ell}\bm{\alpha}_{\vtheta}^{\ell-1}(\bm{x})\right]_{+}$,
where $\relu{x} = \max \set{x, 0}$ denotes the ReLU (and we adopt in general the convention of
writing $\relu{\vx}$ to denote application of the scalar function $\relu{}$ to each entry of the
vector $\vx$), with boundary condition
$\bm{\alpha}_{\vtheta}^{0}(\bm{x})=\bm{x}$, and the network output on an input $\bm x$ is written
$f_{\vtheta}(\vx) = \mathbf{W}^{L+1}\bm{\alpha}_{\vtheta}^{L}(\bm{x})$. 
We will also write $\zeta_{\vtheta}(\vx) = f_{\vtheta}(\vx) - f_{\star}(\vx)$ to denote the
fitting error.
We use Gaussian initialization: if $\ell \in \set{1, 2, \dots, L}$, the weights are initialized as 
$W_{ij}^{\ell}\simiid \mathcal{N}(0,\frac{2}{n})$, and the top level weights are initialized as
$W_{i}^{L+1}\simiid \mathcal{N}(0,1)$ in order to preserve the expected
feature norm.\footnote{This initialization style is common in practice (it might be referred to as
  ``fan-out initialization'' in that context), but less common in the theoretical literature on
  kernel regime training of deep neural networks, where a less-natural ``NTK parameterization'' is
  typically employed. A detailed discussion of these differences, and how to translate results for
  one parameterization into those for another, can be found (for example) in \cite[\S
A.3]{Buchanan2020-er}.} In the sequel, we will write $\vtheta_0$ to denote the collection of these
initial random parameters, and therefore $f_{\vtheta_0}$ to denote the initial random network. 

We will employ a convenient ``empirical measure'' notation to concisely represent finite-sample and
population quantities in the analysis. Let $\mu^N = \frac{1}{N} \sum_{i=1}^N \delta_{\set{\vx_i}}$
denote the empirical measure associated to our i.i.d.\ random sample from the population measure
$\mu$, where $\delta_{\vp}$ denotes a Dirac measure at a point $\vp$.
We train on the square loss $\loss_{\mu^N}(\vtheta) = (1/2) \int_{\manifold}
\left(\zeta_{\vtheta}(\vx)\right)^2 \diff \mu^N(\vx)$ (of course one simply has
$\loss_{\mu^N}(\vtheta) = 1/(2N) \sum_{i=1}^N \left(\zeta_{\vtheta}(\vx_i)\right)^2$), which we
minimize using randomly-initialized ``gradient descent'' starting at $\vtheta_0$ with constant step
size $\tau>0$. We put gradient descent in quotations here because the loss $\loss_{\mu^N}$ is only
almost-everywhere differentiable, due to the nondifferentiability of the ReLU activation
$\relu{}$: in this sense our algorithm for minimization is `gradient-like', in that it corresponds
to a gradient descent iteration at almost all values of the parameters. Concretely, we define
\begin{equation*}
  \bfeature{\ell}{\vx}{\vtheta} 
  = \left(
    \weight{L+1} \fproj{L}{\vx} \weight{L} \fproj{L-1}{\vx} \dots \weight{\ell+2} \fproj{\ell+1}{\vx}
  \right)\adj
\end{equation*}
for $\ell = 0, 1, \dots, L-1$, where
\begin{equation*}
  \fsupp{\ell}{\vx} = 
  \set*{
    i \in [n] \given \ip*{\ve_i}{\feature{\ell}{\vx}{\vtheta}}> 0 
  },
  \qquad
  \fproj{\ell}{\vx} = \sum_{i \in \fsupp{\ell}{\vx}} \ve_i \ve_i \adj
\end{equation*}
denotes the orthogonal projection onto the set of coordinates where the $\ell$-th activation at
input $\vx$ is positive (above, $\ve_i$ denotes the $i$-th canonical basis vector, having its
$j$-th entry equal to $1$ if $j=i$ and $0$ otherwise).
Then we define `formal gradients' of the network output with respect to the parameters (denoted by
an operator $\wt{\nabla}$) by 
\begin{equation*}
  \wt{\nabla}_{\weight{\ell}} f_{\vtheta}(\vx)
  =
  \bfeature{\ell-1}{\vx}{\vtheta} \feature{\ell-1}{\vx}{\vtheta}\adj
\end{equation*}
for $\ell \in [L]$, and
\begin{equation*}
  \wt{\nabla}_{\weight{L+1}} f_{\vtheta}(\vx)
  =
  \feature{L}{\vx}{\vtheta}\adj.
\end{equation*}
As stated above, these expressions agree with the actual gradients at points of differentiability
(to see this, apply the chain rule).  We then define a formal gradient of $\loss_{\mu^N}$ by
\begin{equation*}
  \wt{\nabla} \loss_{\mu^N}(\vtheta) = \int_{\manifold} \wt{\nabla} f_{\vtheta}(\vx)
  \zeta_{\vtheta}(\vx) \diff \mu^N(\vx).
\end{equation*}
Thus, our gradient-like algorithm we study here is given by the sequence of parameters
$\vtheta_{k+1} = \vtheta_k - \tau \wt{\nabla} \loss_{\mu^N}(\vtheta_k)$, with $\vtheta_0$ given by
the Gaussian initialization we describe above.

Our study of this gradient-like iteration is facilitated by using kernel regime techniques, which
we will describe now. Formally, the gradient descent iteration implies the following ``error
dynamics'' equation:
\begin{equation*}
  \prederr_{\vtheta^{N}_{k+1}}(\vx)
  =
  \prederr_{\vtheta^{N}_{k}}(\vx) - \tau  
  \int_{\manifold} \ntk_k^N(\vx, \vx') \prederr_{\vtheta^{N}_{k}}(\vx') \diff \mu^N(\vx'),
\end{equation*}
where 
$ \ntk_k^N(\vx, \vx') = \int_0^1 \ip{ \wt{\nabla}{f_{\vtheta_k^N}(\vx')} } { \wt{\nabla}{f_{
\vtheta_{k}^N - t \tau \wt{\nabla} \loss_{\mu^N}(\vtheta_k^N) }(\vx)} } \diff t $.
For a proof of this claim, see \cite[Lemma B.8]{Buchanan2020-er}. As we describe in
\Cref{sec:two_curves}, under suitable conditions on the network width, depth, and the number of
samples, this error dynamics update is well-approximated by a ``nominal dynamics'' update equation 
defined by $\zeta_{k+1} = \left( \Id - \tau \vTheta_{\mu}^{\mathrm{NTK}}\right)[\zeta_k]$ with
boundary condition $\zeta_0 = \zeta_{\vtheta_0}$,
where $\Theta^{\mathrm{NTK}}(\vx, \vx') = \ip{\wt{\nabla} f_{\vtheta_0}(\vx)}{\wt{\nabla}
f_{\vtheta_0}(\vx')}$ is the ``neural tangent kernel''. The analysis of this nominal evolution
leads us to the certificate problem that we have posed in \Cref{sec:two_curves}, and which we
resolve for the two curve problem in this work.

In the remainder of this section, we introduce several notations for quantities related to the
certificate problem which we will refer to throughout these appendices.
We let $\ntk$ denote the following approximation to the neural tangent kernel: 
\begin{equation}\label{def:ntk}
\ntk(\vx, \vx')
=
\frac{n}{2} \sum_{\ell = 0}^{L-1} \prod_{\ell' = \ell}^{L-1}
\left(
1 - \frac{\ivphi{\ell'}( \angle(\vx, \vx') )}{\pi}
\right),
\end{equation}
where $\ivphi{\ell}$ denotes the $\ell$-fold composition of the angle evolution function
$\varphi(t)=\cos^{-1}\left((1-\frac{t}{\pi})\cos t+\frac{\sin t}{\pi}\right)$.
We let $\prederr$ denote the following piecewise constant approximation to $\zeta_{0}$:
\begin{equation}
\prederr(\vx) = -\target(\vx) + \int_{\manifold} f_{\vtheta_0}(\vx') \diff \mu(\vx'). \label{eq:zeta}
\end{equation}
We also use the notation
\begin{align*}
\ixi{\ell}(t) &= \prod_{\ell' = \ell}^{L-1}\left(
1 - \frac{\ivphi{\ell'}(t)}{\pi}
\right) 
\\
\psi(t) &= \frac{n}{2} \sum_{\ell = 0}^{L-1}\ixi{\ell}(t)
\end{align*}
for convenience. We find it convenient in our analysis to consider $\psi$ and its ``DC
component'', i.e., its value at $\pi$, separately. To this end, we write $\psi^\circ = \psi -
\psi(\pi)$. We also write the subtracted approximate NTK as $\Theta^\circ(\vx, \vx') =
\psi^\circ(\angle(\vx,\vx'))$. As a consequence, we have
\begin{equation}
\psi^{\circ}(\angle(\mb x, \mb x')) = \ \Theta^{\circ}(\mb x, \mb x') =  \Theta(\mb x, \mb x') -
\psi(\pi). \label{eqn:psi_circ}
\end{equation}
We use $\mb \Theta_\mu$ to represent the integral operator with
\begin{equation*}
\mb \Theta_\mu[g](x) = \int_{\manifold} \Theta(\mb x, \mb x')g(\mb x')\diff\mu(\mb x'),
\end{equation*}
and similarly for $\mb \Theta_\mu^\circ$.
An omitted subscript/measure will denote the Riemannian measure on $\manifold$.

\subsection{Geometric Properties}
\label{sec:geo_defs}

	We assume our data manifold $\manifold = \manifold_+ \cup \manifold_-$, where $\manifold_+$ and
  $\manifold_-$ each is a smooth, regular, simple closed curve on the unit sphere $\mathbb{S}^{n_0
    - 1}$. Because the curves are regular, it is without loss of generality to assume they are
    unit-speed and parameterized with respect to arc length $s$, giving parameterizations as maps from $[0, \len(\manifold_\sigma)]$ to $\Sph^{\adim - 1}$, as we have defined them in \Cref{sec:geometry_def} of the main body. 
    Throughout the appendices, we will find it convenient to consider periodic extensions of these arc-length parameterizations, which are smooth and well-defined by the fact that our manifolds are smooth, closed curves:
    for $\sigma \in \{\pm\}$, we use $\mb x_{\sigma}(s): \R \to \mathbb{S}^{n_0-1}$ to
    represent these parameterizations of the two manifolds.\footnote{We clarify an abuse of notation we will commit with these
  parameterizations throughout the analysis, which stems from the fact that the curves are closed
  (i.e. topologically circles). That is, there is no preferred basepoint (i.e.\ the points $\vx_{\sigma}(0)$) for the arc
  length parameterizations (the curves are only defined up to translation): because our primary use for these parameterizations
  is in the analysis of extrinsic distances between points on the curves, the basepoint will be irrelevant. }
We require that the two curves are disjoint. Notice that as
the two curves do not self intersect, we have $\mb x_{\sigma}(s) = \mb x_{\sigma'}(s')$ if and
only if $\sigma = \sigma'$  and $s' = s + k \len(\manifold_{\sigma})$ for some $k \in \Z$.
Precisely, our arguments will require our curves to have `five orders' of smoothness, in other
words $\mb x_{\sigma}(s)$ must be five times continuously differentiable for $\sigma \in\set{ +,-}$. 
	
	For a differentiable function $\mb h: \manifold \to \bbC^p$ with $p \in \N$, we define its derivative $\tfrac{d}{ds} \mb h$ as 
	\begin{align*}
	\frac{d}{ds} \mb h(\mb x) &= \left[ \frac{d}{dt} \Bigr|_s \mb h\bigl( \mb x_\sigma(t)\bigr) \right] \Biggr|_{\mb x_{\sigma}(s) = \mb x} \\
	&= \left[ \lim_{t \to 0} \frac{1}{t}(\mb h(\mb x_\sigma(s+t)) - \mb h(\mb x_\sigma(s))) \right] \Biggr|_{\mb x_{\sigma}(s) = \mb x} . \labelthis \label{eq:dh}
	\end{align*}
	We call attention to the ``restriction'' bar used in this notation: it should be read as ``let $s$ and $\sigma$ be such that $\vx_{\sigma}(s) = \vx$'' in the definition's context. This leads to a valid definition in 
  \eqref{eq:dh} because our curves are simple and disjoint, so for any choice $s, s'$ with $\mb x_{\sigma}(s) = \mb
  x_{\sigma}(s') = \mb x$, we have $\mb x_{\sigma}(s+t) = \mb x_{\sigma}(s' + t)$ for all $t$. 
  We will use this notation systematically throughout these appendices.
  We
  further denote its $i$-th order derivative by $\mb h^{(i)}(\mb x)$. For $i \in \N$, we use $\mc
  C^i(\manifold)$ to represent the collection of real-valued functions $h: \manifold \to \R$ whose
  derivatives $h^{(1)}, \dots, h^{(i)}$ exist and are continuous. 

	In particular, consider the inclusion map $\mb \iota: \manifold \to \R^{n_0}$, which is the
  identification $\mb \iota(\vx) = \vx$. Following the definition as above, we have
	\begin{equation}
	\mb \iota^{(i+1)}(\mb x) =\left[  \lim_{t \to 0} \frac{1}{t}(\mb \iota^{(i)}(\mb x_{\sigma}(s+t)) - \mb \iota^{(i)}(\mb x_{\sigma}(s)))\right] \Biggr|_{\mb x_{\sigma}(s) = \mb x} . 
	\label{eq:inclusion_abuse}
	\end{equation}
	In the sequel, with abuse of notation we will use $\mb x^{(i)}$ to represent $\mb \iota^{(i)}(\mb x)$. 
	For example, we will write expressions such as $\sup_{\vx \in \manifold} \norm{\vx^{(2)}}_2$ to denote the quantity $\sup_{\vx \in \manifold} \norm{\mb \iota^{(2)}(\vx)}_2$.
  This notation will enable increased concision, and it is benign, in the sense that it is
  essentially an identification. We call attention to it specifically to note a possible conflict
  with our notation for the parameterizations and their derivatives $\vx^{(i)}_{\sigma}$, which
  are maps from $\bbR$ to $\bbR^{\adim}$ (say), rather than maps defined on $\manifold$.
	In this context, we also use $\dot{\mb x}$ and $\ddot{\mb x}$ to represent first and second derivatives $\mb x^{(1)}$ and $\mb x^{(2)}$ for brevity. 
	We have $\norm{\mb x}_2 = \norm{\dot{\mb x}}_2 = 1$ from the fact that $\manifold \subset
  \Sph^{\adim-1}$ and that we have a unit-speed parameterization. This and associated facts are
  collected in \Cref{lem:x_deriv}.
	
For any real or complex-valued function $h$, the integral operator over manifold can be written as
\begin{align*}
	\int_{\mb x \in \manifold} h(\mb x) d\mu(\mb x) &= \sum_{\sigma = \pm} \int_{s = 0}^{\len(\manifold_{\sigma})} h(\mb x_{\sigma}(s))\rho(\mb x_\sigma(s)) ds, \\
	\int_{\mb x \in \manifold} h(\mb x) d \mb x &= \sum_{\sigma = \pm} \int_{s = 0}^{\len(\manifold_{\sigma})} h(\mb x_{\sigma}(s)) ds. 
\end{align*}

We have defined key geometric properties in the main body, in \Cref{sec:geometry_def}. Our
arguments will require slightly more technical definitions of these quantities, however. In the
remainder of this section, we introduce the same definition of angle injectivity radius and \fco-number
with a variable scale, which helps us in proofs in \Cref{sec:proof_certificate_theorem}. 

First, we give a precise definition for the intrinsic distance $d_{\manifold}$ on the curves.
To separate the notions of ``close over the sphere'' and ``close over the manifold'', we use the
extrinsic distance (angle) $\angle(\mb x, \mb x') = \acos\innerprod{\mb x}{\mb x'}$ to
measures closeness between two points $\mb x$, $\mb x'$ over the sphere. The distance over the
manifold is measured through the intrinsic distance $d_{\manifold}(\mb x, \mb x')$, which takes
$\infty$ when $\mb x$ and $\mb x'$ reside on different components $\manifold_+$ and $\manifold_-$
and the length of the shortest curve on the manifold connecting the two points when they belong to
the same component. More formally, we have 
\begin{equation}
  d_{\manifold}(\mb x, \mb x') 
  = \begin{cases} 
    \inf \setcolon{\abs{s - s'} \given \vx_{\sigma}(s) = \vx,\, \vx_{\sigma}(s') = \vx'}
    & f_{\star}(\vx) = f_{\star}(\vx'), \\
    +\infty & \ow,%
  \end{cases}
  \label{eq:intrinsic_distance_function}
\end{equation}
where the infimum is taken over all valid $\sigma \in \set{+,-}$ and $(s, s') \in \bbR^2$. 
Notice that as the curves $\manifold_{\sigma}$ do not intersect themselves, one has
$x_{\sigma}(s_1) = x_{\sigma}(s_2)$ if and only if $s_1 = s_2 + k \len(\manifold_\sigma)$ for some
$k \in \Z$. Thus for any two points $\mb x, \mb x'$ that belong to the same component
$\manifold_{\sigma}$, the above infimum is attained: there exist $s, s'$ such that $\mb
x_{\sigma}(s) = \mb x, \mb x_{\sigma}(s') = \mb x'$, and $d_{\manifold}(\mb x, \mb x') = \abs{s -
s'}$.

\paragraph{Angle Injectivity Radius}
For $\veps \in (0, 1)$ we define the angle injectivity radius of scale $\veps$ as
\begin{equation}\label{eq:inj_radius_eps_def}
\Delta_{\veps} = \min\left\{\frac{\sqrt{\veps}}{\hat{\kappa}}, \inf_{\mb x, \mb x' \in \manifold}
\left\{ \angle(\mb x, \mb x') \,\middle|\, d_{\mc M}( \mb x, \mb x') \ge \frac{
\sqrt{\veps}}{\hat{\kappa}} \right\}\right\},
\end{equation}
which is the smallest extrinsic distance between two points whose intrinsic distance exceeds $\frac{\sqrt{\veps}}{\hat{\kappa}}$ with
\begin{equation}\label{eq:khat_def}
\hat{\kappa} = \max\left\{\kappa, \frac{2}{\pi}\right\}.
\end{equation}
Observe that for any scale $\veps$, 
$\Delta_{\veps}$ is smaller than inter manifold separation $\min_{\mb x \in \manifold_+, \mb  x' \in \manifold_-} \angle(\mb x, \mb x') $. 

\paragraph{\fco-number}

For $\veps \in (0, 1), \delta \in (0, (1-\veps)]$, we define \fco-number of scale $\veps, \delta$ as
\begin{equation}\label{eq:cnumber_eps_def}
  \fco_{\veps, \delta}(\mc M) = \sup_{\mb x \in \mc M} N_{\mc M}\left( \left\{ \mb x' \; \middle| \; d_{\mc M}( \mb x, \mb x') \ge \frac{\sqrt{\veps}}{\hat{\kappa}} \; \text{and} \; \angle(\mb x, \mb x') \le \frac{\delta\sqrt{\veps}}{\hat{\kappa}} \right\}, \frac{1}{\sqrt{1+\kappa^2}}  \right). 
\end{equation}
Here, $N_{\mc M}( T, \veps )$ is the size of a minimal $\veps$ covering of $T$ in the intrinsic
distance on the manifold. We call the set $ \left\{ \mb x' \, \middle | \,  d_{\manifold}(\mb x,
\mb x') \ge \frac{\sqrt{\veps}}{\hat{\kappa}}, \; \angle( \mb x, \mb x') \le \frac{\delta
\sqrt{\veps}}{\hat{\kappa}} \right\} $ appearing in this definition the winding piece of scale
$\veps$ and $\delta$: it contains points that are far away in intrinsic distance but close in
extrinsic distance. We will give it a formal definition in \Cref{eqn:winding_term}, where it will
play a key role in our arguments.

In the sequel, we denote $\Delta, \fco(\manifold)$ to be the angle injectivity radius and \fco-number
with the specific instantiations $\veps = \frac{1}{20}$ and $\delta = 1 - \veps$. These are key
geometric features used in \Cref{thm:certificate_main} and \Cref{thm:main_formal}. 

\subsection{Subspace of Smooth Functions and Kernel Derivatives}\label{sec:subspace_def}
As the behavior of the kernel and its approximation is easier to understand when constrained in a
low frequency subspace, we first introduce the notion of low-frequency subspace formed by the Fourier
basis on the two curves.

\paragraph{Fourier Basis and Subspace of Smooth Functions} 
We define a Fourier basis of functions over the manifold as 
\begin{equation}\label{eq:fourier_basis}
  \phi_{\sigma, k}( \mb x_{\sigma'}(s) ) = \begin{cases}
    \frac{1}{\sqrt{\mr{len}(\mc M_{\sigma})}} \exp\left(  \frac{ \mf i 2 \pi k s }{\mr{len}(\mc M_{\sigma})} \right), & \sigma' = \sigma\\
    0, & \sigma' \ne \sigma
  \end{cases}
\end{equation}
for each $k = 0, 1, \dots$, 
and further define a subspace of low frequency functions
\begin{equation}\label{eq:smooth_subspace}
  S_{K_{+}, K_{-}} = \mr{span}_{\bbC}\{ \phi_{+, 0}, \phi_{+, -1}, \phi_{+, 1}, \dots, \phi_{+, -K_+}, \phi_{+, K_+}, \phi_{-, 0}, \dots, \phi_{-, K_-} \}
\end{equation}
for $K_+, K_- \ge 0$. Using the fact that our curves are unit-speed, one can see that indeed
\Cref{eq:fourier_basis} defines an orthonormal basis for $L^2$ functions on $\manifold$.

\section{Main Results} \label{sec:main_thm}

\begin{theorem}[Generalization]\label{thm:we_generalize}
  Let $\manifold$ be two disjoint smooth, regular, simple closed curves, satisfying $\angle(\mb x, \mb x') \le \pi/2$ for all $\mb x, \mb x' \in \manifold$.
  For any $0 < \delta \leq 1/e$, choose $L$ so that
    \begin{align*}
    L &\geq K \max\set*{
      \frac{1}{\left( \Delta(1 + \kappa^2) \right)^{C \fco(\manifold)}},
      C_{\mu} \log^9(\tfrac{1}{\delta}) 
      \log^{24} ( C_{\mu} \adim \log(\tfrac{1}{\delta}) ),
      e^{C' \max \set{\len(\sM)\hat{\kappa}, \log(\hat{\kappa})}},
    P }\\
 n &= K'L^{99}\log^{9}(1/\delta)\log^{18}(L\adim) \\
 N &\geq L^{10},
  \end{align*}
  and fix $\tau > 0$ such that $\frac{C''}{nL^2} \leq \tau \leq \frac{c}{nL}$.
  Then with probability at least $1 - \delta$, the parameters obtained at iteration
  $\floor{L^{39/44} / (n\tau)}$ of gradient descent on the finite sample loss %
  yield a classifier that separates the two manifolds.

  The constants $c, C, C', C'', K, K' > 0$ are absolute, 
  and the constant $C_{\mu}$ is equal to
  $ \frac{\max \set{\rho_{\min}^{19}, \rho_{\min}^{-19}} (1 + \rho_{\max})^{12}  }
  {\left(\min\,\set*{\mu(\mplus), \mu(\mminus)} \right)^{11/2}}$. 
  $P$ is a polynomial $\poly\{M_3, M_4, M_5, \len(\manifold), \Delta^{-1}\}$ of degree at most $36$, with degree at most $12$ when viewed as a polynomial in $M_3, M_4, M_5$ and $\len(\manifold)$, and of degree at most $24$ as a polynomial in $\Delta^{-1}$.
\end{theorem}	

\begin{proof}
  The proof is an application of \Cref{thm:1}; we note that the
  conditions on $n$, $L$, $\delta$, $N$, and $\tau$ imply all hypotheses of this theorem, except
  for the certificate condition. We will complete the proof by showing that the certificate
  condition is also satisfied, under the additional hypotheses on $L$ and with a suitable choice
  of $\qcert$.

  First, we navigate a difference in the formulation of the two curves' regularity properties
  between our work and \cite{Buchanan2020-er}, from which \Cref{thm:1} is drawn. 
  \Cref{thm:1} includes a condition $L\geq C \kappa^2_{\mathrm{ext}}
  C_\lambda $ for some absolute constant $C$, where $\kappa^2_{\mathrm{ext}}  =
  {\sup}_{\bm{x}\in\mathcal{M}} \norm{\ddot{\mb x}}_2^2$ is a bound on the extrinsic
  curvature (we will discuss $C_{\lambda}$ momentarily). In our context, we have $M_2 =
  \kappa_{\mathrm{ext}}$, and following \Cref{lem:x_deriv} (using that our curves are unit-speed
  spherical curves), we get that it suffices to require $L \gtrsim (1 + \kappa^2) C_{\lambda}$
  instead. In turn, we can pass to $\hat{\kappa}$: since this constant is lower-bounded by a
  positive number and is larger than $\kappa$, it suffices to require $L \gtrsim \hat{\kappa}^2
  C_{\lambda}$.
  As for $C_\lambda$, this is a constant related to the angle injectivity radius $\Delta$, and is defined by $C_\lambda=K_\lambda^2/c_\lambda^2$, where these two
  constants satisfy 
  \begin{equation*}
    \forall s\in\left(0,c_{\lambda}/\kappa_{\mathrm{ext}}\right],(\mb x,\mb x')\in\mathcal{M}_{\blank}\times\mathcal{M}_{\blank}
    ,\,\blank\in\{+,-\}
    \quad:\quad
    \angle(\vx,\vx')\leq s\Rightarrow d_{\mathcal{M}}(\vx,\vx')\leq K_{\lambda}s.
  \end{equation*}
  We will relate this constant to constants in our formulation.
  Consider any $\bm{x},\bm{x}'\in\mathcal{M}$. If $\angle(\bm{x},\bm{x}')\leq\frac{\Delta}{2}$
  then from the definition of $\Delta$ we have
  $d_{\mathcal{M}}(\bm{x},\bm{x}')\leq\frac{\sqrt{\varepsilon}}{\hat{\kappa}}$ and hence by
  \cref{eqn:near-isometry} we find $d_{\mathcal{M}}(\bm{x},\bm{x}')\leq\angle(\bm{x},\bm{x}')$.
  If on the other hand $\angle(\bm{x},\bm{x}')>\frac{\Delta}{2}$, then a trivial bound gives
  \begin{equation}
    \label{eqn:global_reg_BGW}
    d_{\mathcal{M}}(\bm{x},\bm{x}')\leq\mathrm{len}(\mathcal{M})=\frac{2\mathrm{len}(\mathcal{M})}{\Delta}\frac{\Delta}{2}<\frac{2\mathrm{len}(\mathcal{M})}{\Delta}\angle(\bm{x},\bm{x}').
  \end{equation}
  We can thus choose $c_\lambda=1, K_{\lambda}=\max\left\{
  1,\frac{2\mathrm{len}(\mathcal{M})}{\Delta}\right\}$ to satisfy \cref{eqn:global_reg_BGW},
  giving $C_{\lambda}=\max\left\{ 1,\frac{4\mathrm{len}^{2}(\mathcal{M})}{\Delta^{2}}\right\} $.
  Thus the requirement $L>C \kappa^2_{\mathrm{ext}} C_\lambda $ of \Cref{thm:1} is automatically
  satisfied if $L\gtrsim \max \set{P, e^{C \len(\manifold) \hat{\kappa}}}$ for a suitable exponent
  $C$, where $P$ is the polynomial in the hypotheses of our result, and so our hypotheses imply
  this condition.

  Next, we establish the certificate claim. The proof will follow closely the argument of
  \cite[Proposition B.4]{Buchanan2020-er}. Write $\Theta^{\mathrm{NTK}}$ for the network's
  neural tangent kernel, 
  as defined in \Cref{sec:prob_form_extended}, and $\vTheta^{\mathrm{NTK}}_{\mu}$ for the
  associated Fredholm integral operator on $L^2_\mu$. In addition, write $\zeta_{0} =
  f_{\vtheta_0} - f_{\star}$ for the initial random network error.
  Because we have modified some exponents in the constant $C_{\mu}$, and added conditions on
  $L$, all hypotheses of \Cref{thm:skeleton_dc_density} are satisfied: invoking it, we have that
  there exists $g : \manifold \to \bbR$ satisfying
  \begin{equation*}
    \norm{g}_{L^2_\mu}  \leq C \frac{\norm{\zeta}_{L^2_\mu}}{\rho_{\min} n}   
  \end{equation*}
  and
  \begin{equation*}
    \norm{ \vTheta_\mu[g] - \zeta }_{L^2_\mu} \leq \frac{\norm{\zeta}_{L^\infty}}{L}.
  \end{equation*}
  By these bounds, the triangle inequality, the Minkowski inequality, and the fact that
  $\mu$ is a probability measure, we have
  \begin{align*}
    \norm*{ \ntkoper^{\mathrm{NTK}}_{\mu}[g] - \prederr_0 }_{L^2_{\mu}}
    &\leq
    \norm{\ntk - \ntk^{\mathrm{NTK}}}_{L^\infty(\manifold \times \manifold)}
    \norm{g}_{L^2_{\mu}}
    + \norm{\vTheta_\mu[g] - \zeta}_{L^2_\mu}
    + \norm{\prederr - \prederr_0}_{L^2_{\mu}} \\
    &\leq
    C \norm{\ntk - \ntk_{\mathrm{NTK}}}_{L^\infty(\manifold \times \manifold)}
    \frac{ \norm{\prederr}_{L^\infty(\manifold)}}{n \rho_{\min}}
    + \frac{\norm{\prederr}_{L^\infty}}{L}
    + \norm{\prederr - \prederr_0}_{L^\infty(\manifold)}.
    \labelthis \label{eq:certa_implies_acert}
  \end{align*}
  An application of \Cref{thm:2} gives that on an event of probability at least $1 - e^{-cd}$
  \begin{equation*}
    \norm{\ntk - \ntk^{\mathrm{NTK}}}_{L^\infty(\manifold \times \manifold)}
    \leq Cn/L
  \end{equation*}
  if $d \geq K \log(n n_0 \len(\manifold))$ and $n \geq K' d^4 L^5$. 
  In translating this result from \cite{Buchanan2020-er}, we use that in the context of the two
  curve problem, the covering constant $C_{\manifold}$ appearing in \cite[Theorem
  B.2]{Buchanan2020-er} is bounded by a constant multiple of $\len(\manifold)$ (this is how we
  obtain \Cref{thm:2} and some other results in \Cref{sec:aux}).  An application of
  \Cref{lem:zeta_approx_bound} gives 
  \begin{equation*}
    \Pr*{
      \norm*{\prederr_0 - \prederr}_{L^\infty(\manifold)}
      \leq
      \frac{\sqrt{2d}}{L}
    }
    \geq 1 - e^{-cd}
  \end{equation*}
  and
  \begin{equation*}
    \Pr*{
      \norm*{\prederr_0}_{L^\infty(\manifold)}
      \leq
      \sqrt{d}
    }
    \geq 1 - e^{-cd}
  \end{equation*}
  as long as $n \geq K d^4 L^5$ and $d \geq K' \log(n n_0 
  \len(\manifold))$, where we use these conditions to simplify the residual that appears in
  \Cref{lem:zeta_approx_bound}.
  In particular, combining the previous two bounds with the
  triangle inequality and a union bound and then rescaling $d$, which worsens the constant $c$
  and the absolute constants in the preceding conditions, gives 
  \begin{equation*}
    \Pr*{
      \norm*{\prederr}_{L^\infty(\manifold)}
      \leq
      \sqrt{d}
    }
    \geq 1 - 2e^{-cd}.
  \end{equation*}
  Combining these bounds using a union bound and substituting into
  \Cref{eq:certa_implies_acert}, we get that under the preceding conditions, on an event of
  probability at least $1 - 3 e^{-cd}$ we have
  \begin{align*}
    \norm*{ \ntkoper_\mu^{\mathrm{NTK}}[g] - \prederr_0 }_{L^2_{\mu}}
    &\leq
    \frac{C \sqrt{d}}{L}
    \left(
      2+
      \frac{1}{\rho_{\min} }
    \right) \\
    &\leq
    \frac{C\sqrt{d}}{L} \max \set{\rho_{\min}, \rho_{\min}^{-1}},
    \labelthis \label{eq:cert_err_want}
  \end{align*}
  where we worst-case the density constant in the second line,
  and in addition, on the same event, we have by the norm bound on the certificate $g$
  \begin{equation}
    \norm{g}_{L^2_{\mu}} \leq C
    \frac{\sqrt{d}}{n \rho_{\min}}.
    \label{eq:cert_norm_want}
  \end{equation}
  To conclude, we simplify the preceding conditions on $n$ and turn the parameter $d$ into a
  parameter $\delta > 0$ in order to obtain the form of the result necessary to apply 
  \Cref{thm:1}.
  We have in this one-dimensional setting
  \begin{equation*}
    \len(\manifold) \leq \frac{\len(\mplus)}{\mu(\mplus)} + \frac{\len(\mminus)}{\mu(\mminus)}
    \leq \frac{2}{\rho_{\min}} \leq 2 \max \set{\rho_{\min}, \rho_{\min}^{-1}},
  \end{equation*}
  where the second inequality here uses simply
  \begin{equation*}
    \mu(\mplus) = \int_{\mplus} \rho_+(\vx) \diff \vx \geq \len(\mplus) \rho_{\min}
  \end{equation*}
  (say).
  Because $n \geq 1$ and $\adim \geq 3$ and $\max\set{\rho_{\min}, \rho_{\min}^{-1}} \geq 1$, it
  therefore suffices to instead enforce the condition on $d$ as $d \geq K \log(n\adim C_{\mu})$,
  where $C_\mu$ is the constant defined in the lemma statement.  But note from our hypotheses here
  that we have $n \geq L$ and $L \geq C_{\mu}$; so in particular it suffices to enforce $d \geq K
  \log(n\adim)$ for an adjusted absolute constant.
  Choosing $d \geq (1/c) \log(1/\delta)$, we obtain that the previous two bounds \Cref{eq:cert_err_want,eq:cert_norm_want} hold on an
  event of probability at least $1 - 3\delta$.
  When $\delta \leq 1/e$, given that $\adim \geq 3$
  we have
  $n\adim \geq e$ and $\max \set{\log (1/\delta), \log(n \adim )} \leq
  \log(1/\delta)\log(n\adim)$, so that it suffices to enforce the requirement $d
  \geq K \log(1/\delta)\log(n\adim)$ for a certain absolute constant $K>0$.
  We can then substitute this lower bound on $d$ into the two certificate bounds above to obtain
  the form claimed in \Cref{eq:mainthm_cert_condition} in \Cref{thm:1} with the instantiation
  $q_{\mathrm{cert}}=1$, and this setting of $q_{\mathrm{cert}}$ matches the choice of $C_{\mu}$
  that we have enforced in our hypotheses here. For the hypothesis on $n$, we substitute this
  lower bound on $d$ into the condition on $n$ to obtain the sufficient
  condition $n \geq K' L^5 \log^4(1/\delta) \log^4(n\adim)$. Using
  a standard log-factor reduction (e.g.\ \cite[Lemma B.15]{Buchanan2020-er})
  and possibly worsening absolute constants, we then get that it
  suffices to enforce $n \geq K' L^5 \log^4(1/\delta) \log^4(L \adim
  \log(1/\delta))$, which is redundant with the (much larger) condition on $n$ that we have
  enforced here. This completes the proof.
\end{proof}
    
\begin{theorem}[Certificates]\label{thm:skeleton_dc_density}
\flushleft{Let $\manifold$ be two disjoint smooth, regular, simple closed curves, satisfying $ \angle(\vx, \vx') \leq\pi/2 $ for all $\vx, \vx' \in \manifold$. There exist constants $C, C', C'', C'''$ and a polynomial }\linebreak $P = \poly(M_3,
M_4, M_5, \len(\manifold), \Delta^{-1})$ of degree at most $36$, with degree at most $12$ in
$(M_3, M_4, M_5, \len(\manifold))$ and degree at most $24$ in $\Delta^{-1}$, such that when
	\begin{equation*}
	    L \ge \max\Biggl\{ \exp(C'\len(\manifold) \hat{\kappa}), \paren{ \frac{1}{\Delta \sqrt{1 +
      \kappa^2}}}^{C''\fco(\manifold)}, C'''\hat{\kappa}^{10}, P, \rho_{\max}^{12} \Biggr\},
	\end{equation*}
	then for $\zeta$ defined in \eqref{eq:zeta}, there exists a certificate { $g: \manifold \to \R$} with 
	\begin{equation*}
	\| g \|_{L^2_\mu} \le \frac{C \norm{\zeta}_{L^2_\mu}}{ \rho_{\min} n \log L} 
	\end{equation*} 
	such that
	\begin{equation*}
	    \norm{\mb \Theta_{\mu}[g] - \zeta}_{L^2_\mu} \le \norm{\zeta}_{L^\infty}L^{-1}.
	\end{equation*}
\end{theorem}
\begin{proof}
  This is a direct consequence of \Cref{lem:skeleton_dc_density}. Notice that $\zeta$ in
  \Cref{eq:zeta} is a real, piecewise constant function over the manifolds, and therefore has its higher
  order derivatives vanish. This makes it directly belong to $\Phi(\norm{\zeta}_{L^2},
  \frac{1}{20})$ defined in \Cref{def:smoothness} and satisfy the condition in
  \Cref{lem:skeleton_dc_density} with $K = 1$. 
\end{proof}

\section{Proof for the Certificate Problem}
\label{sec:proof_certificate_theorem}

The goal of this section is to prove \Cref{lem:skeleton_dc_density}, a generalized version of
\Cref{thm:skeleton_dc_density}. Instead of showing the certificate exists for the particular
piecewise constant function $\zeta$ defined in \eqref{eq:zeta}, as claimed in
\Cref{thm:skeleton_dc_density}, \Cref{lem:skeleton_dc_density} claims that for \textit{any}
reasonably $\zeta$ with bounded higher order derivatives, there exists a small norm certificate $g$ such that $\mb \Theta_{\mu}[g]
\approx \zeta$.  There are two main technical difficulties in establishing this result. First,
$\mb \Theta$ contains a very large constant term: $\mb \Theta = \mb \Theta^\circ + \psi(\pi)
\indicator{}\indicator{}\adj$.
This renders the operator $\mb \Theta$ somewhat ill-conditioned. Second, the eigenvalues of $\mb
\Theta^\circ$ are not bounded away from zero: because the kernel is sufficiently regular, %
it is possible to
demonstrate high-frequency functions $h$ for which $\| \mb \Theta^\circ [h] \|_{L^2} \ll \| h
\|_{L^2}$. 

Our proof handles these technical challenges sequentially: in \Cref{sec:certificate_inverse_S}, we
restrict attention to the DC subtracted kernel $\Theta^\circ$ and a subspace $S$ containing
low-frequency functions, and show that the restriction $\mb P_S \mb \Theta^\circ \mb P_S$ to $S$
is stably invertible over $S$. In \Cref{sec:certificate_nonsmooth_ub}, we argue that the solution
$g$ to $\mb P_S \mb \Theta^\circ [g] = \zeta$ is regularized enough that $\mb \Theta^\circ [g] \approx
\zeta$, i.e., the restriction to $S$ can be dropped. Finally, in \Cref{sec:certificate_dc_density}
we move from the DC subtracted kernel $\mb \Theta^\circ$ without density to the full kernel $\mb
\Theta_{\mu}$. This
move entails additional technical complexity; to maintain accuracy of approximation, we develop an
iterative construction that successively applies the results of
\Cref{sec:certificate_inverse_S}--\Cref{sec:certificate_nonsmooth_ub} to whittle away approximation
errors, yielding a complete proof of \Cref{lem:skeleton_dc_density}.

\subsection{Invertibility Over a Subspace of Smooth Functions}\label{sec:certificate_inverse_S}
\paragraph{Proof Sketch and Organization.} 
In this section, we solve a restricted version of the certificate problem for DC subtracted kernel
$\mb \Theta^\circ$, over a subspace $S$ of low-frequency functions defined in
\eqref{eq:smooth_subspace}. Namely, for $\zeta \in S$, we demonstrate the existence of a small
norm solution $g \in S$ to the equation
\begin{equation}
\mb P_S \mb \Theta^\circ [g] = \zeta. 
\end{equation}
This equation involves the integral operator $\mb \Theta^\circ$, which acts via
\begin{equation}
\mb \Theta^\circ [g] (\mb x) = \int_{\mb x' \in \mc M} \Theta^\circ(\mb x, \mb x') g(\mb x') d\mb x'.
\end{equation}
We argue that this operator is invertible over $S$, by decomposing this integral into four pieces,
which we call the {\bf Local}, {\bf Near}, {\bf Far}, and {\bf Winding} components. The formal
definitions of these four components follow: for parameters $0 < \veps < 1$, $r>0$, and $\delta >
0$, we define
	\begin{align}
	&\text{\bf [Local]}:&  L_r(\mb x) &= \left\{\mb x' \in \manifold  \, \middle | \,  d_{\manifold}(\mb x, \mb x') < r \right\}, \label{eqn:local_term} \\
	&\text{\bf [Near]}:&  N_{r, \veps}(\mb x) &= \left\{\mb x' \in \manifold  \, \middle | \,  r \le d_{\manifold}(\mb x, \mb x') \le \frac{\sqrt{\veps}}{ \hat{\kappa}}\right\}, \label{eqn:near_term} \\
	&\text{\bf [Far]}:&  F_{\veps, \delta}(\mb x) &= \left\{ \mb x' \in \manifold  \, \middle | \, d_{\manifold}(\mb x, \mb x') \ge \frac{\sqrt{\veps}}{\hat{\kappa}}, \; \angle( \mb x, \mb x' ) > \frac{\delta\sqrt{\veps}}{\hat{\kappa}}\right\}, \label{eqn:faraway_term} \\
	&\text{\bf [Winding]}:&  W_{\veps, \delta}(\mb x) &= \left\{\mb x' \in \manifold  \, \middle | \,  d_{\manifold}(\mb x, \mb x') \ge \frac{\sqrt{\veps}}{\hat{\kappa}}, \; \angle( \mb x, \mb x') \le \frac{\delta \sqrt{\veps}}{\hat{\kappa}}\right\}. \label{eqn:winding_term} 
	\end{align}
It is easy to verify that for any choice of these parameters and any $\mb x \in \mc M$, these four
pieces cover $\manifold$: i.e., $L_r(\mb
x) \cup N_{r,\veps}(\mb x) \cup F_{\veps,\delta}(\mb x) \cup W_{\veps,\delta}(\mb x) = \mc M$.
Intuitively, the {\bf Local} and {\bf Near} pieces contain points that are close to $\mb x$, in
the intrinsic distance on $\mc M$. The {\bf Far} component contains points that are far from $\mb
x$ in intrinsic distance, {\em and} far in the extrinsic distance (angle).  The {\bf Winding}
component contains portions of $\mc M$ that are far in intrinsic distance, but close in extrinsic
distance. Intuitively, this component captures parts of $\mc M$ that ``loop back'' into the
vicinity of $\mb x$.

\paragraph{Parameter choice.} The specific parameters $r,\veps,\delta$ will be chosen with an eye
towards the properties of both $\mc M$ and $\mb \Theta^\circ$. The parameter $\veps \in (0,
\frac{3}{4})$ is a scale parameter, which controls $r = r_\veps$ such that 
\begin{enumerate}
	\item $r$ is large enough to enable the local component $L_r(\mb x)$ to dominate the kernel's
    behavior;
	\item $r$ is not too large, so the kernel stays sharp and localized over the local component
    $L_r(\vx)$. 
\end{enumerate}
Specifically, we choose
\begin{align} 
    a_\veps &= (1- \veps)^3 ( 1 - \veps/ 12 ), \label{eq:a_def}\\
	r_\veps &= 6 \pi L^{-\frac{a_\veps}{a_\veps+1}} \label{eqn:r_def}. 
\end{align} 
Notice that when $\veps \downto 0$, we have $r_\veps
\approx L^{-1/2}$. So with a smaller choice of $\veps$ we may get a larger local component with the
price of a larger constant dependence. 

We further choose $\delta$ to ensure
that the {\bf Near} and {\bf Far} components overlap. To see that this is possible, note that at
the boundary of the {\bf Near} component, $d_{\manifold}(\mb x, \mb x') = \sqrt{\veps} /
\hat{\kappa}$; from \Cref{lem:angle_inverse}, we have
\begin{equation}
\angle(\mb x,\mb x') \;\ge\; d_{\manifold}(\mb x, \mb x') - \hat{\kappa}^2 d_{\manifold}^3(\mb x,\mb x'),
\end{equation}
so at this point $\angle(\mb x,\mb x') \ge (1 - \veps) \sqrt{\veps} / \hat{\kappa}$. Thus as long as
$\delta < 1 - \veps$, {\bf Near} and {\bf Far} overlap. 

\paragraph{Kernel as main and residual.} \label{graf:kernel_main_residual}
The kernel $\Theta^\circ(\mb x,\mb x')$ is a decreasing
function of $\angle(\mb x,\mb x')$: $\Theta^\circ$ is largest over the {\bf Local} component,
smaller over the {\bf Near} and {\bf Winding} components, and smallest over the {\bf Far}
component. By choosing the scale parameter $r_\veps$ as in $\Cref{eqn:r_def}$, we define an operator $\mb M_\veps$ which captures the contribution of the {\bf Local} component to the kernel: 
\begin{equation}
  \mb M_\veps[f](\mb x) = \int_{\mb x' \in L_{r_\veps}(\mb x)} \psi^\circ(\angle(\mb x, \mb x'))f(\mb x')d\mb x'.  \label{eqn:M}
\end{equation}
Because $\angle(\mb x,\mb x')$ is small over $L_{r_\veps}(\vx)$ when $r_\veps$ is chosen to be small compared to inverse curvature $1/\hat{\kappa}$,
on this component, $d_{\manifold}(\mb x, \mb x') \approx \angle(\mb x,\mb x')$ (which we formalize
in \Cref{lem:angle_inverse}). We will use this property to argue that $\mb M_{\veps}$ can be
approximated by a self-adjoint convolution operator, defined as
\begin{align}\label{eq:invariant_op}
  \wh{\mb M}_\veps [ f ]( \mb x) &= \int_{s' = s-r_{\veps}}^{s+r_{\veps}} \psi^\circ(\abs{s-s'}) f(\mb x_{\sigma}(s')) ds' \Biggr|_{\mb x_{\sigma}(s) = \mb x}.
\end{align}
The restriction is valid because for any choice of $\sigma$ and $s$ such that $\mb x_{\sigma}(s) = \mb x$, the RHS has the same value. On the other hand, given that we require $0 < \veps < \frac{3}{4}$, \Cref{eq:a_def,eqn:r_def} show that when $L$ is chosen larger than a certain absolute constant, we have $r_{\eps} \leq \pi$, assuring $\abs{s' - s}$ falls in the domain of $\psi^\circ$, which makes this operator well-defined. We will always assume such a choice has been made in the sequel, and in particular include it as a hypothesis in our results. 

Notice that $\wh{\mb M}_\veps$ is an \textit{invariant operator}: it commutes with
the natural translation action on $\manifold$.
As a result, it diagonalizes in the Fourier basis
defined in \Cref{eq:fourier_basis} (i.e., each of these functions is an eigenfunction of $\wh{\mb
M}_{\veps}$). 
See \Cref{lem:main_term_pd} and its proof for the precise formulation of these properties.
This enables us to study its spectrum on the subspace of smooth functions defined in
\Cref{eq:smooth_subspace} at the specific scale $\veps$, defined as
\begin{equation}\label{eqn:S_eps_def}
S_\veps = S_{K_{\veps, +}, K_{\veps, -}}
\end{equation}
with $K_{\veps, \sigma} = \left\lfloor \frac{ \veps^{1/2} \mr{len}( \mc M_\sigma ) }{ 2 \pi r_\veps }
\right\rfloor$ for $\sigma \in \set{+,-}$.\footnote{Notice that although $\mb \Theta$ and $\zeta$ are real objects, our subspace $S_\veps$ contains complex-valued functions. In the remainder of \Cref{sec:proof_certificate_theorem}, we will work with complex objects for convenience, which means our constructed certificate candidates can be complex-valued. This will not affect our result because (intuitively) the fact that $\mb \Theta$ and $\zeta$ are real makes the imaginary component of the certificate is redundant, and removing it with a projection onto the subspace of real-valued functions will give us the same norm and residual guarantees for the certificate problem. We make this claim rigorous and guarantee the existence of a real certificate in \Cref{lem:skeleton_dc_density}, which is invoked in the proof of our main result on certificates, \Cref{thm:skeleton_dc_density}. } 
In this way, we will establish that $\wh{\mb M}_\veps$
is stably invertible on $S_{\veps}$.

In the remainder of the section, we show the diagonalizability and restricted invertibility of
$\wh{\mb M}_{\veps}$ in \Cref{lem:main_term_pd}, and control the $L^2$ to $L^2$ operator norm of
all four components of $\vTheta^\circ$ in \Cref{lem:local_diff_ub}, \Cref{lem:near_ub},
\Cref{lem:winding_ub} and \Cref{lem:far_ub}. Then we show $\vTheta^\circ$ is stably invertible
using these results by a Neumann series construction (\Cref{lem:main_res_norm_compare}) and
finally prove the main theorem for this section in \Cref{thm:certificate_over_S}. 

\begin{theorem}\label{thm:certificate_over_S} For any $\veps \in (0, \tfrac{3}{4})$, $\delta
  \in (0, 1 - \veps]$, there exist an absolute constant $C$ and constants $C_{\veps},
  C_{\veps, \delta}', C_{\veps}''$ depending only on the subscripted parameters such that if 
	\begin{equation*}
	L \ge \max \Biggl\{ \, \exp\Bigl(C_{\veps, \delta}'\len(\manifold) \hat{\kappa}\Bigr), 
	\paren{1 + \frac{1}{\Delta_{\veps} \sqrt{1 + \kappa^2}}}^{C_{\veps}''\fco_{\veps,
  \delta}(\manifold)},
	 \paren{\veps^{-1/2}12\pi\hat{\kappa}}^{\frac{a_\veps+1}{a_\veps}} ,\; C_\veps \;  \Biggr\},
	\end{equation*}	
	where $a_\veps$, $r_\veps$ as in \Cref{eq:a_def} and \Cref{eqn:r_def} and we set subspace $S_\veps$ and the invariant operator $\wh{\mb M}_{\veps}$ as in  \Cref{eqn:S_eps_def} and \Cref{eq:invariant_op}, 
	we have $\mb P_{S_\veps} \wh{\mb M}_\veps \mb P_{S_\veps}$ is invertible over $S_\veps$, and 
	\begin{equation*}
	\left\| \left(\mb P_{S_\veps}\wh{\mb M}_\veps \mb P_{S_\veps} \right)^{-1} \mb P_{S_\veps} \left(\mb \Theta^\circ - \wh{\mb M}_\veps \right) \mb P_{S_\veps}  \right\|_{L^2 \to L^2} \le 1 - \veps. 
	\end{equation*}
	Moreover, for any $\zeta \in S_\veps$, the equation $\mb P_{S_\veps} \mb \Theta^\circ  [g] = \zeta$ has a unique solution $g_\veps[\zeta] \in S_\veps$ given by the convergent Neumann series
	\begin{equation} \label{eqn:g-neumann}
	g_\veps[\zeta] = \sum_{\ell= 0}^\infty (-1)^\ell \left( \left(\mb P_{S_\veps}\wh{\mb M}_\veps \mb P_{S_\veps} \right)^{-1} \mb P_{S_\veps} (\mb \Theta^\circ - \wh{\mb M}_\veps) \mb P_{S_\veps} \right)^\ell \left(\mb P_{S_\veps}\wh{\mb M}_\veps \mb P_{S_\veps} \right)^{-1} \zeta,
	\end{equation}
	which satisfies 
	\begin{equation}\label{eq:g_def}
	\| g_\veps[\zeta] \|_{L^2} \le \frac{C \norm{\zeta}_{L^2}}{\veps n \log L}.
	\end{equation}
\end{theorem}

\begin{proof}\label{proof:skeleton_dc_subtracted}
	We construct $g \in S_{\veps}$ satisfying $\mb P_{S_{\veps}}\mb \Theta^\circ[g] = \zeta$ by equivalently writing
	\begin{align*}
	    \mb P_{S_{\veps}}\mb \Theta^\circ[g] &= \paren{\mb P_{S_{\veps}}\wh{\mb M}_{\veps}\mb P_{S_{\veps}} + \mb P_{S_{\veps}}\paren{\mb \Theta^\circ - \wh{\mb M}_{\veps}} \mb P_{S_{\veps}}}[g].
	\end{align*}
	Under our hypotheses, \Cref{lem:main_res_norm_compare} implies the invertibility of $\mb P_{S_\veps} \wh{\mb M}_\veps \mb P_{S_\veps}$ with
	\begin{equation}\label{eq:dc_sub_lem_eigen_res_compare}
	\lambda_{\min}\left(\mb P_{S_\veps}\wh{\mb M}_{\veps} \mb P_{S_\veps}\right) \ge \frac{1}{1 - \veps} \left\| \mb \Theta^\circ - \wh{\mb M}_{\veps} \right\|_{L^2 \to L^2},
	\end{equation} 
  where $\lambda_{\min}\left(\mb P_{S_\veps}\wh{\mb M}_{\veps} \mb P_{S_\veps}\right)$ is the minimum
  eigenvalue of the self-adjoint operator $\mb P_{S_\veps}\wh{\mb M}_{\veps} \mb P_{S_\veps} : S_{\veps} \to S_{\veps}$ as shown in
  \Cref{lem:main_term_pd}.  
  In particular, $\mb P_{S_\veps} \wh{\vM}_{\veps}\mb P_{S_\veps}$ is %
  invertible, and the system we seek to solve can be written equivalently as
  \begin{align*}
      \left( \mb P_{S_{\veps}}\wh{\mb M}_{\veps}\mb P_{S_{\veps}} \right)\inv \zeta
	    &=\paren{\Id_{S_{\veps}} +  \paren{\mb P_{S_{\veps}}\wh{\mb M}_{\veps}\mb P_{S_{\veps}}}^{-1}\mb P_{S_{\veps}} \paren{\mb \Theta^\circ - \wh{\mb M}_{\veps}} \mb P_{S_{\veps}}}[g],
  \end{align*}
  where the LHS of the last system is in $S_{\veps}$.
  Next, we argue that the operator that remains on the RHS of the last equation is invertible. Noting that 
	\begin{align*}
    &\left\|\left(\mb P_{S_\veps}\wh{\mb M}_{\veps} \mb P_{S_\veps} \right)^{-1} \mb P_{S_\veps}
  \left(\mb \Theta^\circ - \wh{\mb M}_{\veps}\right) \mb P_{S_\veps}\right\|_{L^2 \to L^2} \\
  &\qquad\qquad\le \quad \left\| \left( \mb P_{S_\veps}\wh{\mb M}_{\veps} \mb P_{S_\veps}\right)^{-1}\right\|_{L^2 \to L^2} \left\| \mb P_{S_\veps} \left(\mb \Theta^\circ - \wh{\mb M}_{\veps}\right) \mb P_{S_\veps}\right\|_{L^2 \to L^2} 
  \\
	&\qquad\qquad\le \quad \lambda_{\min}\left(\mb P_{S_\veps}\wh{\mb M}_{\veps} \mb P_{S_\veps}\right)^{-1}\left\|\mb \Theta^\circ - \wh{\mb M}_{\veps}\right\|_{L^2 \to L^2} \\
	&\qquad\qquad\le \quad 1-\veps
	\labelthis \label{eq:neumann_res_term_bound}
	\end{align*}
	using both \Cref{lem:main_term_pd} and \Cref{eq:dc_sub_lem_eigen_res_compare},
	we have by the Neumann series that 
	\begin{align*}
	    \paren{\Id_{S_{\veps}} + \left(\mb P_{S_\veps}\wh{\mb M}_{\veps} \mb P_{S_\veps} \right)^{-1} \mb P_{S_\veps}
  \left(\mb \Theta^\circ - \wh{\mb M}_{\veps}\right) \mb P_{S_\veps}}^{-1} &= \sum_{i = 0}^{\infty} (-1)^i \paren{\left(\mb P_{S_\veps}\wh{\mb M}_{\veps} \mb P_{S_\veps} \right)^{-1} \mb P_{S_\veps}
  \left(\mb \Theta^\circ - \wh{\mb M}_{\veps}\right) \mb P_{S_\veps}}^i. 
	\end{align*}
	Thus we know $g_{\veps}[\zeta]$ in \Cref{eqn:g-neumann} serves as the solution to the equation $\mb P_{S_{\veps}}\mb \Theta^\circ[g] = \zeta$. 
	
	Furthermore, from \Cref{lem:main_term_pd} when $L \ge C_\veps$, we have 
	\begin{align*}
	\left\| \left(\mb P_{S_\veps}\wh{\mb M}_{\veps} \mb P_{S_\veps} \right)^{-1} \zeta\right\|_{L^2} &\le \lambda_{\min}\left(\mb P_{S_\veps}\wh{\mb M}_{\veps} \mb P_{S_\veps}\right)^{-1} \norm{\zeta}_{L^2} \\
	&\le \frac{1}{c n \log L } \norm{\zeta}_{L^2}.
	\end{align*}
	Combining this bound with
	\Cref{eq:neumann_res_term_bound} and the triangle inequality in the series representation \Cref{eqn:g-neumann}, we obtain the claimed norm bound in \Cref{eq:g_def}:
	\begin{align*}
	\norm{g_\veps[\zeta] }_{L^2} &\le \sum_{\ell = 0}^{\infty} (1 - \veps)^{\ell}\left\|\left(\mb P_{S_\veps}\wh{\mb M}_{\veps} \mb P_{S_\veps} \right)^{-1} \zeta \right\|_{L^2} \\
	&\le \frac{C\norm{\zeta}_{L^2}}{\veps n \log L}. 
	\end{align*}
\end{proof}

\begin{lemma}\label{lem:main_res_norm_compare}
  Let $\veps \in (0, \frac{3}{4})$, $\delta \in (0, 1-\veps]$, and let $a_\veps$, 
  $\wh{\mb M}_\veps$ and $S_\veps$ be as in \Cref{eq:a_def}, \Cref{eq:invariant_op} and \Cref{eqn:S_eps_def}. There are constants
  $C_{\veps}, C_{\veps, \delta}', C_{\veps}''$ depending only on the subscripted
  parameters such that if 
  \begin{equation*}
    L \ge \max \Biggl\{ \, \exp\Bigl(C_{\veps, \delta}'\len(\manifold) \hat{\kappa}\Bigr),
      \paren{1 + \frac{1}{\Delta_{\veps} \sqrt{1 + \kappa^2}}}^{C_{\veps}''\fco_{\veps, \delta}(\manifold)},
    \paren{\veps^{-1/2}12\pi\hat{\kappa}}^{\frac{a_\veps+1}{a_\veps}}, \; C_{\veps} \;  \Biggr\},
  \end{equation*}	
	we have $\mb P_{S_\veps} \wh{\mb M}_\veps \mb P_{S_\veps}$ is invertible over $S_\veps$ with
	\begin{equation*}
	\lambda_{\min}\left(\mb P_{S_\veps}\wh{\mb M}_{\veps} \mb P_{S_\veps} \right) \ge \frac{1}{1 - \veps} \left\|\mb \Theta^\circ - \wh{\mb M}_{\veps} \right\|_{L^2 \to L^2}
	\end{equation*}
  where $\lambda_{\min}\left(\mb P_{S_\veps}\wh{\mb M}_{\veps} \mb P_{S_\veps} \right)$ is defined in
  \Cref{lem:main_term_pd}. 
\end{lemma}
\begin{proof}	
From triangle inequality for the $L^2 \to L^2$ operator norm, we have
\begin{align*}
    \norm*{\mb \Theta^\circ - \wh{\mb M}_\veps}_{L^2 \to L^2} &\le \norm*{\mb \Theta^\circ - {\mb M}_\veps}_{L^2 \to L^2} + \norm*{\mb M_\veps - \wh{\mb M}_\veps}_{L^2 \to L^2}. 
\end{align*}
To bound the first term, we define
\begin{align*}
    M_{\veps}(\mb x, \mb x') &= \indicator{d_{\manifold}(\mb x, \mb x') < r_{\veps}}\psi^\circ(\angle(\mb x, \mb x')). 
\end{align*}
Then it is a bounded symmetric kernel $\manifold \times \manifold \to \bbR$, and following \Cref{eqn:M}, $\mb M_{\veps}$ is its associated Fredholm integral operator. We can thus apply \Cref{lem:young_ineq} and get 
\begin{align*}
    \norm*{\mb \Theta^\circ - {\mb M}_\veps}_{L^2 \to L^2} &\le \sup_{\mb x \in \manifold} \int_{\mb x' \in \manifold} \abs*{ \Theta^\circ(\mb x, \mb x') -  M_{\veps}(\mb x, \mb x') } d\mb x' \\
    &= \sup_{\mb x \in \manifold} \int_{\mb x' \in \manifold \setminus L_{r_\veps}(\mb x)} \abs*{ \Theta^\circ(\mb x, \mb x')} d\mb x'. 
\end{align*}
Because the {\bf Near}, {\bf Far} and {\bf Winding} pieces cover $\manifold \setminus L_{r_\veps}(\mb x)$, we have
	\begin{eqnarray*}
	\int_{\mb x' \in \manifold \setminus L_{r_\veps}(\mb x)} \abs{\Theta^\circ(\mb x, \mb x')} d\mb x' &\le& \int_{\mb x' \in N_{r_\veps, \veps}(\mb x)} \abs{\Theta^\circ(\mb x, \mb x')} d\mb x' + \int_{\mb x' \in W_{\veps, \delta}(\mb x)} \abs{\Theta^\circ(\mb x, \mb x')} d\mb x' \\ &&+ \int_{\mb x' \in F_{\veps, \delta}(\mb x)} \abs{\Theta^\circ(\mb x, \mb x')} d\mb x' 
	\end{eqnarray*}
	From \Cref{lem:local_diff_ub}, \Cref{lem:near_ub}, \Cref{lem:winding_ub} and  \Cref{lem:far_ub},
  we know that there exist constants $C_2, C_3, C_4$ and for any $\veps'' \le 1$ exist numbers
  $C_{\veps''}, C_{\veps''}'$ such that when $L \ge C_{\veps''}$ and $L \ge \paren{\veps^{-1/2}12\pi\hat{\kappa}}^{\frac{a_\veps+1}{a_\veps}}$, we have 
	\begin{align*}
	\left\| \mb \Theta^\circ - \wh{\mb M}_{\veps} \right\|_{L^2 \to L^2} 	&\le \quad  \sup_{\mb x \in \manifold} \int_{\mb x' \in N_{r_\veps, \veps}(\mb x)} \abs{\Theta^\circ(\mb x, \mb x')} d\mb x' \, + \, \sup_{\mb x \in \manifold} \int_{\mb x' \in W_{\veps, \delta}(\mb x)} \abs{\Theta^\circ(\mb x, \mb x')} d\mb x' \\ &\qquad + \,  \sup_{\mb x \in \manifold} \int_{\mb x' \in F_{\veps, \delta}(\mb x)} \abs{\Theta^\circ(\mb x, \mb x')} d\mb x'  \, + \, \left\| \mb M_{\veps} - \wh{\mb M}_\veps \right\|_{L^2\to L^2} \\
	&\le \quad \frac{3\pi n}{4(1 - \veps)}(1 + \veps'') \log\left(\frac{\sqrt{\veps}}{ \hat{\kappa} r_\veps}\right) \; + \; C_{\veps''}' n  \\
	&\qquad +\,  C_2 \, \mr{len}(\manifold) \, n\frac{\hat{\kappa}}{\delta\sqrt{\veps}}  \\
	&\qquad +\,  C_3 \, \fco_{\veps, \delta}(\manifold) \, n \log\left(1 + \frac{\frac{1}{\sqrt{1 + \kappa^2}}}{\Delta_\veps}\right) \\
	&\qquad +\,  C_4 \, (1-\veps)^{-2} \hat{\kappa}^2nr_\veps^2. \labelthis \label{eqn:resid_bound}
	\end{align*}
Meanwhile, from \Cref{lem:main_term_pd} there exists constant $C_{\veps}, C_1$ such that when $L \ge C_{\veps}$
\begin{equation} 
	(1 - \veps) \, \lambda_{\min}\left(\mb P_{S_\veps} \wh{\mb M}_{\veps} \mb P_{S_\veps}\right) \ge (1 - \veps)^2 \frac{3\pi n}{4} \log\left(1 + \frac{L-2}{3\pi}r_\veps \right) - C_1 (1 - \veps) n r_\veps \log^2 L. \label{eqn:lambda_lb}
	\end{equation}
We will treat all named constants appearing in the previous two equations as fixed for the remainder of the proof.
We argue that the first term in this expression is large enough to dominate each of the terms in \eqref{eqn:resid_bound} and the residual term in \eqref{eqn:lambda_lb}. 

Set $\veps' = \frac{\veps}{24}$. We will choose $\veps'' = \frac{\veps'}{1 - 2\veps'} < 1$, so that both $\veps'$ and $\veps''$ depend only on $\veps$. 
Then, since $r_\veps = 6\pi L^{-\frac{a_\veps}{a_\veps+1}}$,
when $L > 4$, we have 
\begin{equation}
\frac{L-2}{3\pi}r_\veps = 2(L-2)L^{-\frac{a_\veps}{a_\veps+1}} > L^{\frac{1}{a_\veps+1}}.
\label{eq:logtermthing_lb}
\end{equation}
Since moreover $a_\veps =  (1-\veps)^3(1-\frac{\veps}{12}) = (1 - \veps)^3(1 - 2\veps') $, we have
$a_{\veps}\veps'' = \veps' (1 - \veps)^3$, and therefore $(1 + \veps'')a_{\veps}=  (1 - \veps')(1 - \veps)^3 $.
Thus 
	\begin{align*}
	(1 - \veps') (1 - \veps)^2 \frac{3\pi n}{4} \log\left(1 + \frac{L-2}{3\pi}r_\veps \right) &=\frac{3\pi n}{4(1 - \veps)}(1+\veps'') \, a_\veps  \log\left(1 + \frac{L-2}{3\pi}r_\veps \right) \\
	&\ge \frac{3\pi n}{4(1 - \veps)}(1+\veps'') \log \left(L^{\frac{a_\veps}{a_\veps+1}} \right) \\
	&\ge \frac{3\pi n}{4(1 - \veps)}(1+\veps'')
  \log\left(\frac{\sqrt{\veps}}{\hat{\kappa}r_\veps}\right),\labelthis \label{eq:main_res_compare_1}
	\end{align*}
	where in the last bound we use $\sqrt{\veps} \hat{\kappa}\inv \leq 6\pi$, given that $\veps < 1$ and $\hat{\kappa} \leq \pi/2$.
	The RHS at the end of this chain of inequalities
	is the first term of the RHS of the last bound in \eqref{eqn:resid_bound}. 
	Since the LHS has a leading coefficient of $(1 - \veps')$, we can conclude provided we can split the remaining $\veps'$ across the remaining terms. 
	
	Next, we will cover the negative term in \Cref{eqn:lambda_lb} and the second and fifth terms in \Cref{eqn:resid_bound}.
	Using \Cref{eq:logtermthing_lb}, we have
	\begin{equation}
	\frac{\veps'}{3} (1 - \veps)^2 \frac{3\pi n}{4} \log\left(1 + \frac{L-2}{3\pi}r_\veps \right)  
	\ge \frac{\veps'}{3} (1 - \veps)^2 \frac{3\pi n}{4} \frac{1}{a_{\veps}+1}\log(L).
	\label{eq:logtermthing_lb_derived}
	\end{equation}
	There exists a constant $C_{\veps}$ such that when $L \geq C_{\veps}$, we have for the RHS
	\begin{equation*}
	    \frac{\veps'}{3} (1 - \veps)^2 \frac{3\pi n}{4} \frac{1}{a_{\veps}+1}\log(L)
	    \ge (C_1 + C_4 + C_{\veps''}')n.
	\end{equation*}
	In particular, we can take
	\begin{equation*}
	     C_\veps \ge \exp\paren{\frac{C_1 + C_4 + C_{\veps''}'}{\frac{\veps'}{3} (1 - \veps)^2 \frac{3\pi}{4} \frac{1}{a_{\veps}+1}}}. 
	\end{equation*}
	Next, there exists another constant $C_{\veps} > 0$ such that when $L \geq C_{\veps}$, we have $r_{\veps} \log^2 L \leq 1$, whence by the previous bound
	\begin{equation*}
	    \frac{\veps'}{3} (1 - \veps)^2 \frac{3\pi n}{4} \frac{1}{a_{\veps}+1}\log(L)
	    \ge (C_1 r_{\veps} \log^2 L + C_4 + C_{\veps''}')n.
	\end{equation*}
	Finally, notice that
	when $L \ge \paren{\veps^{-1/2}12\pi\hat{\kappa}}^{\frac{a_\veps+1}{a_\veps}}$, we have
	\begin{equation*}
		r_\veps = 6\pi L^{-\frac{a_\veps}{a_\veps + 1}} \le \frac{\sqrt{\veps}}{2\hat{\kappa}},
	\end{equation*}
	so $r_\veps\hat{\kappa} \le \sqrt{\veps}/2$, and since $\veps \in (0, 3/4)$, we have
	\begin{align*}
	    (1 - \veps)C_1 n r_\veps \log ^2 L + C_4(1 - \veps)^{-2}\hat{\kappa}^2 n r_\veps^2 + C'_{\veps''}n &\le \paren{C_1r_\veps \log^2 L + 3C_4 + C_{\veps''}'}n, 
	\end{align*}
	where we used that $\veps \mapsto \veps(1 - \veps)^{-2}$ is increasing. Combining our previous bounds, this gives
	\begin{equation}
	\frac{\veps'}{3} (1 - \veps)^2 \frac{3\pi n}{4} \log\left(1 + \frac{L-2}{3\pi}r_\veps \right)  
	\geq
	    (1 - \veps)C_1 n r_\veps \log ^2 L + C_4(1 - \veps)^{-2}\hat{\kappa}^2 n r_\veps^2 + C'_{\veps''}n,
	    \label{eq:main_res_compare_2}
	\end{equation}
	as desired. 
	
	For the remaining two terms, define
	\begin{equation*}
	C_{\veps, \delta}' = \frac{(a_{\veps}+1)C_2}{(1 - \veps)^2\frac{\veps'}{3} \frac{3\pi }{4} \delta\sqrt{\veps}}, \quad C_{\veps}'' = \frac{(a_{\veps}+1)C_3}{(1 - \veps)^2\frac{\veps'}{3} \frac{3\pi }{4} }.
	\end{equation*}
	We will use the estimate \Cref{eq:logtermthing_lb_derived} as our base.
	Then when 
	\begin{equation*}
	L \ge \max\Biggl\{ \; \exp\Bigl( C_{\veps, \delta}'\len(\manifold) \hat{\kappa}\Bigr),
	\paren{1 + \frac{1}{\Delta_{\veps} \sqrt{1 + \kappa^2}}}^{C_{\veps}''\fco_{\veps, \delta}(\manifold)} \Biggr\}, 
	\end{equation*}
	we have 
	\begin{align*}
	\frac{\veps'}{3} (1 - \veps)^2 \frac{3\pi n}{4} \log\left(1 + \frac{L-2}{3\pi}r_\veps \right) 
	&\ge C_2\, \mr{len}(\manifold) \, n\frac{\hat{\kappa}}{\delta\sqrt{\veps}},\labelthis\label{eq:main_res_compare_3} \\
	\frac{\veps'}{3} (1 - \veps)^2 \frac{3\pi n}{4} \log\left(1 + \frac{L-2}{3\pi}r_\veps \right) 
	&\ge C_3 \, \fco_{\veps, \delta}(\manifold) \, n \log\left(1 + \frac{1}{\Delta_\veps\sqrt{1 + \kappa^2}}\right). \labelthis \label{eq:main_res_compare_4}
	\end{align*}
	Combining \eqref{eq:main_res_compare_1}, \eqref{eq:main_res_compare_2}, \eqref{eq:main_res_compare_3}, \eqref{eq:main_res_compare_4} completes the proof. 
\end{proof}

\begin{lemma}\label{lem:x_deriv}
	For any $\mb x \in \manifold$, we have
	\begin{align*}
	\innerprod{ \mb x}{ \dot{\mb x}} = \innerprod{ \dot{\mb x} }{ \ddot{\mb x} } = \innerprod{ \mb x
  }{ \mb x^{(3)} } &= 0,  \labelthis \label{eq:x_deriv_01}\\
	\innerprod{ \mb x }{ \ddot{\mb x} } &= -1,  \labelthis \label{eq:x_deriv_02}\\
	\innerprod{ \mb x }{ \mb x^{(4)} } = -\innerprod{ \dot{\mb x} }{ \mb x^{(3)}} &= \| \ddot{\mb x} \|_2^2,  \\
	\innerprod{ \ddot{\mb x} }{ \mb x^{(3)} } &= - \frac{1}{3} \innerprod{ \dot{\mb x} }{ \mb x^{(4)}}, \\
	\norm{\mb P_{\mb x^{\perp}}\ddot{\mb x} }_2^2 &= \| \ddot{\mb x} \|_2^2 - 1, \\
	M_2 = \sqrt{1 + \kappa^2} &\le M_4, \labelthis \label{eq:M2_M4_compare} \\
	M_2 &< 2\hat{\kappa}, \labelthis \label{eq:M2_kappa_compare} \\
	\frac{1}{\hat{\kappa}} &\le \min\{\len(\manifold_-), \len(\manifold_+)\} \labelthis \label{eq:kappa_manifold_len_compare} , 
	\end{align*}
	where we use above the notation introduced near \Cref{eq:inclusion_abuse}.
\end{lemma}

\begin{proof}
	As our curve is defined over sphere and has unit speed, we have 
	\begin{align*}
	\norm{\mb x }_2^2 = \norm{\dot{\mb x} }_2^2 = 1. 
	\end{align*}
	Taking derivatives on both sides, we get
	\begin{align*}
	\innerprod{ \mb x }{ \dot{\mb x} } = \innerprod{ \dot{\mb x} }{ \ddot{\mb x} } = 0. 
	\end{align*}
	Continuing to take higher derivatives, we get the following relationships:
	\begin{align*}
	\| \dot{\mb x} \|_2^2 + \innerprod{ \mb x }{ \ddot{\mb x} } &= 0, \\
	3 \innerprod{ \dot{\mb x} }{ \ddot{\mb x} } + \innerprod{ \mb x }{ \mb x^{(3)}} &= 0 , \\
	3\| \ddot{\mb x} \|_2^2 + 4 \innerprod{ \dot{\mb x}}{ \mb x^{(3)}}  + \innerprod{ \mb x }{ \mb x^{(4)} } &= 0, \\
	\| \ddot{\mb x} \|_2^2 + \innerprod{ \dot{\mb x}}{ \mb x^{(3)}}  &= 0,\\
	3 \innerprod{ \ddot{\mb x} }{ \mb x^{(3)} } + \innerprod{ \dot{\mb x} }{ \mb x^{(4)} } &= 0. 
	\end{align*}
	which gives us by plugging in the previous constraints
	\begin{align*}
	\innerprod{ \mb x }{ \ddot{\mb x} } &= -1, \\
	\innerprod{ \mb x }{ \mb x^{(3)} } &= 0,  \\
	\innerprod{ \dot{\mb x} }{ \mb x^{(3)} } &= -\| \ddot{\mb x} \|_2^2,  \\
	\innerprod{ \mb x }{ \mb x^{(4)} } &= \| \ddot{\mb x} \|_2^2, \\
	\innerprod{ \ddot{\mb x} }{ \mb x^{(3)} } &= - \frac{1}{3} \innerprod{ \dot{\mb x} }{ \mb x^{(4)} }. 
	\end{align*}
	As a consequence, the intrinsic curvature $\norm{\mb P_{\mb x^{\perp}} \ddot{\mb x} }_2$ and
  extrinsic curvature $\norm{\ddot{\mb x} }_2$ are related by
	\begin{align*}
	\norm{\mb P_{\mb x^{\perp}} \ddot{\mb x}}_2^2 &= \left\| \Bigl(\mb I - \mb x \mb x^* \Bigr) \ddot{\mb x} \right\|_2^2  \\
	&=\innerprod{ \mb x }{ \ddot{\mb x} }^2 + \innerprod{ \ddot{\mb x} }{ \ddot{\mb x} } - 2\innerprod{\mb x }{ \ddot{\mb x} }^2  \\
	&= \| \ddot{\mb x} \|_2^2 - 1. 
	\end{align*}
	Thus we know 
	\begin{align*}
	M_2 &= \sup_{\mb x \in \manifold} \; \norm{\ddot{\mb x} }_2 \\
	&= \sup_{\mb x \in \manifold}\sqrt{1 + \norm{\mb P_{\mb x^{\perp}} \ddot{\mb x} }_2^2} \\
	&= \sqrt{1 + \sup_{\mb x \in \manifold} \set{\norm{\mb P_{\mb x^{\perp} } \ddot{\mb x} }_2}^2}\\
	&= \sqrt{1 + \kappa^2} \\
	&\le \sqrt{\paren{\frac{\pi}{2}\hat{\kappa}}^2 + \kappa^2} \\
	&< 2\hat{\kappa}. 
	\end{align*}
	Furthermore, the above shows that $M_2 \geq 1$, so we have
	\begin{align*}
	M_2 \le M_2^2 &= \sup_{\mb x \in \manifold} \; \norm{\ddot{\mb x} }_2^2 \\
	&= \sup_{\mb x \in \manifold} \; \innerprod{\mb x}{\mb x^{(4)}} \\
	&\le M_4, 
	\end{align*}
  using one of our previously-derived relationships in the second line and Cauchy-Schwarz in the
  third.
	Finally, for any point $\mb x = \mb x_{\sigma}(s)$, as $\mb x_{\sigma}(s +
  \len(\manifold_\sigma)) = \mb x_{\sigma}(s)$, we have
	\begin{align*}
		\Zero = \mb x_{\sigma}(s + \len(\manifold_\sigma)) - \mb x_{\sigma}(s) &= \int_{s' = s}^{s + \len(\manifold_\sigma)} \dot{\mb x}_{\sigma}(s') d s' \\
		&= \len(\manifold_\sigma) \dot{\mb x}_{\sigma}(s) + \int_{s' = s}^{s + \len(\manifold_\sigma)}  \int_{s'' = s}^{s'} \ddot{\mb x}_{\sigma}(s'') ds'' d s' 
	\end{align*}
	which leads to
	\begin{align*}
		\len(\manifold_\sigma) &= \norm*{ \len(\manifold_{\sigma}) \dot{\mb x}_{\sigma}(s)}_2 \\
		&= \norm*{\int_{s' = s}^{s + \len(\manifold_\sigma)}  \int_{s'' = s}^{s'} \ddot{\mb x}_{\sigma}(s'') ds'' d s' }_2 \\
		&\le {\int_{s' = s}^{s + \len(\manifold_\sigma)}  \int_{s'' = s}^{s'} M_2 ds'' d s' } \\
		&= \frac{\len(\manifold_\sigma)^2}{2} M_2 < \len(\manifold_\sigma)^2 \hat{\kappa}, 
	\end{align*}
	completing the proof, where the first line uses the unit-speed property, the second uses the previous relation, the third uses Jensen's inequality (given that $\norm{}_2$ is convex and $1$-homogeneous),
	and the last line comes from \Cref{eq:M2_kappa_compare}. 
\end{proof}

\begin{lemma}\label{lem:angle_inverse}
	Let $\hat{\kappa} = \max\left\{\kappa,  \frac{2}{\pi}\right\} $. For $\sigma \in \{ \pm \}$ and $\abs{s' - s} \le \frac{1}{\hat{\kappa}}$, we have
	\begin{equation}
		\abs{s - s'} - \hat{\kappa}^2 \abs{s - s'}^3 
    \le \angle\bigl( \mb x_\sigma(s), \mb x_{\sigma}(s') \bigr) 
    \le \abs{s - s'}.
    \label{eq:angleinverse_claim1}
	\end{equation}
	As a consequence, for $\abs{s - s'} \le \frac{\sqrt{\veps}}{\hat{\kappa}}$, 
	\begin{equation}
	(1 - \veps) \abs{s - s'} \le \angle \bigl(\mb x_\sigma(s), \mb x_\sigma(s') \bigr) \le \abs{s - s'}. \label{eqn:near-isometry}
	\end{equation}
	In particular, for any two points $\mb x, \mb x' \in \manifold_{\sigma}$, choosing $s, s'$ such that 
	$\vx_{\sigma}(s) = \vx$, $\vx_{\sigma}(s') = \vx'$, and 
	$\abs{s - s'} = d_{\manifold}(\mb x, \mb x')$, we have when $d_{\manifold}(\mb x, \mb x') \le \frac{\sqrt{\veps}}{\hat{\kappa}}$ %
	\begin{equation*}
	(1 - \veps) d_{\manifold}(\mb x, \mb x') \le \angle \bigl(\mb x, \mb x' \bigr) \le d_{\manifold}(\mb x, \mb x'). 
	\end{equation*}
\end{lemma}

\begin{proof} We prove \Cref{eq:angleinverse_claim1} first.

  The upper bound is direct from the fact that $\manifold$ is a pair of paths in the sphere and
  $\angle(\vx,\vx')$ is the length of a path in the sphere of minimum distance between points
  $\vx$, $\vx'$,
  and then using
  the fact that the distance $\abs{s' - s} \ge d_{\manifold}(\mb x_{\sigma}(s), \mb
  x_{\sigma}(s'))$ from \Cref{eq:intrinsic_distance_function}.

  The lower bound requires some additional estimates. We fix $s, s'$ satisfying our assumptions; as both $\abs{s - s'}$ and $\angle(\mb x_{\sigma}(s), \mb x_{\sigma}(s'))$ are symmetric functions of $(s, s')$, it suffices to assume that $s' \ge s$. Define $t = s' - s$, then by assumption we have $0 \le t \le \frac{1}{\hat{\kappa}} \le \frac{\pi}{2}$. 
  As $\acos$ is strictly decreasing on $[-1, 1]$, we only need to show that 
	\begin{equation}
		 \innerprod{\mb x_\sigma(s)}{\mb x_\sigma(s+t)} \le \cos(t - \hat{\kappa}^2 t^3).
		\label{eq:angle_inverse_needtoshow} 
	\end{equation}
    Using the second order Taylor expansion at $s$, we have
    \begin{align*}
        \mb x_{\sigma}(s + t) = \mb x_{\sigma}(s) + \dot{\mb x}_{\sigma}(s) + \int_{a = s}^{s + t}\int_{b = s}^a \ddot{\mb x}_{\sigma}(b) db\,  da
    \end{align*}
	and so
	\begin{align*}
	\innerprod{\mb x_{\sigma}(s)}{\mb x_\sigma(s+t)} &= \innerprod{\mb x_\sigma(s)}{\mb x_{\sigma}(s) + \dot{\mb x}_{\sigma}(s) + \int_{a = s}^{s + t}\int_{b = s}^a \ddot{\mb x}_{\sigma}(b) db\,  da} \\
	&= \norm{\mb x_{\sigma}(s)}_2^2 + \innerprod{\mb x_{\sigma}(s)}{\dot{\mb x}_{\sigma}(s)} + \innerprod{\mb x_\sigma(s)}{\int_{a = s}^{s + t}\int_{b = s}^a \ddot{\mb x}_{\sigma}(b) d b\, da }  \\
	&= 1 + \int_{a = s}^{s + t}\int_{b = s}^a \innerprod{\mb x_\sigma(s)}{\ddot{\mb x}_{\sigma}(b) } db \, da 
	\labelthis \label{eq:angle_inv_equal1}
	\end{align*}
	where we use properties established in \Cref{lem:x_deriv},
	in particular \Cref{eq:x_deriv_01} in the last line. %
	Take second order Taylor expansion at $b$ for $\mb x_{\sigma}(s)$, we have similarly
    \begin{align*}
    \mb x_{\sigma}(s) = \mb x_{\sigma}(b) + \dot{\mb x}_{\sigma}(b) + \int_{c = s}^b\int_{d = c}^b \ddot{\mb x}_{\sigma}(d) dd\,  dc.
    \end{align*}
    From \Cref{eq:x_deriv_01} and \Cref{eq:x_deriv_02}, we have $\innerprod{\mb x_{\sigma}(b)}{ \ddot{\mb x}_{\sigma}(b)} = -1$ and $\innerprod{\dot{\mb x}_{\sigma}(b)}{ \ddot{\mb x}_{\sigma}(b)} = 0$. Thus uniformly for $b \in [s, s + t]$
    \begin{align*}
        \innerprod{\mb x_{\sigma}(s)}{\ddot{\mb x}_{\sigma}(b)} &= -1 + \innerprod{\int_{c = s}^b\int_{d = c}^b \ddot{\mb x}_{\sigma}(d) dd\,  dc}{\ddot{\mb x}_{\sigma}(b)} \\
        &= -1 + \int_{c = s}^b\int_{d = c}^b \innerprod{\ddot{\mb x}_{\sigma}(d)}{\ddot{\mb x}_{\sigma}(b)} dd\,  dc \\
        &\le -1 + \int_{c = s}^b\int_{d = c}^b \norm{\ddot{\mb x}_{\sigma}(d)}_{2}\norm{\ddot{\mb x}_{\sigma}(b)}_2 dd\,  dc \\
        &\le -1 + \int_{c = s}^b\int_{d = c}^b M_2^2 dd\,  dc \\
        &\leq -1 + \frac{M_2^2}{2} (b - s)^2,
    \end{align*}
    where in the third line we use Cauchy-Schwarz.
    Plugging this last bound into \Cref{eq:angle_inv_equal1},
    it follows
	\begin{align*}
	\ip*{\mb x_{\sigma}(s)}{ \mb x_{\sigma}(s+t) }
	&\le 1 + \int_{a = s}^{s + t}\int_{b = s}^a (-1 + (M_2^2/2) (b-s)^2) db \, da  \\
	&= 1 - \frac{t^2}{2} + \frac{M_2^2}{2}\int_{a = s}^{s + t}\int_{b = s}^a (b - s)^2 db \, da \\
	&= 1 - \frac{t^2}{2} + \frac{M_2^2}{4!}t^4 \\
	&= 1 - \frac{t^2}{2} + \frac{1+\kappa^2}{4!}t^4,
	\labelthis \label{eq:angle_inv_equal2}
	\end{align*}
	with an application of \Cref{lem:x_deriv} in the final equality. 
    To conclude, we derive a suitable estimate for $\cos(t - \hat{\kappa}^2 t^3)$. 
    Because $0 \leq t \leq \hat{\kappa}\inv$, we have that $t\inv ( t - \hat{\kappa}^2 t^3) \in [0, 1]$, and because 
    $t \leq \hat{\kappa}\inv \leq \pi/2$, we can apply
	concavity of $\cos$ on $[0, \pi/2]$ to obtain
	\begin{equation*}
		\cos(t - \hat{\kappa}^2 t^3) \ge \frac{t - \hat{\kappa}^2 t^3}{t} \cos(t) + \left(1 - \frac{t - \hat{\kappa}^2 t^3}{t}\right)\cos(0).
	\end{equation*}
	Next, the estimate $\cos(x) \ge 1 - \frac{x^2}{2} + \frac{x^4}{4!} - \frac{x^6}{6!}$ for all $x$, a consequence of Taylor expansion, gives
	\begin{align*}
	    (1 - \hat{\kappa}^2 t^2) \cos(t) + \hat{\kappa}^2 t^2
	    &\geq
	    (1 - \hat{\kappa}^2 t^2) \left(1 - \frac{t^2}{2} + \frac{t^4}{4!} - \frac{t^6}{6!}\right) + \hat{\kappa}^2 t^2 \\
	    &=  1 - \frac{t^2}{2} + \frac{t^4}{4!} + \frac{\hat{\kappa}^2 t^4}{2} 
	    - \frac{t^6}{6!} - \frac{\hat{\kappa}^2 t^6}{4!} 
	    + \frac{\hat{\kappa}^2 t^8}{6!}
	\end{align*}
	after distributing. Because $\hat{\kappa} \geq \kappa$, we can split terms and write
	\begin{equation*}
	    t^4 /4! + \hat{\kappa}^2 t^4 / 2 
	    \geq 
	    \frac{1 + \kappa^2}{4!}t^4
	    + \hat{\kappa}^2 t^4 / 4,
	\end{equation*}
	and then grouping terms in the preceding estimates gives
	\begin{equation*}
	    \cos(t - \hat{\kappa}^2 t^3) \ge
	    1 - \frac{t^2}{2} + \frac{1 + \kappa^2}{4!}t^4 + \hat{\kappa}^2 t^4 \left(
	        \frac{1}{4}- \frac{t^2 }{ 4!} + \frac{t^4 }{ 6! } - \frac{t^2}{  6! \hat{\kappa}^2}
	    \right).
	\end{equation*}
	By way of \Cref{eq:angle_inv_equal2,eq:angle_inverse_needtoshow}, we will therefore be done if we can show that
	\begin{equation*}
	        \frac{1}{4} - \left(\frac{1}{ 4!} + \frac{1}{6! \hat{\kappa}^2} \right) t^2 + \frac{t^4 }{ 6! } 
	        \geq 0.
	\end{equation*}
	This is not hard to obtain: for example, we can prove the weaker but sufficient bound
	\begin{equation*}
	       1 - \frac{1}{3!} \left(1 + \frac{1}{30 \hat{\kappa}^2} \right)t^2
	    \geq 0
	\end{equation*}
	by noticing that because $t \leq \hat{\kappa}\inv$, it suffices to show
	\begin{equation*}
	\frac{1}{\hat{\kappa}^2} \left( 1 + \frac{1}{30 \hat{\kappa}^2} \right) \leq 6,
	\end{equation*}
	and because the LHS of the previous line is an increasing function of $\hat{\kappa}\inv$ and moreover $\hat{\kappa}\inv \leq \pi/2$, this bound follows by verifying that indeed $(\pi/2)^2 (1 + (1/30) (\pi/2)^2) \leq 6$.
	Because $s, s'$ were arbitrary we have thus proved \Cref{eq:angleinverse_claim1}.

	For the remaining claims, \Cref{eqn:near-isometry} follows naturally from the fact that when $\abs{s - s'} \le \frac{\sqrt{\veps}}{\hat{\kappa}}$, we have $\abs{s - s'} - \hat{\kappa}^2 \abs{s - s'}^3 \ge (1 - \veps)\abs{s - s'}$. 
	The final claim is a restatement of \Cref{eqn:near-isometry} under the additional stated hypotheses.
\end{proof}

\subsection*{Invertibility of $\wh{\mb M}$ over $S$.} 
\begin{lemma}[Young's inequality for Fredholm operators]\label{lem:young_ineq}
  Let $K: \manifold \times \manifold \to \R$ satisfy $K(\vx, \vx') = K(\vx', \vx)$ for
  all $(\vx, \vx') \in \manifold \times \manifold$ and $\sup_{(\vx, \vx') \in \manifold \times
  \manifold} \abs{K(\vx, \vx')} < +\infty$, and let $\mb K$ denote its Fredholm integral operator
  (defined as $g \mapsto \mb K[g] = \int_{\manifold}K(\spcdot, \mb x')g(\mb x')d \mb x'$). For any
  $1 \le p \le +\infty$, we have
  \begin{equation*}
    \norm{\mb K}_{L^p \to L^p} \le \sup_{\mb x \in \manifold}\int_{\mb x' \in
    \manifold}\abs{K(\mb x, \mb x')} d\mb x'.
  \end{equation*}
\end{lemma}
\begin{proof}
  The proof uses the M.\ Riesz convexity theorem for interpolation of operators \cite[\S V,
  Theorem 1.3]{Stein1971-ed}, which we need here in the form of a special case: it states that for
  all $1 \leq p \leq +\infty$, one has
  \begin{equation}
    \norm{\vK}_{L^{p} \to L^{p}} 
    \leq \norm{\vK}_{L^\infty \to L^\infty}^{1/p} \norm{\vK}_{L^1 \to L^1}^{1-1/p}.
    \label{eq:m_riesz_convex}
  \end{equation}
  To proceed, we will bound the two operator norm terms on the RHS. We have
  \begin{align*}
    \norm{\mb K}_{L^1 \to L^1} &= \sup_{\norm{g}_{L^1} = 1} \int_{\mb x \in \manifold} \left| \int_{\mb x' \in \manifold}K(\mb x, \mb x') g(\mb x') d\mb x' \right|  d \mb x \\
    &\le \sup_{\norm{g}_{L^1} = 1} \int_{\mb x \in \manifold} \int_{\mb x' \in \manifold}\abs{K(\mb x, \mb x')} \abs{g(\mb x')} \, d\mb x' d \mb x \\
    &= \sup_{\norm{g}_{L^1} = 1} \int_{\mb x' \in \manifold} \left(\int_{\mb x \in
    \manifold}\abs{K(\mb x, \mb x')} \, d\mb x\right)
    \, \abs{g(\mb x')} \, d \mb x' \\
    &\le 
    \sup_{\norm{g}_{L^1} =1}
    \left(
      \norm{g}_{L^1}
      \sup_{\vx' \in \manifold}
      \abs*{
        \int_{\vx \in \manifold} \abs{K(\vx, \vx')} \diff \vx
      }
    \right) \\
    &= \sup_{\mb x \in \manifold}\int_{\mb x' \in \manifold}\abs{K(\mb x, \mb x')} \, d\mb x'.
    \labelthis \label{eq:kernel_l1_bound}
  \end{align*} 
  The first inequality above uses the triangle inequality for the integral.
  In the third line, we rearrange the order of integration using Fubini's theorem, given that $g$
  is integrable and $K$ is bounded on $\manifold \times \manifold$.
  In the fourth line, we use $L^1$-$L^\infty$ control of the integrand (i.e., H\"{o}lder's
  inequality), and in the final line we use that $\norm{g}_{L^1} = 1$ along with symmetry of $K$
  and nonnegativity of the integrand to to re-index and remove the outer absolute value.
	On the other hand, $L^1$-$L^\infty$ control and the triangle inequality give immediately 
	\begin{align*}
    \norm{\mb K}_{L^{\infty} \to L^{\infty}} &= \sup_{\mb x \in \manifold, \, \norm{g}_{L^\infty} = 1} \left| \int_{\mb x' \in \manifold}K(\mb x, \mb x') g(\mb x') d\mb x' \right|  \\
	&\le \sup_{\mb x \in \manifold}  \int_{\mb x' \in \manifold} \abs{K(\mb x, \mb x')} d\mb x'. 
	\end{align*}
  These two bounds are equal; plugging them into \Cref{eq:m_riesz_convex} thus proves the claim.
\end{proof}

\begin{lemma}\label{lem:main_term_pd} 
    Let $\veps \in (0, \frac{3}{4})$, $r_\veps, S_\veps$ and $\wh{\mb M}_\veps$ be as defined in \Cref{eqn:r_def}, \Cref{eq:smooth_subspace} and \Cref{eq:invariant_op}. Then $\wh{\mb M}_{\veps}$ diagonalizes in the
    Fourier orthonormal basis \eqref{eq:fourier_basis}. Write $\lambda_{\min}\left(\mb P_{S_\veps}
    \wh{\mb M}_{\veps} \mb P_{S_\veps}\right)$ for the minimum eigenvalue of the operator
    $\mb P_{S_\veps} \wh{\mb M}_{\veps} \mb P_{S_\veps} : S_{\veps} \to S_{\veps}$. 
    Then there exist constants $c, C$ and a constant $C_{\veps}$ such that when $L \ge C_{\veps}$, we
    have
	\begin{align*}
	\lambda_{\min}\left(\mb P_{S_\veps} \wh{\mb M}_{\veps} \mb P_{S_\veps}\right) &\ge (1 - \veps)\frac{3\pi n}{4} \log\left(1 + \frac{L-2}{3\pi}r_\veps\right) - Cn r_\veps \log^2 L, \\
	&\ge  c n \log L. 
	\end{align*}
	As a consequence, $\mb P_{S_\veps} \wh{\mb M}_{\veps} \mb P_{S_\veps}$ is invertible over $S_\veps$, and 
	\begin{equation*}
	\left\| \paren{\mb P_{S_\veps} \wh{\mb M}_{\veps} \mb P_{S_\veps}}^{-1} \right\|_{L^2 \to L^2} = \lambda_{\min}^{-1}\left(\mb P_{S_\veps} \wh{\mb M}_{\veps} \mb P_{S_\veps}\right). 
	\end{equation*} 
\end{lemma}
\begin{proof}
Choose $L \gtrsim 1$ to guarantee that $\wh{\vM}_{\veps}$ is well-defined.
We use $\psi^\circ$ to denote the DC subtracted skeleton, as defined in \eqref{eqn:psi_circ}, and 
$(\phi_{\sigma, k})_{\sigma, k}$ the (intrinsic) Fourier basis on $\mc M$, as defined in
\eqref{eq:fourier_basis}. For any Fourier basis function $\phi_{\sigma, k}$, we have
\begin{align*}
\wh{\mb M}_{\veps}[\phi_{\sigma, k}](\mb x_\sigma(s)) &= \int_{s' = s-r_\veps}^{s+r_\veps}
\psi^\circ(\abs{s-s'}) \phi_{\sigma, k}\left(\mb x_\sigma(s') \right) ds'  \\
&= \int_{s' = -r_\veps}^{r_\veps} \psi^\circ(\abs{s'}) \exp\paren{\frac{i2\pi k s'}{\len(\manifold_\sigma)}} d s' \phi_{\sigma, k}(\mb x_\sigma(s))  \\
&=  \phi_{\sigma, k}(\mb x_\sigma(s))
\int_{s' = -r_\veps}^{r_\veps} \psi^\circ(\abs{s'}) \cos\paren{\frac{2\pi k s'}{\len(\manifold_\sigma)}} d s',  
\end{align*}
which shows that each Fourier basis function is an eigenfunction of $\wh{\mb M}_\veps$; because
these functions form an orthonormal basis for $L^2(\manifold)$ (by classical results from Fourier
analysis on the circle), $\wh{\mb M}_{\veps}$ diagonalizes in this basis.  
Moreover, because $S_{\veps}$ is the span of Fourier basis functions, $\vP_{S_{\veps}}$ also
diagonalizes in this basis, and hence so does $\vP_{S_\veps} \wh{\vM}_{\veps} \vP_{S_{\veps}}$.
Because $\wh{\vM}_{\veps}$ is self-adjoint and $\vP_{S_{\veps}}$ is an orthogonal projection, 
$\vP_{S_\veps} \wh{\vM}_{\veps} \vP_{S_{\veps}}$ is self-adjoint; and because $\dim(S_{\veps}) <
+\infty$, the operator $\vP_{S_\veps} \wh{\vM}_{\veps} \vP_{S_{\veps}}$ has finite rank, and
therefore has a well-defined minimum eigenvalue, which we denote as in the statement of the lemma.
As $K_{\veps, \sigma} =\floor{\frac{\veps^{1/2} \len(\manifold_\sigma)}{2\pi r_{\veps}}}$, we
have for any $\abs{k_{\sigma}} \le K_{\veps, \sigma}$ and any $\abs{s'} \le r_\veps$, 
\begin{equation*}
  1 \ge \cos\paren{\frac{2\pi k_{\sigma} s'}{\len(\manifold_\sigma)}} 
  \ge 1 - \paren{\frac{2\pi k_{\sigma} s'}{\len(\manifold_\sigma)}}^2 
  \ge 1 - \veps.
\end{equation*}
Then for $\sigma \in \set{+,-}$ and $ \abs{k} \le K_{\veps, \pm}$, 
\begin{align*}
	\wh{\mb M}_{\veps}[\phi_{\sigma, k}](\mb x_\sigma(s)) &= 
  \phi_{\sigma, k}(\mb x_\sigma(s)) 
  \int_{s' = -r_\veps}^{r_\veps} \psi^\circ(\abs{s'}) \cos\paren{\frac{2\pi k s'}{\len(\manifold_\sigma)}} d s' \\
	&\ge (1 - \veps)
  \phi_{\sigma, k}(\mb x_\sigma(s))
  \int_{s' = -r_\veps}^{r_\veps} \psi^\circ(\abs{s'}) ds', 
\end{align*}
and so
\begin{equation*}
 \lambda_{\min}\left(\mb P_{S_\veps} \wh{\mb M}_{\veps} \mb P_{S_\veps} \right) \; \ge \; 
 2(1 - \veps) \int_{0}^{r_\veps} \psi^\circ(s) \, ds.  
\end{equation*}
From \Cref{lem:skel_r_lb}, we have if $L \gtrsim 1$
\begin{equation*}
2\int_{s = 0}^{r_\veps} \psi^\circ(s) ds \;\ge\;  \frac{3\pi n}{4} \log\left(1 +
\frac{L-2}{3\pi}r_\veps\right) - Cn r_\veps \log^2 L.
\end{equation*}
In particular, as $r_\veps = 6\pi L^{-\frac{a_{\veps}}{a_\veps + 1}}$, there exists a constant
$C_\veps$ such that when $L \ge C_{\veps}'$, we have
\begin{equation*}
	Cnr_{\veps}\log^2 L \le \frac{\veps}{4} \frac{3\pi n}{4} \log \left(1 + \frac{L-2}{3\pi}
  r_\veps \right),
\end{equation*}
and thus 
\begin{align*}
	\lambda_{\min}\left(\mb P_{S_\veps} \wh{\mb M}_{\veps} \mb P_{S_\veps} \right) &\ge (1 - \veps) \left(1 - \frac{\veps}{4}\right) \frac{3\pi n}{4} \log \left(1 + \frac{L-2}{3\pi} r_\veps \right) \\
	&\ge \left(1 - \frac{5 \veps}{4} \right) \frac{3\pi n}{4} \log \left(L^{1 - \frac{a_\veps}{a_\veps+1}} \right) \\
	&= \left(1 - \frac{5 \veps}{4}\right) \frac{3\pi n}{4} \frac{1}{a_\veps + 1}\log L \\
	&\ge \left(1 - \frac{5}{4}\cdot \frac{3}{4} \right) \frac{3\pi n}{4} \cdot \frac{1}{2} \log L \\
	&\ge c n \log L \\
	& > 0, 
\end{align*}
where we used $L \gtrsim 1$ in the second inequality, and $\veps < 3/4$ and $a_{\veps} \leq 1$ in
the third inequality.
So $\mb P_{S_\veps} \wh{\mb M}_{\veps} \mb P_{S_\veps}$ is invertible over $S_\veps$, with 
\begin{equation}
\left(\mb P_{S_\veps} \wh{\mb M}_{\veps} \mb P_{S_\veps} \right)^{-1}[h] = \sum_{\sigma = \pm}\sum_{k = 0}^{K_{\veps, \sigma}}  \paren{\int_{s = -r_\veps}^{r_\veps} \psi^\circ(\abs{s}) \cos\paren{\frac{2\pi k s}{\len(\manifold_\sigma)}} d s}^{-1} \phi_{\sigma, k}\phi_{\sigma, k}^*h. \label{eq:smooth_subspace_main_term}
\end{equation}
The final claim is a consequence of the fact that $\mb P_{S_\veps} \wh{\mb M}_{\veps} \mb
P_{S_\veps}$ is self-adjoint and finite-rank.

\end{proof}

\begin{lemma}\label{lem:local_diff_ub}  
Let $\veps \in (0, \frac{3}{4})$, $a_\veps, r_\veps, \mb M_{\veps}$ and $\wh{\mb M}_\veps$ be as defined in \Cref{eq:a_def}, \Cref{eqn:r_def}, \Cref{eqn:M} and \Cref{eq:invariant_op}. There exist constants $C,C'$, such that when $L \ge C$ and $L \ge \paren{\veps^{-1/2}12\pi\hat{\kappa}}^{\frac{a_\veps+1}{a_\veps}}$, we have
	\begin{equation*}
	\left\| \mb M_{\veps} - \wh{\mb M}_{\veps} \right\|_{L^2\to L^2} \le (1-\veps)^{-2} C' \hat{\kappa}^2nr_\veps^2. 
	\end{equation*}
\end{lemma}
\begin{proof}
We choose $L \gtrsim 1$ to guarantee that $\wh{\vM}_{\veps}$ is well-defined for all $0 < \veps < 3/4$.
We would like to use \Cref{lem:young_ineq} to bound $\norm{\mb M_\veps - \wh{\mb M}_\veps}_{L^2 \to L^2}$, and thus we define two (suggestively-named) bounded symmetric kernels $\manifold \times \manifold \to \R$: 
\begin{equation*}
    M_{\veps}(\mb x, \mb x') = \indicator{d_{\manifold}(\mb x, \mb x') < r_{\veps}}\psi^\circ(\angle(\mb x, \mb x'))
\end{equation*}
and
\begin{equation*}
    \wh{M}_{\veps}(\mb x, \mb x') = \indicator{d_{\manifold}(\mb x, \mb x') < r_{\veps}}\psi^\circ(d_{\manifold}(\mb x, \mb x')). 
\end{equation*}
From \Cref{eqn:M}, $\mb M_{\veps}$ is indeed $M_\veps$'s associated Fredholm integral operator. To show that under our constraints for $L$, $\wh{\mb M}_\veps$ is also $\wh{M}_\veps$'s associated integral operator, we first notice that following \Cref{eq:intrinsic_distance_function}, for any $\mb x, \mb x' \in \manifold$, $d_{\manifold}(\mb x, \mb x') < r_{\veps}$ if and only if there exist $\sigma$, $s$ and $s'$ such that $\mb x = \mb x_{\sigma}(s)$, $\mb x' = \mb x_{\sigma}(s')$ and $\abs{s' - s} < r_{\veps}$. This means for any fixed $\mb x$, if we let $\sigma$ and $s$ be chosen such that $\mb x = \mb x_{\sigma}(s)$, then $L_{r_\veps}(\mb x) =  \set{\mb x_{\sigma}(s') \vert \abs{s' - s} < r_{\veps}}$.  
Furthermore, as $L \ge \paren{\veps^{-1/2}12\pi\hat{\kappa}}^{\frac{a_\veps+1}{a_\veps}}$, by \Cref{eq:kappa_manifold_len_compare} 
	in \Cref{lem:x_deriv}
	we have $r_\veps \le \frac{\sqrt{\veps}}{2\hat{\kappa}} < \min\set{\len(\manifold_+), \len(\manifold_-)}/2$. Under this condition, we can unambiguously express the intrinsic distance $d_{\manifold}$ in terms of arc length at the local scale: for any $\mb x' \in L_{r_{\veps}}(\mb x)$, there is a \textit{unique} $s'$ such that $\abs{s' - s} \le r_{\veps}$. To see this, note that for any other parameter choice that attains the infimum in \Cref{eq:intrinsic_distance_function} $s'' = s' + k \len(\manifold_\sigma)$ with integer $k \ne 0$, 
the triangle inequality implies $\abs{s'' - s} \geq \abs{r_{\eps} - k \len(\sM_\sigma)}$, and one has $\abs{r_{\eps} - k \len(\sM_\sigma)} > r_{\eps}$ for every $k \neq 0$ if $ 0 < r_{\eps} < \len(\sM_\sigma)/2$. Then for $\mb x' \in L_{r_\veps}(\mb x)$ and any $s' \in [s - r_\veps, s+ r_\veps]$ such that $\mb x_\sigma(s') = \mb x'$, we have $d_{\manifold}(\mb x, \mb x') = \abs{s - s'}$. Combining all these points, $\wh{M}_\veps$'s associated Fredholm integral operator $\mb H$ can be written as:
\begin{align*}
    \mb H[f](\mb x_\sigma(s)) &= \int_{d_{\manifold}(\mb x_{\sigma}(s), \mb x') < r_{\veps}} \psi^\circ(d_{\manifold}(\mb x_\sigma(s), \mb x'))f(\mb x') d\mb x' \\
    &= \int_{\mb x' \in \set{\mb x_\sigma(s) \vert \abs{s' - s} < r_\veps}} \psi^\circ(d_{\manifold}(\mb x_{\sigma}(s), \mb x'))f(\mb x') d\mb x' \\
    &= \int_{s' = s - r_\veps}^{s + r_\veps} \psi^\circ(d_{\manifold}(\mb x_\sigma(s), \mb x_\sigma(s'))f(\mb x_\sigma(s')) d s' \\
    &= \int_{s' = s - r_\veps}^{s + r_\veps} \psi^\circ(\abs{s - s'})f(\mb x_\sigma(s')) d s' \\
    &= \wh{\mb M}_\veps[f](\mb x_\sigma(s)),
\end{align*}
which means $\wh{\mb M}_\veps$ is indeed $\wh{M}_\eps$'s associated integral kernel. 
	 
	We can now apply \Cref{lem:young_ineq} and and get
	\begin{align*}
	    \norm{\mb M_{\veps} - \wh{\mb M}_{\veps}}_{L^2 \to L^2} &\le \sup_{\mb x \in \manifold} \int_{\mb x' \in \manifold} \abs{M_{\veps}(\mb x, \mb x') - \wh{M}_{\veps}(\mb x, \mb x')} d \mb x'\\
	    &= \sup_{\mb x \in \manifold} \int_{\mb x' \in L_{r_{\veps}(\mb x)}} \abs{M_{\veps}(\mb x, \mb x') - \wh{M}_{\veps}(\mb x, \mb x')} d \mb x'\\
	    &= \sup_{s,\sigma}  \int_{s' =
  s-r_\veps}^{s+r_\veps} \left| \psi^\circ\Bigl( \angle(\mb x_\sigma(s), \mb x_\sigma(s')) \Bigr)
  - \psi^\circ \left( \abs{s-s'} \right) \right| ds'. 
  \labelthis \label{eq:M_Mhat_diff_bound}
	\end{align*}
	Here, we recall that $r_{\eps} < \pi/4$ (because $\hat{\kappa} \leq \pi/2$), so there is no issue with these evaluations and the domain of $\psi^\circ$ being $[0, \pi]$.
  Note that from \Cref{eq:angleinverse_claim1}, when $\abs{s - s'} \le r_\veps \le \frac{\sqrt{\veps}}{\hat{\kappa}}$, we have
  \begin{align*}
  \angle\left( \mb x_\sigma(s), \mb
  x_\sigma(s') \right) &\ge \abs{s - s'} - \hat{\kappa}^2 \abs{s - s'}^3 \\
  &\ge (1 - \veps)\abs{s - s'}. 
  \labelthis \label{eq:M_Mhat_angle_bound}
  \end{align*}
  As $\psi^\circ$ is nonnegtive, strictly decreasing and convex by \Cref{lem:phi_concave}, we know both $\psi^\circ$ and $\abs{\dot{\psi}^\circ}$ are decreasing. 
  Also, by the upper bound in \Cref{lem:angle_inverse}, we have that 
  ${ \psi^\circ(\angle(\mb x_\sigma(s), \mb x_\sigma(s')) )
  - \psi^\circ ( \abs{s-s'} ) } \geq 0 $, so we can essentially ignore the absolute value in the integrand in \Cref{eq:M_Mhat_diff_bound}.
  We can then calculate
	\begin{align*}
	\int_{s' =
  s-r_\veps}^{s+r_\veps}  \psi^\circ\Bigl( \angle(\mb x_\sigma(s), \mb x_\sigma(s')) \Bigr)
  - \psi^\circ \left( \abs{s-s'} \right)  ds'&\le
	\int_{s' =
  s-r_\veps}^{s+r_\veps}  \psi^\circ\Bigl(\abs{s - s'} - \hat{\kappa}^2 \abs{s - s'}^3 \Bigr)
  - \psi^\circ \left( \abs{s-s'} \right)  ds' \\
    &=
	\int_{t = -r_\veps}^{r_\veps}  \psi^\circ \bigl( \abs{t} - \hat{\kappa}^2 \abs{t}^3 \bigr) - \psi^\circ\bigl(\abs{t} \bigr)  \, dt  \\  &= \int_{t =-r_\veps}^{r_\veps} \int_{a = \abs{t} - \hat{\kappa}^2 \abs{t}^3}^{\abs{t}} \abs*{\dot{\psi}^\circ(a)} \, da \, dt  \\ 
	&\le \hat{\kappa}^2 \int_{t=-r_\veps}^{r_\veps} \abs{t}^3 \abs*{\dot{\psi}^\circ( \abs{t} - \hat{\kappa}^2 \abs{t}^3)} \, dt  \\
	&\le \hat{\kappa}^2 \int_{t=-r_\veps}^{r_\veps} \abs{t}^3 \abs*{\dot{\psi}^\circ( (1 - \veps) \abs{t} )} \, dt  \\
	&= 2(1-\veps)^{-4}\hat{\kappa}^2\int_{t = 0}^{(1-\veps)r_\veps} t^3 \abs{\dot{\psi}^\circ( t )} \, dt.  
	\end{align*}
  Above, the first line comes from \Cref{eq:M_Mhat_angle_bound} and the the fact that $\psi^\circ$ is strictly decreasing, the fourth and fifth line comes from the fact that $\abs{\dot{\psi}^\circ}$ is decreasing and \Cref{eq:M_Mhat_angle_bound}. The last line uses symmetry and a linear transformation. Note that from $\Cref{eq:M_Mhat_angle_bound}$ we always have $\abs{t} - \hat{\kappa}^2\abs{t^3}$ nonnegative when $\abs{t} \le r_{\veps}$ and thus all above formulas are well defined. 
	From \Cref{lem:skel_dot_ub}, we know that there exists $C, C'$ such that when $L \ge C$, we have
    \begin{equation*}
        \int_0^{(1 - \veps)r_\veps} t^3\abs{\dot{\psi}^\circ(t)} dt \le C' n (1 - \veps)^2
        r_\veps^2,
    \end{equation*}
and plugging all bounds back to \Cref{eq:M_Mhat_diff_bound} we get 
    \begin{equation*}
        \left\| \mb M_{\veps} - \wh{\mb M}_{\veps} \right\|_{L^2 \to L^2} \le (1-\veps)^{-2} C' \hat{\kappa}^2nr_\veps^2 
    \end{equation*}
    as claimed. 
\end{proof}

\begin{lemma}\label{lem:near_ub}
Let $\veps \in (0, \frac{3}{4})$, $r_\veps$ and $N_{r_\veps, \veps}$ as defined in \Cref{eqn:r_def} and \Cref{eqn:near_term}. For any $0 < \veps'' \le 1$, there exist numbers $C_{\veps''}, C_{\veps''}'$ such that when $L \ge C_{\veps''}$, we have
	\begin{equation*}
	\sup_{\vx \in \manifold}\,
	\int_{\mb x' \in N_{r_{\veps}, \veps}(\mb x)} \psi^\circ \Bigl( \angle\bigl(\mb x, \mb x' \bigr) \Bigr) ds' \le \frac{3\pi n}{4(1 - \veps)}(1 + \veps'') \log\left(\frac{\sqrt{\veps}}{ \hat{\kappa} r_\veps}\right)+ C_{\veps''}'n.  
	\end{equation*}
\end{lemma}
\begin{proof}
For $\mb x \in \manifold$, assume the parameters are chosen such that the corresponding near piece is nonempty, for otherwise the claim is immediate. 
Recalling \Cref{eqn:near_term}, for any $\mb x' \in N_{r_\veps, \veps}(\mb x)$, we have $d_{\manifold}(\mb x, \mb x') \le \sqrt{\veps}/\hat{\kappa}$. From \Cref{lem:angle_inverse}, this implies $\angle(\mb x, \mb x') \ge (1 - \veps)d_{\manifold}(\mb x, \mb x')$. Let $\sigma, s$ be such that $\mb x_{\sigma}(s) = \mb x$. Notice by the discussion following the definition of the intrinsic distance in \Cref{eq:intrinsic_distance_function} that the near component $N_{r_{\veps}, \veps}(\mb x)$ is contained in the set $\set{\mb x_\sigma(s') \mid \abs{s' - s} \in [r_\veps, \sqrt{\veps}/\hat{\kappa}]}$. And from \Cref{lem:phi_concave}, $\psi^\circ$ is strictly decreasing, thus we have
	\begin{align*}
	\int_{\mb x' \in N_{r_{\veps}, \veps}(\mb x)} \psi^\circ \Bigl( \angle \bigl(\mb x, \mb x' \bigr)
  \Bigr) ds' &\le  \int_{s' = s + r_\veps}^{s + \frac{\sqrt{\veps}}{\hat{\kappa}}}  \psi^\circ
  \Bigl( \angle \bigl(\mb x_{\sigma}(s), \mb x_{\sigma}(s') \bigr) \Bigr) ds' \\&\hphantom{=}+ \int_{s' = s - \frac{\sqrt{\veps}}{\hat{\kappa}}}^{s - r_\veps}  \psi^\circ \Bigl( \angle\bigl(\mb x_{\sigma}(s), \mb x_{\sigma}(s') \bigr) \Bigr) ds' \\
	&\le \int_{s' = s+r_\veps}^{s+\frac{\sqrt{\veps}}{\hat{\kappa}}}  \psi^\circ \left( (1 -
  \veps) |s'-s| ) \right) ds' \\&\hphantom{=}+  \int_{s' = s-\frac{\sqrt{\veps}}{\hat{\kappa}}}^{s-r_\veps}  \psi^\circ \left( (1 - \veps) |s'-s| ) \right) ds'  \\
	&= 2 \int_{t = r_\veps}^{\frac{\sqrt{\veps}}{\hat{\kappa}}}  \psi^\circ \left( (1 - \veps) t \right) dt \\
	&= \frac{2}{1- \veps} \int_{t = (1 - \veps) r_\veps}^{\frac{(1 - \veps)\sqrt{\veps}}{\hat{\kappa}}} \psi^\circ\left( t \right) dt,  
	\end{align*}
	where in the last line we apply a linear change of variables.
	We also note that in the above integrals $\abs{s' - s} \leq \hat{\kappa}\inv \leq \pi/2$, so there are no issues above with the domain of $\psi^\circ$ being $[0, \pi]$.
From \Cref{lem:skel_l1_ub}, for any $0 < \veps'' \leq 1$, there exist numbers $C_{\veps''}, C_{\veps''}'$ such that if $L \ge C_{\veps''}$, then $r_\veps$ satisfies the condition in \Cref{eq:skel_l1_ub_r_to_b} and we have
	\begin{align*}
	\sup_{\mb x \in \manifold} \int_{\mb x' \in N_{r_{\veps}, \veps}(\mb x)} \psi^\circ \Bigl( \angle\bigl(\mb x, \mb x' \bigr) \Bigr) ds'  &\le \frac{2}{1-\veps} \int_{t = (1 - \veps) r_\veps}^{\frac{(1 - \veps)\sqrt{\veps}}{\hat{\kappa}}}\psi^\circ\left( \abs{t} \right) dt  \\
	&\le  \frac{2}{1- \veps} (1 + \veps'')\frac{3\pi n}{8} \log\left(\frac{1+(L-3)\frac{(1 -
  \veps)\sqrt{\veps}/(3\pi)}{ \hat{\kappa}}}{1+(L-3)(1 - \veps) r_\veps/(3\pi)}\right) 
  \\&\hphantom{=}+\enspace C_{\veps''}' n \\
	&\le \frac{3\pi n}{4(1 - \veps)}(1 + \veps'') \log\left(\frac{\sqrt{\veps}}{ \hat{\kappa} r_\veps}\right)+ C_{\veps''}'n. 
	\end{align*}
\end{proof}

\begin{lemma}\label{lem:winding_ub}
Let $\veps \in (0, \frac{3}{4})$, $\delta \in (0, 1-\veps]$. Let $W_{\veps, \delta}$ as in \Cref{eqn:winding_term}. There exist constants $C, C'$ such that when $L \ge C$, for any $\mb x \in \manifold$, 
	\begin{equation*}
	\int_{\mb x' \in W_{\veps, \delta}(\mb x)} \psi^\circ \Bigl( \angle\bigl(\mb x, \mb x' \bigr) \Bigr) ds' \le \fco_{\veps, \delta}(\mc M) C'n \log\left(1 + \frac{\frac{1}{\sqrt{1 + \kappa^2}}}{\Delta_{\veps}}\right). 
	\end{equation*}
\end{lemma}

\begin{proof}
    To bound the integral, we rely on the observation that for each `curve segment' inside the winding component, the angle $\angle(\mb x, \mb x')$ cannot stay small for the whole segment, and thus we can avoid worst case control for the angle as we have employed for the far component in \Cref{lem:far_ub}.\footnote{Within the lemma, a curve segment means $\set{\mb x_{\sigma}(s) \vert s \in[s_1, s_2]} \subseteq \manifold_\sigma$ for certain $\sigma$, $s_1$ and $s_2$ with $\abs{s_1 - s_2} < \len(\manifold_\sigma)$, and we call $\abs{s_1 - s_2}$ the length of the curve segment. } We will begin by constructing a specific finite cover of curve segments for the winding component, then we will bound the integral over each curve segment by providing a lower bound for the angle function. 
    
	As $\sM$ is compact with bounded length, from the definition in \Cref{eq:cnumber_eps_def} we know $\fco_{\eps, \delta}(\manifold)$ is a finite number for any choice of $\veps, \delta$. 
	From the definition of the winding component \Cref{eqn:winding_term}, for any point $\mb x \in \manifold$, we can cover $W_{\veps, \delta}(\mb x)$ by at most $\fco_{\eps, \delta}(\manifold)$ closed balls in the intrinsic distance on the manifold with radii no larger than $1/\sqrt{1 + \kappa^2}$. 
	Topologically, each ball in the intrinsic distance of radii $r$ is a curve segment of length $2r$;
	thus, $W_{\veps, \delta}(\mb x)$ can be covered by at most $2 \fco_{\eps, \delta}(\manifold)$ curve segments, %
    each with length no larger than $1/\sqrt{1 + \kappa^2}$.  
	Formally, this implies that for each $\mb x \in \manifold$, there exists a number $N(\mb x) \le 2\fco_{\veps, \delta}(\manifold)$ and for each $i \in \set{1, \cdots, N(\mb x)}$, there exist a sign $\sigma_i(\mb x) \in \set{\pm}$ and a nonempty interval $I_i(\mb x) = [s_{1, i}(\mb x), s_{2, i}(\mb x)]$ with length no greater than $\frac{1}{\sqrt{1 + \kappa^2}}$ and strictly less than $\len(\manifold_{\sigma_i(\vx)})$ such that
	\begin{align*}
	    W_{\veps, \delta}(\mb x) \subseteq \bigcup_{i = 1}^{N(\mb x)} X_i(\mb x)
	\end{align*}
	where $X_{i}(\mb x) = \{ \mb x_{\sigma_i(\mb x)}(s) \mid s \in I_i(\mb x) \} \subset \manifold$ with $X_{i}(\mb x) \cap W_{\veps, \delta}(\mb x) \ne \varnothing$. 
	For the purpose of minimum coverage, we can further assume without loss of generality that for each $\vx$ and each $i$, the boundary points $\mb x_{\sigma_i(\mb x)}(s_{1, i}(\mb x))$ and $\mb x_{\sigma_i(\mb x)}(s_{2, i}(\mb x))$ belong to $W_{\veps, \delta}(\mb x)$: we can always set $p_{1, i}(\mb x) = \inf\set{s \given s \in [s_{1, i}(\mb x), s_{2, i}(\mb x)], \mb x_{\sigma_i(\mb x)}(s) \in W_{\veps, \delta}(\mb x)}$ and $p_{2, i}(\mb x) = \sup\set{s \vert s \in [s_{1, i}(\mb x), s_{2, i}(\mb x)], \mb x_{\sigma_i(\mb x)}(s) \in W_{\veps, \delta}(\mb x)}$, then the curve segment associated with $\sigma_i(\mb x)$ and interval $ [p_{1, i}(\mb x), p_{2, i}(\mb x)]$ still covers $X_i(\mb x) \cap W_{\veps, \delta}(\mb x)$. 
	As $W_{\veps, \delta}(\mb x)$ is closed, we have the boundary points $\mb x_{\sigma_i(\mb x)}(p_{1, i}(\mb x)), \mb x_{\sigma_i(\mb x)}(p_{2, i}(\mb x)) \in W_{\veps, \delta}(\mb x)$ and as $X_{i}(\mb x)$ intersect with $W_{\veps, \delta}(\mb x)$, the definition above is well defined. 
	
	We will next increase the number of sets in these coverings, so that they are guaranteed not to fall into any of the ``local pieces'' at $\vx$: although by the definitions \Cref{eqn:winding_term,eqn:local_term} the local and winding pieces at any $\vx$ are disjoint, it may be the case that when we pass to the covering sets $(X_i(\vx))_{i \in [N(\vx)]}$, we overlap with the local piece.
	In particular, consider a ``local piece'' $L_{\sqrt{\veps}/\hat{\kappa}}(\mb x)$ defined as in \Cref{eqn:local_term}, which from the definition does not intersect with $W_{\veps, \delta}(\mb x)$. 
	For each $i$, as the boundary points of $X_{i}(\mb x)$ fall in $W_{\veps, \delta}(\mb x)$, these boundary points do not belong to $L_{\sqrt{\veps}/\hat{\kappa}}(\mb x)$. 
	And as $L_{\sqrt{\veps}/\hat{\kappa}}(\mb x)$ is topologically connected and one dimensional, if $X_i(\mb x)$ intersects with $L_{\sqrt{\veps}/\hat{\kappa}}(\mb x)$, it must contains the whole local piece. 
	As $X_i(\mb x)$ itself is a curve segment and one dimensional, and $L_{\sqrt{\veps}/\hat{\kappa}}(\mb x)$ is open, removing $L_{\sqrt{\veps}/\hat{\kappa}}(\mb x)$ would leave two curve segments with smaller length. 
	Then these two curve segments lie in $\manifold \setminus L_{\sqrt{\veps}/\hat{\kappa}}(\mb x)$, and cover $X_i(\mb x) \setminus L_{\sqrt{\veps}/\hat{\kappa}}(\mb x)$. 
	In other words, for any $\mb x \in \manifold$, there exists $N'(\mb x) \le 4\fco_{\veps, \delta}(\manifold)$ and for $i \in \set{1, \cdots, N'(\mb x)}$, there exist signs $\sigma'_i(\mb x) \in \set{\pm}$ and intervals $I'_i(\mb x) = [s'_{1, i}(\mb x), s'_{2, i}(\mb x)]$ with length no greater than 
	$\frac{1}{\sqrt{1 + \kappa^2}}$
	such that %
	\begin{equation*}
		W_{\veps, \delta}(\mb x) \subseteq \bigcup_{i = 1}^{N'(\mb x)} X'_i(\mb x), 
	\end{equation*}
	where $X'_{i}(\mb x) = \{ \mb x_{\sigma'_i(\mb x)}(s) \mid s \in I'_i(\mb x) \} \subset \manifold \setminus L_{\sqrt{\veps}/\hat{\kappa}}(\mb x)$ with $X'_{i}(\mb x) \cap W_{\veps, \delta}(\mb x) \ne \varnothing$. 
	We therefore have
	\begin{equation}
	\int_{\mb x' \in W_{\veps, \delta}(\mb x)} \psi^\circ \Bigl( \angle\bigl(\mb x, \mb x' \bigr) \Bigr) d s'  \le \sum_{i=1}^{N'(\mb x)} \int_{s \in I'_{i}(\mb x)} \psi^\circ \Bigl( \angle \bigl(\mb x, \mb x_{\sigma'_i(\mb x)} (s) \bigr) \Bigr) ds. \label{eq:winding_piece_integralsplit}
	\end{equation}
	
	We next derive additional properties of the pieces $X_i'(\vx)$ that will allow us to obtain suitable estimates for the integrals on the RHS of \Cref{eq:winding_piece_integralsplit}.
	As each $X'_i(\mb x)$ is a compact set, we let $$s_i^*(\mb x) \in \arg \min_{s \in I'_i(\mb x)} \angle\bigl( \mb x, \mb x_{\sigma'_i(\mb x)}(s)\bigr)$$ and denote $\mb x_i^*(\mb x) = \mb x_{\sigma_i(\vx)}(s_i^*(\mb x))$. Below we will abbreviate $\mb x_i^*(\mb x), s_i^*(\mb x)$ and $\sigma'_i(\mb x)$ as $\mb x_i^*, s_i^*$ and $\sigma'_i$ when the base point $\mb x$ is clear. We further abbreviate $\dot{\mb x}_i^* = \dot{\vx}_{\sigma_i(\vx)}(s_i^*(\vx))$. As $X'_i(\mb x)$ intersects with the winding component, we have $\angle\bigl( \mb x, \mb x_i^* \bigr) \le \frac{\delta\sqrt{\veps}}{\hat{\kappa}} < \frac{\pi}{2}$. And as $X'_{i}(\mb x) \cap L_{\sqrt{\veps}/\hat{\kappa}}(\mb x) = \varnothing$, we have $d_{\manifold}(\mb x, \mb x_i^*) \ge \sqrt{\veps}/\hat{\kappa}$. This means $\mb x_i^* \in W_{\veps, \delta}(\mb x)$ from \Cref{eqn:winding_term}. 
	As $\cos$ is strictly decreasing from $0$ to $\pi$ and $s_i^*$ minimizes $\angle(\mb x, \mb x_{\sigma'_i}(s))$, it also maximizes $\ip{\mb x}{\mb x_{\sigma'_i}(s)}$. For any $s \in I'_i(\mb x)$, from the second order Taylor expansion of $\mb x_{\sigma'_i}(s)$ around $\mb x_i^*$ we have 
	\begin{align*}
	    \ip{\mb x}{\mb x_i^*} &\ge \ip{\mb x}{\mb x_{\sigma'_i}(s)} \\
	    &= \ip{\mb x}{\mb x_i^*} + (s - s_i^*)\ip{\mb x}{\dot{\mb x}_{i}^*} + \ip*{\mb x}{\int_{a = s_i^*}^s \int_{b = s_i^*}^a \mb x_{\sigma'_i}^{(2)}(b) \, db \, da} \\
	    &\ge \ip{\mb x}{\mb x_i^*} + (s - s_i^*)\ip{\mb x}{\dot{\mb x}_{i}^*} - \frac{(s - s_i^*)^2}{2}M_2,
	\end{align*}
	with the last line following from Cauchy-Schwarz. 
	In the previous equations, we are of course using the convention that for a real-valued function $f$ and numbers $a < b$, the notation $\int_b^a f(x) \diff x$ denotes the integral $-\int_a^b f(x) \diff x$.
	We are going to use this bound to reprove a classical first-order optimality condition for interval-constrained problems. We split into cases depending on where the point $s_i^*$ lies:
	if $s_i^*$ is not the right end point $s_{2, i}'$, by taking $s$ approaching $s_i^*$ from above, we would have $\innerprod{\mb x}{\dot{\mb x}_i^*} \le 0$. Similarly, if $s_i^*$ is not the left end point $s_{1, i}'$, by taking $s$ approaching $s_i^*$ from below, we would have $\innerprod{\mb x}{\dot{\mb x}_i^*} \ge 0$. This gives 
	\begin{align*}
	    \ip{\mb x}{\dot{\mb x}_i^*} \text{ is } \begin{cases}\le 0 &\quad s_i^* = s_{2, i} \\
	    \ge 0 &\quad s_i^* = s_{1, i} \\
	    = 0 &\quad o.w.
	    \end{cases}
	\end{align*}
    which implies
    \begin{align}\label{eq:winding_xstar_dot_0}
        (s - s_i^*)\innerprod{\mb x}{\dot{\mb x}_i^*} \le 0, \qquad \forall s \in I'_i(\mb x). 
    \end{align} 
    We use again the Taylor expansion at $s_i^*$ and get
	\begin{align*} 
	\left\| \mb x_{\sigma'_i}(s) - \mb x_i^* - (s - s_i^*) \dot{\mb x}_i^* \right\|_2 &= \left\|\int_{a = s_i^*}^s \int_{b = s_i^*}^a \mb x_{\sigma'_i}^{(2)}(b) \, db \, da \right\|_2   \\
	&\le \frac{(s - s_i^*)^2}{2}M_2   \\
	&= \frac{1}{2}(1 + \kappa^2)^{1/2} (s - s_i^*)^2
	\labelthis \label{eq:winding_xs_linear_approx}
	\end{align*} 
	with an application of \Cref{eq:M2_M4_compare} in the last line. 
	Moreover, we have %
	\begin{align*}
	\norm{\mb x -\mb x_{i}^*}_2 &= 2 \sin \left( \frac{\angle(\mb x, \mb x_{i}^*)}{2} \right) \\
	&\ge  \frac{4}{\pi} \sin\left(\frac{\pi}{4}\right) {\angle(\mb x, \mb x_{i}^*)} \\
	&= \frac{2\sqrt{2}}{\pi} \angle\bigl(\mb x, \mb x_{i}^*\bigr) \\
	&\ge \frac{2\sqrt{2}}{\pi} \Delta_{\veps},
	\labelthis \label{eq:winding_xstar_bound}
	\end{align*}
	where the first line is a trigonometric identity,
	the first inequality uses $\angle(\mb x, \mb x_i^*) < \pi/2$ together with the fact that $\sin$ function is concave from $0$ to $\pi$ and thus $\sin(at) \ge a \sin(t)$ for $a \in [0, 1]$ and $t \in [0, \pi]$ (applied to $a = \angle(\vx,\vx_i\adj)/(\pi/2) $ and $t = \pi/4$), and the last line follows directly from the definition of $\Delta_{\veps}$ in \Cref{eq:inj_radius_eps_def}. 
	Making use of the preceding estimates,
	for any $s \in I'_i(\mb x)$ we can finally calculate
	\begin{align*}
	\norm{\mb x_{\sigma'_i}(s) - \mb x}_2^2 &=  \norm{ \mb x_i^* - \mb x + (s - s_i^*)\dot{\mb x}_i^* + ( \mb x_{\sigma'_i}(s) - \mb x_i^* - (s - s_i^*)\dot{\mb x}_{i}^{*})}_2^2 \\
	&\ge \norm{ \mb x_i^* - \mb x + (s - s_i^*)\dot{\mb x}_{i}^* }_2^2 - \norm{ \mb x_{\sigma'_i}(s) - \mb x_i^* - (s - s_i^*)\dot{\mb x}_i^*}_2^2 \\
	&\ge \norm{\mb x_i^* - \mb x}_2^2 + \norm{(s - s_i^*) \dot{\mb x}_i^*}_2^2 
   - 2\innerprod{\mb x}{(s - s_i^*)\dot{\mb x}_i^*} \\
  &\hphantom{=}-\enspace  \paren{\frac{1}{2}(1 + \kappa^2)^{1/2} {(s - s_i^*)}^2}^2 \\
	&\ge \paren{\frac{2\sqrt{2}}{\pi}\Delta_\veps}^2 + (s - s_i^*)^2 - \frac{1}{4}(1 + \kappa^2) {(s - s_i^*)}^4 \\
	&\ge \paren{\frac{2\sqrt{2}}{\pi}\Delta_\veps}^2 + \frac{3}{4}(s - s_i^*)^2 \\
	&\ge \paren{\frac{2}{\pi}\Delta_\veps + \frac{\sqrt{3}}{2\sqrt{2}}\abs{s - s_i^*}}^2.
	\labelthis \label{eq:winding_xs_dist_lb}
	\end{align*}
	Above, the second line uses the triangle inequality, the third line uses the parallelogram identity plus \linebreak\Cref{lem:x_deriv} (first term) and \Cref{eq:winding_xs_linear_approx} (second term), the fourth line comes from \Cref{eq:winding_xstar_bound} and \Cref{eq:winding_xstar_dot_0}, and the fifth line comes from our construction that the length of each interval $I'_i(\mb x)$ is no greater than $1/\sqrt{1 + \kappa^2}$ and therefore the same is true of $\abs{s - s_i^*}$. The last line is an application of inequality of arithmetic and geometric means. 
	Additionally, for any $\vx, \vx'$ of unit norm, one has
	\begin{align*}
	    \angle(\mb x, \mb x') &\ge 2 \sin \left( \frac{\angle(\mb x, \mb x')}{2} \right) \\
	    &= \norm{\mb x - \mb x'}_{2}. 
	\end{align*}
	Combining this and \Cref{eq:winding_xs_dist_lb}, for all $s \in I'_i(\mb x)$ we have
	\begin{align*}
	\angle(\mb x_{\sigma'_i}(s), \mb x) &\ge \norm{\mb x_{\sigma'_i}(s)- \mb x}_2 \; \ge\; \frac{2}{\pi}\Delta_\veps + \frac{\sqrt{3}}{2\sqrt{2}}\abs{s - s_i^*} \\
	&\ge \frac{1}{\sqrt{3}}\Delta_\veps + \frac{1}{\sqrt{3}} \abs{s - s_i^*},  
	\end{align*}
	where the last line just worst-cases constants for simplicity.
	From \Cref{lem:phi_concave}, $\psi^\circ$ is nonnegative and strictly decreasing, so
	\begin{align*}
	\int_{s \in I'_i(\mb x)} \psi^\circ\Bigl( \angle\bigl( \mb x, \mb x_{\sigma'_i}(s) \bigr) \Bigr) ds &= \int_{s = s_i^*}^{s'_{2, i}(\mb x)} \psi^\circ \Bigl( \angle\bigl(\mb x, \mb x_{\sigma'_i}(s)\bigr) \Bigr) ds + \int_{s=s'_{1, i}(\mb x)}^{s_i^*} \psi^\circ \Bigl( \angle\bigl( \mb x, \mb x_{\sigma'_i}(s)\bigr) \Bigr) ds \\
	&\le \int_{s=s_i^*}^{s'_{2, i}(\mb x)} \psi^\circ \left( \frac{1}{\sqrt{3}}\Delta_\veps +
  \frac{1}{\sqrt{3}}\abs{s - s_i^*}\right) ds \\&\hphantom{=}+\enspace \int_{s = s'_{1, i}(\mb x)}^{s_i^*} \psi^\circ\left( \frac{1}{\sqrt{3}}\Delta_\veps + \frac{1}{\sqrt{3}}\abs{s - s_i^*} \right) ds \\
	&\le 2 \int_{s=0}^{\frac{1}{\sqrt{1+ \kappa^2}}} \psi^\circ \left(\frac{1}{\sqrt{3}}\Delta_\veps + \frac{1}{\sqrt{3}}s \right) ds \\
	&= 2\sqrt{3} \int_{t = \frac{1}{\sqrt{3}} \Delta_\veps}^{\frac{1}{\sqrt{3}} \Delta_\veps + \frac{1}{\sqrt{3} \sqrt{1 + \kappa^2}}} \psi^\circ(t) dt 
	\end{align*} 
	where again, the second to third line comes from the fact that our intervals has length at most $1/\sqrt{1 + \kappa^2}$. 
	From \Cref{eq:skel_l1_ub_a_to_b} in \Cref{lem:skel_l1_ub} and a summation over all $N'(\mb x) \le 4\fco_{\veps,\delta}(\mc M)$ segments in the covering, there exists constant $C'$ such that when $L \ge C$, 
	\begin{align*}
	\sum_{i=1}^{N'(\mb x)} \int_{s \in I'_i(\mb x)} \psi^\circ\Bigl( \angle( \mb x, \mb
  x_{\sigma'_i}(s) ) \Bigr) ds &\le \fco_{\veps, \delta}(\mc M) C'n\log\left(\frac{1+(L-3)\paren{\frac{1}{\sqrt{3}} \Delta_\veps + \frac{1}{\sqrt{3} \sqrt{1 + \kappa^2}}}/(3\pi)}{1+(L-3)\frac{1}{\sqrt{3}}\Delta_{\veps}/(3\pi)}\right) \\
	&\le \fco_{\veps, \delta}(\mc M) C'n \log\left(1 + \frac{\frac{1}{\sqrt{1 + \kappa^2}}}{\Delta_{\veps}}\right). 
	\end{align*}
	Recalling our bound \Cref{eq:winding_piece_integralsplit}, we can thus take a supremum over $\vx \in \manifold$ and conclude.
\end{proof}

\begin{lemma}\label{lem:far_ub}
Let $\veps \in (0, 1)$, $\delta \in (0, 1-\veps]$. Let $F_{\veps, \delta}$ as in \Cref{eqn:faraway_term}. There exist constants $C, C'$ such that when $L \ge C$, we have for any $\mb x \in \manifold$, 
	\begin{equation*}
	\int_{\mb x' \in F_{\veps, \delta}(\mb x)} \psi^\circ \Bigl( \angle\bigl(\mb x, \mb x' \bigr) \Bigr) ds' \le C' \mr{len}(\manifold) n\frac{\hat{\kappa}}{\delta\sqrt{\veps}}. 
	\end{equation*}
\end{lemma}
\begin{proof}
	We have the simple bound from \Cref{lem:skel_r_ub} and decreasingness of $\skel^\circ$ from \Cref{lem:phi_concave}, that there exists constant $C'$, with
	\begin{align*}
	\int_{\mb x' \in F_{\veps, \delta}(\mb x)} \psi^\circ \Bigl( \angle\bigl(\mb x, \mb x' \bigr) \Bigr) ds' &\le \mr{len}(\manifold) \psi^\circ\left( \frac{\delta\sqrt{\veps}}{\hat{\kappa}} \right)  \\
	&\le \mr{len}(\manifold)C' n\frac{L-3}{1 +(L-3)\frac{\delta\sqrt{\veps}}{\hat{\kappa}}/(3\pi)}  \\
	&\le \mr{len}(\manifold)C' n\frac{\hat{\kappa}}{\delta\sqrt{\veps}}, 
	\end{align*}
	as claimed. 
	
\end{proof}

\subsection{Certificates for the DC-Subtracted Kernel}\label{sec:certificate_nonsmooth_ub}
\paragraph{Proof Sketch and Organization}
In \Cref{sec:certificate_inverse_S}, we constructed a certificate for the DC subtracted kernel $\mb \Theta^\circ$ over the subspace $S_\veps$. In this section, we show that the certificate $g = g_{\veps}[\zeta]$ defined in \Cref{thm:certificate_over_S} can also be viewed as the certificate without subspace constraints, satisfying
	\begin{equation*}
		\mb \Theta^\circ [g_{\veps}[\zeta]] \approx \zeta. 
	\end{equation*}
As $\mb P_{S_\veps} \mb \Theta^\circ[g_{\veps}[\zeta]] = \zeta$, we only need $\mb P_{S_\veps^\perp} \mb \Theta^\circ[g_{\veps}[\zeta]]$ to be small. The subspace $S_\veps$ is formed by all Fourier basis with low frequency, and thus contains functions that do not oscillate rapidly, in the sense that for any function $h$ and integer $k$
\begin{align*}
\norm{\mb P_{S_{\veps}^\perp} h}_{L^2} \lesssim \frac{\norm{\frac{d^k}{ds^k} h}_{L^2}}{\dim(S_\veps)^k}. 
\end{align*}
This argument is made rigorous in \Cref{lem:fourier_deriv_lb}; by choosing $k=3$ and extracting the dimension of the subspace from \Cref{eqn:S_eps_def}, we obtain the estimate we are looking for. %
This leaves us to show the derivatives of $\mb \Theta^{\circ}[g_{\veps}[\zeta]]$ are small compared to its norm. 

The remainder of this subsection is organized as follows. We define a relevant notion of derivatives for the kernel $\Theta^\circ$ in \Cref{def:deriv-R}. These derivatives can be represented as a function of the higher order derivatives of $\psi$ and that of the angle function (\Cref{lem:deriv-lambda}). We bound the derivatives of the angle by higher order curvatures in \Cref{lem:lambda1,lem:lambda2,lem:lambda3}, and borrow results in \Cref{lem:skel_dot_ub,lem:skel_ddot_ub,lem:skel_dddot_ub} that $\psi$'s higher order derivatives decrease rapidly since $\psi$ is localized when the network is deep enough. These bounds together allow us to control the $L^2$ to $L^2$ operator norm of operators corresponding to the $i$-th order derivatives of $\Theta^\circ$ in \Cref{lem:R1,lem:R2,lem:R3} by geometric parameters of the manifold $\mc M$, including higher order regularity constants $M_i$ and the angle injectivity radius $\Delta_{\veps}$. 
In \Cref{lem:invariant_deriv_commute}, we show that the projection operator $\mb P_{S_{\veps}}$ and main invariant operator $\wh{\mb M}_{\veps}$ commute with differential operators on functions on $\manifold$, and thus the ``low oscillation'' property of the target function $\zeta$ can be transferred to the ``low oscillation'' of $g_{\veps}[\zeta]$ and further down to that of $\mb \Theta[g_{\veps}[\zeta]]$. 
To simplify the language, we introduce \Cref{def:smoothness} to represent the required regularity property, and prove that $g_{\veps}[\zeta]$ and $\mb \Theta^\circ[g_{\veps}[\zeta]]$ satisfy such regularity in \Cref{lem:g_deriv,lem:Theta_Sperp_proj}. Finally, we get control of $\mb P_{S_\veps^\perp} \mb \Theta^\circ[g_{\veps}[\zeta]]$ in \Cref{lem:Theta_Sperp_proj_cont}.

\begin{definition} \label{def:deriv-R} For any $\mb x, \mb x' \in \mc M$, let $\sigma, \sigma' \in \set{\pm 1}$ denote the class memberships of $\vx$ and $\vx'$, let $s, s' \in \bbR$ be such that $\vx_{\sigma}(s) = \vx$, $\vx_{\sigma'}(s') = \vx'$, and write $\Theta^{(0)}(\mb x, \mb x') = \Theta^\circ(\mb x, \mb x') = \psi^\circ( \angle(\mb x, \mb x'))$. We consider higher order derivatives of the kernel with respect to a ``simultaneous advance''. For $i = 1,2,3$, define inductively
\begin{equation*}
\Theta^{(i)}(\mb x, \mb x') = \left[ \frac{d}{dt} \Bigr|_0  \Theta^{(i-1)}\bigl( \mb x_\sigma(s+t), \mb x_{\sigma'}(s'+t) \bigr) \right] \Biggr|_{\mb x_{\sigma}(s) = \mb x, \mb x_{\sigma'}(s') = \mb x'} .
\end{equation*}
Let $\mb \Theta^{(i)}$ denote the Fredholm integral operator associated to $\Theta^{(i)}$: 
\begin{equation*}
\mb \Theta^{(i)}[h] (\mb x) = \int_{\mb x' \in \mc M} \Theta^{(i)}(\mb x, \mb x') h(\mb x') d \mb x'.
\end{equation*}
It is clear that these definitions do not depend on the choice of $s, s' \in \bbR$ among `equivalent' points (c.f.\ \Cref{eq:dh} and surrounding discussion).

\end{definition}
\begin{remark}
    For the moment, we have elided the issue that due to differentiability issues with the angle
    function $(\vx, \vx') \mapsto \angle(\vx, \vx')$, the kernels $\Theta^{(i)}$ defined in
    \Cref{def:deriv-R} may not be well-defined on all of $\manifold \times \manifold$. This issue
    is resolved in \Cref{lem:deriv-lambda}.
\end{remark}

\begin{lemma} \label{lem:deriv-lambda}%
Let
\begin{align*}
\lambda_0(\mb x, \mb x') &= \angle(\mb x,\mb x') \\
\lambda_{i+1}(\mb x, \mb x') &= \left[ \frac{d}{dt} \Bigr|_0 \lambda_i(\mb x_{\sigma}(s + t),\mb x_{\sigma'}(s' + t)) \right] \Biggr|_{\mb x_{\sigma}(s) = \mb x, \mb x_{\sigma'}(\mb s') = \mb x'} \\
&= \left[ \paren{\frac{\partial}{\partial s} + \frac{\partial}{\partial s'}}\lambda_i(\mb x_{\sigma}(s),\mb x_{\sigma'}(s')) \right] \Biggr|_{\mb x_{\sigma}(s) = \mb x, \mb x_{\sigma'}(s') = \mb x'}, \; i = 0, 1, 2. 
\end{align*}
denote derivatives of the angle function with respect to a ``simultaneous advance''. 
Then when the parameterizations $\vx_{\sigma}$ are five times continuously differentiable (as required in \Cref{sec:prob_form_extended}), these functions are well-defined on $\manifold \times \manifold$. 

In addition, 
the kernels $\Theta^{(i)}$ defined in \Cref{def:deriv-R} are well-defined on $\manifold \times \manifold$ and can be expressed in terms of the derivatives of $\psi$ and the functions $\lambda_i$ as 
\begin{align*}
\Theta^{(0)}(\mb x, \mb x') &= \psi^{\circ}(\angle(\mb x,\mb x'))  \\ 
\Theta^{(1)}(\mb x,\mb x') &= \dot{\psi}(\angle(\mb x,\mb x')) \lambda_1(\mb x,\mb x') \\
\Theta^{(2)}(\mb x,\mb x') &= \ddot{\psi}(\angle(\mb x,\mb x')) \lambda_1^2(\mb x,\mb x')  + \dot{\psi}(\angle(\mb x,\mb x')) \lambda_2(\mb x,\mb x') \\
\Theta^{(3)}(\mb x,\mb x') &= \dddot{\psi}(\angle(\mb x,\mb x')) \lambda_1^3(\mb x,\mb x') + 3
\ddot{\psi}(\angle(\mb x,\mb x')) \lambda_2(\mb x,\mb x') \lambda_1(\mb x,\mb x') \\&+ \dot{\psi}(\angle(\mb x,\mb x')) \lambda_3(\mb x,\mb x'), 
\end{align*}
where $\dot{\psi}, \ddot{\psi}, \dddot{\psi}$ denote the first three derivatives of $\psi^{\circ}$. 
\end{lemma}
\begin{proof}
    Because the function $t\mapsto \acos(t)$ is infinitely differentiable except at $\set{-1, 1} \subset[-1, +1]$ and $\psi$ is 3 times continuously differentiable on $[0, \pi]$ (\Cref{lem:phi_concave}), and given the differentiability assumption on the curves and the fact that \Cref{eq:max_angle} precludes $\manifold$ from containing any antipodal points, the claim follows immediately by the chain rule except on the diagonal $\set{(\vx, \vx) \given \vx \in \manifold}$. 
    Here, suppose $s, s'$ are such that
	$\mb x_{\sigma}(s) = \mb x_{\sigma}(s')$. 
	Then we have $\mb x_{\sigma}(s + t) = \mb x_{\sigma}(s' + t)$ for every $t \in \bbR$. In particular,
    $\angle(\vx_{\sigma}(s+t), \vx_{\sigma}(s'+t))=0$ for all $t \in \bbR$, which implies that $\lambda_i(\vx, \vx) = 0$ for all $i$.
    A similar argument implies well-definedness of $\Theta^{(i)}(\vx,\vx)$ for all $i$, which establishes the claimed formulas on all of $\manifold \times \manifold$. 
    
\end{proof}

\begin{lemma}\label{lem:cos_angle_lb}
    For points $\mb x, \mb x' \in \manifold$ and $\veps \in (0, 1)$ we have
    \begin{align*}
    \sqrt{1 - (\mb x^*\mb x')^2} &\ge \begin{cases} \frac{d_{\manifold}(\mb x, \mb x')}{3}& d_{\manifold}(\mb x, \mb x') \le \frac{1}{\hat{\kappa}} \\
    \frac{2}{\pi}\Delta_\veps & d_{\manifold}(\mb x, \mb x') \ge \frac{\sqrt{\veps}}{\hat{\kappa}} \end{cases}
    \end{align*}
\end{lemma}
\begin{proof}
	When $d_{\manifold}(\mb x, \mb x') \geq \frac{\sqrt{\veps}}{\hat{\kappa}}$, from definition of the angle injectivity radius in \Cref{eq:inj_radius_eps_def} we have $\angle(\mb x, \mb x') \ge \Delta_\veps$. From \Cref{eq:max_angle} we also have $\angle(\mb x, \mb x') \le \pi/2$, then
	\begin{align*}
	\sqrt{1 - (\mb x^*\mb x')^2} &= \sin(\angle(\mb x, \mb x'))\\
	&\ge \sin(\Delta_\veps) \\
	&\geq \frac{2}{\pi}\Delta_\veps \labelthis \label{eq:sinangle_injectivityrad_lb},
	\end{align*}
	where the first inequality comes from the monotonicity of $\sin(t)$ from $0$ to $\pi/2$. The second inequality uses concavity of $\sin$ to get $\sin(t) \geq (2/\pi)t$ for $0 \leq t \leq \pi/2$, and the fact that $\veps < 1$ and hence $\Delta_{\veps} \leq \pi/2$.
	
	When $d_{\manifold}(\mb x, \mb x') \le \frac{1}{\hat{\kappa}} \le \frac{\pi}{2}$, assume $\mb x, \mb x'$ are parameterized by $\mb x_\sigma(s), \mb x_{\sigma}(s')$ separately with $\abs{s - s'} = d_{\manifold}(\mb x, \mb x')$, then $\abs{s' - s} \le \frac{1}{\hat{\kappa}}$. 
	Assuming without loss of generality that $s' \ge s$, using a second-order Taylor expansion and properties from \Cref{lem:x_deriv} gives
	\begin{align*}
	\mb x_{\sigma}(s)^*\mb x_{\sigma}(s') &=\mb x_{\sigma}(s)^*\left(\mb x_{\sigma}(s) + (s' - s)\dot{\mb x}_\sigma(s) + \int_{a = s}^{s'}\int_{b = s}^a \ddot{\mb x}_{\sigma}(b) \, db \, da \right) \\
	&= 1 + \int_{a = s}^{s'}\int_{b = s}^a \innerprod{\mb x_{\sigma}(s)}{\ddot{\mb x}_{\sigma}(b)} \, db \, da\\
	&= 1 + \int_{a = s}^{s'}\int_{b = s}^a  \innerprod{\mb x_{\sigma}(b) + (s - b) \dot{\mb x}_{\sigma}(b) + \int_{c = b}^{s}\int_{d = b}^c \ddot{\mb x}_{\sigma}(d) \, dd \, dc}{ \ddot{\mb x}_{\sigma}(b)} \, db \, da \\
	&= 1 - \frac{(s' - s)^2}{2} + \int_{a = s}^{s'}\int_{b = s}^a \int_{c = b}^{s}\int_{d = b}^c \ddot{\mb x}_{\sigma}(d)^*\ddot{\mb x}_{\sigma}(b) \, dd \, dc \, db \, da, 
	\end{align*}
	with a Taylor expansion at $b$ used in the third line, and using
	the convention that for a real-valued function $f$ and numbers $a < b$, the notation $\int_b^a f(x) \diff x$ denotes the integral $-\int_a^b f(x) \diff x$.
	As $\hat{\kappa} = \max\{\kappa, \frac{2}{\pi}\}$, we can use the previous expression (with a bound of the integrand in the last line by $M_2$, and \Cref{lem:x_deriv} again) to obtain after an integration
	\begin{align*}
	\abs{\mb x_{\sigma}(s)^* \mb x_{\sigma}(s') -  1 + \tfrac{1}{2} (s'-s)^2} &\le \frac{(s'-s)^4}{4!} (1 + \kappa^2) \\
	&\le \frac{(s'-s)^2}{4!}\frac{1 + \kappa^2}{\hat{\kappa}^2} \\
	&\le \frac{(s'-s)^2}{4!}(\frac{\pi^2}{4} + 1) \\
	&< \frac{(s'-s)^2}{6}, 
	\end{align*} 
	and thus
	\begin{align*}
	\sqrt{1 - (\mb x_{\sigma}(s)^*\mb x_{\sigma}(s'))^2} 
	 &= \sqrt{1 + \mb x_{\sigma}(s)^*\mb x_{\sigma}(s')}\sqrt{1 - \mb x_{\sigma}(s)^*\mb x_{\sigma}(s')} \\
	& \ge \sqrt{\left(1 + \paren{1 - \frac{1}{2}(s'-s)^2 - \frac{1}{6}(s' - s)^2}\right)\paren{\frac{1}{2}(s' - s)^2 - \frac{1}{6}(s' - s)^2}}\\
	&= \sqrt{\left(2 - \frac{2}{3}(s' - s)^2\right)\frac{1}{3}(s' - s)^2} \\
	&\ge \sqrt{\left(2 - \frac{2}{3}\paren{\frac{\pi}{2}}^2\right)\frac{(s'-s)^2}{3}} \\
	&> \frac{\abs{s'-s}}{3}.
	\end{align*}
\end{proof}

\begin{lemma}\label{lem:lambda1}
	For any $\mb x, \mb x' \in \manifold$, we have
	\begin{equation*}
	\abs{\lambda_1(\mb x, \mb x')} \le \begin{cases} \frac{7 d_{\manifold}(\mb x, \mb x')^3}{12}  M_4 & d_{\manifold}(\mb x, \mb x') \le \frac{1}{\hat{\kappa}}\\
	2 & \forall \mb x, \mb x' \in \manifold. \end{cases} 
	\end{equation*}
\end{lemma}

\begin{proof} 
	Let $s, s'$ be such that $\vx = \mb x_{\sigma}(s)$ and $\vx'= \mb x_{\sigma'}(s')$, with $\abs{s - s'} = d_{\manifold}(\mb x, \mb x')$ when in addition $\sigma = \sigma'$. As $\angle(\mb x, \mb x') = \acos \paren{\mb x^*\mb x'}$, 
	\begin{align*}
	\lambda_1(\mb x, \mb x') &= \frac{\partial}{\partial s} \angle(\mb x_{\sigma}(s), \mb x_{\sigma'}(s')) + \frac{\partial}{\partial s'} \angle(\mb x_{\sigma}(s), \mb x_{\sigma'}(s'))  \\
	&= - \frac{\dot{\mb x}^* \mb x' + \dot{\mb x}'^* \mb x}{\sqrt{1 - (\mb x^* \mb x')^2}}. \labelthis \label{eq:lambda1}
	\end{align*}
	Notice that $\sqrt{1 - (\mb x^* \mb x')^2} = \norm{(\mb I - \mb x \mb x^*) \mb x'}_{2}$, and therefore
	by \Cref{lem:x_deriv} and Cauchy-Schwarz
	\begin{equation*}
	\left| \frac{\dot{\mb x}^* \mb x'}{\sqrt{1 - (\mb x^* \mb x')^2}}\right| = \left| \innerprod{\dot{\mb x}}{\frac{(\mb I - \mb x \mb x^*) \mb x'}{\norm{(\mb I - \mb x \mb x^*) \mb x'}_{2}}}\right| \le 1,
	\end{equation*}
	and thus $\abs{\lambda_1(\mb x, \mb x')} \le 2$ by symmetry. 
	
	When $d_{\manifold}(\mb x, \mb x') \le \frac{1}{\hat{\kappa}}$, we have $\sigma = \sigma'$, (as above) $\abs{s' - s} \le \frac{1}{\hat{\kappa}}$. 
	By symmetry, we may assume $s' \ge s$. 
	From \Cref{lem:x_deriv}, we have $\dot{\mb x}^* \mb x = \dot{\mb x}^*\ddot{\mb x} = 0$, $\dot{\mb x}^*\dot{\mb x} = 1$, $\ddot{\mb x}^* \mb x^{(3)} = -\frac{1}{3}\dot{\mb x}^* \mb x^{(4)}$. 
	In the remainder of the proof, with an abuse of notation we will write $\dot{\vx} = \dot{\vx}_{\sigma}(s)$, $\dot{\vx}' = \dot{\vx}_{\sigma}(s')$, and so on for the higher derivatives to represent the \textit{specific} points of interest concisely.
	Thus by a fourth-order Taylor expansion (respectively, of $\vx' = \vx_\sigma(s')$ at $s$, and of $\vx = \vx_\sigma(s)$ at $s'$)
	\begin{align*}
	 &\abs{\dot{\mb x}^* \mb x' + \dot{\mb x}'^* \mb x} \\
	 &\qquad = \left| \dot{\mb x}^*\Biggl(\mb x + (s' - s) \dot{\mb x} + \frac{(s' - s)^2}{2} \ddot{\mb x} + \frac{(s' - s)^3}{3!} \mb x^{(3)}+  \int_{a=s}^{s'} \int_{b=s}^a \int_{c = s}^b \int_{d = s}^c \mb x_{\sigma}^{(4)}(d) dd\, dc \, db\, da \Biggr) \right. \\
	 &\qquad \qquad + \dot{\mb x}'^*\Biggl(\mb x' - (s' - s)\dot{\mb x}' + \frac{(s' - s)^2}{2} \ddot{\mb x}' - \frac{(s' - s)^3}{3!} \mb x'^{(3)} \left. +  \int_{a=s}^{s'} \int_{b=a}^{s'} \int_{c = b}^{s'} \int_{d = c}^{s'} \mb x_{\sigma}^{(4)}(d) dd\, dc \, db\, da \Biggr) \right| \\
	 &\qquad =\left| s' - s  + \dot{\mb x}^*\Biggl(  \frac{(s' - s)^3}{3!} \mb x^{(3)} +  \int_{a=s}^{s'} \int_{b=s}^a \int_{c = s}^b \int_{d = s}^c \mb x_{\sigma}^{(4)}(d) \, dd\, dc \, db\, da \Biggr)\right. \\
	 &\qquad \qquad - (s' - s) + \left. \dot{\mb x}'^*\Biggl( - \frac{(s' - s)^3}{3!} \mb x'^{(3)} +  \int_{a=s}^{s'} \int_{b=a}^{s'} \int_{c = b}^{s'} \int_{d = c}^{s'}  \mb x_{\sigma}^{(4)}(d) \, dd\, dc \, db\, da\Biggr) \right| \\
	 &\qquad \le \frac{(s' - s)^3}{3!} \abs{\dot{\mb x}'^*\mb x'^{(3)} - \dot{\mb x}^*\mb x^{(3)}} + \frac{2(s' - s)^4}{4!}M_4 \\
	 &\qquad \leq \frac{(s' - s)^3}{3!}\int_{a = s}^{s'} \abs{\dot{\mb x}_{\sigma}(a)^*\mb x_{\sigma}^{(4)}(a) + \ddot{\mb x}_{\sigma}(a)^*\mb x_{\sigma}^{(3)}(a)} da + \frac{2(s' - s)^4}{4!}M_4 \\
	 &\qquad = \frac{(s' - s)^3}{3!}\int_{a = s}^{s'} \left|\frac{2}{3}\dot{\mb x}_{\sigma}(a)^*\mb x_{\sigma}^{(4)}(a) \right| da + \frac{2(s' - s)^4}{4!}M_4 \\
	 &\qquad \le \frac{7}{36}M_4 (s' - s)^4.  \labelthis \label{eq:lambda_1_extra_term}
	\end{align*}
	Above, the first inequality uses the triangle inequality and Cauchy-Schwarz; 
	the second inequality Taylor expands the first term in the difference at $s$ (which leads to a cancellation with the second term)
	and uses the triangle inequality to move the absolute value inside the integral;
	the following line rewrites using \Cref{lem:x_deriv};
	and then the final line uses Cauchy-Schwarz, integrates and collects constants.
	Using \Cref{lem:cos_angle_lb}, we obtain that when $d_{\manifold}(\mb x, \mb x') \le \frac{1}{\hat{\kappa}}$, 
	\begin{equation*}
	\left| \lambda_1(\mb x, \mb x') \right| \le \frac{7 d_{\manifold}^3(\mb x, \mb x')}{12}  M_4. 
	\end{equation*}
	
\end{proof}

\begin{lemma}\label{lem:lambda2}
	There exists an absolute constant $C$ such that for any $\veps \in (0, 1)$ and $\mb x, \mb x' \in \manifold$ we have
	\begin{equation}
	\abs{\lambda_2(\mb x, \mb x')} \le  \begin{cases} C(M_4^2 + M_5) d_{\manifold}(\mb x, \mb x')^3, & d_{\manifold}(\mb x, \mb x')\le \frac{1}{\hat{\kappa}}  \\
	\frac{\pi}{4}\Delta_{\veps}^{-1} + 2 M_2, & d_{\manifold}(\mb x, \mb x') > \frac{\sqrt{\veps}}{\hat{\kappa}} \end{cases}. 
	\end{equation}
\end{lemma}

\begin{proof}
Let $s, s'$ be such that $\vx = \mb x_{\sigma}(s)$ and $\vx'= \mb x_{\sigma'}(s')$, with $\abs{s - s'} = d_{\manifold}(\mb x, \mb x')$ when in addition $\sigma = \sigma'$.
	From \eqref{eq:lambda1}, 
	\begin{align*}
	\lambda_2(\mb x, \mb x') &= \frac{d}{dt} \Bigr|_0 \lambda_1(\mb x_{\sigma}(s+t), \mb x_{\sigma'}(s'+t)) \\ 
	&= -\frac{\left( \dot{\mb x}^* \mb x' + \dot{\mb x}'^* \mb x \right)^2 \mb x^* \mb x'}{\left( 1 - ( \mb x^* \mb x' )^2 \right)^{3/2}} - \frac{ \dot{\mb x}^* \dot{\mb x}' + \ddot{\mb x}^* \mb x' + \ddot{\mb x}'^* \mb x + \dot{\mb x}'^* \dot{\mb x} }{\sqrt{ 1- (\mb x^* \mb x')^2 }}. \labelthis \label{eq:lambda2}
	\end{align*}
	
	First consider the case where $d_{\manifold}(\mb x, \mb x') > \frac{\sqrt{\veps}}{\hat{\kappa}}$. 
	As $\sqrt{1 - (\mb x^*\mb x')^2} = \norm{(\mb I - \mb x \mb x^*) \mb x'}_{2}$, we can write
	\begin{align*}
		\frac{\ddot{\mb x}^*\mb x'}{\sqrt{ 1- (\mb x^* \mb x')^2 }} &= \innerprod{\ddot{\mb x}}{\frac{(\mb I - \mb x \mb x^*) \mb x'}{\norm{(\mb I - \mb x \mb x^*) \mb x'}_{2}}} + \frac{\ddot{\mb x}^*\mb x \mb x^* \mb x'}{\norm{(\mb I - \mb x \mb x^*) \mb x'}_{2}} \\
		&= \innerprod{\ddot{\mb x}}{\frac{(\mb I - \mb x \mb x^*) \mb x'}{\norm{(\mb I - \mb x \mb x^*) \mb x'}_{2}}} - \frac{\mb x^* \mb x'}{\norm{(\mb I - \mb x \mb x^*) \mb x'}_{2}} 
	\end{align*}
	using \Cref{lem:x_deriv}.
	Thus following \eqref{eq:lambda1} and \Cref{eq:sinangle_injectivityrad_lb} and \Cref{lem:cos_angle_lb,lem:lambda1,lem:x_deriv}, 
	\begin{align*}
		\abs*{\lambda_2(\mb x, \mb x')} &\le \abs*{\lambda_1(\mb x, \mb x')^2 \frac{\mb x^*\mb x'}{\sqrt{1 - (\mb x^*\mb x')^2}}} + \abs*{\frac{2\dot{\mb x}^*\dot{\mb x}' - 2 \mb x^* \mb x'}{\sqrt{1 - (\mb x^*\mb x')^2}}} \\
		&\qquad + \abs*{\innerprod{\ddot{\mb x}}{\frac{(\mb I - \mb x \mb x^*) \mb x'}{\norm{(\mb I - \mb x \mb x^*) \mb x'}_{2}}}} + \abs*{\innerprod{\ddot{\mb x}'}{\frac{(\mb I - \mb x' \mb x'^*) \mb x}{\norm{(\mb I - \mb x' \mb x'^*) \mb x}_{2}}}}  \\
		&\le (4 + 4) \paren{\frac{2}{\pi}\Delta_\veps}^{-1} + 2M_2 \\
		&= \frac{\pi}{4} \Delta_\veps^{-1} + 2M_2.
	\end{align*}

	When $d_{\manifold}(\mb x, \mb x') \le \frac{1}{\hat{\kappa}}$, we have $\sigma = \sigma'$ and $\abs{s' - s} \le \frac{1}{\hat{\kappa}}$. By symmetry, we may assume $s' \ge s$. 
	Following \Cref{lem:x_deriv}, we have $\dot{\mb x}^* \ddot{\mb x} = 0$, $\mb x^* \ddot{\mb x} = - 1$. 
	In the remainder of the proof, with an abuse of notation we will write $\dot{\vx} = \dot{\vx}_{\sigma}(s)$, $\dot{\vx}' = \dot{\vx}_{\sigma}(s')$, and so on for the higher derivatives to represent the \textit{specific} points of interest concisely.
	We can calculate by Taylor expansion and \Cref{lem:x_deriv}
	\begin{align*}
	\ddot{\mb x}^* \mb x'  &= \ddot{\mb x}^* \left(\mb x + (s' - s) \dot{\mb x} + \frac{(s' - s)^2}{2} \ddot{\mb x} + \frac{(s' - s)^3}{6} \mb x^{(3)}
	+ \int_{a = s}^{s'}\int_{b = s}^a\int_{c = s}^b\int_{d = s}^c \mb x_{\sigma}^{(4)}(d)\, dd\,dc \, db \, da\right) \\
	&= -1 + \ddot{\mb x}^*\left(\frac{(s' - s)^2}{2} \ddot{\mb x} + \frac{(s' - s)^3}{6} \mb x^{(3)}
  + \int_{a = s}^{s'}\int_{b = s}^a\int_{c = s}^b\int_{d = s}^c \mb x_{\sigma}^{(4)}(d)\, dd\,dc\,
db \, da\right), \labelthis \label{eq:lambda_2_extra_term_1}\\
	\ddot{\mb x}'^* \mb x  &= \ddot{\mb x}'^* \left(\mb x' - (s' - s) \dot{\mb x}' + \frac{(s' - s)^2}{2} \ddot{\mb x}' - \frac{(s' - s)^3}{6} \mb x'^{(3)}
	+ \int_{a = s}^{s'}\int_{b = a}^{s'}\int_{c = b}^{s'}\int_{d = c}^{s'} \mb x_{\sigma}^{(4)}(d) \, dd\,dc \, db \, da \right) \\
	&= -1 + \ddot{\mb x}^* \left( \frac{(s' - s)^2}{2} \ddot{\mb x}' - \frac{(s' - s)^3}{6} \mb
  x'^{(3)} + \int_{a = s}^{s'}\int_{b = a}^{s'}\int_{c = b}^{s'}\int_{d = c}^{s'} \mb
x_{\sigma}^{(4)}(d) \, dd\,dc \, db \, da \right), \labelthis \label{eq:lambda_2_extra_term_2} \\
	\dot{\mb x}^* \dot{\mb x}' &= \dot{\mb x}^* \left(\dot{\mb x} + (s' - s) \ddot{\mb x} + \frac{(s' - s)^2}{2} \mb x^{(3)} + \frac{(s' - s)^3}{6} \mb x^{(4)}
	+ \int_{a = s}^{s'}\int_{b  = s}^a\int_{c = s}^b\int_{d = s}^c \mb x_{\sigma}^{(5)}(d) \, dd\,dc \, db \, da \right) \\
	&= 1 + \dot{\mb x}^*\left( \frac{(s' - s)^2}{2} \mb x^{(3)} + \frac{(s' - s)^3}{6} \mb x^{(4)} +
  \int_{a = s}^{s'}\int_{b = s}^a\int_{c = s}^b\int_{d = s}^c \mb x_{\sigma}^{(5)}(d) \, dd\,dc\,
db \, da \right) \labelthis \label{eq:lambda_2_extra_term_3} \\
	&= 1 + \dot{\mb x}'^*\left( \frac{(s' - s)^2}{2} \mb x'^{(3)} - \frac{(s' - s)^3}{6} \mb
  x'^{(4)} + \int_{a = s}^{s'}\int_{b = a}^{s'}\int_{c = b}^{s'}\int_{d = c}^{s'} \mb
x_{\sigma}^{(5)}(d) \, dd\,dc \, db \, da \right). \labelthis \label{eq:lambda_2_extra_term_4}
	\end{align*}
	In addition
	\begin{align*}
		\abs{\dot{\mb x}^* \mb x^{(4)} - \dot{\mb x}'^* \mb x'^{(4)}} &= \left| - \int_{a = s}^{s'} \ddot{\mb x}(a)^* \mb x^{(4)}(a) + \dot{\mb x}(a)^* \mb x^{(5)}(a)da \right|, \\
		&\le \abs{s' - s}(M_2 M_4 + M_5)\labelthis \label{eq:lambda_2_extra_term_5}
	\end{align*}
	by Taylor expansion of the first term in the difference on the LHS at $s'$.
	From \Cref{lem:x_deriv}, $\ddot{\mb x}^* \mb x^{(3)} = -\frac{1}{3}\dot{\mb x}^* \mb x^{(4)}$, $\ddot{\mb x}^*\ddot{\mb x} = -\dot{\mb x}^* \mb x^{(3)}$. Whence adding \eqref{eq:lambda_2_extra_term_1}, \eqref{eq:lambda_2_extra_term_2}, \eqref{eq:lambda_2_extra_term_3}, \eqref{eq:lambda_2_extra_term_4} and applying \eqref{eq:lambda_2_extra_term_5} we get
	\begin{align*}
		\abs{\ddot{\mb x}^* \mb x' + \mb x^* \ddot{\mb x}' + 2 \dot{\mb x}^* \dot{\mb x}'} &\le \left| \frac{(s' - s)^2}{2}\Bigl(\ddot{\mb x}^*\ddot{\mb x}  + \dot{\mb x}^* \mb x^{(3)}\Bigr)
		+ \frac{(s' - s)^2}{2}\Bigl(\ddot{\mb x}'^*\ddot{\mb x}' + \dot{\mb x}' \mb x'^{(3)} \Bigr)\right|  \\
		&\qquad + \frac{(s' - s)^3}{6} \paren{1 - \frac{1}{3}} \abs{\dot{\mb x}^* \mb x^{(4)} - \dot{\mb x}'^* \mb x'^{(4)}}
		+ \frac{2 (s' - s)^4}{4!}(M_2M_4 + M_5) \\
		&\le \frac{(s' - s)^4}{9} \paren{M_2M_4 + M_5} + \frac{(s' - s)^4}{12} \paren{M_2M_4 + M_5} \\
		&= \frac{7(s' - s)^4}{36} \paren{M_2M_4 + M_5}. \labelthis \label{eq:lambda_2_extra_term}
	\end{align*}
	From \Cref{lem:x_deriv}, $M_2 \le M_4$. Plugging \eqref{eq:lambda_2_extra_term} and \eqref{eq:lambda_1_extra_term} into the bound \Cref{eq:lambda2} and using \Cref{lem:cos_angle_lb},
	we obtain that when $d_{\manifold}(\mb x, \mb x') \le \frac{1}{\hat{\kappa}} \le \frac{\pi}{2}$
	\begin{align*}
	\abs{\lambda_2(\mb x, \mb x')} &\le \paren{\frac{3}{\abs{s'-s}}}^3 \paren{\frac{7M_4\abs{s'-s}^4}{36}}^2 + \frac{3}{\abs{s'-s}} \frac{7\paren{M_2M_4 + M_5}}{36}\abs{s'-s}^4 \\
	&= \frac{49}{48} M_4^2 \abs{s'-s}^5 + \frac{7}{12} (M_2M_4 + M_5) \abs{s'-s}^3 \\
	&\le C(M_4^2 + M_5) d_{\manifold}^3(\mb x, \mb x') 
	\end{align*}
	for some absolute constant $C > 0$. 
\end{proof}

\begin{lemma}\label{lem:lambda3}
	There exists an absolute constant $C > 0$ such that for any $\veps \in (0, 1)$ and $\mb x, \mb x' \in \manifold$ we have
	\begin{equation*}
	\abs{\lambda_3(\mb x, \mb x')} \le  \begin{cases} C (M_4^3  + M_4M_5 + M_3M_4)d_{\manifold}(\mb x, \mb x')^3,  & d_{\manifold}(\mb x, \mb x')  \le \frac{1}{\hat{\kappa}} \\
	\frac{3\pi^2}{4}\Delta_\veps^{-2} + 9\pi M_2\Delta_\veps^{-1} + 2M_3 + 8,  &  d_{\manifold}(\mb x, \mb x')  > \frac{\sqrt{\veps}}{\hat{\kappa}} \end{cases}. 
	\end{equation*}
\end{lemma}

\begin{proof}

Let $s, s'$ be such that $\vx = \mb x_{\sigma}(s)$ and $\vx'= \mb x_{\sigma'}(s')$, with $\abs{s - s'} = d_{\manifold}(\mb x, \mb x')$ when in addition $\sigma = \sigma'$.
	Then from \eqref{eq:lambda2}, 
	\begin{equation}
	\begin{split}
	\lambda_3 (\mb x, \mb x') &= \frac{d}{dt} \Bigr|_0 \lambda^2(\mb x_{\sigma}(s + t), \mb x_{\sigma'}(s' + t)) \\
	&= - \frac{\left( \dot{\mb x}^* \mb x' + \dot{\mb x}'^* \mb x \right)^3}{\left( 1 - ( \mb x^* \mb x' )^2 \right)^{3/2}} \;-\; 3 \frac{\left( \dot{\mb x}^* \mb x' + \dot{\mb x}'^* \mb x \right)^3\left(\mb x^* \mb x'\right)^2}{\left( 1 - ( \mb x^* \mb x' )^2 \right)^{5/2}} \\
	&\hphantom{=}- 3 \frac{\left( \dot{\mb x}^* \mb x' + \dot{\mb x}'^* \mb x \right)\mb x^* \mb x'\left(2 \dot{\mb x}^* \dot{\mb x}' + \ddot{\mb x}^* \mb x' + \ddot{\mb x}'^* \mb x \right)}{\left( 1 - ( \mb x^* \mb x' )^2 \right)^{3/2}} \\
	&\hphantom{=}-  \frac{ 3 \dot{\mb x}^* \ddot{\mb x}' + 3 \ddot{\mb x}^*\dot{\mb x}' + \mb x^{(3)*} \mb x' + \mb x'^{(3)*} \mb x }{\sqrt{ 1- (\mb x^* \mb x')^2 }}. 
	\end{split}
	\label{eq:lambda3_expr}
	\end{equation}
	When $d_{\manifold}(\mb x, \mb x')> \frac{\sqrt{\veps}}{\hat{\kappa}}$, as $\sqrt{1 - (\mb x^*\mb x')^2} = \norm{(\mb I - \mb x \mb x^*) \mb x'}_{2}$ and from \Cref{lem:x_deriv} $\mb x^{(3)*}\mb x = 0$, 
	\begin{equation*}
	\frac{\mb x^{(3)*}\mb x'}{\sqrt{ 1- (\mb x^* \mb x')^2 }} = \innerprod{\mb x^{(3)}}{\frac{(\mb I - \mb x \mb x^*) \mb x'}{\norm{(\mb I - \mb x \mb x^*) \mb x'}_{2}}}. 
	\end{equation*}
	Thus from \eqref{eq:lambda1}, \eqref{eq:lambda2}, \Cref{eq:sinangle_injectivityrad_lb}, \Cref{lem:cos_angle_lb}, \Cref{lem:lambda1} and \Cref{lem:lambda2}, 
	\begin{align*}
		\abs*{\lambda_3(\mb x, \mb x')} &\le \abs*{\lambda_1(\mb x, \mb x')^3} + \abs*{3 \lambda_1(\mb x, \mb x')\lambda_2(\mb x, \mb x')\frac{\mb x^*\mb x'}{\sqrt{ 1- (\mb x^* \mb x')^2 }}} + \abs*{\frac{ 3 \dot{\mb x}^* \ddot{\mb x}'+ 3 \ddot{\mb x}^*\dot{\mb x}'}{\sqrt{ 1- (\mb x^* \mb x')^2 }}} \labelthis \label{eq:lambda3_term_reuse}  \\
		& \quad+ \abs*{\innerprod{\mb x^{(3)}}{\frac{(\mb I - \mb x \mb x^*) \mb x'}{\norm{(\mb I - \mb x \mb x^*) \mb x'}_{2}}}} + \abs*{\innerprod{\mb x'^{(3)}}{\frac{(\mb I - \mb x' \mb x'^*) \mb x'}{\norm{(\mb I - \mb x' \mb x'^*) \mb x}_{2}}}} \\
	&\le 8 + 6\left(\frac{\pi}{4}\Delta_\veps^{-1} + 2M_2\right)\paren{\frac{2}{\pi}\Delta_\veps}^{-1} + 6M_2\paren{\frac{2}{\pi}\Delta_\veps}^{-1} + 2M_3 \\
	&\le \frac{3\pi^2}{4}\Delta_\veps^{-2} + 9\pi M_2\Delta_\veps^{-1} + 2M_3 + 8. 
	\end{align*}
	
	When $d_{\manifold}(\mb x, \mb x') \le \frac{1}{\hat{\kappa}}$, we have $\sigma = \sigma'$ and $\abs{s' - s} \le \frac{1}{\hat{\kappa}}$. By symmetry, we may assume $s' \ge s$. 
	Following \Cref{lem:x_deriv}, $\dot{\mb x}^* \ddot{\mb x} = 0$ and $\ddot{\mb x}^*\ddot{\mb x} = -\dot{\mb x}^* \mb x^{(3)}$. 
	In the remainder of the proof, with an abuse of notation we will write $\dot{\vx} = \dot{\vx}_{\sigma}(s)$, $\dot{\vx}' = \dot{\vx}_{\sigma}(s')$, and so on for the higher derivatives to represent the \textit{specific} points of interest concisely.
	Because we can reuse bounds for lower-order $\lambda_i$ terms to bound the first three terms in \Cref{eq:lambda3_expr}, we will focus on controlling the last term.
	We can calculate by Taylor expansion
	\begin{align*}
	3 \dot{\mb x}'^* \ddot{\mb x} &+ 3 \ddot{\mb x}'^*\dot{\mb x} + \mb x'^{(3)*} \mb x + \mb x^{(3)*} \mb x' \\
	&= \quad 3 \ddot{\mb x}^* \left( \dot{\mb x} + (s' - s) \ddot{\mb x} + \frac{(s' - s)^2}{2}\mb x^{(3)} + \frac{(s' - s)^3}{6}\mb x^{(4)}  + \int_{a = s}^{s'} \int_{b = s}^a \int_{c = s}^b \int_{d = s}^c  \mb x_{\sigma}^{(5)}(d) \, dd \, dc \, db \, da \right) \\
	&\quad + 3 \ddot{\mb x}'^* \left( \dot{\mb x}' - (s' - s) \ddot{\mb x}' + \frac{(s' - s)^2}{2}\mb x'^{(3)} - \frac{(s' - s)^3}{6}\mb x'^{(4)} + \int_{a = s}^{s'} \int_{b = a}^{s'} \int_{c = b}^{s'} \int_{d = c}^{s'}  \mb x_{\sigma}^{(5)}(d) \, dd \, dc \, db \, da \right) \\
	&\quad + \mb x^{(3)*} \left( \mb x + (s' - s) \dot{\mb x} + \frac{(s' - s)^2}{2}\ddot{\mb x} + \frac{(s' - s)^3}{6}\mb x^{(3)}  + \int_{a = s}^{s'} \int_{b = s}^a \int_{c = s}^b \int_{d = s}^c  \mb x_{\sigma}^{(4)}(d) \, dd \, dc \, db \, da \right)\\
	&\quad + \mb x'^{(3)*} \left( \mb x' - (s' - s) \dot{\mb x}' + \frac{(s' - s)^2}{2}\ddot{\mb x}' - \frac{(s' - s)^3}{6}\mb x'^{(3)}  + \int_{a = s}^{s'} \int_{b = a}^{s'} \int_{c = b}^{s'} \int_{d = c}^{s'}  \mb x_{\sigma}^{(4)}(d) \, dd \, dc \, db \, da \right)\\
	&= 2(s' - s)\paren{\ddot{\mb x}^*\ddot{\mb x} - \ddot{\mb x}'^*\ddot{\mb x}'} + 2(s' - s)^2 \paren{\ddot{\mb x}^*\mb x^{(3)} + \ddot{\mb x}'^*\mb x'^{(3)}} \\
	&\quad + \frac{(s' - s)^3}{6} \paren{3 \ddot{\mb x}^*\mb x^{(4)} - 3\ddot{\mb x}'^*\mb x'^{(4)} + \mb x^{(3)*}\mb x^{(3)} - \mb x'^{(3)*}\mb x'^{(3)}} \\
	&\quad + 3 \ddot{\mb x}^* \int_{a = s}^{s'} \int_{b = s}^a \int_{c = s}^b \int_{d = s}^c  \mb x_{\sigma}^{(5)}(d) \, dd \, dc \, db \, da 
	+ 3 \ddot{\mb x}'^* \int_{a = s}^{s'} \int_{b = a}^{s'} \int_{c = b}^{s'} \int_{d = c}^{s'}  \mb x_{\sigma}^{(5)}(d) \, dd \, dc \, db \, da \\
	&\quad + \mb x^{(3)*}  \int_{a = s}^{s'} \int_{b = s}^a \int_{c = s}^b \int_{d = s}^c  \mb x_{\sigma}^{(4)}(d) \, dd \, dc \, db \, da
	+ \mb x'^{(3)*} \int_{a = s}^{s'} \int_{b = a}^{s'} \int_{c = b}^{s'} \int_{d = c}^{s'}  \mb x_{\sigma}^{(4)}(d) \, dd \, dc \, db \, da.  \labelthis \label{eq:lambda_3_extra_term_lhs} 
	\end{align*}
	We expand the first term by successive Taylor expansion
	\begin{align*}
	\Biggl| \paren{\ddot{\mb x}^*\ddot{\mb x} - \ddot{\mb x}'^*\ddot{\mb x}'} & +  (s' - s)\paren{\ddot{\mb x}^*\mb x^{(3)} + \ddot{\mb x}'^*\mb x'^{(3)}} \Biggr| \\
	&= \left|- \int_{a = s}^{s'}2\ddot{\mb x}_{\sigma}(a)^*\mb x_{\sigma}^{(3)}(a) da + (s' - s)\paren{\ddot{\mb x}^*\mb x^{(3)} + \ddot{\mb x}'^*\mb x'^{(3)}}\right| \\
	&= \left|\int_{a = s}^{s'} \left[\paren{\ddot{\mb x}'^*\mb x'^{(3)} -\ddot{\mb x}_{\sigma}(a)^*\mb x_{\sigma}^{(3)}(a) } - \paren{\ddot{\mb x}_{\sigma}(a)^*\mb x_{\sigma}^{(3)}(a) -\ddot{\mb x}^*\mb x^{(3)}} \right]da\right| \\
	&= \left| \int_{a = s}^{s'}\int_{b = a}^{s'} \left(\ddot{\mb x}_{\sigma}(b)^*\mb x_{\sigma}^{(4)}(b) + \mb x_{\sigma}^{(3)}(b)^*\mb x_{\sigma}^{(3)}(b)\right) \, db \, da  \right. \\& \qquad \left. - \int_{a = s}^{s'}\int_{b = 
	s}^{a}\left(\ddot{\mb x}_{\sigma}(b)^*\mb x_{\sigma}^{(4)}(b) + \mb x_{\sigma}^{(3)}(b)^*\mb x_{\sigma}^{(3)}(b)\right) \, db \, da \right|.  \\
	&= \left| \int_{b = s}^{s'}\int_{a = s}^{b} \paren{\ddot{\mb x}_{\sigma}(b)^*\mb x_{\sigma}^{(4)}(b) + \mb x_{\sigma}^{(3)}(b)^*\mb x_{\sigma}^{(3)}(b)} \, da \, db  \right.  \\& \qquad \left. - \int_{b = s}^{s'}\int_{a = 
	b}^{s'}\paren{\ddot{\mb x}_{\sigma}(b)^*\mb x_{\sigma}^{(4)}(b) + \mb x_{\sigma}^{(3)}(b)^*\mb x_{\sigma}^{(3)}(b)} \, da \, db \right|  \\
	&=  \left| \int_{b = s}^{s'}\paren{(b - s) - (s' - b)} \paren{ \ddot{\mb x}_{\sigma}(b)^*\mb x_{\sigma}^{(4)}(b) + \mb x_{\sigma}^{(3)}(b)^*\mb x_{\sigma}^{(3)}(b)} \, db  \right|.
	\end{align*}
	Above, in the fourth equality we rewrite the preceding integrals by switching the limits of integration; the fifth equality then just integrates over $a$.
	As $2b - s - s'$ stays positive when $b > (s + s')/2$ and negative otherwise, we divide the integral into two parts, change variables using $b' = s + s' - b$ and get 
	\begin{align*}
	\Biggl| \paren{\ddot{\mb x}^*\ddot{\mb x} - \ddot{\mb x}'^*\ddot{\mb x}'} & +  (s' - s)\paren{\ddot{\mb x}^*\mb x^{(3)} + \ddot{\mb x}'^*\mb x'^{(3)}} \Biggr| \\
	&= \left| \int_{b = (s + s')/2}^{s'}\paren{2b - s - s'} \paren{ \ddot{\mb x}_{\sigma}(b)^*\mb x_{\sigma}^{(4)}(b) + \mb x_{\sigma}^{(3)}(b)^*\mb x_{\sigma}^{(3)}(b)} \, db  \right.  \\
	&\qquad \left. - \int_{b = s}^{(s + s')/2}\paren{s + s' - 2b} \paren{ \ddot{\mb x}_{\sigma}(b)^*\mb x_{\sigma}^{(4)}(b) + \mb x_{\sigma}^{(3)}(b)^*\mb x_{\sigma}^{(3)}(b)} \, db \right| \\
	&= \left| \int_{b' = s}^{(s + s')/2}\paren{s + s' - 2b'} \Bigl( \ddot{\mb x}_{\sigma}(s + s' - b')^*\mb x_{\sigma}^{(4)}(s + s' - b') \right.\\
	&\qquad \qquad \left. + \mb x_{\sigma}^{(3)}(s + s' - b')^*\mb x_{\sigma}^{(3)}(s + s' - b')\Bigr) \, db'  \right.  \\
	&\qquad \left. - \int_{b = s}^{(s + s')/2}\paren{s + s' - 2b} \paren{ \ddot{\mb x}_{\sigma}(b)^*\mb x_{\sigma}^{(4)}(b) + \mb x_{\sigma}^{(3)}(b)^*\mb x_{\sigma}^{(3)}(b)} \, db \right| \\
	&= \left| \int_{b = s}^{(s + s')/2}\paren{s + s' - 2b} \int_{c = b}^{s + s' - b} \paren{ \ddot{\mb x}_{\sigma}(c)^*\mb x_{\sigma}^{(5)}(c) + 3 \mb x_{\sigma}^{(3)}(c)^*\mb x_{\sigma}^{(4)}(c)} \, dc \, db  \right| \\
	&\le \left| \int_{b = s}^{(s + s') / 2}\paren{s + s' - 2b}^2 \paren{M_2 M_5 + 3 M_3M_4} \, db \, da \right| \\
	&= -\paren{M_2 M_5 + 3 M_3M_4} \frac{(s + s' - 2b)^3}{6} \Big|_{b = s}^{(s + s')/2}\\
	&= \frac{(s' - s)^3}{6}\paren{M_2 M_5 + 3 M_3M_4} \labelthis \label{eq:lambda_3_extra_term_1}
	\end{align*}
	We use Taylor expansion again for the second term of \Cref{eq:lambda_3_extra_term_lhs}
	\begin{equation} 
	3 \ddot{\mb x}^*\mb x^{(4)} - 3\ddot{\mb x}'^*\mb x'^{(4)} + \mb x^{(3)*}\mb x^{(3)} - \mb x'^{(3)*}\mb x'^{(3)} = -\int_{a = s}^{s'} \left[ 3 \ddot{\mb x}_{\sigma}(a)^*\mb x_{\sigma}^{(5)}(a) + 5 \mb x_{\sigma}^{(3)}(a)\mb x_{\sigma}^{(4)}(a) \right] da.  \label{eq:lambda_3_extra_term_2}
	\end{equation}
	Plug \Cref{eq:lambda_3_extra_term_1,eq:lambda_3_extra_term_2} back to \Cref{eq:lambda_3_extra_term_lhs} and we conclude that 
	\begin{align*}
	&\abs{3 \dot{\mb x}'^* \ddot{\mb x} + 3 \ddot{\mb x}'^*\dot{\mb x} + \mb x'^{(3)*} \mb x + \mb x^{(3)*} \mb x'} \\
	&\qquad \qquad \le \frac{(s' - s)^4}{3} \paren{M_2M_5 + 3M_3M_4} + \frac{(s' - s)^4}{6} \paren{3M_2M_5 + 5M_3M_4}  \\ &\qquad  \qquad  \qquad + \frac{(s' - s)^4}{4!}\paren{6M_2M_5 + 2M_3M_4} \\
	&\qquad \qquad = \frac{(s' - s)^4}{12}\paren{13M_2M_5 + 23M_3M_4} \labelthis \label{eq:lambda_3_extra_term}. 
	\end{align*}
	As from \Cref{lem:x_deriv} $M_2 \le M_4$, when $\abs{s' - s} \le \frac{1}{\hat{\kappa}} \le \frac{\pi}{2}$, plugging \eqref{eq:lambda_1_extra_term}, \eqref{eq:lambda_2_extra_term}, \eqref{eq:lambda_3_extra_term}, and \Cref{lem:cos_angle_lb} into \Cref{eq:lambda3_expr}, we have
	\begin{align*}
	\lambda_3(\mb x, \mb x') &\le \paren{\frac{\abs{s' - s}}{3}}^{-3} \paren{\frac{7}{36}M_4 \abs{s'-s}^4}^3 + 3 \paren{\frac{\abs{s' - s}}{3}}^{-5} \paren{\frac{7}{36}M_4 \abs{s'-s}^4}^3 \\
	& + 3 \paren{\frac{\abs{s' - s}}{3}}^{-3}\paren{\frac{7}{36}M_4 \abs{s'-s}^4} \paren{\frac{7\abs{s'-s}^4}{36}(M_2M_4 + M_5)} \\
	&+ \paren{\frac{\abs{s' - s}}{3}}^{-1} \frac{\abs{s' - s}^4}{12}(13M_2M_5 + 23M_3M_4) \\
	&= (\frac{343}{1728}\abs{s' - s}^9 + \frac{343}{64}\abs{s' - s}^7 )M_4^3 + \frac{49}{16}M_4 (M_2M_4 + M_5)\abs{s' - s}^5 \\
	& + \frac{1}{4}\abs{s' - s}^3(13M_2M_5 + 23M_3M_4) \\
	&\ltsim (M_4^3 + M_4M_5 + M_3M_4)d_{\manifold}^3(\mb x, \mb x'),
	\end{align*}
	where the last line uses the fact that we can adjust constants to keep only the lowest-order term involving the distance, given that the distance is bounded.
\end{proof}

\begin{lemma}\label{lem:R1}  
For $\veps \in (0, \frac{3}{4})$, there exist positive constants $C, C_{1}$ such that when $L \ge C$, we have
	\begin{equation*}
	\norm{\mb \Theta^{(1)}}_{L^2 \to L^2} \le P_1 ( M_4, \len(\manifold), \Delta_\veps^{-1} ) \, n,
	\end{equation*} 
 where $P_1( M_4, \len(\manifold), \Delta_{\veps}^{-1} ) = C_{1}\paren{M_4 + \len(\manifold)\Delta_\veps^{-2}}$ is a polynomial in $M_4, \len(\manifold)$ and $\Delta_\veps^{-1}$.
\end{lemma}

\begin{proof} 
	From \Cref{lem:young_ineq} and \Cref{lem:deriv-lambda} we have 
	\begin{align*}
	\norm{\mb \Theta^{(1)}}_{L^2 \to L^2} &\le \sup_{\mb x \in \manifold} \int_{\mb x' \in \manifold} \abs{\Theta^{(1)}(\mb x, \mb x')} d\mb x' \\
	&= \sup_{\mb x \in \manifold} \int_{\mb x' \in \manifold}  \abs{\dot{\psi}(\angle(\mb x, \mb x'))\lambda_1(\mb x, \mb x')} d \mb x' \\
	&\le \sup_{\mb x \in \manifold} \int_{\mb x' \in \manifold}  \abs{\dot{\psi}(\angle(\mb x, \mb x'))}\abs{\lambda_1(\mb x, \mb x')} d \mb x'
	\end{align*}
	
	\Cref{lem:skel_dot_ub,lem:lambda1} provide us the control for $\dot{\psi}$ and $\lambda_1$. From \Cref{lem:skel_dot_ub}, there exist constants $C, C_1$, such that when $L > C$, we have
	\begin{align*}
	\max_{t \ge r} \abs{\dot{\psi}(t)} \le \frac{C_1 n}{r^2}. 
	\end{align*}
	From \Cref{lem:lambda1}, we have 
	\begin{equation*}
	\abs{\lambda_1(\mb x, \mb x')} \le \begin{cases} \frac{7 d_{\manifold}(\mb x, \mb x')^3}{12}  M_4 & d_{\manifold}(\mb x, \mb x') \le \frac{1}{\hat{\kappa}}\\
	2 & \forall \mb x, \mb x' \in \manifold. \end{cases} 
	\end{equation*}
	
	In order to get a lower bound for the angle $\angle(\mb x, \mb x')$, for a fixed point $\mb x \in \manifold$, we decompose the integral into nearby piece $\mc N(\mb x) = \set{\mb x' \in \manifold \, \vert \, d_{\manifold}(\mb x, \mb x') \le \frac{\sqrt{\veps}}{\hat{\kappa}}}$ and faraway piece $\mc F(\mb x) = \set{\mb x' \in \manifold \, \vert \, d_{\manifold}(\mb x, \mb x') > \frac{\sqrt{\veps}}{\hat{\kappa}}}$ and have
	\begin{align*}
	    \norm{\mb \Theta^{(1)}}_{L^2 \to L^2} &\le \sup_{\mb x \in \manifold} \paren{ \int_{\mb x' \in \mc N(\mb x)}  \abs{\dot{\psi}(\angle(\mb x, \mb x'))}\abs{\lambda_1(\mb x, \mb x')} d \mb x' + \int_{\mb x' \in \mc F(\mb x)}  \abs{\dot{\psi}(\angle(\mb x, \mb x'))}\abs{\lambda_1(\mb x, \mb x')} d \mb x'}.\labelthis \label{eq:R1_integral_bound}
	\end{align*}
	Then for any point $\mb x'$ in faraway piece $\mc F(\mb x)$, $d_{\manifold}(\mb x, \mb x') > \frac{\sqrt{\veps}}{\hat{\kappa}}$, we have $\angle(\mb x, \mb x') \ge \Delta_\veps$ from \Cref{eq:inj_radius_eps_def} with $\Delta_\veps \le \frac{\sqrt{\veps}}{\hat{\kappa}} \le \frac{\pi}{2}$. From \Cref{lem:skel_dot_ub,lem:lambda1} we get
	\begin{align*}
	\int_{\mb x' \in \mc F(\mb x)}  \abs{\dot{\psi}(\angle(\mb x, \mb x'))}\abs{\lambda_1(\mb x, \mb x')} d \mb x' &\le \int_{\mb x' \in \mc F(\mb x)} \frac{C_1 n }{\Delta_{\veps}^2}\abs{\lambda_1(\mb x, \mb x')} d \mb x'  \\
	&\le \len(\manifold) \paren{2 \frac{C_1 n }{\Delta_{\veps}^2}} \\
	&\le  C' \Delta_\veps^{-2} \len(\manifold) n \labelthis \label{eq:theta1_far}
	\end{align*}
	for some constant $C'$. 
	
	For the integral over nearby piece, let $s, \sigma$ be such that $\mb x = \mb x_{\sigma}(s)$. Follow \Cref{lem:angle_inverse} we have \linebreak$\angle(\mb x_{\sigma}(s), \mb x_{\sigma}(s')) \ge (1 - \veps) \abs{s - s'}$ when $\abs{s - s'} \le \frac{\sqrt{\veps}}{\hat{\kappa}}$. As $d(\mb x, \mb x_{\sigma}(s')) \le \abs{s - s'}$, from \Cref{lem:skel_dot_ub,lem:lambda1} we get
	\begin{align*}
	\int_{\mb x' \in \mc N(\mb x)}  \abs{\dot{\psi}(\angle(\mb x, \mb x'))}\abs{\lambda_1(\mb x, \mb x')} d \mb x' &\le \int_{s' = s - \frac{\sqrt{\veps}}{\hat{\kappa}}}^{s - \frac{\sqrt{\veps}}{\hat{\kappa}}}\abs{\dot{\psi}(\angle(\mb x, \mb x_{\sigma}(s')))}\abs{\lambda_1(\mb x, \mb x_{\sigma}(s'))} ds' \\
	&\le \int_{s' = s - \frac{\sqrt{\veps}}{\hat{\kappa}}}^{s - \frac{\sqrt{\veps}}{\hat{\kappa}}} \frac{C_1 n }{(1 - \veps)^2\abs{s' - s}^2}\abs{\lambda_1(\mb x, \mb x_{\sigma}(s'))} d s' \\
	&\le \int_{s' = s - \frac{\sqrt{\veps}}{\hat{\kappa}}}^{s - \frac{\sqrt{\veps}}{\hat{\kappa}}} \frac{C_1 n }{(1 - \veps)^2\abs{s' - s}^2}\abs*{\frac{7 \abs{s' - s}^3}{12}  M_4 } d s'  \\
	&\le C''M_4 n \int_{s' = s - \frac{\sqrt{\veps}}{\hat{\kappa}}}^{s - \frac{\sqrt{\veps}}{\hat{\kappa}}} \abs{s' - s}ds'\\
	&\le C''M_4 n \paren{\frac{\sqrt{\veps}}{\hat{\kappa}}}^2 \\
	&\le  C'' M_4 n  \labelthis \label{eq:theta1_near}
	\end{align*}
	for some constant $C''$, where the fourth line comes from $\veps < \frac{3}{4}$ and the last line comes from $\hat{\kappa} \ge \frac{2}{\pi}$ from definition. Plugging \Cref{eq:theta1_far,eq:theta1_near} back in \Cref{eq:R1_integral_bound} proves the claim. 
\end{proof}

\begin{lemma}\label{lem:R2} 
For $\veps \in (0, \frac{3}{4})$, there exist positive constants $C, C_{2}$ such that when $L \ge C$, we have 
	\begin{equation*}
	\norm{\mb \Theta^{(2)}}_{L^2 \to L^2}  \le P_2(M_2, M_4, M_5, \len(\manifold), \Delta_{\veps}^{-1} ) n
	\end{equation*} 
	where $$P_2(M_2, M_4, M_5, \len(\manifold), \Delta_{\veps}^{-1} ) = C_{2}\paren{M_4^2 + M_5 + \len(\manifold)\paren{\Delta_\veps^{-3} + M_2\Delta_\veps^{-2}}}$$ is a polynomial in $M_2, M_4, M_5, \len(\manifold)$ and $\Delta_\veps^{-1}$.
\end{lemma}

\begin{proof}
From \Cref{lem:young_ineq} we have
\begin{equation*}
	\norm{\mb \Theta^{(2)}}_{L^2 \to L^2} = \sup_{\mb x \in \manifold} \int_{\mb x' \in \manifold} \abs{\Theta^{(2)}(\mb x, \mb x')} d\mb x'. 
\end{equation*}
From \Cref{lem:deriv-lambda}, we know
\begin{align*}
  \abs{\Theta^{(2)}(\mb x, \mb x')} &= \abs{\ddot{\psi}(\angle(\mb x, \mb x'))\lambda_1^2(\mb x, \mb x') + \dot{\psi}(\angle(\mb x, \mb x'))\lambda_2(\mb x, \mb x')}\\
    &\le \abs{\ddot{\psi}(\angle(\mb x, \mb x'))}\abs{\lambda_1^2(\mb x, \mb x')} +
    \abs{\dot{\psi}(\angle(\mb x, \mb x'))}\abs{\lambda_2(\mb x, \mb x')}.
\end{align*}

\Cref{lem:skel_dot_ub,lem:skel_ddot_ub} provide bounds for derivatives of $\psi(t)$: there exist constants $C, C_1, C_2$, such that when $L > C$, we have
\begin{align*}
\max_{t \ge r} \abs{\dot{\psi}(t)} \le \frac{C_1 n}{r^2} 
\end{align*}
and 
\begin{align*}
\max_{t \ge r} \abs{\ddot{\psi}(t)} \le\frac{C_2 n}{r^3}. 
\end{align*} 

To utilize the bound above, we need to get a lower bound for the angle $\angle(\mb x, \mb x')$. For a fixed point $\mb x \in \manifold$, we decompose the integral into nearby piece $\mc N(\mb x) = \set{\mb x' \in \manifold \, \vert \, d_{\manifold}(\mb x, \mb x') \le \frac{\sqrt{\veps}}{\hat{\kappa}}}$ and faraway piece $\mc F(\mb x) = \set{\mb x' \in \manifold \, \vert \, d_{\manifold}(\mb x, \mb x') > \frac{\sqrt{\veps}}{\hat{\kappa}}}$. Then for $\mb x' \in \mc F(\mb x)$, $d_{\manifold}(\mb x, \mb x') > \frac{\sqrt{\veps}}{\hat{\kappa}}$, and we have $\angle(\mb x, \mb x' ) \ge \Delta_\veps$ with $\Delta_\veps \le \frac{\sqrt{\veps}}{\hat{\kappa}} \le \frac{\pi}{2}$. From \Cref{lem:lambda1,lem:lambda2} we get
\begin{align*}
	\int_{\mb x' \in \mc F(\mb x)}\abs{\Theta^{(2)}(\mb x, \mb x')} d \mb x' &\le \int_{\mb x' \in \mc F(\mb x)} \frac{C_2 n }{\Delta_{\veps}^3}\abs{\lambda_1^2(\mb x, \mb x')} + \frac{C_1 n }{\Delta_{\veps}^2} \abs{\lambda_2(\mb x, \mb x')} d \mb x' \\
	&\le \len(\manifold) \paren{4 \frac{C_2 n }{\Delta_{\veps}^3} + \frac{C_1 n }{\Delta_{\veps}^2} \paren{\frac{\pi}{4}\Delta_{\veps}^{-1} + 2 M_2} }\\
	&\le  C' \paren{\Delta_\veps^{-3} + M_2\Delta_\veps^{-2}} \len(\manifold) n \labelthis \label{eq:theta2_far}
\end{align*}
for some constant $C'$. 

For nearby piece, let $\sigma, s$ be such that $\mb x = \mb x_{\sigma}(s)$. Follow \Cref{lem:angle_inverse} we have $\angle(\mb x, \mb x_{\sigma}(s')) \ge (1 - \veps)\abs{s - s'}$ when $\abs{s - s'} \le \frac{\sqrt{\veps}}{\hat{\kappa}}$. From \Cref{lem:skel_dot_ub,lem:skel_ddot_ub,lem:lambda1,lem:lambda2} there exists constant $c, C''$ such that 
\begin{align*}
	&\int_{\mb x' \in \mc N(\mb x)}\abs{\Theta^{(2)}(\mb x, \mb x')} d \mb x'  \\
	&\quad \le \int_{s' = s - \frac{\sqrt{\veps}}{\hat{\kappa}}}^{s - \frac{\sqrt{\veps}}{\hat{\kappa}}} \abs{\Theta^{(2)}(\mb x, \mb x_{\sigma}(s'))} ds' \\
	&\quad \le \int_{s' = s - \frac{\sqrt{\veps}}{\hat{\kappa}}}^{s - \frac{\sqrt{\veps}}{\hat{\kappa}}} \frac{C_2 n }{(1 - \veps)^3\abs{s' - s}^3}\abs{\lambda_1^2(\mb x, \mb x_\sigma(s'))} + \frac{C_1 n }{(1 - \veps)^2\abs{s' - s}^2} \abs{\lambda_2(\mb x, \mb x_\sigma(s') )} d s'  \\
	&\quad \le \int_{s' = s - \frac{\sqrt{\veps}}{\hat{\kappa}}}^{s - \frac{\sqrt{\veps}}{\hat{\kappa}}} \frac{C_2 n }{(1 - \veps)^3\abs{s' - s}^3}\abs*{\frac{7 \abs{s' - s}^3}{12}  M_4 }^2 + \frac{C_1 n }{(1 - \veps)^2\abs{s' - s}^2} \abs*{c(M_4^2 + M_5) \abs{s' - s}^3} d s'  \\
	&\quad \le  C'' \paren{ 1 + \paren{\frac{\sqrt{\veps}}{\hat{\kappa}}}^4}\paren{M_4^2 + M_5} n \\
	&\quad \le  C'' \paren{M_4^2 + M_5} n.  \labelthis \label{eq:theta2_near}
\end{align*}
combining \eqref{eq:theta2_far} and \eqref{eq:theta2_near} directly proves the claim. 
\end{proof}

\begin{lemma}\label{lem:R3} 
	For $\veps \in (0, \frac{3}{4})$, there exist positive constants $C, C_{3}$ such that when $L \ge C$, we have \begin{equation*}
		\norm{\mb \Theta^{(3)}}_{L^2 \to L^2}  \le P_3(M_2, M_3, M_4, M_5, \len(\manifold), \Delta_{\veps}^{-1}) n
	\end{equation*} 
	where $$P_3(M_2, M_3, M_4, M_5, \len(\manifold), \Delta_{\veps}^{-1} ) = C_{3}\paren{M_4^3 +M_3M_4 + M_4M_5 +  \len(\manifold)\paren{\Delta_\veps^{-4} + M_2 \Delta_\veps^{-3} + M_3 \Delta_\veps^{-2}}}$$ is a polynomial in $M_2, M_3, M_4, M_5, \len(\manifold)$ and $\Delta_\veps^{-1}$.
\end{lemma}
\begin{proof}
	From \Cref{lem:young_ineq} we have
	\begin{equation*}
	\norm{\mb \Theta^{(3)}}_{L^2 \to L^2} = \sup_{\mb x \in \manifold} \int_{\mb x' \in \manifold} \abs{\Theta^{(3)}(\mb x, \mb x')} d\mb x'. 
	\end{equation*}
	From \Cref{lem:deriv-lambda}, we know
	\begin{align*}
	&\abs{\Theta^{(3)}(\mb x, \mb x')} \\
	&\qquad = \abs{\dddot{\psi}(\angle(\mb x, \mb x'))\lambda_1^3(\mb x, \mb x') + 3\ddot{\psi}(\angle(\mb x, \mb x'))\lambda_2(\mb x, \mb x')\lambda_1(\mb x, \mb x') + \dot{\psi}(\angle(\mb x, \mb x'))\lambda_3(\mb x, \mb x')}\\
	&\qquad \le \abs{\dddot{\psi}(\angle(\mb x, \mb x'))}\abs{\lambda_1^3(\mb x, \mb x')} + \abs{3\ddot{\psi}(\angle(\mb x, \mb x'))}\abs{\lambda_2(\mb x, \mb x')}\abs{\lambda_1(\mb x, \mb x')} \\
	&\qquad \qquad + \abs{\dot{\psi}(\angle(\mb x, \mb x'))}\abs{\lambda_3(\mb x, \mb x')} 
	\end{align*}
	
	From \Cref{lem:skel_dot_ub,lem:skel_ddot_ub,lem:skel_dddot_ub}, there exist constants $C, C_1, C_2, C_3$, such that when $L > C$, we have
	\begin{align*}
	\max_{t \ge r} \abs{\dot{\psi}(t)} &\le \frac{C_1 n}{r^2} \\
	\max_{t \ge r} \abs{\ddot{\psi}(t)} &\le\frac{C_2 n}{r^3} \\
	\max_{t \ge r} \abs{\dddot{\psi}(t)} &\le\frac{C_3 n}{r^4}. 
	\end{align*}  
	
	To get lower bound for the angle $\angle(\mb x, \mb x')$, for a fixed point $\mb x \in \manifold$, we decompose the integral into a nearby piece $\mc N(\mb x) = \set{\mb x' \in \manifold \, \vert \, d_{\manifold}(\mb x, \mb x') \le \frac{\sqrt{\veps}}{\hat{\kappa}}}$ and a faraway piece $\mc F(\mb x) = \set{\mb x' \in \manifold \, \vert \, d_{\manifold}(\mb x, \mb x') > \frac{\sqrt{\veps}}{\hat{\kappa}}}$. 
	When $d_{\manifold}(\mb x, \mb x') > \frac{\sqrt{\veps}}{\hat{\kappa}}$, we have $\angle(\mb x, \mb x') \ge \Delta_\veps$ with $\Delta_\veps \le \frac{\sqrt{\veps}}{\hat{\kappa}} \le \frac{\pi}{2}$. From \Cref{lem:lambda1,lem:lambda2,lem:lambda3}, there exist constants $c, C'$ such that
	\begin{align*}
	\int_{\mb x' \in \mc F(\mb x)}&\abs{\Theta^{(3)}(\mb x, \mb x')} d \mb x'  \\
	&\le \int_{\mb x' \in \mc F(\mb x)} \frac{C_3 n }{\Delta_{\veps}^4}\abs{\lambda_1^3(\mb x, \mb x')} + \frac{C_2 n }{\Delta_{\veps}^3} \abs{\lambda_2(\mb x, \mb x')}\abs{\lambda_1(\mb x, \mb x')} + \frac{C_1 n }{\Delta_{\veps}^2} \abs{\lambda_3(\mb x, \mb x')} d \mb x'  \\
	&\le \len(\manifold) \left(8 \frac{C_3 n }{\Delta_{\veps}^4} + \frac{2 C_2 n }{\Delta_{\veps}^3} \paren{\frac{\pi}{4}\Delta_\veps^{-1} + 2M_2} \right.  \\
		&\qquad \qquad \qquad \qquad \left.+ \frac{C_1 n }{\Delta_{\veps}^2} \paren{	\frac{3\pi^2}{4}\Delta_\veps^{-2} + 9\pi M_2\Delta_\veps^{-1} + 2M_3 + 8}\right) \\
	&\le  C' \paren{\Delta_\veps^{-4} + M_2 \Delta_\veps^{-3} + M_3 \Delta_\veps^{-2}} \len(\manifold) n \labelthis \label{eq:theta3_far}
	\end{align*}
	for some constant $C'$. 
	
	For the nearby piece, let $\sigma, s$ be such that $\mb x = \mb x_{\sigma}(s)$. From \Cref{lem:angle_inverse}, when $\abs{s - s'} \le \frac{\sqrt{\veps}}{\hat{\kappa}}$ we have $\angle(\mb x, \mb x_{\sigma}(s')) \ge (1 - \veps)\abs{s - s'}$. From \Cref{lem:lambda1,lem:lambda2,lem:lambda3} there exists constant $c', c'', C''$ such that 
	\begin{align*}
	\int_{\mb x' \in \mc N(\mb x)}\abs{\Theta^{(3)}(\mb x, \mb x')} d \mb x' &\le \int_{s' = s - \frac{\sqrt{\veps}}{\hat{\kappa}}}^{s - \frac{\sqrt{\veps}}{\hat{\kappa}}} \abs{\Theta^{(3)}(\mb x, \mb x_{\sigma}(s'))} d s'\\
	&\le \int_{s' = s - \frac{\sqrt{\veps}}{\hat{\kappa}}}^{s - \frac{\sqrt{\veps}}{\hat{\kappa}}} \frac{C_3 n }{(1 - \veps)^4\abs{s' - s}^4}\abs{\lambda_1^3(\mb x, \mb x_\sigma(s') )} \\
	&\qquad + \frac{C_2 n }{(1 - \veps)^3\abs{s' - s}^3} \abs{\lambda_2(\mb x, \mb x_\sigma(s'))} \abs{\lambda_1(\mb x, \mb x_\sigma(s'))}   \\ &\qquad + \frac{C_1 n }{(1 - \veps)^2\abs{s' - s}^2} \abs{\lambda_3(\mb x, \mb x_\sigma(s'))} \; d s' \\
	&\le \int_{s' = s - \frac{\sqrt{\veps}}{\hat{\kappa}}}^{s - \frac{\sqrt{\veps}}{\hat{\kappa}}} \frac{C_2 n }{(1 - \veps)^4\abs{s' - s}^4}\abs*{\frac{7 \abs{s' - s}^3}{12}  M_4 }^3  \\ &\qquad + \frac{C_2 n }{(1 - \veps)^3\abs{s' - s}^3} \abs*{c'(M_4^2 + M_5) \abs{s' - s}^3}\abs*{\frac{7 \abs{s' - s}^3}{12}  M_4 } \\ &\qquad + \frac{C_1 n }{(1 - \veps)^2\abs{s' - s}^2} c'' (M_4^3  + M_4M_5 + M_3M_4)\abs{s' - s}^3 d s' \\
	&\le  C'' \paren{M_4^3 + M_3M_4 + M_4M_5} n.  \labelthis \label{eq:theta3_near}
	\end{align*}
	where the last line comes from the fact that $\int_{s' = s - \frac{\sqrt{\veps}}{\hat{\kappa}}}^{s - \frac{\sqrt{\veps}}{\hat{\kappa}}} \abs{s - s'}^i ds' < 2\hat{\kappa}^{-i-1} \le 2(2/\pi)^{-i-1}$. Combining \Cref{eq:theta3_far,eq:theta3_near} directly proves the claim. 
\end{proof}

\begin{lemma}\label{lem:deriv-R}
	For $i = 0, 1, 2$ and any differentiable $h : \manifold \to \bbC$, we have 
	\begin{equation*}
	\frac{d}{ds} \mb \Theta^{(i)}[h](\mb x) = \mb \Theta^{(i)}\left[ \frac{d}{ds} h \right](\mb x) + \mb \Theta^{(i+1)}[h](\mb x ), 
	\end{equation*}
	where we recall the notation defined in \Cref{eq:dh}.
\end{lemma}
\begin{proof}
	Let $s$ be such that $\mb x_{\sigma}(s) = \mb x$. We have 
	\begin{align*}
	\lefteqn{ \frac{d}{ds} \mb \Theta^{(i)}[h](\mb x_{\sigma}(s)) = \dac \int_{\mb x' \in \manifold} \Theta^{(i)}(\mb x_{\sigma}(s+t),\mb x') h(\mb x') d\mb x'  } \\
	&= \dac \sum_{\sigma'} \int_{s'} \Theta^{(i)}(\mb x_{\sigma}(s+t), \mb x_{\sigma'}(s')) h(\mb x_{\sigma'}(s')) ds' \\
	&= \sum_{\sigma'} \lim_{t \to 0} \frac{1}{t} \left[ \int_{s'} \Theta^{(i)}(\mb x_\sigma(s+t), \mb x_{\sigma'}(s') )  h(\mb x_{\sigma'}(s') ) ds' \right. \\
	&\qquad \qquad \qquad \qquad \qquad \left. - \int_{s'} \Theta^{(i)}(\mb x_{\sigma}(s), \mb x_{\sigma'}(s')) h(\mb x_{\sigma'}(s')) ds' \right] \\
	&= \sum_{\sigma'} \lim_{t \to 0} \frac{1}{t} \left[ \int_{s'} \Theta^{(i)}(\mb x_\sigma(s+t), \mb x_{\sigma'}(s'+t) )  h(\mb x_{\sigma'}(s'+t) ) ds' \right. \\
	&\qquad \qquad \qquad \qquad \qquad \left. - \int_{s'} \Theta^{(i)}(\mb x_{\sigma}(s), \mb x_{\sigma'}(s')) h(\mb x_{\sigma'}(s')) ds' \right] \\
	&= \sum_{\sigma'} \lim_{t \to 0} \frac{1}{t} \Biggl[ \int_{s'} \Theta^{(i)}(\mb x_\sigma(s+t), \mb x_{\sigma'}(s'+t) )  h(\mb x_{\sigma'}(s'+t) ) ds'  \\
	&\qquad \qquad \qquad \qquad \qquad - \int_{s'} \Theta^{(i)}(\mb x_\sigma(s), \mb x_{\sigma'}(s') )  h(\mb x_{\sigma'}(s'+t) ) ds'  \\ 
	& \qquad \qquad  \qquad \qquad \qquad + \int_{s'} \Theta^{(i)}(\mb x_\sigma(s), \mb x_{\sigma'}(s') )  h(\mb x_{\sigma'}(s'+t) ) ds'   \\
	&\qquad \qquad \qquad \qquad \qquad  - \int_{s'} \Theta^{(i)}(\mb x_{\sigma}(s), \mb x_{\sigma'}(s')) h(\mb x_{\sigma'}(s')) ds' \Biggr] \\
	&= \sum_{\sigma'} \lim_{t \to 0} \frac{1}{t} \Biggl[ \int_{s'} \left[ \Theta^{(i)}(\mb x_\sigma(s+t), \mb x_{\sigma'}(s'+t) )  -  \Theta^{(i)}(\mb x_\sigma(s), \mb x_{\sigma'}(s') )  \right] h(\mb x_{\sigma'}(s'+ t)) ds' \\ 
	& \qquad \qquad \qquad \int_{s'} \Theta^{(i)}(\mb x_{\sigma}(s), \mb x_{\sigma}(s')) \left[ h(\mb x_{\sigma'}(s'+t)) - h(\mb x_{\sigma'}(s')) \right] ds' \Biggr]. 
	\end{align*}
	Above, the domain of each of the $s'$ integrals is a fundamental domain for the circles $\bbR / (\len(\sM_{\sigma}) \cdot \bbZ) $ (by periodicity of the parameterizations, the specific fundamental domain is irrelevant).
	For $i = 0, 1, 2$ we have by the mean value theorem
	\begin{align*}
		&\left| \frac{ \Theta^{(i)}(\mb x_\sigma(s+t), \mb x_{\sigma'}(s'+t) )  -  \Theta^{(i)}(\mb x_\sigma(s), \mb x_{\sigma'}(s') ) }{t}h(\mb x_{\sigma'}(s'+ t)) \right| \\
		&\qquad = \quad \left| \Theta^{(i+1)}(\mb x_\sigma(s+t_0), \mb x_{\sigma'}(s'+t_0) ) h(\mb x_{\sigma'}(s'+ t)) \right| \\
		&\qquad \le \quad \sup_{\substack{\mb x_1, \mb x_2 \in \manifold \\ d_{\manifold}(\mb x_1, \mb x_2) = d_{\manifold}(\mb x, \mb x')}} \abs*{\Theta^{(i+1)}(\mb x_1, \mb x_2)} \sup_{\mb x_3 \in \manifold} \abs*{h(\mb x_3)} 
	\end{align*}
	for some $\abs{t_0} \le \abs{t}$. 
	As $\manifold$ is closed with bounded length and $h$ is differentiable, $\sup_{\mb x_3 \in \manifold} h(\mb x_3)$ is bounded. 
	By the formulas in \Cref{lem:deriv-lambda}, the fact that $\skel$ is $\mathcal{C}^3$ by \Cref{lem:phi_concave}, and \Cref{lem:lambda1,lem:lambda2,lem:lambda3}, it follows that 
	the former supremum is finite as well.
	From the dominated convergence theorem, we then have
	\begin{align*}
	\sum_{\sigma'} &\lim_{t \to 0} \frac{1}{t} \int_{s'} \left[ \Theta^{(i)}(\mb x_\sigma(s+t), \mb x_{\sigma'}(s'+t) )  -  \Theta^{(i)}(\mb x_\sigma(s), \mb x_{\sigma'}(s') )  \right] h(\mb x_{\sigma'}(s'+ t)) ds' \\
	&=\sum_{\sigma'} \int_{s'}  \lim_{t \to 0} \frac{1}{t} \left[ \Theta^{(i)}(\mb x_\sigma(s+t), \mb x_{\sigma'}(s'+t) )  -  \Theta^{(i)}(\mb x_\sigma(s), \mb x_{\sigma'}(s') )  \right] h(\mb x_{\sigma'}(s'+ t)) ds' \\
	&= \sum_{\sigma'} \int_{s'}\Theta^{(i+1)}(\mb x_\sigma(s), \mb x_{\sigma'}(s') )   h(\mb
  x_{\sigma'}(s')) ds' \labelthis \label{eq:deriv_R_limit_interchange_1}
	\end{align*}
	Similarly, as $\esssup_{\mb x \in \manifold} \abs*{\frac{d}{ds} h(\mb x)}$ is finite and $\manifold$ is compact, from the dominated convergence theorem we also have
	\begin{align*}
		\sum_{\sigma'} &\lim_{t \to 0} \frac{1}{t} \int_{s'} \Theta^{(i)}(\mb x_{\sigma}(s), \mb x_{\sigma}(s')) \left[ h(\mb x_{\sigma'}(s'+t)) - h(\mb x_{\sigma'}(s')) \right] ds' \\
		&=\sum_{\sigma'}\int_{s'} \lim_{t \to 0} \frac{1}{t} \Theta^{(i)}(\mb x_{\sigma}(s), \mb x_{\sigma}(s')) \left[ h(\mb x_{\sigma'}(s'+t)) - h(\mb x_{\sigma'}(s')) \right] ds' \\
		&= \sum_{\sigma'}\int_{s'}\Theta^{(i)}(\mb x_{\sigma}(s), \mb x_{\sigma}(s'))\frac{d}{ds}h(\mb
    x_{\sigma'}(s')) ds' \labelthis \label{eq:deriv_R_limit_interchange_2}
	\end{align*}
	Summing \eqref{eq:deriv_R_limit_interchange_1} and \eqref{eq:deriv_R_limit_interchange_2} shows the claim. 
\end{proof}

\begin{lemma}\label{lem:invariant_deriv_commute}
    There is an absolute constant $C>0$ such that for all $L \geq C$,
	any differentiable $h : \manifold \to \bbC$ and any $\veps \in (0, \frac{3}{4})$, if the operator $\wh{\mb M}_{\veps}$ and the subspace $S_{\veps}$ are as defined in \Cref{eq:invariant_op,eqn:S_eps_def}, then we have
	\begin{align*}
	\frac{d}{ds} \mb P_{S_\veps} [h] &= \mb P_{S_\veps} \left[\frac{d}{ds} h \right],  \\
	\frac{d}{ds} \wh{\mb M}_{\veps} \mb P_{S_\veps} [h] &= \wh{\mb M}_{\veps} \mb P_{S_\veps} \left[\frac{d}{ds} h \right]. 
	\end{align*}
	Also, suppose the hypotheses of \Cref{lem:main_term_pd} are satisfied, so that $\mb P_{S_\veps} \wh{\mb M}_{\veps}\mb P_{S_\veps}$ is invertible over $S_\veps$. Then one has in particular
	\begin{equation*}
	\frac{d}{ds} \left( \mb P_{S_\veps} \wh{\mb M}_{\veps}\mb P_{S_\veps} \right)^{-1} [h] = \left( \mb P_{S_\veps} \wh{\mb M}_{\veps} \mb P_{S_\veps} \right)^{-1} \left[\frac{d}{ds} h \right]. 
	\end{equation*}
\end{lemma}

\begin{proof}
    The condition on $L$ implies that $\wh{\vM}_{\veps}$ is well-defined.
	For any operator $\mb T$ that diagonalizes in the Fourier basis for $S_{\veps}$, i.e. for any $h \in L^2(\manifold)$, $\vT$ satisfies
	\begin{eqnarray}
	\mb T[h] = \sum_{\sigma \in \{\pm\}} \sum_{k = -K_{\veps, \sigma}}^{K_{\veps, \sigma}} m_{\sigma, k} \phi_{\sigma, k} \phi_{\sigma, k}^* h \label{eq:operator_diagonal}
	\end{eqnarray}
	for some coefficients $m_{\sigma, k} \in \bbC$ independent of $h$,\footnote{Here and in the sequel, we recall that we are using the notation $\phi\adj_{\sigma,k} h = \ip{\phi\adj_{\sigma,k}}{h} = \int_{\sM} \conj{\phi}_{\sigma,k}(\vx) h(\vx) \diff \vx$ for the standard inner product on complex-valued functions on $\manifold$.} we have 
	\begin{align*}
	\frac{d}{ds} \mb T[h] &= \sum_{\sigma \in \{\pm\}} \sum_{k = -K_{\veps, \sigma}}^{K_{\veps, \sigma}} m_{\sigma, k} 
	\left[\frac{d}{ds} \phi_{\sigma, k}\right] \phi_{\sigma, k}^* h \\
	&= \sum_{\sigma \in \{\pm\}} \sum_{k = -K_{\veps, \sigma}}^{K_{\veps, \sigma}} m_{\sigma, k} \frac{i2\pi k}{\len(\manifold_{\sigma})} \phi_{\sigma, k} \phi_{\sigma, k}^* h, 
	\end{align*}
	where we recall the definition of the Fourier basis functions from \Cref{eq:fourier_basis} for the second equality.
	Now fix $h$ differentiable as in the statement of the lemma.
	On the other hand, since $\phi_{\sigma, k}^* h$ is simply some complex number, which does not depend on $s$, we have 
	\begin{equation*}
	\paren{\frac{d}{ds} \phi_{\sigma, k}}^\ast h + \phi_{\sigma, k}^* \frac{d}{ds}h = 0,
	\end{equation*}
	and so
	\begin{align*}
	\mb T\left[\frac{d}{ds}  h \right] &= \sum_{\sigma \in \{\pm\}} \sum_{k = -K_{\veps, \sigma}}^{K_{\veps, \sigma}} m_{\sigma, k} \phi_{\sigma, k} \phi_{\sigma, k}^\ast \frac{d}{ds} h \\
	&= -\sum_{\sigma \in \{\pm\}} \sum_{k = -K_{\veps, \sigma}}^{K_{\veps, \sigma}} m_{\sigma, k} \phi_{\sigma, k} \paren{\frac{d}{ds} \phi_{\sigma, k}}^\ast h \\
	&=-\sum_{\sigma \in \{\pm\}} \sum_{k = -K_{\veps, \sigma}}^{K_{\veps, \sigma}} m_{\sigma, k} \phi_{\sigma, k} \paren{\frac{i2\pi k}{\len(\manifold_{\sigma})} \phi_{\sigma, k}}^\ast h \\
	&= \frac{d}{ds} \mb T[h].
	\end{align*}
	The operators $\wh{\mb M}_\veps$ and $\mb P_{S_\veps}$ both diagonalize in the Fourier basis for $S_{\veps}$, following the arguments in the proof of \Cref{lem:main_term_pd}. 
	By the same token, $(\mb P_{S_\veps} \wh{\mb M}_{\veps} \mb P_{S_\veps})^{-1}$ also diagonalizes in the Fourier basis for $S_{\veps}$ when it is well defined (recall \eqref{eq:smooth_subspace_main_term}), which concludes the proof. 
\end{proof}

\begin{lemma}\label{lem:fourier_deriv_lb}
    There is an absolute constant $C>0$ such that if $L \geq C$, and for any
	$\veps \in (0, \frac{3}{4})$ if $a_\veps, r_\veps, S_\veps$ defined as in \Cref{eq:a_def},   \Cref{eqn:r_def}, and \Cref{eqn:S_eps_def}, respectively, then when in addition $$L \ge \paren{\veps^{1/2}\min\{\len(\manifold_+), \len(\manifold_-)\}/(12\pi^2)}^{-\frac{a_\veps+1}{a_\veps}},$$ we have for any differentiable function $f : \manifold \to \bbC$ 
	\begin{align*}
	\left\| \frac{d}{ds}\mb P_{S_\veps} f \right\|_{L^2} &\le \frac{\sqrt{\veps}}{r_\veps} \norm*{\mb P_{S_\veps} f}_{L^2}, \\
	\norm*{\mb P_{S_{\veps}^\perp} f}_{L^2} &\le \frac{2r_\veps}{\sqrt{\veps}}\left\| \frac{d}{ds}\mb P_{S_{\veps}^\perp} f \right\|_{L^2}. 
	\end{align*}
\end{lemma}
\begin{proof}
    The condition on $L$ guarantees that $\wh{\vM}_{\veps}$ is well-defined.
    From \Cref{eqn:S_eps_def}, $S_{\veps} = S_{K_{\veps, +}, K_{\veps, -}}$ with $K_{\veps, \sigma} = \left\lfloor \frac{ \veps^{1/2} \mr{len}( \mc M_\sigma ) }{ 2 \pi r_\veps } \right\rfloor$ for $\sigma \in \set{+, -}$, then
    by orthonormality of the Fourier basis functions \Cref{eq:fourier_basis}, we have
	\begin{align*}
	\left\| \frac{d}{ds} \mb P_{S_\veps} f \right\|_{L^2} &= \left\| \frac{d}{ds}\sum_{\sigma = \pm}\sum_{k = -K_{\veps,\sigma}}^{K_{\veps, \sigma}}\phi_{\sigma, k}\phi_{\sigma, k}^{*}f \right\|_{L^2} \\
	&= \left\| \sum_{\sigma = \pm}\sum_{k = -K_{\veps,\sigma}}^{K_{\veps, \sigma}} \frac{i2\pi k}{\len(\manifold_\sigma)}\phi_{\sigma, k}\phi_{\sigma, k}^*f \right\|_{L^2} \\
	&\le \frac{\sqrt{\veps}}{r_\veps}\paren{\sum_{\sigma = \pm}\sum_{k = -K_{\veps,\sigma}}^{K_{\veps, \sigma}} \norm{\phi_{\sigma, k}\phi_{\sigma, k}^{*}f}_{L^2}^2}^{1/2} \\
	&=  \frac{\sqrt{\veps}}{r_\veps} \norm{\mb P_{S_\veps} f}_{L^2}. 
	\end{align*}
	Above, the inequality follows because $\abs{k} \leq K_{\veps, \sigma}$ implies $2\pi \abs{k} / \len(\sM_{\sigma}) \leq \sqrt{\veps} / r_{\veps}$, and because the Fourier basis functions are mutually orthogonal (and $\norm{f}_{L^2}^2 = \ip{f}{f}$).
	This establishes the first claim.
	
	For the second claim, we have 
	$$\mb P_{S_{\veps}^\perp} f = \sum_{\sigma = \pm}\sum_{\abs{k} > K_{\veps, \sigma}} \phi_{\sigma, k} \phi_{\sigma, k}^* f.$$ 
	When $L \ge \paren{\veps^{1/2}\min\{\len(\manifold_+), \len(\manifold_-)\}/(12\pi^2)}^{-\frac{a_\veps+1}{a_\veps}}$, we have $K_{\veps, \pm} = \floor*{\frac{\veps^{1/2}\len(\manifold_\pm)}{12\pi^2}L^{\frac{a_\veps}{a_\veps + 1}}} \ge 1$ and thus $K_{\veps, \pm} =  \left\lfloor \frac{ \veps^{1/2} \mr{len}( \mc M_\pm ) }{ 2 \pi r_\veps } \right\rfloor \ge \frac{ \veps^{1/2} \mr{len}( \mc M_\pm ) }{ 4 \pi r_\veps }$, whence
	\begin{align*}
	\left\| \frac{d}{ds} \mb P_{S_{\veps}^\perp} f\right\|_{L^2} &=  \left\| \sum_{\sigma = \pm} \sum_{\abs{k} > K_{\veps, \sigma}} \left( \frac{d}{ds} \phi_{\sigma, k} \right) \phi_{\sigma, k}^* f  \right\|_{L^2} \\
	&= \left\| \sum_{\sigma = \pm}\sum_{\abs{k} > K_{\veps, \sigma}} \frac{i 2 \pi k}{\mr{len}(\mc M_{\sigma})} \phi_{\sigma, k}  \phi_{\sigma, k}^* f \right\|_{L^2} \\
	&= \left( \sum_{\sigma = \pm}\sum_{\abs{k} > K_{\veps, \sigma}} \left| \frac{i 2 \pi k}{\mr{len}(\mc M_{\sigma})}\right|^2\left(  \phi_{\sigma, k}^* f \right)^2 \right)^{1/2} \\
	&\ge \frac{\sqrt{\veps}}{2r_\veps} \left( \sum_{\sigma = \pm}\sum_{\abs{k} > K_{\veps, \sigma}}\left(  \phi_{\sigma, k}^* f \right)^2 \right)^{1/2} \\
	&= \frac{\sqrt{\veps}}{2r_\veps} \left\| \mb P_{S_{\veps}^\perp} f \right\|_{L^2},
	\end{align*}
	as claimed. Above, the first equality entails an interchange of limit processes---a formal justification for the validity of this interchange follows from the assumed differentiability of $f$ (which implies that its coefficients $\phi\adj_{\sigma, k} f$ have a faster rate of decay $o(\abs{k}^{-3/2})$) and a dominated convergence argument, where the difference quotient involving $\phi_{\sigma, k}$ is bounded by $O(\abs{k})$, which together with the extra smoothness of $f$ leads to an integrable upper bound. 
\end{proof}

\begin{definition}\label{def:smoothness}
	For any $\veps \in (0, \frac{3}{4})$, let $P_1 = P_1\Bigl( M_4, \mr{len}(\mc M), \Delta_{\veps}^{-1} \Bigr), P_2 = P_2\Bigl( M_4, M_5, \mr{len}(\mc M), \Delta_{\veps}^{-1} \Bigr), P_3 = P_3\Bigl( M_3, M_4, M_5, \mr{len}(\mc M), \Delta_{\veps}^{-1} \Bigr)$ as defined in \Cref{lem:R1,lem:R2,lem:R3}. 
	We let $\Phi( C_\zeta, \veps )$ for some constant $C_\zeta \ge 0$ denote the set of all functions $\zeta \in C^3(\manifold)$ which satisfy
	\begin{align*}
	\norm{\zeta}_{L^2} &\le C_\zeta \\
	\norm{\zeta^{(1)}}_{L^2} &\le  C_\zeta \frac{P_1}{\log L}  \\
	\norm{\zeta^{(2)}}_{L^2} &\le  C_\zeta \paren{\frac{P_2}{\log L} + \frac{P_1^2}{\log^2 L}} \\
	\norm{\zeta^{(3)}}_{L^2} &\le  C_\zeta \paren{\frac{P_3}{\log L} + \frac{P_2P_1}{\log^2 L} + \frac{P_1^3}{\log^3 L}}.  
	\end{align*}
	Furthermore, for $\veps \ge \veps' > 0$, one has $\Delta_{\veps}^{-1} \le \Delta_{\veps'}^{-1}$. As $P_1, P_2, P_3$ have positive coefficients, $\zeta \in \Phi( C_\zeta, \veps)$ implies $\zeta \in \Phi( C_\zeta, \veps' )$, i.e., $\Phi(C_\zeta, \veps' ) \subseteq \Phi(C_\zeta, \veps)$. 
\end{definition}

\begin{lemma} \label{lem:g_deriv} For any $\veps \in (0,\tfrac{3}{4})$, there exist numbers $C_\veps, C_\veps' > 0$ such that when the conditions of \Cref{thm:certificate_over_S} are in force, for any $\zeta \in S_\veps \cap \Phi( C_\zeta, \veps )$, the certificate $g_\veps[\zeta]$ defined in \Cref{thm:certificate_over_S} satisfies 
\begin{equation}
g_\veps[\zeta] \in \Phi \left(\frac{C_{\veps}'C_\zeta}{n \log L}, \veps \right). 
\end{equation}  
\end{lemma}

\begin{proof}
	Following \Cref{lem:invariant_deriv_commute} and \Cref{lem:deriv-R}, we have that 
	\begin{align*}
	&g_{\veps}^{(1)}[\zeta] \\
	&\quad = \sum_{\ell=1}^{\infty}(-1)^\ell \sum_{a=0}^{\ell-1} \left( \left(\mb P_{S_\veps} \wh{\mb M}_{\veps} \mb P_{S_\veps} \right)^{-1} \mb P_{S_\veps} (\mb \Theta^\circ - \wh{\mb M}_{\veps}) \mb P_{S_\veps} \right)^{a} \left(\mb P_{S_\veps} \wh{\mb M}_{\veps} \mb P_{S_\veps} \right)^{-1} \mb P_{S_\veps} \mb \Theta^{(1)} \mb P_{S_\veps} \\&  \qquad \qquad \qquad \qquad \times \; \left( \left(\mb P_{S_\veps} \wh{\mb M}_{\veps} \mb P_{S_\veps} \right)^{-1} \mb P_{S_\veps} (\mb \Theta^\circ - \wh{\mb M}_\veps) \mb P_{S_\veps} \right)^{\ell - a - 1} \left( \mb P_{S_\veps} \wh{\mb M}_{\veps} \mb P_{S_\veps} \right)^{-1} \zeta \\
	&\quad \qquad  +  g_\veps[\zeta^{(1)}]. 
	\end{align*}
	As $\zeta \in S_\veps$, we have $\zeta^{(1)} = \frac{d}{ds} \zeta = \frac{d}{ds} \mb P_{S_\veps} \zeta = \mb P_{S_\veps} \frac{d}{ds} \zeta \in S_\veps$, and thus following \Cref{thm:certificate_over_S}, there exists constant $c$ such that
	\begin{align*}
		\norm{g_\veps[\zeta^{(1)}]}_{L^2} \le \frac{\norm{\zeta^{(1)}}_{L^2}}{\veps c n\log L}. 
	\end{align*}
    From \Cref{lem:main_term_pd}, there exists $C, c' > 0$ such that when $L \ge C$, 
	\begin{align*}
	    \lambda_{\min}(\mb P_{S} \wh{\mb M}_{\veps} \mb P_{S_\veps}) &\ge c n \log L. 
	\end{align*}
	Under the conditions of \Cref{thm:certificate_over_S}, we have $$\left\|  \left(\mb P_{S_\veps} \wh{\mb M}_{\veps} \mb P_{S_\veps} \right)^{-1} \mb P_{S_\veps} (\mb \Theta^\circ - \wh{\mb M}_\veps) \mb P_{S_\veps} \right\|_{L^2 \to L^2} \le 1-\veps. $$ 
	Let $P_1 = P_1\Bigl( M_4, \mr{len}(\mc M), \Delta_{\veps}^{-1} \Bigr), P_2 = P_2\Bigl( M_4, M_5, \mr{len}(\mc M), \Delta_{\veps}^{-1} \Bigr), P_3 = P_3\Bigl( M_3, M_4, M_5, \mr{len}(\mc M), \Delta_{\veps}^{-1} \Bigr)$ be the polynomials in \Cref{lem:R1,lem:R2,lem:R3}. From  \Cref{lem:R1} , we have
	\begin{align*}
	\norm{g_{\veps}^{(1)}[\zeta]}_{L^2} &\le \sum_{\ell=1}^{\infty} \sum_{a=0}^{\ell-1} (1-\veps)^{\ell -1} \norm{\Theta^{(1)}}_{L^2 \to L^2}\frac{\norm{\zeta}_{L^2}}{\left(cn\log L\right)^2} \\
	& + \frac{\norm{\zeta^{(1)}}_{L^2}}{\veps c n\log L} \\
	&= \frac{\norm{\Theta^{(1)}}_{L^2 \to L^2}\norm{\zeta}_{L^2}}{\left(cn\log L\right)^2} \sum_{\ell=0}^{\infty} \ell (1-\veps)^{\ell-1} + \frac{\norm{\zeta^{(1)}}_{L^2}}{\veps cn \log L } \\
	&= \frac{\norm{\Theta^{(1)}}_{L^2 \to L^2}\norm{\zeta}_{L^2}}{\veps^2\left(cn\log L\right)^2} + \frac{\norm{\zeta^{(1)}}_{L^2}}{\veps cn \log L} \\
	&\le \frac{P_1 \norm{\zeta}_{L^2}}{\veps^2c^2 n \log^2 L} +
  \frac{\norm{\zeta^{(1)}}_{L^2}}{\veps cn \log L}. \labelthis \label{eq:g1}
	\end{align*}
	From the fact that $\zeta \in \Phi( C_\zeta, \veps )$, we further obtain
	\begin{align*}
		\norm{g_{\veps}^{(1)}[\zeta]}_{L^2} &\le \frac{P_1 C_\zeta}{\veps^2c^2 n \log^2 L} + \frac{C_\zeta}{\veps cn \log L} \frac{P_1}{\log L} \\
		&\le C_\veps \frac{P_1 C_\zeta}{n \log^2 L}. 
	\end{align*}
	For the second derivative, we have
	\begin{align*}
	g_{\veps}^{(2)}[\zeta] &= \sum_{\ell=1}^{\infty}(-1)^\ell \sum_{a=0}^{\ell - 1} \left( \left(\mb P_{S} \wh{\mb M}_{\veps} \mb P_{S_\veps} \right)^{-1} \mb P_{S_\veps} (\mb \Theta^\circ - \wh{\mb M}) \mb P_{S_\veps} \right)^{a}  \\&\qquad \qquad \qquad \times \quad \left(\mb P_{S} \wh{\mb M}_{\veps} \mb P_{S_\veps} \right)^{-1} \mb P_{S_\veps} \mb \Theta^{(2)}  \mb P_{S_\veps} \\& \qquad \qquad \qquad \times \quad \left( \left(\mb P_{S} \wh{\mb M}_{\veps} \mb P_{S_\veps} \right)^{-1} \mb P_{S_\veps} (\mb \Theta^\circ - \wh{\mb M}) \mb P_{S_\veps} \right)^{\ell - a - 1} \left( \mb P_{S_\veps} \wh{\mb M}_{\veps} \mb P_{S_\veps} \right)^{-1} \zeta \\
	&\quad  + 2\sum_{\ell=1}^{\infty}(-1)^\ell \sum_{a=0}^{\ell - 2} \sum_{a'=a}^{\ell - 2} \left( \left(\mb P_{S} \wh{\mb M}_{\veps} \mb P_{S_\veps} \right)^{-1} \mb P_{S_\veps} (\mb \Theta^\circ - \wh{\mb M}) \mb P_{S_\veps} \right)^{a} \\
	&\qquad \qquad \qquad \times \quad  \left(\mb P_{S} \wh{\mb M}_{\veps} \mb P_{S_\veps} \right)^{-1}  \mb P_{S_\veps} \mb \Theta^{(1)} \mb P_{S_\veps}  \\&  \qquad \qquad \qquad \times \quad  \left( \left(\mb P_{S} \wh{\mb M}_{\veps} \mb P_{S_\veps} \right)^{-1} \mb P_{S_\veps} (\mb \Theta^\circ - \wh{\mb M}) \mb P_{S_\veps} \right)^{a'-a}  \\
	& \qquad \qquad \qquad \times \quad \left(\mb P_{S} \wh{\mb M}_{\veps} \mb P_{S_\veps} \right)^{-1} \mb P_{S_\veps} \mb \Theta^{(1)} \mb P_{S_\veps} \\
	& \qquad \qquad \qquad \times \quad  \left( \left(\mb P_{S} \wh{\mb M}_{\veps} \mb P_{S_\veps} \right)^{-1} \mb P_{S_\veps} (\mb \Theta^\circ - \wh{\mb M}) \mb P_{S_\veps} \right)^{\ell - a' - 2} \left( \mb P_{S_\veps} \wh{\mb M}_{\veps} \mb P_{S_\veps} \right)^{-1} \zeta \\
	&\quad  + 2 g_\veps^{(1)}[\zeta^{(1)}] \; - \; g_\veps[\zeta^{(2)}]. 
	\end{align*}
	From \eqref{eq:g1}, as $\zeta^{(1)}, \zeta^{(2)} \in S_\veps$ we have
	\begin{align*}
		\norm{g_\veps^{(1)}[\zeta^{(1)}]}_{L^2} &\le \frac{P_1 \norm{\zeta^{(1)}}_{L^2}}{\veps^2c^2 n \log^2 L} + \frac{\norm{\zeta^{(2)}}_{L^2}}{\veps cn \log L}, \\
		\norm{g_\veps[\zeta^{(2)}]}_{L^2} &\le \frac{\norm{\zeta^{(2)}}_{L^2 \to L^2}}{\veps cn\log L}.
	\end{align*}
	 which leads to
	\begin{align*}
	\norm{g_{\veps}^{(2)}[\zeta]}_{L^2} &\le \frac{\norm{\Theta^{(2)}}_{L^2 \to L^2}\norm{\zeta}_{L^2}}{\veps^2\left(cn\log L\right)^2} + \sum_{\ell = 1}^{\infty} \ell(\ell-1)(1 - \veps)^{\ell-2}\frac{\norm{\Theta^{(1)}}_{L^2 \to L^2}^2\norm{\zeta}_{L^2}}{\left(cn\log L\right)^3} \\
	& + \quad 2 \norm{g_\veps^{(1)}[\zeta^{(1)}]}_{L^2} + \norm{g_\veps[\zeta^{(2)}]}_{L^2} \\
	&\le \frac{\norm{\Theta^{(2)}}_{L^2 \to L^2}\norm{\zeta}_{L^2}}{\veps^2\left(cn\log L\right)^2} +  \frac{2\norm{\Theta^{(1)}}_{L^2 \to L^2}^2\norm{\zeta}_{L^2}}{\veps^3\left(cn\log L\right)^3} \\
	& + \quad 2 \paren{\frac{P_1 \norm{\zeta^{(1)}}_{L^2}}{\veps^2c^2 n \log^2 L} + \frac{\norm{\zeta^{(2)}}_{L^2}}{\veps cn \log L}} + \frac{\norm{\zeta^{(2)}}_{L^2}}{\veps cn\log L} \\
	&\le \frac{P_2\norm{\zeta}_{L^2}}{\veps^2c^2n\log^2 L} +
  \frac{2P_1^2\norm{\zeta}_{L^2}}{\veps^3c^3n\log^3 L} + \frac{2P_1
  \norm{\zeta^{(1)}}_{L^2}}{\veps^2c^2 n \log^2 L} + \frac{3\norm{\zeta^{(2)}}_{L^2}}{\veps cn\log
L}. \labelthis \label{eq:g2}
	\end{align*}
	Again as $\zeta \in \Phi(C_\zeta, \veps)$ we have
	\begin{align*}
	\norm{g_{\veps}^{(2)}[\zeta]}_{L^2} &\le  \frac{P_2C_\zeta}{\veps^2c^2n\log^2 L} +  \frac{2P_1^2C_\zeta}{\veps^3c^3n\log^3 L} \\
	& + \frac{2 P_1 C_\zeta}{\veps^2c^2 n \log^2 L}\frac{P_1 }{\log L} + \frac{C_\zeta}{\veps cn\log L}\paren{\frac{P_2}{\log L} + \frac{P_1^2}{\log^2 L}} \\
	&\le C_\veps C_\zeta \paren{\frac{P_1^2}{n \log^3 L} + \frac{P_2}{n \log ^2 L}}. 
	\end{align*}
	For third derivative, we have
	\begin{align*}
	g_{\veps}^{(3)}[\zeta] &= \sum_{\ell=1}^{\infty}(-1)^\ell \sum_{a=0}^{\ell - 1} \left( \left(\mb P_{S} \wh{\mb M}_{\veps} \mb P_{S_\veps} \right)^{-1} \mb P_{S_\veps} (\mb \Theta^\circ - \wh{\mb M}) \mb P_{S_\veps} \right)^{a} \\
	& \qquad \qquad \qquad \times \quad \left(\mb P_{S} \wh{\mb M}_{\veps} \mb P_{S_\veps} \right)^{-1} \mb P_{S_\veps} \mb \Theta^{(3)}  \mb P_{S_\veps} \\& \qquad \qquad \qquad \times \quad \left( \left(\mb P_{S} \wh{\mb M}_{\veps} \mb P_{S_\veps} \right)^{-1} \mb P_{S_\veps} (\mb \Theta^\circ - \wh{\mb M}) \mb P_{S_\veps} \right)^{\ell - a - 1} \left( \mb P_{S_\veps} \wh{\mb M}_{\veps} \mb P_{S_\veps} \right)^{-1} \zeta \\
	& + 3\sum_{\ell=1}^{\infty}(-1)^\ell \sum_{a=0}^{\ell - 2} \sum_{a'=a}^{\ell - 2} \left( \left(\mb P_{S} \wh{\mb M}_{\veps} \mb P_{S_\veps} \right)^{-1} \mb P_{S_\veps} (\mb \Theta^\circ - \wh{\mb M}) \mb P_{S_\veps} \right)^{a} \\
	& \qquad \qquad \qquad \times \quad \left(\mb P_{S} \wh{\mb M}_{\veps} \mb P_{S_\veps} \right)^{-1} \mb P_{S_\veps} \mb \Theta^{(2)} \mb P_{S_\veps} \\&  \qquad \qquad \qquad \times \quad \left( \left(\mb P_{S} \wh{\mb M}_{\veps} \mb P_{S_\veps} \right)^{-1} \mb P_{S_\veps} (\mb \Theta^\circ - \wh{\mb M}) \mb P_{S_\veps} \right)^{a'-a} \\
	& \qquad \qquad \qquad \times \quad \left(\mb P_{S} \wh{\mb M}_{\veps} \mb P_{S_\veps} \right)^{-1} \mb P_{S_\veps} \mb \Theta^{(1)} \mb P_{S_\veps}  \\
	&  \qquad \qquad \qquad \times \quad \left( \left(\mb P_{S} \wh{\mb M}_{\veps} \mb P_{S_\veps} \right)^{-1} \mb P_{S_\veps} (\mb \Theta^\circ - \wh{\mb M}) \mb P_{S_\veps} \right)^{\ell - a' - 2} \left( \mb P_{S_\veps} \wh{\mb M}_{\veps} \mb P_{S_\veps} \right)^{-1} \zeta \\
	& + 3\sum_{\ell=1}^{\infty}(-1)^\ell \sum_{a=0}^{\ell - 2} \sum_{a'=a}^{\ell - 2} \left( \left(\mb P_{S} \wh{\mb M}_{\veps} \mb P_{S_\veps} \right)^{-1} \mb P_{S_\veps} (\mb \Theta^\circ - \wh{\mb M}) \mb P_{S_\veps} \right)^{a} \\
	& \qquad \qquad \qquad \times \quad \left(\mb P_{S} \wh{\mb M}_{\veps} \mb P_{S_\veps} \right)^{-1} \mb P_{S_\veps} \mb \Theta^{(1)} \mb P_{S_\veps} \\& \qquad \qquad \qquad \times \quad \left( \left(\mb P_{S} \wh{\mb M}_{\veps} \mb P_{S_\veps} \right)^{-1} \mb P_{S_\veps} (\mb \Theta^\circ- \wh{\mb M}) \mb P_{S_\veps} \right)^{a'-a} \\
	& \qquad \qquad \qquad \times \quad \left(\mb P_{S} \wh{\mb M}_{\veps} \mb P_{S_\veps} \right)^{-1} \mb P_{S_\veps} \mb \Theta^{(2)} \mb P_{S_\veps}  \\
	& \qquad \qquad \qquad \times \quad \left( \left(\mb P_{S} \wh{\mb M}_{\veps} \mb P_{S_\veps} \right)^{-1} \mb P_{S_\veps} (\mb \Theta^\circ - \wh{\mb M}) \mb P_{S_\veps} \right)^{\ell - a' - 2} \left( \mb P_{S_\veps} \wh{\mb M}_{\veps} \mb P_{S_\veps} \right)^{-1} \zeta \\
	& + 6\sum_{\ell=1}^{\infty}(-1)^\ell \sum_{a=0}^{\ell - 2} \sum_{a'=a}^{\ell - 2} \sum_{a''=a'}^{\ell - 2} \left( \left(\mb P_{S} \wh{\mb M}_{\veps} \mb P_{S_\veps} \right)^{-1} \mb P_{S_\veps} (\mb \Theta^\circ - \wh{\mb M}) \mb P_{S_\veps} \right)^{a} \\
	& \qquad \qquad \qquad \times \quad \left(\mb P_{S} \wh{\mb M}_{\veps} \mb P_{S_\veps} \right)^{-1} \mb P_{S_\veps} \mb \Theta^{(1)} \mb P_{S_\veps} \\& \qquad \qquad \qquad \times \quad \left( \left(\mb P_{S} \wh{\mb M}_{\veps} \mb P_{S_\veps} \right)^{-1} \mb P_{S_\veps} (\mb \Theta^\circ - \wh{\mb M}) \mb P_{S_\veps} \right)^{a'-a} \\
	& \qquad \qquad \qquad \times \quad \left(\mb P_{S} \wh{\mb M}_{\veps} \mb P_{S_\veps} \right)^{-1} \mb P_{S_\veps} \mb \Theta^{(1)} \mb P_{S_\veps}  \\& \qquad \qquad \qquad \times \quad \left( \left(\mb P_{S} \wh{\mb M}_{\veps} \mb P_{S_\veps} \right)^{-1} \mb P_{S_\veps} (\mb \Theta^\circ - \wh{\mb M}) \mb P_{S_\veps} \right)^{a''-a'} \\
	& \qquad \qquad \qquad \times \quad \left(\mb P_{S} \wh{\mb M}_{\veps} \mb P_{S_\veps} \right)^{-1} \mb P_{S_\veps} \mb \Theta^{(1)} \mb P_{S_\veps}  \\
	& \qquad \qquad \qquad \times \quad \left( \left(\mb P_{S} \wh{\mb M}_{\veps} \mb P_{S_\veps} \right)^{-1} \mb P_{S_\veps} (\mb \Theta^\circ - \wh{\mb M}) \mb P_{S_\veps} \right)^{\ell - a'' - 2} \left( \mb P_{S_\veps} \wh{\mb M}_{\veps} \mb P_{S_\veps} \right)^{-1} \zeta \\
	& +3 g_{\veps}^{(2)}[\zeta^{(1)}] - 3g_{\veps}^{(1)}[\zeta^{(2)}] + g_\veps[\zeta^{(3)}]. 
	\end{align*}
	Similarly, as $\zeta^{(1)}, \zeta^{(2)}$ and $\zeta^{(3)} \in S_\veps$, plug in results in \eqref{eq:g1} and \eqref{eq:g2} we can control
	\begin{align*}
	\norm{g_{\veps}^{(3)}[\zeta]}_{L^2} &\le \frac{\norm{\Theta^{(3)}}_{L^2 \to L^2}\norm{\zeta}_{L^2}}{\veps^2\left(cn\log L\right)^2} \\
	&\qquad +   \frac{3\norm{\Theta^{(2)}}_{L^2 \to L^2}\norm{\Theta^{(1)}}_{L^2 \to L^2}\norm{\zeta}_{L^2}}{\veps^3\left(cn\log L\right)^3} \\
	&\qquad + \sum_{\ell = 1}^{\infty} (\ell+1)(\ell-1)(\ell - 1)(1 - \veps)^{\ell-3}\frac{\norm{\Theta^{(1)}}_{L^2\to L^2}^3\norm{\zeta}_{L^2}}{\left(cn\log L\right)^4} \\
	&\qquad + 3 \paren{\frac{P_2\norm{\zeta^{(1)}}_{L^2}}{\veps^2c^2n\log^2 L} +  \frac{2P_1^2\norm{\zeta^{(1)}}_{L^2}}{\veps^3c^3n\log^3 L} + \frac{2P_1 \norm{\zeta^{(2)}}_{L^2}}{\veps^2c^2 n \log^2 L} + \frac{3\norm{\zeta^{(3)}}_{L^2}}{\veps cn\log L}} \\
	&\qquad + 3 \paren{\frac{P_1 \norm{\zeta^{(2)}}_{L^2}}{\veps^2c^2 n \log^2 L} + \frac{\norm{\zeta^{(3)}}_{L^2}}{\veps cn \log L}} + \frac{\norm{\zeta^{(3)}}}{\veps cn \log L}\\
	&=  \frac{\norm{\Theta^{(3)}}_{L^2 \to L^2}\norm{\zeta}_{L^2}}{\veps^2\left(cn\log
  L\right)^2} +  \frac{3\norm{\Theta^{(2)}}_{L^2 \to L^2}\norm{\Theta^{(1)}}_{L^2 \to
  L^2}\norm{\zeta}_{L^2}}{\veps^3\left(cn\log L\right)^3}\\&\qquad + \frac{6\norm{\Theta^{(1)}}_{L^2\to L^2}^3\norm{\zeta}_{L^2}}{\veps^4\left(cn\log L\right)^4} \\
	&\qquad + \frac{\norm{\zeta^{(1)}}_{L^2}}{\veps^2c^2n\log^2 L} \paren{\frac{6P_1^2}{\log L} + 3P_2} + \frac{9 P_1 \norm{\zeta^{(2)}}_{L^2}}{\veps^2c^2 n \log^2 L} + \frac{\norm{13 \zeta^{(3)}}}{\veps cn \log L} \\
	&\le \frac{P_3\norm{\zeta}_{L^2}}{\veps^2c^2n\log^2 L} +  \frac{3P_2P_1\norm{\zeta}_{L^2}}{\veps^3c^3n\log^3 L} + \frac{6P_1^3\norm{\zeta}_{L^2}}{\veps^4c^4n\log^4 L} \\
	&\qquad + \frac{\norm{\zeta^{(1)}}_{L^2}}{\veps^2c^2n\log^2 L} \paren{\frac{6P_1^2}{\log L} + 3P_2} + \frac{9 P_1 \norm{\zeta^{(2)}}_{L^2}}{\veps^2c^2 n \log^2 L} + \frac{\norm{13 \zeta^{(3)}}}{\veps cn \log L}. 
	\end{align*}
	Plug in bounds for norms of $\zeta^{(1)}, \zeta^{(2)}$ and $\zeta^{(3)}$ we get
	\begin{align*}
	\norm{g_{\veps}^{(3)}[\zeta]}_{L^2} &\le C_\veps C_\zeta \paren{\frac{P_1^3}{n \log^2 L} + \frac{P_1 P_2}{n \log^3 L} + \frac{P_3}{n \log^2 L}}. 
	\end{align*}
	Combined with zero's order condition of $g_\veps[\zeta]$, which follows directly from \Cref{thm:certificate_over_S}, and we know that there exists $C_\veps$ such that $g \in \Phi(\frac{C_\veps C_\zeta}{n \log L}, \veps)$. 
\end{proof}

\begin{lemma}\label{lem:Theta_ub}
	For $\veps \in (0, \frac{3}{4})$, when $L$ satisfies conditions in \Cref{thm:certificate_over_S}, there exists positive constant $C$ such that
	\begin{align*}
	\norm{\mb \Theta^\circ}_{L^2 \to L^2} \le Cn \log (L).
	\end{align*}
\end{lemma}
\begin{proof}
	As $\wh{\mb M}_\veps$ in \Cref{eq:invariant_op} is invariant in Fourier basis as shown in \Cref{lem:main_term_pd}, we have $\norm{\mb P_{S_\veps} \wh{\mb M}_\veps \mb P_{S_\veps}}_{L^2 \to L^2} \le \norm{\wh{\mb M}_\veps}_{L^2 \to L^2} $. From \Cref{lem:main_res_norm_compare}, 
	\begin{align*}
	\norm{\mb \Theta^\circ}_{L^2 \to L^2} &= \norm{\wh{\mb M}_{\veps} + \mb \Theta^\circ - \wh{\mb M}_\veps}_{L^2 \to L^2} \\
	&\le \norm{\wh{\mb M}_\veps}_{L^2 \to L^2} + \norm{\mb \Theta^\circ - \wh{\mb M}}_{L^2 \to L^2} \\
	&\le \norm{\wh{\mb M}_\veps}_{L^2 \to L^2} + (1- \veps) \lambda_{\min}(\mb P_{S_\veps} \wh{\mb M}_\veps \mb P_{S_\veps}) \\
	&\le 2 \norm{\wh{\mb M}_\veps}_{L^2 \to L^2}. 
	\end{align*}
	As when $L \ge (\veps^{-1/2}6\pi \hat{\kappa})^{\frac{a_\veps + 1}{a_\veps}}$ we have $r_\veps \le \frac{\sqrt{\veps}}{\hat{\kappa}}$, where $a_\veps, r_\veps$ are defined in \Cref{eq:a_def} and \Cref{eqn:r_def}, and following \Cref{lem:angle_inverse} we have 
	\begin{equation*}
		\angle\left({\mb x_{\sigma}(s), \mb x_{\sigma}(s')}\right) \ge (1-\veps)\abs{s - s'}
	\end{equation*}
	for any $\mb x_{\sigma}(s)$ and $\abs{s - s'} \le r_\veps$. Then follow \Cref{lem:young_ineq} and \Cref{eq:skel_l1_ub_a_to_b}\Cref{lem:skel_l1_ub} and monotonicity of $\psi^\circ$ in \Cref{lem:phi_concave}, we have
	\begin{align*}
	\norm{\wh{\mb M}_\veps}_{L^2 \to L^2} &\le \max_{\mb x_{\sigma}(s) \in \manifold} \int_{s' = s - r_{\veps}}^{s + r_\veps} \psi^\circ(\angle{\mb x_\sigma(s), \mb x_{\sigma}(s')}) ds' \\
	&\le \max_{\mb x_{\sigma}(s) \in \manifold} \int_{s' = s - r_{\veps}}^{s + r_\veps} \psi^\circ((1 - \veps)\abs{s' - s}) ds' \\
	&= (1 - \veps)^{-1} \int_{t = -(1-\veps)r_{\veps}}^{(1 - \veps)t_{\veps}}\psi^{\circ}(t) dt \\
	&\le (1-\veps)^{-1}C \log \left(1 + \frac{(L-3)(1-\veps)r_\veps}{3\pi} \right) \\
	&\le C' n \log L
	\end{align*}
	for some constant $C$, which concludes the claim. 
\end{proof}

\begin{lemma} \label{lem:Theta_Sperp_proj} 
	For any $\veps \in (0,\tfrac{3}{4})$, there exists constant $C \ge 0$ such that for any $g \in \Phi( C_g, \veps )$, under the conditions of \Cref{thm:certificate_over_S}, we have 
	\begin{equation*}
		\mb \Theta^\circ[g] \in \Phi( C C_g n \log L, \veps ) 
	\end{equation*} 
	As a consequence, for $\zeta \in S_\veps \cap \Phi ( C_\zeta, \veps )$ letting $g_{\veps}[\zeta]$ be the certificate in the statement of \Cref{thm:certificate_over_S}, there exists number $ C_\veps \ge 0$ such that
	\begin{equation*}
	\mb \Theta^{\circ}[g_{\veps}[\zeta]] - \zeta \in \Phi \Bigl( C_\veps C_\zeta, \veps \Bigr).
	\end{equation*}
\end{lemma} 

\begin{proof}
    From \Cref{lem:Theta_ub} we have $\norm{\mb \Theta^\circ}_{L^2 \to L^2} \le  Cn\log(L)$ for some constant $C$. 
	Let $P_1, P_2, P_3$ be the polynomials in \Cref{lem:R1,lem:R2,lem:R3}. Following \Cref{lem:deriv-R,lem:R1,lem:R2,lem:R3} and the fact that $g \in \Phi( C_g, \veps )$, we have following control for derivatives of $\mb \Theta^\circ[g]$: 
	\begin{align*}
		\left\|\frac{d}{ds} \mb \Theta^\circ [g] \right\|_{L^2 \to L^2} &=  \norm{\mb \Theta^\circ [g^{(1)}] + \mb \Theta^{(1)} [g]}_{L^2 \to L^2} \\
		&\le C n \log L \frac{P_1 C_g}{\log L} + P_1 n C_g\\
		&= (C + 1) P_1 n C_g,  
	\end{align*}
	\begin{align*}
	\left\|\frac{d^2}{ds^2} \mb \Theta^\circ [g]  \right\|_{L^2 \to L^2} &=  \norm{\mb \Theta^\circ [g^{(2)}] + 2 \mb \Theta^{(1)} [g^{(1)}] + \mb \Theta^{(2)} [g]}_{L^2 \to L^2} \\
	&\le C n \log L \paren{\frac{P_2C_g}{\log L} + 2\frac{P_1^2C_g}{\log^2 L}} + P_1 n \frac{P_1C_g}{ \log L}  + P_2 n C_g \\
	&= (C + 2) \paren{P_2 C_g+ \frac{P_1^2 C_g}{ \log L}} n,  
	\end{align*}
	\begin{align*}
	\left\| \frac{d^3}{ds^3} \mb \Theta^\circ [g]  \right\|_{L^2 \to L^2} &=  \norm{\mb \Theta^\circ [g^{(3)}] +  3 \mb \Theta^{(1)} [g^{(2)}] + 3 \mb \Theta^{(2)} [g^{(1)}] + \mb \Theta^{(3)} [g]}_{L^2 \to L^2} \\
	&\le C n \log L \paren{\frac{P_3C_g}{\log L}} + 3P_1 n \paren{\frac{P_2C_g}{\log L} +
  \frac{P_1^2C_g}{\log^2 L}} \\&+ 3 P_2 n \frac{P_1C_g}{\log L}  + P_3 n C_g \\
	&= (C + 3) \paren{P_3 C_g + \frac{P_1P_2 C_g}{\log L} + \frac{P_1^3 C_g}{\log^2 L}} n. 
	\end{align*}
	which leads to $\mb \Theta^\circ[g] \in \Phi((C+3)C_g n \log L, \veps)$ and finish the claim. The other part of the claim follows directly from \Cref{lem:g_deriv} as $g_\veps[\zeta] \in \Phi \left(\frac{C_{\veps}C_\zeta}{n \log L}, \veps \right)$. 
\end{proof} 

\begin{lemma}\label{lem:Theta_Sperp_proj_cont}
	Let $\veps \in (0, \frac{3}{4})$, $a_\veps$, $r_\veps$, $S_\veps$  be as in \Cref{eq:a_def,eqn:r_def,eqn:S_eps_def}.  There exist numbers $C_\veps, C_\veps'$ such that when $L \ge \paren{\veps^{1/2}\min\{\len(\manifold_+), \len(\manifold_-)\}/(12\pi^2)}^{-\frac{a_\veps+1}{a_\veps}}$ and the conditions of \Cref{thm:certificate_over_S} are in force, for any $w \in \Phi( C_w, \veps )$, we have
	\begin{equation*}
	\norm{\mb P_{S_\veps^\perp} w}_{L^2} \le C_\veps C_w \frac{r_{\veps}^3}{\log L}\paren{P_3 + \frac{P_1P_2}{\log L} + \frac{P_1^3}{\log^2 L}}. 
	\end{equation*}
	where $P_1, P_2, P_3$ are the polynomials from \Cref{lem:R1,lem:R2,lem:R3} respectively.
	
	As a consequence, for $\zeta \in S_{\veps} \cap \Phi( C_\zeta, \veps )$, letting $g_{\veps}[\zeta]$ be as in \Cref{thm:certificate_over_S}, we have 
	\begin{equation*}
	\norm{\mb \Theta[g_{\veps}[\zeta]] - \zeta}_{L^2} \le C_\veps' C_\zeta \frac{r_{\veps}^3}{\log L}\paren{P_3 + \frac{P_1P_2}{\log L} + \frac{P_1^3}{\log^2 L}}. 
	\end{equation*}
	for some $C_\veps' > 0$. 
\end{lemma}	

\begin{proof}
	When $L \ge \paren{\veps^{1/2}\min\{\len(\manifold_+), \len(\manifold_-)\}/(12\pi^2)}^{-\frac{a_\veps+1}{a_\veps}}$, from \Cref{lem:fourier_deriv_lb} we have
	\begin{align*}
	\norm{\mb P_{S_{\veps}^\perp} w} &\le \paren{\frac{2r_\veps}{\sqrt{\veps}}}^3\norm{\mb P_{S_{\veps}^\perp} w^{(3)}}_{L^2} \\
	&\le \paren{\frac{2r_\veps}{\sqrt{\veps}}}^3\norm{w^{(3)}}_{L^2} \\
	&\le C_w \paren{\frac{2r_\veps}{\sqrt{\veps}}}^3\paren{\frac{P_3}{\log L} + \frac{P_2P_1}{\log^2 L} + \frac{P_1^3}{\log^3 L}}. 
	\end{align*}
	As $\zeta \in \Phi(C_\zeta, \veps)$, from \Cref{lem:Theta_Sperp_proj} we get $\mb \Theta^\circ[g] - \zeta \in \Phi(C_\veps'C_\zeta, \veps)$ for some $C_\veps' > 0$. The rest follows from the fact that $\mb P_{S_\veps} [\mb \Theta^\circ[g] - \zeta] = 0$ and thus $\norm{\mb \Theta[g_{\veps}[\zeta]] - \zeta}_{L^2} = \norm{\mb P_{S_\veps^\perp} [\mb \Theta^\circ[g] - \zeta] }_{L^2}$. 
\end{proof}

\subsection{Certificates with Density and DC}\label{sec:certificate_dc_density}

\paragraph{Proof Sketch and Organization.} In this section, we leverage the calculations in the previous sections to prove \Cref{thm:skeleton_dc_density}, which gives a near solution to the equation 
\begin{equation*}
\mb \Theta_\mu[g](\mb x) = \int \Theta(\mb x,\mb x') g(\mb x') \rho(\mb x') d\mb x' = \zeta(\mb x). 
\end{equation*}
To accomplish this, we need to account for two factors: the presence of a constant (DC) term in $\Theta(\mb x, \mb x') = \Theta^\circ(\mb x,\mb x') + \psi(\pi)$, and the presence of the data density $\rho$ in $\mb \Theta_\mu$. 

Our approach is conceptually straightforward: since $\mb \Theta = \mb \Theta^\circ + \psi(\pi) \indicator{}\indicator{}\adj $, we produce near solutions to two equations 
\begin{align*}
\mb \Theta^\circ[g](\mb x) &= \zeta(\mb x), \\
\mb \Theta^\circ[g_1](\mb x) &= 1, 
\end{align*}
and then combine them to nearly solve $\mb \Theta[h] = \zeta$, by setting $h = g + \alpha g_1$ for an appropriate choice of $\alpha$, 
\begin{equation} \label{eqn:alpha-def}
\alpha = - \frac{\psi(\pi) \indicator{}[g] }{ \psi(\pi) \indicator{}[g_1] + 1 }. 
\end{equation}
Here and in the rest of this section, we write $\indicator{}[g]$ to denote $\indicator{}\adj g$.

The statement of \Cref{thm:skeleton_dc_density} makes two demands on $h$: small approximation error $\| \mb \Theta[h] - \zeta \|_{L^2}$ and small size $\|h\|_{L^2}$. These demands introduce a tension, which forces us to work with DC subtracted solutions $g_{\eps}[\zeta]$ defined in \Cref{eqn:g-neumann} at multiple scales $\veps$. We will set $g = g_{\veps_0}[ \zeta ]$ with $\veps_0$ small, which ensures that both $\| \mb \Theta^\circ[ g] - \zeta \|_{L^2}$ and $\| g \|_{L^2}$ are small. We would like to similarly set $g_1 = g_{\veps_1}[ \zeta_1 ]$, with $\zeta_1 \equiv 1$. In order to ensure that $h$ is small, we need to ensure that the coefficient $\alpha$ defined in \eqref{eqn:alpha-def} is also small, which in turn requires a lower bound on $\indicator{}[g_1]$. This is straightforward if $g_1$ is (pointwise) nonnegative, but challenging if $g_1$ can take on arbitrary signs. The function $g_1$ is defined by the Neumann series 
\begin{align*}
  g_1 &= g_{\veps_1}[\zeta_1] \\ &= \sum_{\ell= 0}^\infty (-1)^\ell \left( \left(\mb P_{S_{\veps_1}}\wh{\mb M}_{\veps_1} \mb P_{S_{\veps_1}} \right)^{-1} \mb P_{S_{\veps_1}} \left(\mb \Theta^\circ - \wh{\mb M}_{\veps_1}\right) \mb P_{S_{\veps_1}} \right)^\ell \left(\mb P_{S_{\veps_1}}\wh{\mb M}_{\veps_1} \mb P_{S_{\veps_1}} \right)^{-1} \zeta_{1}.
\end{align*}
Although this expression is complicated, the first ($\ell=0$) summand is {\em always} nonnegative. If we choose $\veps_1$ large, this expression will be dominated by the first term, providing the necessary control on $\indicator{}[g_1]$. So, we will use two different scales, $\veps_0 < \veps_1$ in constructing $g$ and $g_1$, respectively. 

The issue introduced by the use of a large scale $\veps_1$ is that the approximation error $\| \mb \Theta^{\circ}[g_1] - \zeta_1\|_{L^2}$ is not sufficiently small for our purposes. To address this issue, we introduce an iterative construction, which produces a sequence of increasingly accurate solutions $h_{(i)}$, each of which removes some portion of the approximation error in the previous solution. This sequence converges to our promised certificate $h$. 

More concretely, we will set $\veps_0 = \tfrac{1}{20}$ and $\veps_1 = \tfrac{51}{100}$. For parameters $a_\veps, r_\veps$ defined in \Cref{eq:a_def} and \Cref{eqn:r_def}, these choices of $\veps$ ensure that $a_{\veps_0} > \tfrac{4}{5}$ and $a_{\veps_1} > \tfrac{1}{9}$, and so 
\begin{align*}
r_{\veps_0} &< 6 \pi L^{-\tfrac{4}{9}}, \\
r_{\veps_1} &< 6 \pi L^{-\tfrac{1}{10}}.
\end{align*}

We further choose $\delta_0 = 1 - \veps_0$  and $\delta_1 = \delta_0\sqrt{\veps_0}/\sqrt{\veps_1} < 1 - \veps_1$. This setting satisfies $\delta_0\sqrt{\veps_0} = \delta_1\sqrt{\veps_1}$ and thus allows 
\begin{equation}\label{eq:cnum_compare}
    \fco(\manifold) = \fco_{\veps_0, \delta_0}(\manifold) \ge \fco_{\veps_1, \delta_1}(\manifold). 
\end{equation}

In the remainder of this section, we carry out the argument described above. \Cref{lem:constant_inverse} constructs the aforementioned certificate $g_1$ for the constant function $\zeta_1$. \Cref{lem:approx_certificate} combines this construction with a certificate $g$ for $\zeta$ to give a (loose) approximate certificate, for the kernel $\mb \Theta$. \Cref{lem:certificate_dc} amplifies this construction to reduce the approximation error to an appropriate level. Finally, we finish by incorporating the density $\rho(\mb x)$ to prove our main result on certificates, \Cref{thm:skeleton_dc_density}.

	\begin{lemma}\label{lem:constant_inverse}
		Let $\zeta_1 \equiv 1$ denote the constant function over $\mc M$. When $L > C$ and the conditions of \Cref{thm:certificate_over_S} are satisfied for $\veps = \veps_1 = \tfrac{51}{100}$, then $g_1 = g_{\veps_1}[ \zeta_1]$ satisfies 
		\begin{equation*}
		g_1 \in \Phi\left( \frac{C'}{n \log L} \norm{\zeta_1}_{L^2}, \veps_1 \right) 
		\end{equation*}
		and 
		\begin{equation*} 
		\mb \Theta^{\circ}[g_1] - \zeta_1 \in \Phi\left( C'' \| \zeta_1 \|_{L^2}, \veps_1 \right). 
		\end{equation*}
		We also have 
		\begin{eqnarray}
		\indicator{}[g_1] &\ge& \left(2 - \frac{1}{\veps_1}\right)\frac{\norm{\zeta_1}_{L^1}}{C''' n \log(L)}, \label{eqn:g1-sum} 
		\end{eqnarray}
		where $C, C', C''$ and $C'''$ are positive numerical constants. 
	\end{lemma}

\begin{proof} 
	Applying \Cref{thm:certificate_over_S}, as conditions of \Cref{thm:certificate_over_S} for $\veps = \veps_1$ is satisfied, we know $\mb P_{S_{\veps_1}} \wh{\mb M}_{\veps_1} \mb P_{S_{\veps_1}}$ is invertible over $S_{\veps_1}$. Noting that $\zeta_1$ is a constant function and thus $\zeta_1 \in S_{\veps_1}$, we set 
	\begin{align*}
    g_1 &= g_{\veps_1}[\zeta_1] \\ &= \sum_{\ell= 0}^\infty (-1)^\ell \left( \left(\mb P_{S_{\veps_1}}\wh{\mb M}_{\veps_1} \mb P_{S_{\veps_1}} \right)^{-1} \mb P_{S_{\veps_1}} \left(\mb \Theta^\circ - \wh{\mb M}_{\veps_1}\right) \mb P_{S_{\veps_1}} \right)^\ell \left(\mb P_{S_{\veps_1}}\wh{\mb M}_{\veps_1} \mb P_{S_{\veps_1}} \right)^{-1} \zeta_{1}.
	\end{align*}
	Since $\zeta_{1}$ is a constant function, all its derivatives are zero and thus $\zeta_{1} \in \Phi\left( \| \zeta_1 \|_{L^2}, \veps_1 \right)$. Applying \Cref{lem:g_deriv}, we have that 
	\begin{equation*}
	g_1 \in \Phi \left( \frac{C_{\veps_1}}{n \log L}\norm{\zeta_1}_{L^2}, \veps_1 \right) 
	\end{equation*} 
	for certain $C_{\veps_1} > 0$. 
	The condition $\mb \Theta^\circ[g_1] - \zeta_1 \in \Phi( C'' \| \zeta_1 \|_{L^2}, \veps_1) $ follows from \Cref{lem:Theta_Sperp_proj} directly. 
	
	To control  $\indicator{}[g_1]$, notice that because $\mb P_{S_{\veps_1}} \wh {\mb M}_{\veps_1} \mb P_{S_{\veps_1}}$ is an invariant operator stably invertible in $S_{\veps_1}$, and $\zeta_1 \in S_{\veps_1}$, we can set
	\begin{align*}
	\wh{g}_1 = \left(\mb P_{S_{\veps_1}}\wh{\mb M}_{\veps_1}\mb P_{S_{\veps_1}} \right)^{-1} \zeta_1 = \paren{\int_{s = -r_{\veps_1}}^{r_{\veps_1}} \psi^{\circ}(\abs{s}) ds}^{-1} \zeta_1
	\end{align*}
	which is also a positive constant function. We have
	\begin{align*}
	g_1 &=  \sum_{\ell= 0}^\infty (-1)^\ell \left( \left(\mb P_{S_{\veps_1}}\wh{\mb M}_{\veps_1} \mb P_{S_{\veps_1}} \right)^{-1} \mb P_{S_{\veps_1}} (\mb \Theta^\circ - \wh{\mb M}_{\veps_1}) \mb P_{S_{\veps_1}} \right)^\ell \wh{g}_1 \\
	&= \wh{g}_1 + \sum_{\ell= 1}^\infty (-1)^\ell \left( \left(\mb P_{S_{\veps_1}}\wh{\mb M}_{\veps_1} \mb P_{S_{\veps_1}} \right)^{-1} \mb P_{S_{\veps_1}} \left(\mb \Theta^\circ - \wh{\mb M}_{\veps_1}\right) \mb P_{S_{\veps_1}} \right)^\ell \wh{g}_1, 
	\end{align*}
	and from \Cref{lem:main_res_norm_compare}, we have 
	\begin{equation*}
	\left\| \left(\mb P_{S_{\veps_1}}\wh{\mb M}_{\veps_1} \mb P_{S_{\veps_1}} \right)^{-1} \mb P_{S_{\veps_1}} \left(\mb \Theta^\circ - \wh{\mb M}_{\veps_1} \right) \mb P_{S_{\veps_1}} \right\|_{L^2 \to L^2} \le 1 - \veps_1,
	\end{equation*}
	and so 
	\begin{align*}
	\norm{g_1 - \wh{g}_1}_{L^1} &\le \sqrt{\len(\manifold)}\norm{g_1 - \wh{g}_1}_{L^2} \\
	&\le  \sqrt{\len(\manifold)} \sum_{\ell = 1}^\infty (1 - \veps_1)^\ell \norm{\wh g_1}_{L^2} \\
	&=  \sqrt{\len(\manifold)} \left(\frac{1}{\veps_1} - 1\right)\norm{\wh g_1}_{L^2} \\
	&= \left(\frac{1}{\veps_1} - 1\right)\norm{\wh g_1}_{L^1} \\
	&=  \left(\frac{1}{\veps_1} - 1\right)\indicator{}[\wh g_1].
	\end{align*}
	The first inequality comes from the equivalence of norms, and the last two lines come from the fact that $\wh{g}_1$ is a positive constant function. 
	Thus, following \Cref{lem:skel_l1_ub}, there exist constants $C, C'$ such that
	\begin{align*}
	\indicator{}[g_1] &\ge \indicator{}[\wh g_1] - \norm{g_1 - \wh{g}_1}_{L^1} \\
	&\ge \left(2 - \frac{1}{\veps_1}\right)\indicator{}[\wh g_1] \\
	&= \left(2 - \frac{1}{\veps_1}\right)\paren{\int_{s = -r_{\veps_1}}^{r_{\veps_1}} \psi^{\circ}(\abs{s}) ds}^{-1} \norm{\zeta_1}_{L^1} \\
	&\ge \frac{\left(2 - \frac{1}{\veps_1}\right) \norm{\zeta_1}_{L^1}}{\frac{3\pi n}{8}C'\log\left(1 + \tfrac{1}{3 \pi} (L-3)\, r_{\veps_1} \right)}\\
	&\ge \frac{\left(2 - \frac{1}{\veps_1}\right)\norm{\zeta_1}_{L^1}}{Cn \log(L)}
	\end{align*}
	as claimed.
\end{proof}

	\begin{lemma}\label{lem:a_ub} Suppose that the conditions of \Cref{thm:certificate_over_S} are satisfied for both $\veps = \veps_0 = \tfrac{1}{20}$ and $\veps = \veps_1 = \tfrac{51}{100}$, and let $g_1, \zeta_1$ be as in \Cref{lem:constant_inverse}. Let $\zeta \in S_{\veps_0}$, and $g = g_{\veps_0}[\zeta]$. Then 
		\begin{equation*}
		\alpha = -\frac{\psi(\pi)\indicator{}[g]}{\psi(\pi)\indicator{}[g_1] + 1},
		\end{equation*}
		satisfies 
		\begin{align*}
		\abs{\alpha}&\le \frac{C\norm{\zeta}_{L^2}}{\norm{\zeta_1}_{L^2} },
		\end{align*}
		where $C$ is a numerical constant. 
	\end{lemma}

\begin{proof}
	Set
	\begin{align*}
	\wh{g} = \left(\mb P_{S_{\veps_0}}\wh{\mb M}_{\veps_0} \mb P_{S_{\veps_0}} \right)^{-1} \zeta 
	\end{align*}
	again, we have
	\begin{align*}
	g &=  \sum_{\ell= 0}^\infty (-1)^\ell \left( \left(\mb P_{S_{\veps_0}} \wh{\mb M}_{\veps_0} \mb P_{S_{\veps_0}} \right)^{-1} \mb P_{S_{\veps_0}} \left(\mb \Theta^\circ - \wh{\mb M}_{\veps_0} \right) \mb P_{S_{\veps_0}} \right)^\ell \wh{g}. 
	\end{align*}
	Then there exist constants $c, C > 0$ that
	\begin{align*}
	\norm{g}_{L^1} &\le \sqrt{\len(\manifold)}\norm{g}_{L^2} \\
	&\le \sqrt{\len(\manifold)} \sum_{\ell = 0}^\infty (1 - \veps_0)^\ell \norm{\wh{g}}_{L^2} \\
	&= \sqrt{\len(\manifold)} \frac{1}{\veps_0}\norm{\wh{g}}_{L^2} \\
	&\le \sqrt{\len(\manifold)} \frac{1}{\veps_0 \, \lambda_{\min}\left(\mb P_{S_{\veps_0}} \wh{\mb M}_{\veps_0} \mb P_{S_{\veps_0}} \right)} \norm{\zeta}_{L^2} \\
	&\le \sqrt{\len(\manifold)} \frac{\norm{\zeta}_{L^2}}{\veps_0 cn \log L}\\
	&\le \sqrt{\len(\manifold)} \frac{C\norm{\zeta}_{L^2}}{n \log L},
	\end{align*}
	where in the penultimate inequality, we have used \Cref{lem:main_term_pd}. Applying the previous lemma, we obtain
	\begin{align*}
	\abs{\alpha} &\le \frac{\psi(\pi)\abs{\indicator{}[g]}}{\psi(\pi)\indicator{}[g_1]} \\
	&\le \frac{C n\log(L)\norm{g}_{L^1}}{(2 - \frac{1}{\veps_1})\norm{\zeta_1}_{L^1}} \\
	&\le \frac{C \sqrt{\len(\manifold)}\norm{\zeta}_{L^2}}{\norm{\zeta_1}_{L^1}} \\
	&= \frac{C \norm{\zeta}_{L^2}}{\norm{\zeta_1}_{L^2}},
	\end{align*}
	where in the final equation we have used that the constant function $\zeta_1$ satisfies $\| \zeta_1 \|_{L^1} = \sqrt{\mr{len}(\mc M)} \| \zeta_1 \|_{L^2}$. 
\end{proof}

	\begin{lemma}\label{lem:approx_certificate} 
		Suppose that the conditions of \Cref{thm:certificate_over_S} are satisfied for both $\veps = \veps_0 = \tfrac{1}{20}$ and $\veps = \veps_1 = \tfrac{51}{100}$, and $L \ge \paren{\veps^{1/2}\min\{\len(\manifold_+), \len(\manifold_-)\}/(12\pi^2)}^{-\frac{a_\veps+1}{a_\veps}}$ for both $\veps = \veps_0$ and $\veps = \veps_1$. There exist numerical constants $C, C' > 0$, such that for every $\zeta \in S_{\veps_0} \cap \Phi( C_\zeta, \veps_1 )$, there exists $h$ such that 
		\begin{eqnarray}
		\norm{h}_{L^2} &\le& \frac{C\norm{\zeta}_{L^2}}{n\log L},\label{eqn:h-bound} \\		
		\norm{\mb \Theta[h] - \zeta}_{L^2} &\le& C C_\zeta \frac{P\Bigl( M_3, M_4,  M_5, \mr{len}(\mc M), \Delta_{\veps_0}^{-1} \Bigr) \times r_{\veps_1}^3}{\log L},\label{eqn:h_orth_e1}\\
		\norm{ \mb P_{S_{\veps_0}^\perp}\left[ \mb \Theta[h] - \zeta \right] }_{L^2} &\le& C C_\zeta P\Bigl( M_3, M_4,  M_5, \mr{len}(\mc M), \Delta_{\veps_0}^{-1} \Bigr) \times L^{-\frac{4}{3}}, \label{eqn:h_orth_e0}
		\end{eqnarray}
		where $P$ is a polynomial $\poly\{M_3, M_4, M_5, \len(\manifold), \Delta_\veps^{-1}\}$ of degree $\le 9$, with degree $\le 3$ in $M_3, M_4, M_5$ $ \len(\manifold)$, and degree $\le 6$ in $\Delta_{\veps}^{-1}$. 
		Furthermore, we have 
		\begin{equation} \label{eqn:PS0_resid_smooth}
		\mb P_{S_{\veps_0}}\left [ \, \mb \Theta[h] - \zeta \, \right] \in \Phi \left( C'\norm{\zeta}_{L^2}, \veps_0 \right). 
		\end{equation}
			\end{lemma}
		
\begin{proof}
	Recall that $\mb \Theta = \mb \Theta^\circ + \psi(\pi)\mb{\indicator{}}$, and let $g_1$ denote the solution for $\zeta_1 \equiv 1$ as in \Cref{lem:constant_inverse}. Set $h = g + \alpha g_1$, where $g = g_{\veps_0}[ \zeta ]$, and 
	\begin{equation*}
	\alpha = -\frac{\psi(\pi)\indicator{}[g]}{\psi(\pi)\indicator{}[g_1] + 1}.
	\end{equation*}
	Using \Cref{thm:certificate_over_S} to control the norms of $g$ and $g_1$, and using \Cref{lem:a_ub} to control $|\alpha|$, we have 
	\begin{align*}
	\norm{h}_{L^2} &\le \norm{g}_{L^2} + \abs{a}\norm{g_1}_{L^2} \\ &\le \frac{C\norm{\zeta}_{L^2}}{n\log L} + \abs{\alpha}\frac{C\norm{\zeta_1}_{L^2}}{n\log L} \\
	&\le \frac{C\norm{\zeta}_{L^2}}{n\log L},
	\end{align*}
	establishing \eqref{eqn:h-bound}. 
	
	From our choice of $\alpha$,	
	\begin{align*}
	\mb \Theta[h] - \zeta &= \mb \Theta^\circ[g] - \zeta + \alpha \mb \Theta^{\circ}[g_1] + \psi(\pi)\mb{\indicator{}}[g] + \alpha \psi(\pi)\mb{\indicator{}}[g_1]\\
	&= \mb \Theta^\circ[g] -\zeta + \alpha \paren{\mb \Theta^{\circ}[g_1] - \zeta_1}. \labelthis \label{eqn:theta_g_minus_h}
	\end{align*}
	
	Using \Cref{lem:Theta_Sperp_proj_cont} and the fact that $\zeta_1 \in \Phi(\norm{\zeta_1}, \veps_1)$ we have
	\begin{eqnarray*}
	\norm{\mb \Theta[h] - \zeta}_{L^2} &\le& \norm{\mb \Theta^{\circ}[g] - \zeta}_{L^2} + \abs{\alpha} \norm{\mb \Theta^{\circ}[g_1] - \zeta_1}_{L^2} \\
	&\le& C_{\veps_0} C_\zeta\frac{ r_{\veps_0}^3  }{\log L}\paren{P_3 + \frac{P_1 P_2}{\log L} + \frac{P_1^3}{\log^2 L}} \\ &&+ \abs{\alpha} C_{\veps_1} \norm{\zeta_1}_{L^2}\frac{  r_{\veps_1}^3}{\log L}\paren{P_3 + \frac{P_1 P_2}{\log L} + \frac{P_1^3}{\log^2 L}}. 
	\end{eqnarray*}
	
	Further using \Cref{lem:a_ub} to bound $\alpha \| \zeta_1 \|_{L^2} \le C' \| \zeta \|_{L^2} \le C' C_\zeta$ and $r_{\veps_0} \le r_{\veps_1}$, we have	
	\begin{equation}\label{eq:approx_certificate_approx_zeta}
	\norm{\mb \Theta[h] - \zeta}_{L^2} \le C' C_\zeta \frac{ r_{\veps_1}^3}{\log L}\paren{P_3 + \frac{P_1 P_2}{\log L} + \frac{P_1^3}{\log^2 L}}
	\end{equation}
	for some absolute constant $C' > 0$. 
	Notice that from \Cref{lem:x_deriv} $M_2 < 2\hat{\kappa} \le 2\Delta_{\veps}^{-2}$ and thus
	\begin{align*}
	P_3 + \frac{P_1 P_2}{\log L} + \frac{P_1^3}{\log^2 L} &\le P_3 + P_1P_2 + P_1^3 \\
	&= C_3(M_4^3 + M_3M_4 + M_4M_5 + \len(\manifold)\paren{\Delta_\veps^{-4} + M_2 \Delta_\veps^{-3} + M_3 \Delta_\veps^{-2}})  \\ &\qquad + C_2(M_4^2 + M_5 + \len(\manifold)\paren{\Delta_{\veps}^{-3} + M_2 \Delta_{\veps}^{-2}})C_1(M_4 + \len(\manifold)\Delta_{\veps}^{-2}) \\ &\qquad + C_1^3(M_4 + \len(\manifold)\Delta_{\veps}^{-2})^3  \\
  &\le C(M_4^3 + M_3M_4 + M_4M_5  \\ &\qquad + \len(\manifold)^3\Delta_{\veps}^{-6}  +
\len(\manifold)^2\paren{\Delta_{\veps}^{-5} + M_2 \Delta_{\veps}^{-4}}  \\ &\qquad +
\len(\manifold)(\Delta_\veps^{-4} + M_2 \Delta_\veps^{-3} + M_3 \Delta_\veps^{-2}  \\&\qquad + M_4^2 \Delta_\veps^{-2} + M_5 \Delta_{\veps}^{-2} + M_4 \Delta_\veps^{-3} + M_2M_4 \Delta_{\veps}^{-2}))  \\
	&\le C (M_4^3 + M_3M_4 + M_4M_5 + \len(\manifold)^3\Delta_{\veps}^{-6}  + \Delta_{\veps}^{-6}
   \\ &\qquad + \len(\manifold)\Delta_\veps^{-2}\paren{M_3+ M_5})  \\
	&\stackrel{\text{def}}{=} P\Bigl(M_3, M_4,  M_5, \mr{len}(\mc M), \Delta_{\veps_0}^{-1} \Bigr) \labelthis \label{eq:P_poly_def}
	\end{align*}
	where $P\Bigl(M_3, M_4,  M_5, \mr{len}(\mc M), \Delta_{\veps_0}^{-1} \Bigr)$ is a polynomial of $M_3, M_4, M_5, \len(\manifold), \Delta_\veps^{-1}$ of degree $\le 9$, with degree $\le 3$ in $M_3, M_4, M_5$ $ \len(\manifold)$, and degree $\le 6$ in $\Delta_{\veps}^{-1}$. Here $P_1, P_2$ and $P_3$ are polynomials defined in \Cref{lem:R1}, \Cref{lem:R2} and \Cref{lem:R3}. This together with \Cref{eq:approx_certificate_approx_zeta} give us \eqref{eqn:h_orth_e1}. 	
	
	To obtain the tighter bound \eqref{eqn:h_orth_e0} on $\mb P_{S_{\veps_0}^\perp} ( \mb \Theta[ h] - \zeta )$, we begin by applying \Cref{lem:Theta_Sperp_proj} with $\zeta \in \Phi( C_\zeta, \veps_0)$ and $\zeta_1 \in \Phi( \| \zeta_1 \|_{L^2}, \veps_1 )$, we have 
	\begin{align*}
	\mb \Theta^\circ[g] - \zeta &\in \Phi\Bigl( C_{\veps_0}' C_\zeta, \veps_0 \Bigr), \\
	\mb \Theta^\circ[g_1] - \zeta_1 &\in \Phi\Bigl( C_{\veps_1}'\norm{\zeta_1}_{L^2}, \veps_1 \Bigr)
  \quad \subseteq \quad  \Phi\Bigl( C_{\veps_1}'\norm{\zeta_1}_{L^2}, \veps_0 \Bigr) \labelthis \label{eqn:g1-regular}
	\end{align*}  
	for certain $C_{\veps_0}', C_{\veps_1}' > 0$. Using $\mb \Theta [h ] - \zeta = \mb \Theta^\circ [ g ] - \zeta + \alpha ( \mb \Theta^\circ [ g_1 ] - \zeta_1 )$, we have 
	\begin{equation*}
	C_{\veps_0}' C_\zeta + \frac{C'\norm{\zeta}_{L^2}}{\norm{\zeta_1}_{L^2}} \times C'_{\veps_1}\norm{\zeta_1}_{L^2}  \le C''C_\zeta.
	\end{equation*}
	for some constant $C'' > 0$ and thus 
	\begin{equation}
	\mb \Theta[h] - \zeta \in \Phi(C''C_\zeta, \veps_0). 
	\end{equation}
Applying \Cref{lem:Theta_Sperp_proj_cont} with $w = \mb \Theta [h ] - \zeta$ and simplifying with \Cref{eq:P_poly_def} we obtain 
\begin{equation*}
 	\| \mb P_{S_{\veps_0}^\perp} ( \mb \Theta [ h ] - \zeta ) \|_{L^2} \le C'' C_\zeta \frac{ P \Bigl(M_3, M_4, M_5, \mr{len}(\mc M), \Delta_{\veps_0}^{-1} \Bigr) \times r_{\veps_0}^3 }{ \log L }. 
\end{equation*}
	\eqref{eqn:h_orth_e0} follows from $r_{\veps_0} < 6 \pi L^{-4/9}$, which implies that $r_{\veps_0}^3 / \log L \le L^{-4/3}$ when $L$ is larger than an appropriate numerical constant.

	Finally, since $\mb P_{S_{\veps_0}} \mb \Theta^\circ[g]  = \mb P_{S_{\veps_0}} \zeta$, 
	\begin{eqnarray} \label{eqn:residual_on_S0}
		\mb P_{S_{\veps_0}}[\mb \Theta[h] - \zeta] &=& \alpha \mb P_{S_{\veps_0}}\left[\mb \Theta^{\circ}[g_1] - \zeta_1\right]. 
	\end{eqnarray}
	From \Cref{lem:invariant_deriv_commute}, $\mb P_{S_{\veps_0}}$ commutes with differentiation, and so for any $i$-times differentiable $w$, 
	\begin{equation*}
	\| ( \mb P_{S_{\veps_0}} w )^{(i)} \|_{L^2} \le \| w^{(i)} \|_{L^2}.
	\end{equation*} 
	Applying this with $i=0,1,2,3$, we see that for any $w \in \Phi( C_w, \veps)$, $\mb P_{S_{\veps_0}} w \in \Phi( C_w, \veps )$. Applying this observation to \Cref{eqn:g1-regular}, we have 
	\begin{equation*}
	\mb P_{S_{\veps_0}} \left( \mb \Theta^\circ [ g_1 ] - \zeta_1 \right) \in \Phi\Bigl( C_{\veps_1}' \| \zeta_1 \|_{L^2}, \veps_0 \Bigr).
	\end{equation*}
	 Combining with \eqref{eqn:residual_on_S0}, we have
	\begin{equation*}
	\mb P_{S_{\veps_0}}( \mb \Theta[h] - \zeta ) \in \Phi \Bigl( |\alpha| C_{\veps_1}' \| \zeta_1 \|_{L^2}, \veps_0 \Bigr) \subseteq \Phi \Bigl( C' C_{\veps_1}' \| \zeta \|_{L^2}, \veps_0 \Bigr),
	\end{equation*}
	which is \eqref{eqn:PS0_resid_smooth}. Here, we have used the bound on $\alpha$ from the previous lemma. 	This completes the proof. 
\end{proof}

	\begin{theorem}[Certificates for DC kernel]\label{lem:certificate_dc} 
		Suppose that the conditions of \Cref{thm:certificate_over_S} are satisfied for both $\veps = \veps_0 = \tfrac{1}{20}$ and $\veps = \veps_1 = \tfrac{51}{100}$ and $L \ge \paren{\veps^{1/2}\min\{\len(\manifold_+), \len(\manifold_-)\}/(12\pi^2)}^{-\frac{a_\veps+1}{a_\veps}}$ for both $\veps = \veps_0$ and $\veps = \veps_1$. There exist constants $C, C''$ such that for any number $K > 0$, when 
		\begin{equation} \label{eqn:certificate_dc_L}
		L \ge C K^4 P\Bigl( M_3, M_4, M_5, \mr{len}(\mc M), \Delta_{\veps_0}^{-1} \Bigr)^4,
		\end{equation} 
		then for any $\zeta \in S_{\veps_0} \cap \Phi( K\| \zeta \|_{L^2}, \veps_0 )$ there exists a certificate $h$ satisfying
		\begin{eqnarray} \label{eqn:dc-cert-resid}
		\norm{\mb \Theta[h] - \zeta}_{L^2} \le \norm{\zeta}_{L^{\infty}}L^{-1}, 
		\end{eqnarray}
		with
		\begin{equation*}
		\norm{h}_{L^2} \le \frac{C'' \norm{\zeta}_{L^2}}{n\log L}.
		\end{equation*}
		In \eqref{eqn:certificate_dc_L}, $P$ is the polynomial defined in \Cref{lem:approx_certificate}. 
	\end{theorem}
\begin{proof}
	Let $\zeta_{(0)} = \zeta$, and iteratively define 
	\begin{equation*}
	\zeta_{(i+1)} = - \mb P_{S_{\veps_0}} \left( \mb \Theta[ h_{(i)} ] - \zeta_{(i)} \right) \in S_{\veps_0}. 
	\end{equation*}
	where $h_{(i)}$ is the approximate certificate of $\zeta_{(i)} \in S_{\veps_0}$ constructed in \Cref{lem:approx_certificate}. From \eqref{eqn:PS0_resid_smooth}, we have
	\begin{equation*}
	\zeta_{(i+1)} \in \Phi\Bigl( C_1 \| \zeta_{(i)} \|_{L^2}, \veps_0 \Bigr),
	\end{equation*}
	where $C_1$ is a numerical constant. 
	Hence, for $i \ge 1$, from \Cref{lem:approx_certificate} we have 
	\begin{align*}
		\left\| h_{(i)} \right\|_{L^2} &\le \left( \frac{C_2}{n \log L} \right) \| \zeta_{(i)} \|_{L^2}, \\
	 \left\| \mb \Theta[ h_{(i)} ] - \zeta_{(i)} \right\|_{L^2} &\le  C_2 C_1 \| \zeta_{(i-1)}
   \|_{L^2} \frac{Pr_{\veps_1}^3}{\log L}, \labelthis \label{eqn:resid-i-1} \\
	\left\| \mb P_{S_{\veps_0}^\perp} \left[ \mb \Theta [ h_{(i)} ] - \zeta_{(i)} \right] \right\|_{L^2} &\le C_2 C_1 \| \zeta_{(i-1)} \|_{L^2} P L^{-4/3}.  
	\end{align*}
	For $i= 0$, as $\zeta_{(0)} \in \Phi\Bigl(  K \| \zeta \|_{L^2}, \veps_0 \Bigr),$ this simplifies to 
	\begin{align*}
		\left\| h_{(0)} \right\|_{L^2} &\le \left( \frac{C_2}{n \log L} \right) \| \zeta \|_{L^2}, \\
	 \left\| \mb \Theta[ h_{(0)} ] - \zeta_{(0)} \right\|_{L^2} &\le  C_2 K \| \zeta \|_{L^2}
   \frac{Pr_{\veps_1}^3}{\log L},  \labelthis \label{eqn:resid-i-0} \\
	\left\| \mb P_{S_{\veps_0}^\perp} \left[ \mb \Theta [ h_{(0)} ] - \zeta_{(0)} \right] \right\|_{L^2} &\le C_2 K \| \zeta \|_{L^2} P L^{-4/3}.  
	\end{align*}
	We use these relationships to control $\| \zeta_{(i)} \|_{L^2}$. As $r_{\veps_1} \le 6\pi L^{-1/10}$, there exists a constant $C$ such that when $L \ge C K^4 P^4$, $C_2 K \frac{Pr_{\veps_1}^3}{\log L} \le \tau = \tfrac{1}{2}$ and $C_2C_1 \frac{Pr_{\veps_1}^3}{\log L} \le \tau^2$. We argue by induction that 
	\begin{equation*}
	\| \zeta_{(i)} \|_{L^2} \le \tau^i \| \zeta \|_{L^2} \quad \forall \; i \ge 0. 
	\end{equation*}
	 This is true by construction for $i = 0$, while for $i = 1$ it follows from \eqref{eqn:resid-i-0}. Finally, for $i \ge 2$, using \eqref{eqn:resid-i-1} and inductive hypothesis, we have 
	\begin{align*}
	\| \zeta_{(i)} \|_{L^2} &\le \| \mb \Theta[ h_{(i-1)} ] - \zeta_{(i-1)} \|_{L^2} \\ &\le C_2 C_1 \| \zeta_{(i-2)} \|_{L^2} \frac{Pr_{\veps_1}^3}{\log L}, \\
	 &\le \tau^2  \| \zeta_{(i-2)} \|_{L^2} \\ &\le \tau^i \| \zeta \|_{L^2}, 
	\end{align*}
	as claimed. 
	
	We set 
	\begin{equation*}
	h = \sum_{i=0}^k h_{(i)},
	\end{equation*}
	where $k$ will be specified below. By construction, 
	\begin{align*}
	\left\| h \right\|_{L^2} &\le \sum_{i=0}^k \left\| h_{(i)} \right\|_{L^2} \\ 
	&\le \frac{C_1}{n \log L} \sum_{i=0}^k \left\| \zeta_{(i)} \right\|_{L^2} \\
	&\le \frac{2C_1}{n \log L} \left\| \zeta \right\|_{L^2},
	\end{align*}
	as claimed. 
	
	We next verify that $\mb \Theta[h]$ is an accurate approximation to $\zeta$:
	\begin{align*}
	\left\| \mb \Theta[ h ] - \zeta_{(0)} \right\|_{L^2} &\le \left\| \mb P_{S_{\veps_0}} [\mb \Theta [ h ] - \zeta_{(0)}] \right\|_{L^2} + \left\| \mb P_{{S_{\veps_0}^\perp}} \mb \Theta [ h ] \right\|_{L^2} \\
	&\le \left\| \mb P_{S_{\veps_0}} \left[ \mb \Theta [ h_{(k)} ] - \zeta_{(k)} \right] \right\|_{L^2} + \sum_{i = 0}^k \left\| \mb P_{{S_{\veps_0}^\perp}} \mb \Theta [ h_{(i)} ]\right\|_{L^2} \\ 
	&\le \tau^{k+1} \left\| \zeta_{(0)} \right\|_{L^2} + C_2 C_1 P L^{-4/3} \sum_{i = 1}^k  \| \zeta_{(i-1)} \|_{L^2} \\
	& \qquad + \quad C_2 K P L^{-4/3} \| \zeta \|_{L^2}  \\ 
	&\le \tau^{k+1} \left\| \zeta_{(0)} \right\|_{L^2} + C_2 P L^{-4/3} (2 C_1 + K) \| \zeta \|_{L^2}.  
	\end{align*}
	Choosing $k$ appropriately, and ensuring that $C_2 (2 C_1 + K) P L^{-4/3} < \tfrac{1}{2} L^{-1}$, establishes \eqref{eqn:dc-cert-resid}. The latter condition follows immediately, from $L \ge C K^4 P^4$ for appropriately large $C$. \end{proof} 

\begin{lemma}\label{lem:skeleton_dc_density}
		\flushleft{There exist constants $C, C', C'', C'''$ and a polynomial} $P = \poly\{M_3, M_4, M_5, \len(\manifold), \Delta^{-1}\}$ of degree $\le 36$, with degree $\le 12$ in $M_3, M_4, M_5$ $ \len(\manifold)$, and degree $\le 24$ in $\Delta^{-1}$ such that for any number $K > 0$, when
	\begin{equation*}
	L \ge \max\Biggl\{ \exp(C'\len(\manifold) \hat{\kappa}), \paren{ \frac{1}{\Delta \sqrt{1 + \kappa^2}}}^{C''\fco(\manifold)}, C'''\hat{\kappa}^{10}, K^4 P, \rho_{\max}^{12} \Biggr\}
	\end{equation*}
	then for any real $\zeta \in \Phi(K\norm{\zeta}_{L^2}, \frac{1}{20})$, there exists a real certificate $g: \manifold \to \R$ with 
	\begin{equation*}
	\| g \|_{L^2_\mu} \le \frac{C \norm{\zeta}_{L^2_\mu}}{ \rho_{\min} n \log L} 
	\end{equation*} 
	such that
	\begin{equation*}
	\norm{\mb \Theta_{\mu}[g] - \zeta}_{L^2_\mu} \le \norm{\zeta}_{L^\infty}L^{-1}
	\end{equation*}
\end{lemma}
\begin{proof}
	We first show that there exists constants $C, C', C'', C'''$ that under condition of the lemma, conditions of \Cref{lem:certificate_dc} are satisfied. From definition, $\Delta = \Delta_{\veps_0} \le \Delta_{\veps_1}$, as we choose $\delta_0 = 1 - \veps_0, \delta_1 = \delta_0\sqrt{\veps_0}/\sqrt{\veps_1} < 1 - \veps_1$, we have $\delta_0\sqrt{\veps_0} = \delta_1\sqrt{\veps_1}$ and thus $\fco(\manifold) = \fco_{\veps_0, \delta_0}(\manifold) \ge \fco_{\veps_1, \delta_1}(\manifold)$ as in \Cref{eq:cnum_compare}. Thus conditions of \Cref{thm:certificate_over_S} for $\veps = \veps_1$ can be absorbed into conditions for $\veps = \veps_0$ with a change of constant factors. 
	
	Furthermore, as $\Delta \le \frac{\sqrt{\veps_0}}{\hat{\kappa}} < \frac{1}{2\sqrt{1 + \kappa^2}}$, we have the condition in \Cref{thm:certificate_over_S}$ \paren{1 + \frac{1}{\Delta\sqrt{1 + \kappa^2}}}^{C''\fco(\manifold)} \le \paren{\frac{1}{\Delta\sqrt{1+\kappa^2}}}^{2C''\fco(\manifold)}$ and thus reduce to the form in the statement. 
	
	We then notice that from \Cref{lem:x_deriv} $\min\{\len(\manifold_+), \len(\manifold_-)\} \ge \frac{1}{\hat{\kappa}}$, and we can choose $C'$ such that $\exp(C' \len(\manifold)\hat{\kappa}) \ge \exp(C') \ge C_{\veps_0}$. Similarly, by choosing $C'''$ appropriately, $L \ge C''' \hat{\kappa}^{10}$ implies both $L \ge ({\veps}^{-1/2}6\pi\hat{\kappa})^{\frac{a_{\veps} + 1}{a_{\veps}}}$ and $L \ge ({\veps}^{-1/2}12\pi^2\min\{\len(\manifold_+), \len(\manifold_+)\})^{\frac{a_{\veps} + 1}{a_{\veps}}}$ for both $\veps = \veps_0$ and $\veps = \veps_1$ as $a_{\veps_0} > \frac{4}{5}$ and $a_{\veps_1} > \frac{1}{9}$. 

	From \Cref{lem:certificate_dc}, we know there exists $g$ such that
	\begin{equation}\label{eq:skel_dc_density_g_cond1}
	\norm{\mb \Theta[g] - \zeta}_{L^2} \le \norm{\zeta}_{L^{\infty}}L^{-1}
	\end{equation}
	with 
	\begin{equation}\label{eq:skel_dc_density_g_cond2}
	\norm{g}_{L^2} \le \frac{C\norm{\zeta}_{L^2}}{n\log L}. 
	\end{equation}
	We can further require this $g$ to be a real function over the manifold. To see this, notice that for any $\mb x, \mb x' \in \manifold$, both the kernel $\Theta(\mb x, \mb x')$ and $\zeta(\mb x)$ are real, thus if we take the real component of $g$ as $\wh{g} = (g + \conj{g})/2$, then we have $\norm{\wh{g}}_{L^2} \le \norm{g}_{L^2}$ and further by the triangle inequality
	\begin{align*}
	    \abs{\mb \Theta[\wh{g}](\mb x) - \zeta(\mb x)} &= \abs*{\int_{\mb x' \in \manifold} \Theta(\mb x, \mb x') \wh{g}(\mb x') d\mb x' - \zeta(\mb x)} \\
	    &\le \paren{\abs*{\int_{\mb x' \in \manifold} \Theta(\mb x, \mb x') {\wh{g}}(\mb x') d\mb x' - \zeta(\mb x)}^2 + \abs*{\int_{\mb x' \in \manifold} \Theta(\mb x, \mb x') (1/2)(g - \conj{g})(\mb x') d\mb x'}^2}^{1/2} \\
	    &= \abs*{\int_{\mb x' \in \manifold} \Theta(\mb x, \mb x') g(\mb x') d\mb x' - \zeta(\mb x)}. 
	\end{align*}
	The last line comes from the fact that $\int_{\mb x' \in \manifold}\Theta(\mb x, \mb x') {\wh{g}}(\mb x')d \mb x' - \zeta(\mb x)$ and $\int_{\mb x' \in \manifold} \Theta(\mb x, \mb x') {(1/2)(g - \conj{g})}(\mb x') d\mb x'$ are the pure real and imaginary part of $\int_{\mb x' \in \manifold} \Theta(\mb x, \mb x') {g}(\mb x') d\mb x' - \zeta(\mb x)$. Thus $\wh{g}$ is real and also satisfies \Cref{eq:skel_dc_density_g_cond1} and \Cref{eq:skel_dc_density_g_cond2}. 
	
	To include the density, define $g_\mu(\mb x) = g(\mb x)/\rho(\mb x)$. We get
	\begin{align*}
	\norm{g_\mu}_{L^2_\mu}^2 &= \int_{\mb x \in \manifold} \abs{g_\mu(\mb x)}^2 \rho(\mb x)d\mb x \\
	&= \int_{\mb x \in \manifold} \abs{g(\mb x)}^2 \rho^{-1}(\mb x)d\mb x \\
	&\le \min_{\mb x' \in \manifold} \rho^{-1}(\mb x') \int_{\mb x \in \manifold} \abs{g(\mb x)}^2 d\mb x \\
	&= \rho_{\min}^{-1} \norm{g}_{L^2}^2
	\end{align*}
	Then for $\zeta$, we have
	\begin{align*}
	\norm{\zeta}_{L^2_\mu}^2  &= \int_{\mb x \in \manifold} \abs{\zeta(\mb x)}^2 \rho(\mb x)d\mb x \\
	&\ge \rho_{\min} \norm{\zeta}_{L^2}^2. 
	\end{align*}
	This gives
	\begin{align*}
	\norm{g_\mu}_{L^2_\mu} &\le \rho_{\min}^{-1/2}\norm{g}_{L^2} \\
	&\le \rho_{\min}^{-1/2}\frac{C\norm{\zeta}_{L^2}}{n \log L} \\
	&\le \frac{C\norm{\zeta}_{L^2_\mu}}{\rho_{\min} n \log L}. 
	\end{align*}
	On the other hand, we have
	\begin{align*}
	\norm{\mb \Theta_\mu[g] - \zeta}_{L^2_\mu}^2 &= \int_{\mb x \in \manifold}\paren{\int_{\mb x' \in \manifold}\Theta(\mb x, \mb x')g_\mu(\mb x) \rho(\mb x')d \mb x' - \zeta(\mb x)}^2 \rho(\mb x) d\mb x \\
	&\le \max_{\mb x'' \in \manifold} \rho(\mb x'') \int_{\mb x \in \manifold}\paren{\int_{\mb x' \in \manifold}\Theta(\mb x, \mb x')g(\mb x')d \mb x' - \zeta(\mb x)}^2 d \mb x\\
	&= \rho_{\max} \norm{\mb \Theta[g] - \zeta}_{L^2}^2.
	\end{align*}
	Notice that $P' = CP+\len(\manifold)^3$ is still a polynomial of degree $3$ in $\len(\manifold)$. Then when $L \ge \max\set{P'^4, \rho_{\max}^{12}}$, we get
	\begin{align*}
	\norm{\mb \Theta_\mu[g] - \zeta}_{L^2_\mu}^2 &\le \rho_{\max}^{1/2} CPL^{-\frac{4}{3}}\norm{\zeta}_{L^2} \\
	&= L^{-1}\norm{\zeta}_{L^\infty} \frac{CP}{L^{1/4}}\frac{\rho_{\max}^{1/2}}{L^{1/24}}\frac{ \sqrt{\len(\manifold)}}{L^{1/24}} \\
	&\le L^{-1}\norm{\zeta}_{L^\infty}. 
	\end{align*}
	As $P$ in \Cref{lem:approx_certificate} is a polynomial $\poly\{M_3, M_4, M_5, \len(\manifold), \Delta^{-1}\}$ of degree $\le 9$, with degree $\le 3$ in $M_2, M_4, M_5$ $ \len(\manifold)$, and degree $\le 6$ in $\Delta^{-1}$, we have $P'^4$ is of the right degree requirement as in the statement of the theorem. 
\end{proof}

\section{Bounds for the Skeleton Function \texorpdfstring{$\psi$}{the Skeleton}}

\label{sec:appendix_skel_bounds}

In this section, we are going to provide sharp bounds on the ``skeleton'' function $\psi$ and its
higher-order derivatives. We recall that the \textit{angle evolution function} is defined as
\begin{equation*}
    \varphi(t) = \arccos\paren{\left(1 - \frac{t}{\pi}\right)\cos t + \frac{1}{\pi}\sin t}, \quad
    t \in [0, \pi].
\end{equation*}
Define $\varphi^{[0]} = \Id$, $\varphi^{[\ell]}$ as $\varphi$'s $\ell$-fold composition with
itself (which will be referred to as the \textit{iterated angle evolution function}). Then the
skeleton is defined as 
\begin{equation*}
  \psi(t) = \frac{n}{2}\sum_{\ell = 0}^{L-1}\xi_{\ell}(t),
\end{equation*}
where 
\begin{equation*}
    \xi_{\ell}(t) = \prod_{\ell' = \ell}^{L-1}\paren{1 - \frac{1}{\pi}\varphi^{[\ell']}(t)},
    \quad \ell = 0, \cdots, L-1.
\end{equation*}
To analyze the function $\psi$, we will establish in this section several
``sharp-modulo-constants'' estimates that connect $\psi$ to a much simpler function, derived using
the local behavior of $\varphi$ at $0$ and its consequences for the iterated compositions
$\varphi^{[\ell]}$ that appear in the definition of $\psi$. In particular, let us define
$\wh{\varphi} : [0, \pi] \to [0, \pi]$ by $\wh{\varphi}(t) = t / (1 + t/(3\pi))$, so that
\begin{equation*}
  \wh{\varphi}^{[\ell]}(t) = \frac{t}{1+\ell t/(3\pi)},
\end{equation*}
and moreover define
\begin{equation*}
  \wh{\xi}_{\ell}(t) = \prod_{\ell' = \ell}^{L-1}\paren{1 - \frac{\wh{\varphi}^{[\ell']}(t)}{\pi}},
  \qquad
  \wh{\psi}(t) = \frac{n}{2} \sum_{\ell = 0}^{L-1}\wh{\xi}_{\ell}(t).
\end{equation*}
We will prove that $\wh{\varphi}^{[\ell]}$ provides a sharp approximation to $\varphi^{[\ell]}$
(\Cref{lem:phi_l_ub,lem:phi_lower_bound}), and then work out a corresponding sharp approximation
of $\wh{\psi}$ to $\psi$ (\Cref{lem:skel_r_lb,lem:skel_l1_ub}). We will then derive estimates for
the low-order derivatives of $\psi$ in \Cref{sec:skel_deriv_bounds}. Unfortunately, it is
impossible to obtain $L^1$ estimates for $\psi$ in terms of $\wh{\psi}$ that are sharp enough to
facilitate operator norm bounds for $\vTheta_\mu$, which would let us construct certificates for
an operator with kernel $\wh{\psi}$ rather than the NTK $\vTheta_\mu$; but the estimates we derive
in this section will be nonetheless sufficient to enable our localization and certificate
construction arguments in \Cref{sec:proof_certificate_theorem}.

We note that bounds similar to a subset of the bounds in this section have been developed in an
$L$-asymptotic, large-angle setting by \cite{Huang2020-pn}. The bounds we develop here are
non-asymptotic and hold for all angles, and are established using elementary arguments that we
believe are slightly more transparent. We reuse (and restate in \Cref{sec:aux}) some estimates
from \cite[Section C]{Buchanan2020-er} here, but the majority of our estimates will be
fundamentally improved (a representative example is \Cref{lem:phi_lower_bound}).

Throughout this section, we use $\dot{\varphi}, \ddot{\varphi}, \dddot{\varphi}$ to represent
first, second and third derivatives of $\varphi$ (see \linebreak\Cref{lem:phi_concave} for basic regularity
assertions for this function and its iterated compositions) and likewise for $\xi$ and $\psi$.
In particular, for example, in our notation the function $\dot{\varphi}^{[\ell]}$ refers to the
derivative of $\varphi^{[\ell]}$, \textit{not} the $\ell$-fold iterated composition of
$\dot{\varphi}$. Although this leads to an abuse of notation, the concision it enables in our
proofs will be of use.

\subsection{Sharp Lower Bound for the Iterated Angle Evolution Function}

\begin{lemma}
	\label{lemma:upper_bound}
  One has
  \begin{align*}
    \varphi(t) &\le \frac{t}{1 + t/(3\pi)}, \quad  t \in [0, \pi].
  \end{align*}
\end{lemma}
\begin{proof}
	As $\cos$ is monotonically decreasing in $[0, \pi)$, it is the same as proving
	\begin{align*}
	(1 - \frac{t}{\pi})\cos(t) + \frac{\sin t}{\pi} -  \cos \frac{t}{1 + \frac{t}{3\pi}} \ge 0
	\end{align*}
	We have the gradient as
	\begin{align*}
	&-(1 - \frac{t}{\pi})\sin t + \sin(\frac{t}{1 + t/(3\pi)}) \frac{1}{(1 + t/(3\pi))^2} \\
	&\quad \ge -(1 - \frac{t}{\pi})\sin t + \sin t \frac{1}{(1 + t/(3\pi))^3} \\
	&\quad \ge (-(1 - \frac{t}{\pi}) + \frac{1}{(1 + \frac{t}{3\pi})^3})\sin t \\
	&\quad \ge 0
	\end{align*}
	For the first inequality, we use the estimate
	\begin{equation}
	   \sin(ax) \geq x \sin a; \quad 0 \leq x \leq 1,\enspace 0 \leq a \leq \pi, 
	   \label{eq:sinc_lb}
	\end{equation}
	which is easily established using concavity of $\sin$ on $[0, \pi]$ and the secant line characterization, and for the final inequality, we use the estimate
	$1 - 3a \le \frac{1}{(1+a)^3}$ for any $a > -1$, which follows from convexity of $a \mapsto (1+a)^{-3}$ on this domain and the tangent line characterization (at $a=0$).
	Since at $t = 0$, we have the inequality holds, we know it holds for the whole interval $[0, \pi]$ by the mean value theorem. 
\end{proof}

\begin{lemma}[Corollary of {\cite[Lemma C.12]{Buchanan2020-er}}]\label{lem:phi_l_ub}%
    If $\ell \in \N_0$, one has the "fluid" estimate for the iterated angle evolution function
    \begin{equation*}
      \varphi^{[\ell]}(t) \le \frac{t}{1 + \ell t/(3\pi)}.
    \end{equation*}
\end{lemma}
\begin{proof}
  Follow the argument of \cite[Lemma C.12]{Buchanan2020-er}, but use \Cref{lemma:upper_bound} as
  the basis for the argument instead of \Cref{lem:aevo_properties}.
\end{proof}

\begin{lemma}\label{lem:phi_lower_bound}%
There exists an absolute constant $C_0>0$ such that for all $\ell \in \bbN$
	\begin{equation}
	\wh{\varphi}^{[\ell]} - \varphi^{[\ell]} \le C_0 \frac{\log (1+\ell)}{\ell^2}.
	\label{eq:hatphi_phi_bound_1}
	\end{equation}
	As a consequence, there exist absolute constants $C, C', C'' > 0$ such that for any $ 0 < \veps \leq 1/2$, if $L \geq C \veps^{-2}$ then for every $t \in [0, C' \veps^2]$ one has
\begin{equation*}
    \dot{\varphi}^{[L]}(t) \leq \frac{1 + \veps}{(1 + Lt / (3\pi))^2},
\end{equation*}
and for every $t \in [0, \pi]$ one has
\begin{equation*}
    \dot{\varphi}^{[L]}(t) \leq \frac{C''}{(1 + Lt / (3\pi))^2}.
\end{equation*}
	Finally, we have for $\ell > 0$
	\begin{align}
	\xi_{\ell}(t)&\le (1 + e^{6C_0} \frac{\log (1+\ell)}{\ell})\wh{\xi}_{ \ell }(t),
    \label{eq:xi_niceub}
	\end{align}
	and if $L \geq 3$ %
	\begin{equation*}
	\psi(t) \le \wh{\psi}(t) + 4 n e^{6C_0} \log^2 L.
	\end{equation*}
	
\end{lemma}

\begin{proof}
  Fix $L \in \bbN$ arbitrary.
  We prove \Cref{eq:hatphi_phi_bound_1} first, then use it to derive the remaining estimates.
  The main tool is an inductive decomposition: start by writing
  \begin{align*}
    \wh{\varphi}^{[L]}(t)-\varphi^{[L]}(t)
    &= \wh{\varphi}\circ\wh{\varphi}^{[L-1]}(t)-\varphi\circ\varphi^{[L-1]}(t) \\
    &=\wh{\varphi}\circ\wh{\varphi}^{[L-1]}(t)-\wh{\varphi}\circ\varphi^{[L-1]}(t) +
    \wh{\varphi}\circ\varphi^{[L-1]}(t)-\varphi\circ\varphi^{[L-1]}(t),
  \end{align*}
  and then use the definition of $\wh{\varphi}$ to simplify the first term  on the RHS of the
  final equation (via direct algebraic manipulation) to
  \begin{equation*}
    \wh{\varphi}\circ\wh{\varphi}^{[L-1]}(t)-\wh{\varphi}\circ\varphi^{[L-1]}(t)
    =
    \frac{\wh{\varphi}^{[L-1]}(t)-\varphi^{[L-1]}(t)}{\left({1+\frac{1}{3\pi}\varphi^{[L-1]}(t)}\right)\left(1+\frac{1}{3\pi}\wh{\varphi}^{[L-1]}(t)\right)}.
  \end{equation*}
  This gives an expression for the difference 
  $\wh{\varphi}^{[L]}(t)-\varphi^{[L]}(t)$ as an affine function of the previous difference
  $\wh{\varphi}^{[L-1]}(t)-\varphi^{[L-1]}(t)$:
  \begin{equation*}
    \wh{\varphi}^{[L]}(t)-\varphi^{[L]}(t)\\
    =
    \frac{\wh{\varphi}^{[L-1]}(t)-\varphi^{[L-1]}(t)}{\left({1+\frac{1}{3\pi}\varphi^{[L-1]}(t)}\right)\left(1+\frac{1}{3\pi}\wh{\varphi}^{[L-1]}(t)\right)}+\wh{\varphi}\circ\varphi^{[L-1]}(t)-\varphi\circ\varphi^{[L-1]}(t),
  \end{equation*}
  and unraveling inductively, we obtain
  \begin{equation*}
    \wh{\varphi}^{[L]}(t) - \varphi^{[L]}(t)
    =
    \sum_{\ell=0}^{L-1}
    \left(
      \prod_{\ell'=\ell+1}^{L-1}
      \frac{1}{
        \left(1 + \frac{1}{3\pi}\wh{\varphi}^{[\ell']}(t)\right)
        \left (1 + \frac{1}{3\pi}\varphi^{[\ell']}(t) \right)
      }
    \right)
    \left(
      \wh{\varphi} - \varphi
    \right) \circ \ivphi{\ell}(t),
  \end{equation*}
  where for concision we write $(\wh{\varphi} - \varphi)(t) = \wh{\varphi}(t) - \varphi(t)$.
  Note that all the product coefficients in this expression are nonnegative numbers. Denoting by
  $\tilde{C}_1$ the constant attached to $t^3$ in the result \Cref{lemma:phi_lb} and defining
  $C_{1}=\max\left\{ \tilde{C}_{1},1\right\}$, \Cref{lemma:phi_lb} gives
  \begin{equation}
    \wh{\varphi}^{[L]}(t) - \varphi^{[L]}(t)
    \leq
    C_1
    \sum_{\ell=0}^{L-1}
    \left(
      \prod_{\ell'=\ell+1}^{L-1}
      \frac{1}{
        \left(1 + \frac{1}{3\pi}\wh{\varphi}^{[\ell']}(t)\right)
        \left (1 + \frac{1}{3\pi}\varphi^{[\ell']}(t) \right)
      }
    \right)
    \left(
      \varphi^{[\ell]}(t)
    \right)^3.
    \label{eq:deltaL_inductive}
  \end{equation}
  To prove \Cref{eq:hatphi_phi_bound_1}, we will use a two-stage approach:
  \begin{enumerate}
    \item (\textbf{First pass}) First, we will control only the first factor in the product term
      in \Cref{eq:deltaL_inductive} using \Cref{lemma:phi_lb}, given that ${\varphi} \geq 0$
      allows us to upper bound by the product term without the second factor. The resulting bound
      on the LHS of \Cref{eq:deltaL_inductive} will be weaker (in terms of its dependence on $L$)
      than \Cref{eq:hatphi_phi_bound_1}.
    \item (\textbf{Second pass}) After completing this control, we will have obtained a lower
      bound on $\varphi^{[L]}$; we can then return to \Cref{eq:deltaL_inductive} and use this
      lower bound to get control of both factors in the product term, which will allow us to
      sharpen our previous analysis and establish the claimed bound \Cref{eq:hatphi_phi_bound_1}.
  \end{enumerate}

  \subparagraph{First pass.} %
  We have
  \begin{equation*}
    \prod_{\ell'=\ell+1}^{L-1} \frac{1}{\left( 1 + \frac{1}{3\pi} \wh{\varphi}^{[\ell']}(t) \right)}
    =
    \frac{1 + \frac{(\ell+1) t}{3\pi} }{1 + \frac{Lt}{3\pi} }.
  \end{equation*}
  Tossing the product term involving $\varphi^{[\ell']}$ and applying
  \Cref{lem:phi_l_ub} in \Cref{eq:deltaL_inductive}, we thus have a bound
  \begin{equation*}
    \wh{\varphi}^{[L]}(t) - \varphi^{[L]}(t) 
    \leq
    \frac{C_1}{1 + \frac{Lt}{3\pi} }
    \sum_{\ell=0}^{L-1} \frac{t^3}{\left( 1 + \frac{\ell t}{3\pi} \right)^2}
    +
    \frac{C_1 t / (3\pi)}{1 + \frac{Lt}{3\pi} }
    \sum_{\ell=0}^{L-1} \frac{t^3}{\left( 1 + \frac{\ell t}{3\pi} \right)^3}.
  \end{equation*}
  For the first term in this expression, we calculate using an estimate from the integral test
  \begin{align*}
    \sum_{\ell=0}^{L-1} \frac{t^3}{\left( 1 + \frac{\ell t}{3\pi} \right)^2}
    &\leq
    t^3 + 
    \int_0^{L} \frac{t^3}{\left( 1 + \frac{\ell t}{3\pi} \right)^2} \diff \ell \\
    &= t^3 + 3\pi t^2 \left( 1 - \frac{1}{1 + \frac{Lt}{3\pi}} \right) \\
    &= t^3 + \frac{Lt^3}{1 + \frac{Lt}{3\pi}},
  \end{align*}
  and for the second term, we calculate similarly
  \begin{align*}
    \sum_{\ell=0}^{L-1} \frac{t^3}{\left( 1 + \frac{\ell t}{3\pi} \right)^3}
    &\leq
    t^3 + \int_0^L \frac{t^3}{\left(1 + \frac{\ell t}{3\pi} \right)^3} \diff \ell \\
    &=
    t^3 + \frac{3\pi t^2}{2} \left( 1 - \frac{1}{\left( 1 + \frac{Lt}{3\pi} \right)^2 } \right) \\
    &=
    t^3 + L t^3 \frac{ 1 + \frac{Lt }{6\pi} }{\left( 1 + \frac{Lt}{3\pi} \right)^2} \\
    &\leq t^3 + \frac{L t^3}{1 + \frac{Lt}{3\pi}}.
  \end{align*}
  Combining these results gives 
  \begin{align*}
    \wh{\varphi}^{[L]}(t) - \varphi^{[L]}(t) 
    &\leq
    \frac{C_1 t^3}{\left(1 + \frac{Lt}{3\pi} \right)}
    +
    \frac{C_1 Lt^3 }{\left(1 + \frac{Lt}{3\pi} \right)^2}
    +
    \frac{C_1 t^4 / (3\pi)}{ \left(1 + \frac{Lt}{3\pi} \right) }
    +
    \frac{C_1 L t^4/(3\pi) }{ \left(1 + \frac{Lt}{3\pi} \right)^2} \\
    &\leq
    \frac{3\pi C_1 t}{L} \left(3\pi + 2t + \frac{1}{3\pi} t^2 \right).
    \labelthis \label{eq:ub_error_est_smallangle}
  \end{align*}
  This bound gives us a nontrivial estimate as far out as $t = \pi$, but the result is weaker
  there than what we need. We can proceed with a bootstrapping approach to improve our result for
  large angles.  To begin, we have shown via \Cref{eq:ub_error_est_smallangle}
  \begin{equation*}
    \varphi^{[L]}(t) \geq \wh{\varphi}^{[L]}(t) - \frac{16\pi^2 C_1 t}{L}.%
  \end{equation*}
  Let us write $t_0 = C / \sqrt{L}$, where $C>0$ is a constant we will optimize below, and define
  \begin{equation*}
    \widecheck{\varphi}_{L}(t)
    =
    \begin{cases}
      \wh{\varphi}^{[L]}(t) - \frac{16C_1\pi^2 t}{L}  & 0 \leq t \leq t_0 \\
      \wh{\varphi}^{[L]}(t_0) - \frac{16C_1\pi^2 t_0}{L}  & t_0 \leq t \leq \pi.
    \end{cases}
  \end{equation*}
  The notation here is justified by noticing that $\varphi^{[L]}$ is concave and nondecreasing, so
  that
  our previous estimates imply $\varphi^{[L]} \geq \widecheck{\varphi}_{ L }$. It follows
  \begin{equation*}
    \wh{\varphi}^{[L]}- \varphi^{[L]}
    \leq 
    \wh{\varphi}^{[L]} - \widecheck{\varphi}_{ L }.
  \end{equation*}
  Our previous bound \Cref{eq:ub_error_est_smallangle} is an increasing function of $t$, and
  sufficient for $0 \leq t \leq t_0$. For $t \geq t_0$, we have
  \begin{equation*}
    \wh{\varphi}^{[L]}(t) - \widecheck{\varphi}_{ L }(t)
    \leq
    \frac{16 C_1 \pi^2 t_0}{L}
    +
    \wh{\varphi}^{[L]}(t) - \wh{\varphi}^{[L]}(t_0),
  \end{equation*}
  and we can calculate using increasingness of $\wh{\varphi}^{[L]}$
  \begin{align*}
    \wh{\varphi}^{[L]}(t) - \wh{\varphi}^{[L]}(t_0)
    &\leq
    \frac{\pi}{1 + L/3} - \frac{C}{\sqrt{L} + CL/(3\pi)} \\
    &=
    \frac{\pi\sqrt{L} - C}{(1 + L/3)(\sqrt{L} + CL/(3\pi))} \\
    &\leq
    \frac{9\pi^2}{CL^{3/2}},
  \end{align*}
  whence the bound
  \begin{align*}
    \wh{\varphi}^{[L]} - \varphi^{[L]}
    &\leq
    \frac{\pi}{L^{3/2}} \left( 16 \pi C_1 C + \frac{9\pi}{C} \right) \\
    &\leq
    \frac{24\pi^{2}\sqrt{C_{1}}}{L^{3/2}}
    \labelthis \label{eq:deltaL_controlled}
  \end{align*}
  valid on the entire interval $[0, \pi]$; the final inequality corresponds to the choice $C=\frac{3}{4\sqrt{C_{1}}}$.%

  \subparagraph{Second pass.}
  To start, with an eye toward the unused product term in \Cref{eq:deltaL_inductive}, we have from
  \Cref{eq:deltaL_controlled}
  \begin{equation*}
    1 + \frac{1}{3\pi}\varphi^{[L]}(t) \geq 1 + \frac{1}{3\pi}\wh{\varphi}^{[L]}(t) -
    \frac{8\pi\sqrt{C_{1}}}{L^{3/2}}.
  \end{equation*}
  Using the numerical inequality $e^{-2x} \leq 1 - x$, valid for $0 \leq x \leq 1/2$ at least, we
  have if $L\geq\left(256\pi^{2}C_{1}\right)^{1/3}$%
  \begin{equation*}
    1 + \frac{1}{3\pi}\varphi^{[L]}(t) \geq 
    \exp \left( - \frac{16 \pi \sqrt{C_1}}{L^{3/2}} \right) \left(1 +  \frac{1}{3\pi}\wh{\varphi}^{[L]}(t)
    \right).
  \end{equation*}
  Applying this bound to terms in the second product term in \Cref{eq:deltaL_inductive} with index
  $\ell \geq \left\lceil \left(256\pi^{2}C_{1}\right)^{1/3}\right\rceil \equiv r(C_{1})$, we
  therefore have \footnote{Although it has a different meaning in our argument at large, here and
    in some subsequent bounds $\zeta(x) = \sum_{n=1}^{\infty} n^{-x}$ denotes the Riemann zeta
  function. In this setting, we have $\zeta(3/2) \leq e$.}
  \begin{align*}
    \wh{\varphi}^{[L]}(t) - \varphi^{[L]}(t)
    &\leq
    C_1 \sum_{\ell=0}^{L-1} 
    \left( \varphi^{[\ell]}(t) \right)^3 
    \left(
      \prod_{\ell' = \max \set{r(C_1), \ell+1}}^{L-1}
      \frac{1}{
        \left(1 + \frac{1}{3\pi}\wh{\varphi}^{[\ell']}(t)\right)
        \left (1 + \frac{1}{3\pi}\varphi^{[\ell']}(t) \right)
      }
    \right) \\
    &\leq
    C_1 \sum_{\ell=0}^{L-1} 
    \left( \varphi^{[\ell]}(t) \right)^3 
    \exp \left(
      16 \pi \sqrt{C_1} \sum_{\ell' = \max \set{r(C_1), \ell+1}}^{L-1} \frac{1}{(\ell')^{3/2}}
    \right) \\
    &\quad\qquad\times
    \left(
      \prod_{\ell' = \max \set{r(C_1), \ell+1}}^{L-1}
      \frac{1}{
        \left(1 + \frac{1}{3\pi}\wh{\varphi}^{[\ell']}(t)\right)^2
      }
    \right)  \\
    &=
    C_1 e^{16 \pi \sqrt{C_1} \zeta(3/2)} \sum_{\ell=0}^{L-1} 
    \left( \varphi^{[\ell]}(t) \right)^3 
    \left( \frac{1 + \frac{\max \set{r(C_1), \ell+1}t}{3\pi}}{1 + \frac{Lt}{3\pi}} \right)^2 \\
    &=
    \frac{C_1 e^{16 \pi \sqrt{C_1} \zeta(3/2)}}
    {\left( 1 + \frac{Lt}{3\pi} \right)^2} 
    \Biggl(
    \left(1 + \frac{(r(C_1) - 1)t}{3\pi}\right)^2 \sum_{\ell=0}^{r(C_1)-2} \left( \varphi^{[\ell]}(t)
    \right)^3 \\
    &\quad\qquad+
    \sum_{\ell=r(C_1)-1}^{L-1}
    \left( \varphi^{[\ell]}(t) \right)^3 \left( 1 + \frac{(\ell+1)t}{3\pi} \right)^2 \Biggr).
  \end{align*}
  Now, since $\varphi \leq \wh{\varphi}$, we have
  \begin{equation*}
    \left(1 + \frac{(r(C_1) - 1)t}{3\pi}\right)^2 \sum_{\ell=0}^{r(C_1)-2} \left( \varphi^{[\ell]}(t) \right)^3 
    \leq r(C_1)^2 t^3,
  \end{equation*}
  and
  \begin{align*}
    \sum_{\ell=r(C_1)}^{L-1} \left( \varphi^{[\ell]}(t) \right)^3 \left( 1 + \frac{(\ell+1)t}{3\pi} \right)^2
    &\leq
    2t^3 \sum_{\ell = 0}^{L-1} \frac{1}{1 + \frac{\ell t}{3\pi} } \\
    &\leq
    2t^3 + 2t^3 \int_0^L \frac{1}{1 + \frac{\ell t}{3\pi}} \diff \ell \\
    &= 2t^3 + 6\pi t^2 \log( 1 + Lt / 3\pi),
  \end{align*}
  whence
  \begin{align*}
    \wh{\varphi}^{[L]}(t) - \varphi^{[L]}(t)
    &\leq
    \frac{C_1e^{16 \pi \sqrt{C_1} \zeta(3/2)}}{\left( 1 + \frac{Lt}{3\pi} \right)^2} 
    \left( (2+r(C_1)^2) t^3 + 6\pi t^2 \log(1 + Lt / 3\pi) \right) \\
    &\leq
    \frac{ 9 \pi^2 C_1 e^{16 \pi \sqrt{C_1} \zeta(3/2)}}{L^2} \left( (2+r(C_1)^2)t + 6\pi \log(1 + Lt / 3\pi) \right)  \\
    &\leq 54 \pi^3 C_1(2+r(C_1)^2)  e^{16\pi \sqrt{C_1} \zeta(3/2)} \frac{ \log(1+ L)}{L^2}.
  \end{align*}
  In the final line, we are simply shuffling constants using $t \leq \pi$. 
  This completes the proof of \Cref{eq:hatphi_phi_bound_1}.

  \subparagraph{Derived estimates.} The remaining claims can be derived from the main claim we have just established; we will do so now.
  Below, we write $C_0=54 \pi^3 C_1(2+r(C_1)^2) e^{16\pi \sqrt{C_1} \zeta(3/2)}$. %
  We will also assume $\ell \geq 1$.

  We prove the claim about $\xi_{ \ell }$ first. First, notice that for nonnegative
  numbers $a, b$, one has $1 - a + b \leq e^{2b}(1 - a)$ provided $a \leq 1/2$. Since $\varphi \leq
  \pi/2$, we have for each $\ell > 0$
  \begin{align*}
    \xi_{ \ell }(t) 
    &\leq \prod_{\ell'=\ell}^{L-1} \left(1 - \frac{\wh{\varphi}^{[\ell']}(t)}{\pi} +
    \frac{C_0 \log(1 + \ell')}{\pi (\ell')^{2}} \right) \\
    &\leq \exp\left( 2C_0 \sum_{\ell'=\ell}^{L-1} \frac{\log(1 + \ell')}{(\ell')^{2}} \right)
    \wh{\xi}_{ \ell }(t).
  \end{align*}
  By the integral test estimate, we have for $\ell > 0$
  \begin{align*}
    \sum_{\ell'=\ell}^{L-1} \frac{\log(1 + \ell')}{(\ell')^{2}}
    &\leq \frac{\log(1 + \ell)}{\ell^{2}} + \int_\ell^{L} \frac{\log(1+\ell')}{(\ell')^{2}} \diff \ell' \\
    &\leq \frac{\log(1+\ell)}{\ell^{2}} 
    + \log\left(\frac{1 + \frac{1}{\ell}}{1 + \frac{1}{L}} \right)
    + \log(1 + \ell) / \ell - \log(1 + L)/L
    \\
    &\leq \frac{3\log(1 + \ell)}{\ell},
  \end{align*}
  where we applied $\log(1 + x) \leq x$ for all $x > -1$, whence for $\ell > 0$
  \begin{equation*}
    \xi_{ \ell }(t) \leq e^{6C_0 \log(1 + \ell) / \ell} \wh{\xi}_{ \ell }(t).
  \end{equation*}
  In particular, using the fact that $\log(1+\ell)/\ell \leq 1$ and the estimate $e^{cx} \leq 1 + xe^{c}$ for $x \in [0, 1]$ (by convexity of the exponential function), we obtain
  \begin{equation}
    \xi_{ \ell }(t) \leq \left(1 + e^{6 C_0} \frac{\log(1 + \ell)}{\ell}\right) \wh{\xi}_{ \ell
    }(t),
  \end{equation}
  as claimed. The proof of the second inequality is very similar: first, repeated application of
  the chain rule gives
  \begin{equation*}
    \dot{\varphi}^{[L]} = \prod_{\ell=0}^{L-1} \dot{\varphi} \circ \varphi^{[\ell]}.
  \end{equation*}
  Using the expression
  \begin{equation*}
    \dot{\varphi}(t) = \frac{(1 - t/\pi) \sin t}{\sin \varphi(t)},
  \end{equation*}
  we can exploit a telescopic cancellation in the preceding expression for $\dot{\varphi}^{[L]}$,
  obtaining
  \begin{equation*}
    \dot{\varphi}^{[L]} = 
    \frac{\sin t}{\sin \varphi^{[L]}(t)}
    \prod_{\ell=0}^{L-1} \left( 1 - \frac{\varphi^{[\ell]}(t)}{\pi} \right).
  \end{equation*}
  As the form of this upper bound is identical to the one we controlled for $\xi_{ \ell }$, only
  with a different constant factor, we can now apply the first part of that argument to the
  present setting, obtaining a bound
  \begin{equation*}
    \dot{\varphi}^{(L)} \leq
    \frac{\sin t}{\sin \ivphi{L}(t)}
    \exp\left( 
      \sum_{\ell=1}^{L-1} \wh{\varphi}^{[\ell]} - \varphi^{[\ell]}
    \right)
    \exp
    \left(
      -\frac{1}{\pi}\sum_{\ell=0}^{L-1}  \frac{t}{1 + \ell t/(3\pi)}
    \right)
  \end{equation*}
  where in simplifying we also used that $\varphi^{[0]} = \wh{\varphi}^{[0]}$. 
  To proceed, we split the first sum, obtaining for any index $1 \leq \ell_{\star}\leq L-1$
  \begin{align*}
    \sum_{\ell=1}^{L-1} \wh{\varphi}^{[\ell]} - \varphi^{[\ell]}
    &=
    \sum_{\ell=1}^{\ell_{\star}-1} (\wh{\varphi}^{[\ell]} - \varphi^{[\ell]})
    + \sum_{\ell=\ell_{\star}}^{L-1} (\wh{\varphi}^{[\ell]} - \varphi^{[\ell]}) \\
    &\leq
    Ct\sum_{\ell=1}^{\ell_{\star}-1} \frac{1}{\ell}
    + 3C_0 \frac{\log(1 + \ell_{\star})}{\ell_{\star}} \\
    &\leq
    Ct \log(\ell_{\star}) +
    3C_0 \frac{\log(1 + \ell_{\star})}{\ell_{\star}} \\
    &\leq C \log(1 + \ell_{\star}) \left( t + \frac{1}{\ell_{\star}} \right),
  \end{align*}
  where in the second line the bound on the first sum used \Cref{eq:ub_error_est_smallangle}, and
  the second used the estimate we proved in the previous section and the integral test estimate
  above; in the third line we estimated the harmonic series with the integral test; and in the
  fourth line we worst-cased. Next, for any $t \leq 1/\ell_{\star}$, we have by the above
  \begin{equation*}
    \sum_{\ell=1}^{L-1} \wh{\varphi}^{[\ell]} - \varphi^{[\ell]}
    \leq C \log(1 + \ell_{\star}) / \ell_{\star},
  \end{equation*}
  and because the RHS approaches $0$ as $\ell_{\star} \to \infty$, for any $0 < \veps \leq 1$ there is an integer $N(\veps) > 0$ such that for all $\ell_{\star} \geq N$ we have
  \begin{equation*}
    \sum_{\ell=1}^{L-1} \wh{\varphi}^{[\ell]} - \varphi^{[\ell]}
    \leq \log(1 + \veps).
  \end{equation*}
  In particular, obtaining a lower bound for the RHS by concavity of $\log$, it is sufficient to
  take $\ell_{\star} \geq C \veps^{-2}$ for a suitably large absolute constant $C>0$.  To ensure
  there exists such a value of $\ell_{\star}$, it suffices to choose $L \geq C \veps^{-2}$ and
  therefore $t \leq C' \veps^2$.  In particular, plugging this estimate into our previous bound,
  we have shown that for any $\veps > 0$, if $L \geq C' \veps^{-2}$ then for all all $t \leq C
  \veps^2$ we have
  \begin{equation*}
    \dot{\varphi}^{[L]} \leq
    (1 + \veps)
    \frac{\sin t}{\sin \ivphi{L}(t)}
    \exp
    \left(
      -\frac{1}{\pi}\sum_{\ell=0}^{L-1}  \frac{t}{1 + \ell t/(3\pi)}
    \right).
  \end{equation*}
  We then calculate
  by an estimate from the integral test
  \begin{equation*}
    \sum_{\ell=0}^{L-1} \frac{t}{1 +  \ell t / (3\pi)}
    \geq
    \int_0^L \frac{t}{1 + \ell t/(3\pi)} \diff \ell
    =
    3\pi \log(1 + L t / (3\pi) ),
  \end{equation*}
  which establishes under the previous conditions on $L$ and $t$ that
  \begin{equation*}
    \dot{\varphi}^{[L]} \leq 
    \frac{\sin t}{\sin \ivphi{L}(t)}
    \frac{1 + \veps}{(1 + Lt / (3\pi))^3}.
  \end{equation*}
  To conclude, we need to simplify the $\sin$ ratio term. 
  Using \Cref{lem:sin_ratio_estimate}, for any $0 < \veps' \leq 1/2$, we have for $0 \leq t \leq C\veps'$ that
  \begin{equation*}
    \frac{\sin t}{\sin \ivphi{L}(t)}
    \leq (1 + 2\veps') (1 + Lt/(3\pi)),
  \end{equation*}
  which suffices to prove the claim for small $t$ after noting $(1 + 2\veps')(1+\veps) = 1 +
  2\veps' + \veps + 2\veps' \veps$, choosing whichever is smaller, and adjusting the preceding
  conditions on $t$ and $L$ (i.e.\ the absolute constants in the previous bounds may grow/shrink as
  necessary).  To show the claimed bound on the entire interval $[0, \pi]$, we can follow exactly
  the argument above, but instead of partitioning the sum of errors $\wh{\varphi}^{[\ell]} -
  \varphi^{[\ell]}$ as above we simply use bound the sum of errors as in the bound on
  $\wh{\xi}_{ \ell }$ previously to obtain a large constant in the numerator; the $\sin$ ratio is
  controlled in this case using the first conclusion in \Cref{lem:sin_ratio_estimate}, which is
  valid on the whole interval $[0, \pi]$.

  Finally, we obtain the estimate on $\psi$ by calculating using the estimate involving
  $\xi_{ \ell }$ and $\wh{\xi}_{ \ell }$ that we proved earlier.  First, we note that although we
  required $\ell > 0$ above, the fact that $\wh{\varphi}^{[0]} = \varphi^{[0]}$ implies that we
  have an estimate $\xi_{ 0 } \leq (1 + \log(2) e^{6C_0} )\wh{\xi}_{ 0 }$. We therefore have
  \begin{align*}
    \psi(t) = \frac{n}{2}\sum_{\ell=0}^{L-1} \xi_{ \ell }(t) 
    &\leq \frac{n}{2}\sum_{\ell=0}^{L-1} \left( 1 + e^{6C_0} \log \frac{(1 + \ell)}{\ell} \right) \wh{\xi}_{ \ell }(t) \\
    &\leq \wh{\psi}(t) + (n/2)e^{6C_0} \left( \log(2) \wh{\xi}_{ 0 }(t) + \sum_{\ell=1}^{L-1} \frac{\log(1+\ell)}{\ell} \wh{\xi}_{ \ell }(t) \right).
  \end{align*}
  It is easy to see that $\wh{\xi}_{ \ell } \leq 1$. Hence
  \begin{align*}
    \psi(t) 
    &\leq
    \wh{\psi}(t) + (n/2)e^{6C_0} \left( \log(2) + \sum_{\ell=1}^{L-1} \frac{\log(1+\ell)}{\ell} \right) \\
    &\leq
    \wh{\psi}(t) + ne^{6C_0} \left( \log(2) + \sum_{\ell=2}^{L-1} \frac{\log(\ell)}{\ell} \right) \\
    &\leq
    \wh{\psi}(t) + ne^{6C_0} \left( 2\log(2) + \int_{\ell=2}^{L-1} \frac{\log(\ell)}{\ell} \diff \ell \right) \\
    &\leq
    \wh{\psi}(t) + ne^{6C_0} \left( 2\log(2) - (1/2) \log^2 2 + (1/2) \log^2 (L-1)  \right) \\
    &\leq \wh{\psi}(t) + 4n e^{6C_0} \log^2 L,
  \end{align*}
  where the final bound requires $L \geq 3$.

\end{proof}

\begin{lemma}
    \label{lem:sin_ratio_estimate}
    For $\ell \in \bbN_0$, one has for $t \in [0, \pi]$
    \begin{equation*}
    \frac{\sin(t)}{\sin(\varphi^{[\ell]}(t))} 
    \leq
    3(1 + \ell t/(3\pi))
    \end{equation*}
    and there exists an absolute constant $C>0$ such that for any $0 < \veps \leq 1/2$, if $0 \leq
    t \leq C\veps$ one has
    \begin{equation*}
    \frac{\sin(t)}{\sin(\varphi^{[\ell]}(t))} 
    \leq
    (1+2\veps)(1 + \ell t/(3\pi)).
    \end{equation*}
\end{lemma}

\begin{proof}
  We prove the bound on $[0, \pi]$ first.
  Because $t \mapsto t\inv \sin t$ is decreasing on $[0, \pi]$, we apply \Cref{lem:phi_l_lb} to
  get
  \begin{align*}
    \frac{\sin(t)}{\sin(\varphi^{[\ell]}(t))} &\le \frac{t}{\varphi^{[\ell]}(t)} \\
    &\leq \frac{t}{\frac{t}{1+\ell t/(\pi)}} \\
    &= 1+\ell t/(\pi) \le 3(1 + \ell t/(3\pi)).
  \end{align*}

  Now fix $0 < \epsilon \leq 1/2$. We claim that there is an absolute constant $C>0$ such that if
  $t \leq C\epsilon$, we have
  \begin{align*}
    \varphi^{[\ell]}(t) \ge (1-\epsilon)\wh{\varphi}^{[\ell]}(t).
  \end{align*}
  Assuming this claim, we have for $t \le C\epsilon$
  \begin{align*}
    \frac{\sin(t)}{\sin(\varphi^{[\ell]}(t))} &\le \frac{t}{\varphi^{[\ell]}(t)} \\
    &\le \frac{t}{(1-\epsilon)\wh{\varphi}^{[\ell]}(t)} \\
    &\le \frac{1}{1-\epsilon}(1 + \ell t/(3\pi)) \\
    &\le (1+2\epsilon)(1 + \ell t/(3\pi)),
  \end{align*}
  which is enough to conclude after rescaling.
  Now we want to show the claim. Let $C_{0}=\max\left\{ 1,C_{1}\right\} $ where $C_1$ denotes the
  constant on $t^3$ in 
  \Cref{lemma:phi_lb}. We first notice that 
  \begin{align*}
    \varphi(t) &\ge \wh{\varphi}(t) - C_0 t^3 \\
    &= \frac{t}{1+t/(3\pi)} - C_0 t^3 \\
    &= \frac{t}{1+t/(3\pi)} - \frac{t}{1+\pi/(3\pi)} \frac{4}{3} C_0 t^2 \\
    &\ge (1 - 
    \frac{4C_0}{3}t^2
    )\frac{t}{1+t/(3\pi)}.
  \end{align*}
  We are going to proceed with an induction-like approach. Put $\epsilon_1 = 4 C_0 t^2 / 3$, and
  choose $t \leq \sqrt{3/(4C_0)}$ so that $1-\epsilon_1 \geq 0$.
  Supposing that it holds
  $\varphi^{[\ell-1]} \ge (1-\epsilon_{\ell-1})\wh{\varphi}^{[\ell-1]}(t)$ for a positive $\epsilon_{\ell-1}$ such that $1 - \epsilon_{\ell-1} \geq 0$
  (we have shown there is such $\veps_1$ and controlled it), we have
  by some applications of the induction hypothesis, \Cref{lemma:upper_bound}, and the previous small-$t$ estimate (we use below that $t \leq \sqrt{3/(4C_0)}$)
  \begin{align*} \varphi^{[\ell]}(t) &\ge \paren{1 -
    \frac{4C_0}{3}\paren{\varphi^{[\ell-1]}(t)}^2}\wh{\varphi}(\varphi^{[\ell - 1]}(t)) \\ &= \paren{1 -
    \frac{4C_0}{3}\paren{\varphi^{[\ell-1]}(t)}^2} \frac{(1 - \epsilon_{\ell - 1})\frac{t}{1+(\ell
    - 1)t/(3\pi)}}{1 + (3\pi)^{-1}(1 - \epsilon_{\ell - 1})\frac{t}{1+(\ell - 1)t/(3\pi)}} \\
    &\ge \paren{1 - \frac{4C_0}{3}\frac{t^2}{(1+(\ell -1)t/(3\pi))^2}} \frac{(1 - \epsilon_{\ell -
    1})t}{1 + \ell t/(3\pi) - \epsilon_{\ell - 1}t/(3\pi)}\\ &\ge \paren{1 -
    \frac{4C_0}{3}\frac{t^2}{(1+(\ell -1)t/(3\pi))^2}}(1 - \epsilon_{\ell - 1})
    \wh{\varphi}^{[\ell]}(t)\\ &\ge \paren{1 - \frac{4C_0}{3}\frac{t^2}{(1+(\ell -1)t/(3\pi))^2} -
    \epsilon_{\ell - 1}}\wh{\varphi}^{[\ell]}(t).  
  \end{align*}
  This shows that we can take $\epsilon_{\ell} = \epsilon_{\ell-1} + (4C_0/3)t^2 / (1 +
  (\ell-1)t/(3\pi))^2$ as long as this term is not larger than $1$.  Unraveling inductively to
  check, we get
  \begin{align*}
    \epsilon_\ell &= \sum_{\ell' = 0}^{\ell-1} \frac{4C_0}{3}\frac{t^2}{(1+\ell't/(3\pi))^2} \\
    &\le \frac{4C_0}{3}t^2\paren{1 + \int_{\ell' = 0}^{\ell-1}\frac{1}{(1+\ell't/(3\pi))^2}\diff \ell'}\\
    &\le \frac{4C_0}{3}t^2\paren{1 + \frac{\ell - 1}{1+(\ell - 1)t/(3\pi)}}\\
    &\le \frac{4C_0}{3} \paren{\pi + \frac{1}{3\pi}} t \\
    &= \frac{16\pi C_0}{3}t.
  \end{align*}
  In particular, the induction is consistent as long as $t \leq 3/(16\pi C_0)$.
  Note as well that since $C_0 \geq 1$ we have 
  $\sqrt{3/(4C_0)} \geq 3/(16\pi C_0)$.
  Thus by induction, we know that when $0 < \epsilon < 1$
  and $t \le \frac{3C_0 \epsilon}{16\pi}$, we have
  \begin{align*}
    \varphi^{[\ell]}(t) \ge (1 - \epsilon)\wh{\varphi}^{[\ell]}(t)
  \end{align*}
  as claimed. 
\end{proof}

\subsection{Sharp Lower Bound for \texorpdfstring{$\psi$}{the Skeleton}}
\begin{lemma}\label{lem:skel_pi_lb}
There is an absolute constant $C_0 > 0$ such that
\begin{align*}
\psi(\pi) &\le \frac{n(L-1)}{8} + 6 \pi n e^{6C_0} \log^2 L 
\end{align*}
\end{lemma}

\begin{proof}
    Following \Cref{lem:phi_lower_bound} (worsening constants slightly for convenience), we directly have
    \begin{align*}
        \psi(\pi) &\le \wh{\psi}(\pi) + 6 \pi n e^{6C_0} \log^2 L. 
    \end{align*}
    $\wh{\psi}(t)$ has a closed form expression, by notice that
    \begin{align*}
        \wh{\xi}_{\ell}(t) &= \prod_{\ell' = \ell}^{L-1}\paren{1 - \frac{\wh{\varphi}^{[\ell']}(t)}{\pi}} \\
        &= \prod_{\ell' = \ell}^{L-1}\paren{1 - \frac{t/\pi}{1 + \ell' t / (3\pi)}} \\
        &= \prod_{\ell' = \ell}^{L-1} \frac{1 + (\ell' - 3) t / (3\pi)}{1 + \ell' t / (3\pi)} \\
        &= \frac{\paren{1 + (\ell - 3) t / (3\pi)}\paren{1 + (\ell - 2) t / (3\pi)}\paren{1 + (\ell - 1) t / (3\pi)}}{\paren{1 + (L - 3) t / (3\pi)}\paren{1 + (L - 2) t / (3\pi)}\paren{1 + (L - 1) t / (3\pi)}}
        \labelthis \label{eq:wh_xi_expr}
    \end{align*}
    and 
    \begin{align*}
        \wh{\psi}(t) &= \frac{n}{2}\sum_{\ell = 0}^{L-1}\wh{\xi}_{\ell}(t) \\
    &= \frac{n}{2}\frac{\sum_{\ell = 0}^{L-1}(3\pi + (\ell-3)t)(3\pi + (\ell-2)t)(3\pi + (\ell-1)t)}{(3\pi + (L-3)t)(3\pi + (L-2)t)(3\pi + (L-1)t)} \\
    &=\frac{n}{2} \frac{1}{4t}\frac{\sum_{\ell = 0}^{L-1}(3\pi + (\ell-3)t)(3\pi + (\ell-2)t)(3\pi + (\ell-1)t)(3\pi + \ell t) }{(3\pi + (L-3)t)(3\pi + (L-2)t)(3\pi + (L-1)t)} \\
    &\qquad- \frac{\sum_{\ell=0}^{L-1}(3\pi + (\ell-4)t)(3\pi + (\ell-3)t)(3\pi + (\ell-2)t)(3\pi + (\ell-1)t)}{(3\pi + (L-3)t)(3\pi + (L-2)t)(3\pi + (L-1)t)} \\
    &= \frac{n}{8t}\frac{(3\pi + (L-4)t)(3\pi + (L-3)t)(3\pi + (L-2)t)(3\pi + (L-1) t)}{(3\pi + (L-3)t)(3\pi + (L-2)t)(3\pi + (L-1)t)}\\
    &\qquad- \frac{(3\pi - 4t)(3\pi - 3t)(3\pi - 2t)(3\pi - t)}{(3\pi + (L-3)t)(3\pi + (L-2)t)(3\pi + (L-1)t)} \\
    &= \frac{n}{8t}(3\pi + (L-4)t) - \frac{n}{8t}\frac{(3\pi - 4t)(3\pi - 3t)(3\pi - 2t)(3\pi - t)}{(3\pi + (L-3)t)(3\pi + (L-2)t)(3\pi + (L-1)t)} \\
    &= \frac{n(L-4)}{8} + \frac{n}{8t}\paren{3\pi - \frac{(3\pi - 4t)(3\pi - 3t)(3\pi - 2t)(3\pi - t)}{(3\pi + (L-3)t)(3\pi + (L-2)t)(3\pi + (L-1)t)} }
    \labelthis \label{eq:wh_psi_expr}
    \end{align*}
    From the second to the fourth line above, we used a telescopic series cancellation trick to sum. Then we get the claim as 
    \begin{align*}
        \psi(\pi) &\le \wh{\psi}(\pi) + 6 \pi n e^{6C_0} \log^2 L \\
        &= \frac{n(L - 4)}{8} + \frac{3\pi n }{8t} + 6 \pi n e^{6C_0} \log^2 L \\
        &= \frac{n(L-1)}{8} + 6 \pi n e^{6C_0} \log^2 L. 
    \end{align*}
\end{proof}

\begin{lemma}\label{lem:skel_int_lb}
When $L \ge 2$, we have for any $r>0$
    \begin{align*}
\int_{0}^{r} \psi(t) dt \ge \frac{n(L-4)}{8} r + \frac{3\pi n}{8} \log(1 + \frac{L-2}{3\pi}r)
\end{align*}
\end{lemma}
\begin{proof}
  From \Cref{lem:phi_l_ub}, we have $\varphi^{[\ell]}(t) \le \frac{t}{1 + \ell t/(3\pi)}$.
  Thus we get 
  \begin{align*}
    \xi_{ \ell }(t) &= \prod_{\ell' = \ell}^{L-1}(1 - \frac{1}{\pi} \varphi^{[\ell']}(t)) \\
    &\ge \frac{(3\pi + (\ell-3)t)(3\pi + (\ell-2)t)(3\pi + (\ell-1)t)}{(3\pi + (L-3)t)(3\pi +
    (L-2)t)(3\pi + (L-1)t)}.
  \end{align*}
  As a result, we have
  \begin{align*}
    \psi(t) &= \frac{n}{2}\sum_{\ell = 0}^{L-1}\xi_{ \ell }(t) \\
    &\ge \frac{n}{2}\frac{\sum_{\ell = 0}^{L-1}(3\pi + (\ell-3)t)(3\pi + (\ell-2)t)(3\pi + (\ell-1)t)}{(3\pi + (L-3)t)(3\pi + (L-2)t)(3\pi + (L-1)t)} \\
    &=\frac{n}{2} \frac{1}{4t}\frac{\sum_{\ell = 0}^{L-1}(3\pi + (\ell-3)t)(3\pi + (\ell-2)t)(3\pi + (\ell-1)t)(3\pi + \ell t) }{(3\pi + (L-3)t)(3\pi + (L-2)t)(3\pi + (L-1)t)} \\
    &\qquad- \frac{\sum_{\ell=0}^{L-1}(3\pi + (\ell-4)t)(3\pi + (\ell-3)t)(3\pi + (\ell-2)t)(3\pi + (\ell-1)t)}{(3\pi + (L-3)t)(3\pi + (L-2)t)(3\pi + (L-1)t)} \\
    &= \frac{n}{8t}\frac{(3\pi + (L-4)t)(3\pi + (L-3)t)(3\pi + (L-2)t)(3\pi + (L-1) t)}{(3\pi + (L-3)t)(3\pi + (L-2)t)(3\pi + (L-1)t)}\\
    &\qquad- \frac{(3\pi - 4t)(3\pi - 3t)(3\pi - 2t)(3\pi - t)}{(3\pi + (L-3)t)(3\pi + (L-2)t)(3\pi + (L-1)t)} \\
    &= \frac{n}{8t}(3\pi + (L-4)t) - \frac{n}{8t}\frac{(3\pi - 4t)(3\pi - 3t)(3\pi - 2t)(3\pi - t)}{(3\pi + (L-3)t)(3\pi + (L-2)t)(3\pi + (L-1)t)} \\
    &= \frac{n(L-4)}{8} + \frac{n}{8t}\paren{3\pi - \frac{(3\pi - 4t)(3\pi - 3t)(3\pi - 2t)(3\pi - t)}{(3\pi + (L-3)t)(3\pi + (L-2)t)(3\pi + (L-1)t)} }\\
    &\stackrel{t' = \frac{t}{3\pi}}{=} \frac{n(L-4)}{8} + \frac{n}{8 t'}\paren{1 - \frac{(1-4t')(1 - 3t')(1 - 2t')(1- t')}{(1+ (L-3)t')(1 + (L-2)t')(1 + (L-1)t')} } \\
    &\ge \frac{n(L-4)}{8} + \frac{n}{8 t'}\paren{1 - \frac{(1 - 3t')(1 - 2t')(1- t')}{(1 +
    (L-2)t')^3} } \\&\qquad+ \frac{n}{8 t'}\frac{4t'(1 - 3t')(1 - 2t')(1- t')}{(1+ (L-3)t')(1 + (L-2)t')(1 + (L-1)t')} \\
    &\ge  \frac{n(L-4)}{8} + \frac{n}{8 t'}\paren{1 - \frac{1}{(1 + (L-2)t')^3} }\\
    &\stackrel{L' = L-2}{=} \frac{n(L-4)}{8} + \frac{n}{8}\frac{3L' + 3L'^2t' + L'^3 t'^2}{(1 + L't')^3} \\
    &= \frac{n(L-4)}{8} + \frac{n}{8}\paren{\frac{L'}{1 + L't'} + \frac{L'}{(1 + L't')^2} +
    \frac{L'}{(1+L't')^3}} .
  \end{align*}
  In the third and fourth lines above, we used a splitting and cancellation trick to sum similar
  to what we used in \Cref{lem:skel_pi_lb}.  In moving from the seventh to the eighth line, we
  used the inequality $(x-1)(x+1) \leq x^2$ after splitting off a term that can be negative for
  large $t'$.  In moving from the eighth to the ninth line, we used nonnegativity of the third
  summand and upper bounded the numerator of the term in the second summand.  (In both of the
  previous simplifications, we are using that $t' \leq 1/3$.) The remaining simplifications obtain
  a common denominator in the second term and then cancel.  Integrating, we thus find
  \begin{align*}
    \int_{0}^{r} \psi(t) dt &\geq  \frac{n(L-4)}{8} r + \frac{3\pi n}{8}\paren{\log (1 + L' \frac{r}{3\pi}) + \paren{1 - \frac{1}{1 + \frac{L'r}{3\pi}}} + \frac{1}{2} \paren{1 - \frac{1}{(1 + \frac{L'r}{3\pi})^2}}} \\
    &\stackrel{\text{when } L' \ge 0}{\ge} \frac{n(L-4)}{8} r + \frac{3\pi n}{8} \log(1 +
    \frac{L-2}{3\pi}r). 
  \end{align*}
\end{proof}

\begin{lemma}\label{lem:skel_r_lb}
There exists an absolute constant $C>0$ such that when $L \ge 2$, we have for any $r>0$
    \begin{align*}
      \int_{0}^{r} (\psi(t) - \psi(\pi)) dt 
      \ge \frac{3\pi n}{8} \log(1 + \frac{L-2}{3\pi}r) - C n r \log^2 L
\end{align*}
\end{lemma}
\begin{proof}
    Following \Cref{lem:skel_int_lb} and \Cref{lem:skel_pi_lb}, we directly get
\begin{align*}
\int_{0}^r(\psi(t) - \psi(\pi)) dt  &\ge \int_{0}^r \psi(t) dt - \paren{\frac{n(L-1)}{8} + 6 \pi n e^{6C_0} \log^2 L} r \\
&\ge  \frac{3\pi n}{8} \log(1 + \frac{L-2}{3\pi}r) -6 \pi n e^{6C_0}\log^2 L r -\frac{3n}{8} r \label{eq:skeleton_lb}\\
&\ge \frac{3\pi n}{8} \log(1 + \frac{L-2}{3\pi}r) - \paren{6 \pi e^{6C_0} +\frac{3}{8}} n r \log^2 L
\end{align*}

\end{proof}

\subsection{Nearly-Matching Upper Bound}
\label{sec:approx_upper_bound}
\begin{lemma}\label{lem:skel_r_ub}
There exist absolute constants $C,C'>0$ and absolute constants $K, K'>0$ such that for any $0 < \veps \leq 1$, 
if $L \geq K\veps^{-3}$ then for any $0 \leq t \leq K' \veps^3$ one has
    \begin{align*}
        \psi(t) - \psi(\pi) \le (1 + \veps)\paren{1 + \frac{18}{1 + (L-3)t/(3\pi)} + \frac{C\log^2(L)}{L}}\frac{n}{8}\frac{L-3}{1+(L-3)t/(3\pi)} 
    \end{align*}
and for any $0 \leq t \leq \pi$ one has
    \begin{align*}
        \psi(t) - \psi(\pi) \le C'n\frac{L-3}{1+(L-3)t/(3\pi)}. 
    \end{align*}
\end{lemma}

\begin{proof}
	We try to control the DC subtracted skeleton $\psi(t) - \psi(\pi)$ by its derivative $\dot{\psi}(t)$, which would require us to control the derivatives $\dot{\xi}_{\ell}(t)$ and further $\dot{\varphi}^{[\ell]}(t)$. Fix $0 < \veps \le 1/2$. When $L \ge C_0 \veps^{-2}$ for some constant $C_0 > 0$, \Cref{lem:phi_lower_bound} provides sharp bound for $\dot{\varphi}^{[\ell]}(t)$ with
	\begin{align*}
	\dot{\varphi}^{[\ell]}(t) &\le \frac{1 + \veps}{(1 + c\ell t)^2} \qquad t \in [0, C'\eps^2]  \\
	\dot{\varphi}^{[\ell]}(t) &\le \frac{C_1}{(1 + c\ell t)^2} \qquad t \in [0, \pi] 
	\end{align*}
	with absolute constants $C', C_1 > 0$ and $c = 1/(3\pi)$. For notation convenience, define $t_1 = C'\veps^2$ and write
	\begin{equation*}
	M_t = \begin{cases}
	1+\veps & 0 \leq t \leq t_1 \\
	C_1 & \ow. 
	\end{cases}
	\end{equation*}
	We can compactly write the previous two bounds together as
	\begin{equation*}
	\dot{\varphi}^{[\ell]}(t) \le \frac{M_t}{(1 + c\ell t)^2}. 
	\end{equation*}
	This allows us to separate $\psi(t) - \psi(\pi)$ into two components $\psi(t) - \psi(t_1)$ and $\psi(t_1) - \psi(\pi)$, where we get the correct constant $1 + \eps$ in the first component and control the second component by the fact that $\psi$ becomes sharp when $L$ is large, making the difference between $\psi(t_1)$ and $\psi(\pi)$ negligible 
	
	Now, for $\ell \ge 4$, with $c =\frac{1}{3\pi}$, we have 
	\begin{align*}
	\abs{\dot{\xi}_{ \ell }(t)} &= \frac{\xi_{\ell}(t)}{\pi}\sum_{\ell' = \ell}^{L-1}
	\frac{\dot{\varphi}^{[\ell']}}{1 - \varphi^{[\ell']}/\pi} \\
	&\le \frac{\xi_{\ell}(t)}{\pi}\sum_{\ell' = \ell}^{L-1}  \frac{\dot{\varphi}^{[\ell']}}{1 - \frac{t/\pi}{1 + c\ell' t}} \\
	&= \frac{\xi_{\ell}(t)}{\pi}\sum_{\ell' = \ell}^{L-1}  \frac{1 + c\ell' t}{1 + c (\ell' - 3)
		t} \dot{\varphi}^{[\ell']} \\
	&\le \frac{\xi_{\ell}(t)}{\pi}\sum_{\ell' = \ell}^{L-1}  \frac{1 + c\ell' t}{1 + c (\ell' - 3) t} \frac{M_t}{(1+c\ell' t)^2} 
	\end{align*}
	Let $c_{1, \ell} = 1 + e^{6C_0} \frac{\log(1 + \ell)}{\ell}$. \Cref{eq:xi_niceub} and \Cref{eq:wh_xi_expr} provide control for $\xi_{\ell}(t)$ and we have
	\begin{align*}
	\abs{\dot{\xi}_{ \ell }(t)} 
	&\le \frac{c_{1, \ell} }{\pi}\frac{(1+c(\ell-3)t)(1+c(\ell-2)t)(1+c(\ell-1)t)}{(1+c(L-3)t)(1+c(L-2)t)(1+c(L-1)t)}\sum_{\ell' = \ell}^{L-1}  \frac{1 + c\ell' t}{1 + c (\ell' - 3) t} \frac{M_t}{(1+c\ell' t)^2} \\
	&\le \frac{c_{1, \ell}M_t}{\pi}\frac{(1+c(\ell-3)t)(1+c(\ell-2)t)(1+c(\ell-1)t)}{(1+c(L-3)t)(1+c(L-2)t)(1+c(L-1)t)}\int_{\ell' = \ell-1}^{L-1}  \frac{1}{(1 + c (\ell' - 2) t)^2} d\ell'\\
	&\le \frac{c_{1, \ell}M_t}{\pi}\frac{(1+c(\ell-3)t)(1+c(\ell-2)t)(1+c(\ell-1)t)}{(1+c(L-3)t)(1+c(L-2)t)(1+c(L-1)t)} \frac{L-\ell}{(1+c(L-3)t)(1+c(\ell-3)t)}\\
	&\le
	\frac{M_t}{\pi}\frac{(1+c(\ell-2)t)(1+c(\ell-1)t)(L-\ell)}{(1+c(L-3)t)^2(1+c(L-2)t)(1+c(L-1)t)}
	\\ &\qquad+ \frac{C \log(1+\ell)/\ell}{\pi}\frac{L}{(1+c(L-3)t)^2} .
	\end{align*}
	In moving from the fifth to the sixth line, we used that
	$(1 + c(\ell - 3)t)(1 + c(\ell)t) = 1 + c(2\ell - 3)t + c^2(\ell - 3)\ell t^2$ $\ge 1 + c(2\ell
	- 4)t + c^2(\ell^2 - 4\ell + 4) t^2  = (1 + c(\ell - 2)t)^2$ provided $\ell \geq 4$ and
	subsequently the integral test. In the splitting in the last line, we used that $M_t$ is always
	bounded by a (very large) absolute constant, and worst-cased (as this term will be sub-leading
	in $L$).
	
	To control derivatives of $\psi$, we need to control sums of the derivatives above. We will
	derive some further estimates for this purpose. First, we calculate
	\begin{align*}
	&\sum_{\ell = 1}^{L-1} (1+c(\ell-2)t)(1+c(\ell-1)t)(L-\ell) \\
	&\qquad= \sum_{\ell = 1}^{L-1} \paren{(L-\ell) + (L-\ell)(2\ell -3)ct + (L-\ell)(\ell - 1)(\ell - 2)c^2t^2}\\
	&\quad= \frac{L(L-1)}{2} + \frac{L(L-1)(L-\frac{7}{2})}{3}ct + \frac{L(L-1)(L-2)(L-3)}{12}c^2t^2 \\
	&\quad\le \paren{\frac{\frac{L(L-1)}{2} + \frac{L(L-1)(L-\frac{7}{2})}{3}ct}{\frac{L(L-3)}{12}(1+c(L-1)t)(1+c(L-2)t)} + 1}\frac{L(L-3)}{12}(1+c(L-1)t)(1+c(L-2)t)\\
	&\quad\le \paren{\frac{6\frac{L-1}{L-3} + 4 ct(L-1)}{(1+c(L-1)t)(1+c(L-2)t)} +
	1}\frac{L(L-3)}{12}(1+c(L-1)t)(1+c(L-2)t).
	\end{align*}
	If $L \ge 4$, we can simplify a term in the last line of the previous expression as
	\begin{align*}
	\frac{6\frac{L-1}{L-3} + 4 ct(L-1)}{(1+c(L-1)t)(1+c(L-2)t)} &\le \frac{18}{1 + c(L-1)t} ,
	\end{align*}
	and under $L \geq 4$ we also have
	\begin{align*}
	\sum_{\ell = 2}^{L-1} \frac{\log(1 + \ell)}{\ell} &\ltsim \log(L)\int_{\ell =
		1}^{L-1}\frac{1}{\ell} \ltsim\log^2 L.
	\end{align*}
	Applying the upper bound from before and adding some terms to the sum (because all terms are
	nonnegative), we get
	\begin{align}
	\pi\sum_{\ell = 4}^{L-1}\abs*{\dot{\xi}_{ \ell }(t)} \le \paren{1+ \frac{18}{1 +
		c(L-1)t}}\frac{L(L-3)}{12}\frac{M_t}{(1+c(L-3)t)^2} + \frac{C L \log^2 L}{(1+c(L-3)t
		)^2}.\label{eq:xi_dot_sum_ub}
	\end{align}
	From \Cref{lem:xi_l_old_ub}, for $\ell = \{0, 1, 2, 3\}$, we can bound $\xi_{ \ell }(t) \le \frac{1 + \ell t/\pi}{1 + Lt/\pi}$. 
	Using that $\xi_{ \ell }$ is decreasing for all $\ell \geq 0$ and nonnegative, for $t, t' \in [0, \pi]$, $t' \ge t$, 
	we are now able to control the DC subtracted skeleton as
	\begin{align*}
	\psi(t) - \psi(t') &\le \frac{n}{2}\frac{4+(0+1+2+3) t/\pi}{1+Lt/\pi} -\frac{n}{2}\sum_{\ell = 4}^{L-1} \int_{v = t}^{t'} \dot\xi_{ \ell }(v)dv \\
	&\le \frac{2 + 3t/\pi}{1+Lt/\pi} n \\&\quad+ \frac{n}{2\pi}\int_{v = t}^{t'} \paren{ \paren{1+ \frac{18}{1 + c(L-1)\nu}}\frac{M_\nu L(L-3)/12}{(1+c(L-3)\nu)^2} + \frac{C L \log^2 L}{(1+c(L-3)\nu )^2}} dv  \\
	&\le \frac{2 + 3t/\pi}{1 + Lt/\pi} n + \frac{n}{2\pi}\int_{v = t}^{t'}  \frac{L \paren{ (1+ \frac{18}{1 + c(L-1)t})\frac{M_{t'}(L-3)}{12} + C \log^2 L}}{(1+c(L-3)\nu)^2}  dv \\
	&\le \frac{2 + 3t/\pi}{1 + Lt/\pi} n \\&\quad + \frac{n}{2\pi} L \paren{ \paren{1+ \frac{18}{1 + c(L-1)t}}\frac{M_{t'}(L-3)}{12} + C \log^2 L} \int_{v = t}^{\pi} \frac{1}{(1+c(L-3)\nu)^2}  dv. 
	\end{align*}
	From the second to the third line, we use the fact that $M_{\nu}$ is nondecreasing in $\nu$. 
	Thus we have   \begin{align*}
	\psi(t) - \psi(t') &\le \frac{2 + 3t/\pi}{1+Lt/\pi} n + \frac{n}{2\pi}   \frac{L \paren{(1+ \frac{18}{1 + c(L-1)t} ) \frac{M_{t'} (L-3)}{12} + C \log^2(L)}\nu}{1 + c(L-3)\nu} \Big|_{\nu = t}^{\pi} \\
	&\le  \frac{5n}{1+Lt/\pi} + \frac{n}{2}   \frac{L \paren{(1+ \frac{18}{1 + c(L-1)t} ) \frac{M_{t'} (L-3)}{12} + C \log^2(L)}}{(1 + c(L-3)t)(1 + c(L-3)\pi)} \\
	&\le  \frac{5n }{1+Lt/\pi}  + \frac{L-3}{1 + c(L-3)\pi}\frac{M_{t'} n}{24}\frac{L \left(1+ \frac{18}{1 + c(L-1)t} + \frac{12C \log^2(L)}{L-3}\right)}{1+c(L-3)t} \\
	&\le  \frac{n}{8}\frac{L M_{t'}}{1+c(L-3)t} \left(1+ \frac{18}{1 + c(L-1)t} + \frac{12C \log^2 L}{L-3} \right) + \frac{5n}{1+Lt/\pi} \\
	&\le \frac{nM_{t'}}{8}\frac{L-3}{1+c(L-3)t} \paren{1 +  \frac{18}{1 + c(L-1)t} + \frac{C \log^2(L) }{L-3}}\\
	&\le \frac{nM_{t'} }{8}\frac{L-3}{1+c(L-3)t} \paren{1 + \frac{18}{1 + c(L-3)t} + \frac{C \log^2(L) }{L}}. 
	\end{align*}
	In moving from the second to the third line, we simplified/rearranged and used that $C_1 \geq 1$.
	In moving from the fourth to the fifth line, we replace the numerator of $L$ in the leading term with $L - 3 + 3$, then expand and simplify. 
	In particular we have 
	\begin{align*}
	\psi(t_1) - \psi(\pi)  &\le \frac{nC_1 }{8}\frac{1}{c t_1} \paren{19 + \frac{C \log^2(L) }{L}} \\
	&\le \frac{C_1C'(19 + C)}{8c \eps^2}n
	\end{align*}
	and 
	\begin{align*}
	\psi(t) - \psi(t_1)  &\le \frac{n(1 + \eps)}{8}\frac{1}{c t} \paren{1 + \frac{18}{1 + c(L-3)t} + \frac{C \log^2(L) }{L}} \\
	&\le \frac{n(1 + \eps)}{8}\frac{L-3}{1 + c (L-3) t} \paren{1 + \frac{18}{1 + c(L-3)t} + \frac{C \log^2(L) }{L}}. 
	\end{align*}
	This inequality holds for all $t$ because when $t \ge t_1$, the left hand side is negative. 
	 Notice that when $t \le \eps^3/(2C_1C'(19 + C))$ and $L-3 \ge C_1C'(19 + C)/(c\eps^3)$, we would have
	\begin{align*}
	\frac{C_1C'(19 + C)}{8c \eps^2}n &= \eps \frac{n}{8}\frac{L-3}{2c(L-3)\eps^3/(2C_1C'(19 + C))} \\
	&\le  \eps \frac{n}{8}\frac{L-3}{1+c(L-3) \eps^3/(2C_1C'(19 + C))} \\
	&\le \frac{n\eps}{8}\frac{L-3}{1+c(L-3)t}. 
	\end{align*}
	Thus when $L \ge C_1C'(19 + C)/(c\eps^3) + 3$, for $t \le \eps^3/(2C_1C'(19 + C))$, 
	\begin{align*}
	\psi(t) - \psi(\pi) &=\psi(t) - \psi(t_1) + \psi(t_1) - \psi(\pi) \\
	&\le \frac{n(1 + 2\eps) }{8}\frac{L-3}{1+c(L-3)t} \paren{1 + \frac{18}{1 + c(L-3)t} + \frac{C \log^2(L) }{L}}. 
	\end{align*}
	Combining the results and notice that we can absorb the factor of $2$ into constants by defining $\veps' = 2\veps$ would give us the claim. 
	
\end{proof}

\begin{lemma}\label{lem:skel_l1_ub}
There exist absolute constants $C, C' > 0$ and $K, K' > 0$ such that for any $0 < \veps \leq 1$, if $L \geq K \veps^{-3}$, for any $0 \leq a \leq b \leq \pi$, one has
    \begin{align}\label{eq:skel_l1_ub_a_to_b}
        \int_a^{b} (\psi(t) - \psi(\pi)) dt  &\le  Cn \log\paren{\frac{1+(L-3)b/(3\pi)}{1+(L-3)a/(3\pi)}}.
    \end{align}
    And if $r > 0$ satisfies $r \leq K' \veps^3$, one further has
    \begin{align}\label{eq:skel_l1_ub_r_to_b}
        \int_r^{b} (\psi(t) - \psi(\pi)) dt  &\le  
        (1 + \veps)\frac{3\pi n}{8} \log\paren{\frac{1+(L-3)b/(3\pi)}{1+(L-3)r/(3\pi)}}
         + C' n \log(\pi / (K' \veps^3)). 
    \end{align}
\end{lemma}
\begin{proof}
\Cref{eq:skel_l1_ub_a_to_b} follows directly from \Cref{lem:skel_r_ub} and integration. 
To achieve an upper bound for integral from $r$ to $b$, we cut the integral at $t_1 = K' \veps^3$ and apply bounds from \Cref{lem:skel_r_ub} separately. Specifically, set $b' = \min\set{b, t_1}$, from \Cref{lem:skel_r_ub} and \Cref{eq:skel_l1_ub_a_to_b} we would have 
\begin{align*}
    \int_r^{b} (\psi(t) - \psi(\pi)) dt  &=  \int_r^{b'} (\psi(t) - \psi(\pi)) dt + \int_{b'}^{b} (\psi(t) - \psi(\pi)) dt \\
    &\le (1 + \veps)\paren{1 + \frac{C \log^2(L) }{L}}\frac{3\pi n}{8} \log\paren{\frac{1+(L-3)b'/(3\pi)}{1+(L-3)r/(3\pi)}} \\ 
    &\qquad+ \int_r^b \frac{18n(1+ \veps)}{8} \frac{L-3}{\paren{1 + (L-3)t/(3\pi)}^2}dt + C' n \log\paren{\frac{1+(L-3)b/(3\pi)}{1+(L-3)b'/(3\pi)}} \\
    &\le (1 + \veps)\paren{1 + \frac{C \log^2(L) }{L}}\frac{3\pi n}{8} \log(\frac{1+(L-3)b/(3\pi)}{1+(L-3)r/(3\pi)}) + \frac{9n}{2(3\pi)^2} + C' n \log(b / b') \\
    &\le (1 + \veps)\frac{3\pi n}{8} \log\paren{\frac{1+(L-3)b/(3\pi)}{1+(L-3)r/(3\pi)}} + 2\frac{C \log^2(L) }{L}\frac{3\pi n}{8} \log\paren{1+(L-3)\pi/(3\pi)} \\
    &\qquad + \frac{n}{2\pi^2} + C' n \log(\pi / t_1) \\
    &\le (1 + \veps)\frac{3\pi n}{8} \log\paren{\frac{1+(L-3)b/(3\pi)}{1+(L-3)r/(3\pi)}} + \frac{C \log^3(L) }{L} + C' n \log(\pi / \paren{K' \eps^{-3}}). 
\end{align*}
\Cref{eq:skel_l1_ub_r_to_b} then follows by setting $L \ge K \eps^{-3}$ for some $K > 0$. 
\end{proof}

\subsection{Higher Order Derivatives of \texorpdfstring{$\psi$}{the Skeleton}}
\label{sec:skel_deriv_bounds}
\begin{lemma} \label{lem:skel_dot_ub}%
  There exist absolute constants $C, C'$ such that when $L \ge C$, we have
  for any $r \in [0, \pi]$, 
  \begin{align}
    \max_{t \ge r} \abs*{\dot{\psi}(t)} \le \frac{C' n}{r^2}
    \label{eq:skel_dot_ub_pointwise}
  \end{align} 
  and we can control the integration
  \begin{align}
    \int_{t = 0}^{r} t^3 \abs*{\dot{\psi}(t)} dt &\le C' n r^2 
    \label{eq:skel_dot_ub_l1}
  \end{align}
\end{lemma}

\begin{proof}
From \eqref{eq:xi_dot_sum_ub} (we control $M_t \leq C$ for an absolute constant $C>0$ in this
context, so that we do not need to deal with the conditions on $\veps$ that appear there) and
\Cref{lem:xi_dot_l_old_ub}, we have 
\begin{align*}
  \abs*{\dot{\psi}(t)} &\le \frac{n}{2} \sum_{\ell = 0}^{3} \abs*{\dot{\xi}_{\ell}(t)} +
  \frac{Cn}{2} \paren{\paren{1+ \frac{18}{1 + (L-1)t/(3\pi)}}\frac{ \frac{L(L-3)}{12}}{(1+(L-3)t/(3\pi))^2} +
  \frac{L \log^2 L}{(1+(L-3)t/(3\pi))^2}} \\
&\le \frac{n}{2}\frac{12L}{1 + Lt/\pi} + \frac{Cn}{2} \paren{\paren{1+ \frac{18}{1 +
(L-1)t/(3\pi)}}\frac{ \frac{L(L-3)}{12}}{(1+(L-3)t/(3\pi))^2} + \frac{L\log^2 L}{(1+(L-3)t/(3\pi))^2}} \\
&\le \frac{6\pi n}{t} + \frac{Cn}{2}\paren{\frac{L(L-3)}{(L-3)^2 t^2} + \frac{L \log^2 L}{(L-3)^2 t^2}} \\
&\le \frac{Cn}{t^2}. 
\end{align*}
This directly get us \Cref{eq:skel_dot_ub_pointwise} and \Cref{eq:skel_dot_ub_l1}.
\end{proof}

\begin{lemma} \label{lem:skel_ddot_ub}
  There exist absolute constants $C, C'$ such that when $L \ge C$, we have for any $r >
  0$ \begin{align}\label{lem:skel_ddot_ub_pointwise}
    \max_{t \ge r} \abs*{\ddot{\psi}(t)} &\le C' \frac{n}{r^3}
  \end{align} 
  and
  \begin{align*}
    \int_{t = 0}^{\pi}t^6\abs*{\ddot{\psi}(t)}dt &\le C'n.
  \end{align*}
\end{lemma}
\begin{proof}
  Following \Cref{lem:skeleton_iphi_ixi_derivatives,lem:xi_dot_l_old_ub,lem:xiddot_ub},
  we have
\begin{align}\label{eq:xi_dot_separate}
   \dot{\xi}_{\ell} 
   = -\frac{\xi_{1}}{\pi}\indicator{\ell = 0} - \frac{\xi_{\ell}}{\pi}\sum_{\ell' =
   \max\set{1, \ell}}^{L-1}\frac{\dot{\varphi}^{[\ell']}}{1 - \varphi^{[\ell']}/\pi}
\end{align}
and
\begin{align*}
  \abs*{\ddot{\xi}_{\ell}} &= -\frac{\xi_{\ell}}{\pi}\sum_{\ell' = \max\{1, \ell\}}^{L-1}
  \frac{\ddot{\varphi}^{[\ell']}}{1 - \varphi^{[\ell']}/\pi} +
  \frac{\xi_{\ell}}{\pi^2}\sum_{\substack{\ell', \ell'' = \max\{1, \ell\}\\ \ell' \ne
  \ell''}}^{L-1} \frac{\dot{\varphi}^{[\ell']}\dot{\varphi}^{[\ell'']}}{(1 - \varphi^{[\ell']}/\pi)(1 - \varphi^{[\ell'']}/\pi)} \\
  &\quad - \frac{2\xi_{1}}{\pi^2}\indicator{\ell = 0}\sum_{\ell' = 1}^{L-1}
  \frac{\dot{\varphi}^{[\ell']}}{1 - \varphi^{[\ell']}/\pi}\\
  &\le \abs*{ \frac{\xi_{\ell}}{\pi}\sum_{\ell' = \max\{1, \ell\}}^{L-1}
    \frac{\ddot{\varphi}^{[\ell']}}{1 - \varphi^{[\ell']}/\pi}} \\
    &\quad+  \frac{\xi_{\ell
      }}{\pi^2}\paren{\sum_{\ell' = \max\{1, \ell\}}^{L-1}\frac{\dot{\varphi}^{[\ell']}}{1 -
    \varphi^{[\ell']}/\pi}}^2 + \abs*{\frac{2\xi_{1}}{\pi^2}\indicator{\ell = 0}\sum_{\ell' = 1}^{L-1}
  \frac{\dot{\varphi}^{[\ell']}}{1 - \varphi^{[\ell']}/\pi}},
\end{align*}
where the diagonal is added to obtain the upper bound on the second term for the inequality.
We compute $\ddot{\varphi}^{[\ell]}(t)$ as 
\begin{align*}
    \ddot{\varphi}^{[\ell]} = \paren{\dot{\varphi}^{[\ell - 1]}}^2 \ddot{\varphi} \circ
    \varphi^{[\ell - 1]} + \ddot{\varphi}^{[\ell-1]} \dot{\varphi} \circ \varphi^{[\ell - 1]} 
\end{align*}
and thus 
\begin{align*}
  \frac{\ddot{\varphi}^{[\ell]}}{\dot{\varphi}^{[\ell]}} &= \dot{\varphi}^{[\ell - 1]}
  \frac{\ddot{\varphi}}{\dot{\varphi}} \circ \varphi^{[\ell - 1]} +
  \frac{\ddot{\varphi}^{[\ell-1]}}{\dot{\varphi}^{[\ell - 1]}} \\
  &= \sum_{\ell' = 0}^{\ell - 1}\dot{\varphi}^{[\ell']} \frac{\ddot{\varphi}}{\dot{\varphi}}
  \circ \varphi^{[\ell']},
\end{align*}
which gives
\begin{align*}
  \abs*{\ddot{\varphi}^{[\ell]}} = \abs*{\dot{\varphi}^{[\ell]} \sum_{\ell' =
  0}^{\ell-1}\dot{\varphi}^{[\ell']}\frac{\ddot{\varphi}}{\dot{\varphi}}\circ \varphi^{[\ell']} }.
\end{align*}
From \Cref{lemm:ddot_phi_prop}, we have $\abs{\ddot{\varphi}} \le c_1 = 4$ on $t \in [0, \pi]$ and
$\dot{\varphi} \ge c_2 = \frac{1}{2}$ on $[0, \frac{\pi}{2}]$.  As when $\ell > 0$, we have
$\varphi^{[\ell]}(t) \le \frac{\pi}{2}$, we separate the case when $\ell = 0$. From
\Cref{lemma:phi_dot_old_upper_bound} we get 
\begin{align*}
	\sum_{\ell' = 0}^{\infty }\dot{\varphi}^{[\ell']}(t) &\le \frac{C}{t}.
\end{align*}
Using the chain rule to get the expression for
$\dot{\varphi}^{[\ell]}$, and concavity of $\varphi$ to get that $\varphi(t) \geq t/2$, and
decreasingness of $\dot{\varphi}$, we have
\begin{align*}
  \abs*{\ddot{\varphi}^{[\ell]}(t)} &\le \abs*{\ddot{\varphi}(t)}\prod_{\ell' = 1}^{\ell
  -1}\dot{\varphi} \circ \varphi^{[\ell']}(t) + \frac{c_1}{c_2} \dot{\varphi}^{[\ell]}(t)
  \sum_{\ell' = 1}^{\ell - 1} \dot{\varphi}^{[\ell']}(t)\\
  &\le c_1 \prod_{\ell' = 1}^{\ell -1}\dot{\varphi} \circ \varphi^{[\ell']}(t) + \frac{c_1}{c_2}
  \dot{\varphi}^{[\ell]}(t)  \sum_{\ell' = 1}^{\ell - 1} \dot{\varphi}^{[\ell']}(t) \\
  &\le c_1 \prod_{\ell' = 0}^{\ell -2}\dot{\varphi} \circ \varphi^{[\ell']} \circ \varphi(t) +
  \frac{C}{t}\dot{\varphi}^{[\ell]}(t) \\
  &\le c_1 \prod_{\ell' = 0}^{\ell -2}\dot{\varphi} \circ \varphi^{[\ell']} \paren{ t/2} +
  \frac{C}{t}\dot{\varphi}^{[\ell]}(t) \\
  &\le 2 c_1 \prod_{\ell' = 0}^{\ell -1}\dot{\varphi} \circ \varphi^{[\ell']} \paren{t/2} +
  \frac{C}{t}\dot{\varphi}^{[\ell]}(t/2) \\
  &\le 8 \dot{\varphi}^{[\ell]}(t/2) + \frac{C}{t}\dot{\varphi}^{[\ell]}(t/2) \\
  &\le \frac{C}{t}\dot{\varphi}^{[\ell]}(t/2).\labelthis \label{eq:ddot_phi_ub_dot_phi}
\end{align*}
From \Cref{lem:phi_concave}, we know $\xi_{\ell}$ is monotonically decreasing, so
$\xi_{\ell}(t) \le \xi_{\ell}(t/2)$. 
Thus (proceeding from our previous bound)
\begin{align*}
    \abs*{ \frac{\xi_{\ell}(t)}{\pi}\sum_{\ell' = \max\{1, \ell\}}^{L-1}
    \frac{\ddot{\varphi}^{[\ell']}(t)}{1 - \varphi^{[\ell']}(t)/\pi}} 
    &\leq
    \frac{C}{t} 
    \abs*{ \frac{\xi_{\ell}(t/2)}{\pi}\sum_{\ell' = \max\{1, \ell\}}^{L-1} 
    \frac{\dot{\varphi}^{[\ell']}(t/2)}{1 - \varphi^{[\ell']}(t)/\pi}}  \\
    &\leq
    \frac{2C}{t} 
    \abs*{ \frac{\xi_{\ell}(t/2)}{\pi}\sum_{\ell' = \max\{1, \ell\}}^{L-1} 
    \frac{\dot{\varphi}^{[\ell']}(t/2)}{1 - \varphi^{[\ell']}(t/2)/\pi}},
\end{align*}
where in the second line we used that $1/2 \leq 1 - \varphi^{(\ell')}(t) / \pi \leq 1$ and the
$t/2$ term is no smaller. 
Similarly, applying \Cref{lemma:phi_dot_old_upper_bound} again, we get
\begin{align*}
    \abs*{\frac{2\xi_{1}(t)}{\pi^2}\indicator{\ell = 0}\sum_{\ell' = 1}^{L-1} 
    \frac{\dot{\varphi}^{[\ell']}(t)}{1 - \varphi^{[\ell']}(t)/\pi}}
    \leq \frac{C}{t} \frac{\xi_{1}(t)}{\pi}\indicator{\ell = 0}
    \leq \frac{C}{t} \frac{\xi_{1}(t/2)}{\pi}\indicator{\ell = 0}
\end{align*}
and
\begin{align*}
    \frac{\xi_{\ell }(t)}{\pi^2}\paren{\sum_{\ell' = \max\{1,
    \ell\}}^{L-1}\frac{\dot{\varphi}^{[\ell']}(t)}{1 - \varphi^{[\ell']}(t)/\pi}}^2 
    &\leq  \xi_{\ell}(t)\paren{\frac{C}{t}\frac{1}{1 + \ell t/(3\pi)}}^2 \\
    &\leq \frac{C \xi_{\ell}(t)}{t^2} \paren{1 + \ell t/(3\pi)}^{-2}
\end{align*}
Combining all these bounds and applying \Cref{lem:xi_l_old_ub}, we have obtained
\begin{align*}
\abs*{\ddot{\xi}_{\ell}(t)} &\le 
\frac{C}{t} \abs*{\dot{\xi}_{\ell}(t/2)} + \frac{C'}{t^2}\frac{1}{1 + Lt/\pi} \\
&\ltsim \frac{1}{t}\abs*{\dot{\xi}_{\ell}(t/2)} + \frac{1}{Lt^3}.
\end{align*}
Note this holds for all $\ell = 0, \cdots, L-1$, so we directly get $\abs{\ddot{\psi}(t)} \le
\frac{C}{t}\abs{\dot{\psi}(t/2)} + \frac{nL}{2}\frac{C'}{Lt^3}$. Thus from 
\Cref{lem:skel_dot_ub}, there exists constant $C, C_1', C_1''$, when $L \ge C$, we have
\begin{align*}
    \max_{t \ge r} \abs*{\ddot{\psi}(t)} &\le \max_{t \ge r} \paren{\frac{C''}{t}
    \abs*{\dot{\psi}(t/2)} + \frac{nL}{2} \frac{C'}{Lt^3}} \\
    &\le \frac{1}{r} \frac{C_1'}{(r/2)^2}n + C'\frac{n}{r^3},
\end{align*}
which provides the bound for $L^\infty$ control. For $L^1$ control, we have
\begin{align*}
    \int_{t = r}^\pi t^6 \abs*{\ddot{\psi}(t)} dt &\le \int_{t = 0}^\pi t^6  \paren{\frac{C}{t}
    \abs*{\dot{\psi}(t)} + \frac{nL}{2} \frac{C'}{Lt^3}} dt \\
    &\le C \int_{t = r}^\pi t^3 \abs*{\dot{\psi}(t/2)} dt + C' n \\
    &\le C n,
\end{align*}
which finishes the proof. 
\end{proof}

\begin{lemma} \label{lem:skel_dddot_ub}
    There exist absolute constants $C, C'$ such that when $L \ge C$ we have for any $r>0$
\begin{align}\label{lem:skel_dddot_ub_pointwise}
\max_{t \ge r} \abs*{\dddot{\psi}(t)} &\le\frac{C' n}{r^4}
\end{align} 
and
\begin{align*}
\int_{t = 0}^{\pi}t^9\abs*{\dddot{\psi}(t)}dt &\le C'n .
\end{align*}
\end{lemma}

\begin{proof}
We calculate with the chain rule starting from the representation in
\Cref{lem:skeleton_iphi_ixi_derivatives} (and use the triangle inequality)
\begin{align*}
\abs*{\dddot{\xi}_{\ell}} &\le \abs*{\frac{\xi_{\ell}}{\pi}\sum_{\ell' = \max\{1,
\ell\}}^{L-1} \frac{\dddot{\varphi}^{[\ell']}}{1 - \varphi^{[\ell']}/\pi}} + 3
\abs*{\frac{\xi_{\ell}}{\pi^2}\sum_{\substack{\ell', \ell'' = \max\{1, \ell\} \\ \ell' \ne
\ell''}}^{L-1} \frac{\ddot{\varphi}^{[\ell']}\dot{\varphi}^{[\ell'']}}{(1 - \varphi^{[\ell']}/\pi)(1 - \varphi^{[\ell'']}/\pi)}} \\
& \quad + \abs*{\frac{\xi_{\ell}}{\pi^3}\sum_{\substack{\ell', \ell'', \ell''' = \max\{1, \ell\} \\
\ell' \ne \ell'', \ell' \ne \ell''', \ell'' \ne \ell'''}}^{L-1}
\frac{\dot{\varphi}^{[\ell']}\dot{\varphi}^{[\ell'']}\dot{\varphi}^{[\ell''']}}{(1 -
\varphi^{[\ell']}/\pi)(1 - \varphi^{[\ell'']}/\pi)(1 - \varphi^{[\ell''']}/\pi)}}\\
&\quad + \abs*{\frac{\xi_{1}}{\pi^3}\indicator{\ell = 0}\sum_{\substack{\ell', \ell'' = 1 \\ \ell'
\ne \ell''}}^{L-1} \frac{\dot{\varphi}^{[\ell']}\dot{\varphi}^{[\ell'']}}{(1 -
\varphi^{[\ell']}/\pi)(1 - \varphi^{[\ell'']}/\pi)}} + \abs*{\frac{3\xi_{1}}{\pi^2}\indicator{\ell
= 0}\sum_{\ell' = 1}^{L-1} \frac{\ddot{\varphi}^{[\ell']}}{1 - \varphi^{[\ell']}/\pi}}\\
&\le \abs*{ \frac{\xi_{\ell}}{\pi}\sum_{\ell' = \max\{1, \ell\}}^{L-1}
\frac{\dddot{\varphi}^{[\ell']}}{1 - \varphi^{[\ell']}/\pi}} \\
&\quad+  \frac{3\xi_{\ell }}{\pi^2} \sum_{\ell' = \max\{1,
\ell\}}^{L-1}\frac{\abs*{\ddot{\varphi}^{[\ell']}}}{1 - \varphi^{[\ell']}/\pi}\sum_{\ell'' =
\max\{1, \ell\}}^{L-1}\frac{\dot{\varphi}^{[\ell'']}}{1 - \varphi^{[\ell'']}/\pi} \\&\quad+
\frac{\xi_\ell}{\pi^3}\paren{\sum_{\ell' = \max\{1, \ell\}}^{L-1}\frac{\dot{\varphi}^{[\ell']}}{1
- \varphi^{[\ell']}/\pi}}^3 \\
&\quad + \frac{\xi_{1}}{\pi^3}\indicator{\ell = 0}\sum_{\ell' =
1}^{L-1}\frac{\abs*{\dot{\varphi}^{[\ell']}}}{1 - \varphi^{[\ell']}/\pi}\sum_{\ell'' =
1}^{L-1}\frac{\dot{\varphi}^{[\ell'']}}{1 - \varphi^{[\ell'']}/\pi} +
\abs*{\frac{3\xi_{1}}{\pi^2}\indicator{\ell = 0}\sum_{\ell' = 1}^{L-1}
\frac{\ddot{\varphi}^{[\ell']}}{1 - \varphi^{[\ell']}/\pi}}.
\labelthis\label{eq:dddot_phi_expression}
\end{align*}
Following \Cref{eq:ddot_phi_ub_dot_phi} and \Cref{lemma:phi_dot_old_upper_bound}, we have
\begin{equation*}
  \sum_{\ell' = 1}^{L-1}\abs*{\ddot{\varphi}^{[\ell']}(t)}
  \le \sum_{\ell' = 1}^{L-1}\frac{\abs*{\ddot{\varphi}^{[\ell']}(t)}}{1 - \varphi^{[\ell']}(t)/\pi} 
  \le \frac{2C}{t} \sum_{\ell' = 1}^{L-1} \abs*{\dot{\varphi}(t/2)}
  \le C / t^2, 
  \labelthis \label{eq:ddot_phi_ub_sum}
\end{equation*}
which leaves the main unresolved term in \Cref{eq:dddot_phi_expression} to be $\dddot{\varphi}$. 
On the other hand, we have from the chain rule
\begin{align*}
  \dddot{\varphi}^{[\ell]} 
  = 3 \dot{\varphi}^{[\ell - 1]} \ddot{\varphi}^{[\ell - 1]} (\ddot{\varphi} \circ \varphi^{[\ell
  - 1]}) 
  + \paren{\dot{\varphi}^{[\ell - 1]}}^3 \paren{\dddot{\varphi} \circ \varphi^{[\ell - 1]}} 
  + \dddot{\varphi}^{[\ell - 1]} \paren{\dot{\varphi} \circ \varphi^{[\ell - 1]}}.
\end{align*}
Using the product expression $\ivphid{\ell} = \ivphid{\ell - 1}\dot{\varphi}\circ\ivphi{\ell - 1}$ and the triangle inequality, we have
\begin{align*}
    \abs*{\frac{\dddot{\varphi}^{[\ell]}}{\dot{\varphi}^{[\ell]}}} 
    &\le 3 \abs*{\ddot{\varphi}^{[\ell - 1]}\frac{\ddot{\varphi}}{\dot{\varphi}} \circ
    \varphi^{[\ell - 1]}} 
    + \paren{\dot{\varphi}^{[\ell - 1]}}^2 \abs*{\frac{\dddot{\varphi}}{\dot{\varphi}} \circ \varphi^{[\ell - 1]}} + \abs*{\frac{\dddot{\varphi}^{[\ell - 1]}}{\dot{\varphi}^{[\ell - 1]}}} \\
    &\le \sum_{\ell' = 1}^{\ell - 1} \paren{3 \abs*{\ddot{\varphi}^{[\ell']}\frac{\ddot{\varphi}}{\dot{\varphi}} \circ \varphi^{[\ell']}} + \paren{\dot{\varphi}^{[\ell']}}^2 \abs*{\frac{\dddot{\varphi}}{\dot{\varphi}} \circ \varphi^{[\ell']}}} + \abs*{\frac{\dddot{\varphi}}{\dot{\varphi}}}
\end{align*}
where the second line uses induction.
From \Cref{lemm:ddot_phi_prop,phi_third_der_bound}, we have $\abs{\ddot{\varphi}} \le c_1 = 4$,
$\abs{\dddot{\varphi}} \le c_4$ on $t \in [0, \pi]$ and $\dot{\varphi} \ge c_2 = \frac{1}{2}$,
$\ddot{\varphi} \le -c_3$ on $[0, \frac{\pi}{2}]$, Again, for $\ell > 0$, we have
$\varphi^{[\ell]}(t) \le \frac{\pi}{2}$. Applying \Cref{eq:ddot_phi_ub_sum} and \Cref{lemma:phi_dot_old_upper_bound}, we
get
\begin{align*}
    \abs*{\frac{\dddot{\varphi}^{[\ell]}\paren{t}}{\dot{\varphi}^{[\ell]}(t)}} &\le 3 \frac{c_1}{c_2}
    \sum_{\ell' = 1}^L \abs*{\ddot{\varphi}^{[\ell']}} + \frac{c_4}{c_2}\sum_{\ell' = 1}^L
    \paren{\dot{\varphi}^{[\ell']}(t)}^2 + \frac{c_4}{\dot{\varphi}(t)} \\
    &\le C/t^2 + \frac{c_4}{\dot{\varphi}(t)}.
\end{align*}
Multiplying both side with $\dot{\varphi}^{[\ell]}$, we get the bound
\begin{align*}
    \abs*{\dddot{\varphi}^{[\ell]}(t)} &\le \frac{C}{t^2}\dot{\varphi}^{[\ell]}(t) + c_4
    \prod_{\ell' = 1}^{\ell -1}\dot{\varphi} \circ \varphi^{[\ell']}(t) \\
    &\le \frac{C}{t^2}\dot{\varphi}^{[\ell]}(t) + c_4 \prod_{\ell' = 0}^{\ell -2}\dot{\varphi}
    \circ \varphi^{[\ell']} \circ \varphi(t) \\
    &\le \frac{C}{t^2}\dot{\varphi}^{[\ell]}(t) + 2 c_4 \dot{\varphi}^{[\ell]}(t/2) \\
    &\le \frac{C}{t^2}\dot{\varphi}^{[\ell]}(t/2),
\end{align*}
where the justifications for this argument are very similar to those used in the proof of
\Cref{lemm:ddot_phi_prop}.

Plugging bounds we have here back to \Cref{eq:dddot_phi_expression}. From \Cref{lemma:phi_dot_old_upper_bound} and monotonicity of $\xi_{\ell}$ in \Cref{lem:phi_concave}, we get
\begin{align*}
  \abs*{\dddot{\xi}_{\ell}(t)} 
&\le \abs*{ \frac{\xi_{\ell}(t)}{\pi}\sum_{\ell' = \max\{1,
      \ell\}}^{L-1} \frac{C\dot{\varphi}^{[\ell]}(t/2)/t^2}{1 - \varphi^{[\ell']}(t)/\pi}} \\
      &\quad +  \frac{3\xi_{\ell
    }(t)}{\pi^2} \frac{C}{t^2}\sum_{\ell' = \max\{1, \ell\}}^{L-1}\frac{\dot{\varphi}^{[\ell']}(t)}{1 -
    \varphi^{[\ell']}(t)/\pi} + \frac{\xi_\ell(t)}{\pi^3}\paren{\frac{C}{t}}^2\sum_{\ell' = \max\{1,
  \ell\}}^{L-1}\frac{\dot{\varphi}^{[\ell']}(t)}{1 - \varphi^{[\ell']}(t)/\pi} \\
  &\quad + \frac{\xi_{1}(t)}{\pi^3}\indicator{\ell = 0}\frac{C}{t} \sum_{\ell' = 1}^{L-1}
  \frac{\dot{\varphi}^{[\ell']}(t)}{1 - \varphi^{[\ell']}(t)/\pi} + \abs*{\frac{3\xi_{1}(t)}{\pi^2}
    \indicator{\ell = 0}\sum_{\ell' = 1}^{L-1} \frac{\ddot{\varphi}^{[\ell']}(t)}{1 -
  \varphi^{[\ell']}(t)/\pi}} \\ 
  &\le \frac{C}{t^2}\abs*{\frac{\xi_{\ell}(t)}{\pi} \sum_{\ell' = \max\{1,
      \ell\}}^{L-1} \frac{\dddot{\varphi}^{[\ell']}(t/2)}{1 - \varphi^{[\ell']}(t)/\pi}} + \frac{C\xi_1(t)}{t^2}\indicator{\ell = 0} \\
  &\le \frac{C}{t^2}\abs*{\frac{\xi_{\ell}(t/2)}{\pi} \sum_{\ell' = \max\{1,
      \ell\}}^{L-1} \frac{\dddot{\varphi}^{[\ell']}(t/2)}{1 - \varphi^{[\ell']}(t/2)/\pi}} + \frac{C\xi_1(t/2)}{t^2}\indicator{\ell = 0} \\
  &= \frac{C}{t^2}\abs*{\dot{\xi}_{\ell}(t/2)}
\end{align*}
where from the second to third line we also use the fact that $1/2 \le 1 - \varphi^{[\ell]}(t)/\pi \le 1$ for all $\ell \ge 1$ and the last line follows from the formula of $\dot{\xi}_{\ell}$ in \Cref{eq:xi_dot_separate}. 

From our bounds of $\dot{\psi}(t)$ in \Cref{eq:skel_dot_ub_pointwise}, this leads to 
\begin{align*}
    \max_{t \ge r}\abs*{\dddot{\psi}(t)} &\le \max_{t \ge r} \frac{C}{t^2}\abs*{\dot{\psi}(t/2)} \\
    &\le \frac{C}{r^2} \frac{C'}{(r/2)^2} \\
    &= \frac{C n}{r^4}
\end{align*}
and 
\begin{align*}
    \int_{t = 0}^\pi t^9 \abs*{\dddot{\psi}(t)} dt &\le \int_{t = 0}^\pi t^9 \frac{C}{t^2}\abs*{\dot{\psi}(t/2)} dt \\
    &\le Cn,
\end{align*}
as claimed.

\end{proof}

\subsection{Additional Proofs for Some Bounds}
\begin{lemma}\label{lemma:phi_lb}
  There exists an absolute constant $C_1 > 0$ such that 
	\begin{align*}
    \wh{\varphi}(t) - \varphi(t)
		\le C_1t^3.
	\end{align*}
\end{lemma}
\begin{proof}
  From \Cref{lem:aevo_properties}, $\varphi$ is 3 times continuously differentiable on $(0,\pi)$,
  and 
\begin{equation*}
    \varphi(0)=0,\quad \dot{\varphi}(0)=1, \quad \ddot{\varphi}(0)=-\frac{2}{3\pi}.
\end{equation*}
It is easy to check that 
\begin{equation*}
    \widehat{\varphi}(0)=0,\quad\dot{\widehat{\varphi}}(0)=1,\quad\ddot{\widehat{\varphi}}(0)=-\frac{2}{3\pi}.
\end{equation*}
Since the Taylor expansions of these two functions around $0$ agree to third order, and both are 3 times continuously differentiable on $(0,\pi)$, we obtain by Lagrange's remainder theorem that for any $t\in[0,\pi)$, 
\begin{equation*}
    \wh{\varphi}(t)-\varphi(t)=\underset{0}{\overset{t}{\int}}\left(\dddot{\wh{\varphi}}(s)-\dddot{\varphi}(s)\right)\frac{s^{2}}{2}\mathrm{d}s\leq C_{1}t^{3}
\end{equation*}
for some finite constant $C_{1}=\underset{t\in[0,\pi)}{\sup}\left|\dddot{\wh{\varphi}}(t)-\dddot{\varphi}(t)\right|$. At $t=\pi$ we have $\widehat{\varphi}(\pi)-\varphi(\pi)=\frac{\pi}{1+\pi/3}-\frac{\pi}{2}\leq0$ hence the same bound holds for $t\in[0,\pi]$.
\end{proof}

\begin{lemma}
	\label{lemm:ddot_phi_prop}
    One has	
    \begin{align*}
      \dot{\varphi}(t) &\ge \frac{1}{2}, \quad  t \in [0, \frac{\pi}{2}] \\
      \abs*{\ddot{\varphi}(t)} & \le 4, \quad  t \in [0, \pi] 
    \end{align*}
\end{lemma}

\begin{proof}
  We know $\varphi$ is monotonically increasing and concave on $[0, \pi]$, thus for $t \in [0,
  \pi/2]$, 
	\begin{align*}
	\dot{\varphi}(t) &\ge \dot{\varphi}(\frac{\pi}{2}) \\
	&= \frac{1/2}{\sin(\varphi(\pi/2))} \\
	&\ge \frac{1}{2}.
	\end{align*}
	Using \Cref{lem:phi_l_lb}
	we also have for $t \in [0, \pi]$, $\dot{\varphi}(t) \le \dot{\varphi}(0) = 1$, 
	\begin{align*}
	\varphi(t) &\ge \frac{t}{ 1 + t/\pi} \ge \frac{t}{2},%
	\end{align*}
	and the first bound here can be used to obtain
	\begin{equation*}
	t - \varphi(t) \le \frac{t^2/\pi}{1 + t/\pi} \le t^2/\pi  .
	\end{equation*}
	Thus since $\varphi \leq \pi/2$
	\begin{align*}
	    \cos t \sin \varphi(t) - \dot{\varphi}(t) \sin t \cos \varphi(t) & \ge \cos t \sin \varphi(t) - \sin t \cos \varphi(t) \\
	    &\ge -\sin(t -\varphi(t)),
	\end{align*}
  and in particular, using the expression for $\ddot{\varphi}$ from \Cref{lem:aevo_properties}
  \begin{align*}
	-\ddot{\varphi}(t) &= -(1 - \frac{t}{\pi})\frac{\cos t \sin \varphi(t) - \dot{\varphi}(t) \sin t \cos \varphi(t)}{\sin^2 \varphi(t)} + \frac{\sin t}{\pi \sin(\varphi(t))} \\
	&\le (1 - \frac{t}{\pi}) \frac{\sin(t -\varphi(t))}{\sin^2 \varphi(t)}+\frac{2}{\pi} \\
	&\le \frac{t^2/\pi}{\sin^2(t/2)}+\frac{2}{\pi} \\
	&\le \frac{t^2/\pi}{(t/\pi)^2} + \frac{2}{\pi} \\
	&\le 4.
	\end{align*}
\end{proof}

\begin{lemma}\label{phi_third_der_bound}
  There exist constants $c_3, c_4 > 0$ such that $\ddot{\varphi}(t) < -c_3$ for $t \in [0,
  \frac{\pi}{2}]$ and $\abs{\dddot{\varphi}} \le c_4$ for $t \in [0, \pi]$.
\end{lemma}
\begin{proof}
  The existence of $c_3$ follows from \Cref{lem:aevo_properties} directly. The existence
  of $c_4$ follows from smoothness of $\varphi$ on $(0, \pi)$ and the fact that $\dddot{\varphi}(0) =
  -\frac{1}{3\pi^2}$, $\dddot{\varphi}(\pi) = \frac{2}{\pi}$ both exist. 
\end{proof}

\begin{lemma}
	\label{lemma:phi_dot_old_upper_bound}
	There exists an absolute constant $C>0$ such that for any $0 < t \leq \pi$ and $\ell \in \N_0$, one has
	\begin{align*}
	\sum_{\ell' = \ell}^{\infty }\dot{\varphi}^{[\ell']}(t) &\le \frac{C}{t}\frac{1}{1 + \ell t/(3\pi)}
	\end{align*}
\end{lemma}

\begin{proof}
Using \Cref{lem:phi_lower_bound}, we have
\begin{equation*}
    \dot{\varphi}^{[\ell]}(t) \leq \frac{C}{( 1 + \ell t/ (3\pi))^2}.
\end{equation*}
We can then calculate
\begin{align*}
  \sum_{\ell'=\ell}^{\infty} \dot{\varphi}^{[\ell']}(t) \leq 
  C \sum_{\ell'=\ell}^{\infty} \frac{1}{( 1 + \ell' t/ (3\pi))^2}
  &\leq C \left( \frac{1}{1 + \ell t/(3\pi)} + \int_{\ell' = \ell}^\infty \frac{1}{( 1 + \ell' t / (3\pi) )^2} \diff \ell' \right) \\
  &\leq C \left( \frac{1}{1 + \ell t/(3\pi)} + \frac{3\pi / t}{1 + \ell t/(3\pi)} \right) \\
  &\le \frac{C}{t} \frac{1}{1 + \ell t/(3\pi)},
\end{align*}
as claimed.

\end{proof}

\section{Auxiliary Results}
\label{sec:aux}
Results in this section are reproduced from the literature for self-containedness, and for the
most part are presented without proofs.

\subsection{Certificates Imply Generalization}

\begin{theorem}[{\cite[Theorem B.1]{Buchanan2020-er}}, specialized slightly]
  \label{thm:1}
  Let $\manifold$ be a two curve problem instance. For any $0 < \delta \leq 1/e$, choose $L$ so
  that
  \begin{equation*}
    L \geq C_1 \max \set*{
      C_{\mu} \log^9(1/\delta) 
      \log^{24} \left( C_{\mu}\adim \log(1 / \delta) \right),
      \kappa^2 C_{\lambda}
    },
  \end{equation*}
  let $N \geq L^{10}$, set $n = C_2 L^{99} \log^9 (1/\delta) \log^{18} (L\adim)$, and fix $\tau >
  0$ such that
  \begin{equation*}
    \frac{C_3}{nL^2} \leq \tau \leq \frac{C_4}{nL}.
  \end{equation*}
  Then if there exists a function $g \in \Lmuinfm$ such that
  \begin{equation}
    \norm*{
      \ntkoper_{\mu}^{\mathrm{NTK}}[g] - \prederr_0
    }_{\Lmuinfm}
    \leq C_5 
    \frac{\sqrt{\log(1/\delta) \log(n\adim)}}{L \min \set*{\rho_{\min}^{\qcert},
    \rho_{\min}^{-\qcert}}};
    \enspace
    \norm*{g}_{\Lmuinfm}
    \leq C_6
    \frac{\sqrt{\log(1/\delta) \log(n\adim)}}{n\rho_{\min}^{\qcert}},
    \label{eq:mainthm_cert_condition}
  \end{equation}
  with probability at least $1 - \delta$ over the random initialization of the network and
  the i.i.d.\ sample from $\mu$, the parameters obtained at
  iteration $\floor{L^{39/44} / (n\tau)}$ of gradient descent on the finite sample loss
  $\loss_{\mu^N}$ yield a classifier that separates the two manifolds.

  The constants $C_1, \dots, C_4 > 0$ depend only on the constants $\qcert, C_5, C_6 > 0$,
  the constants $\kappa$, $C_{\lambda}$ are respectively
  the extrinsic curvature constant and the global regularity constant defined in
  \cite[\S 2.1]{Buchanan2020-er}, and the constant
  $C_{\mu}$ is defined as $ \max \set{\rho_{\min}^{q}, \rho_{\min}^{-q}} %
  (1 + \rho_{\max})^{6}  \left( \min\,\set*{\mu(\mplus), \mu(\mminus)}
  \right)^{-11/2}$, where $q = 11+8\qcert$. 
\end{theorem}

\subsection{Concentration of the Initial Random Network and Its Gradients}

\begin{theorem}[{Corollary of \cite[Theorem B.2, Lemma C.11]{Buchanan2020-er}}]
  \label{thm:2}
  Let $\manifold$ be a two curve problem instance. For any $d \geq K\log (n\adim
  \len(\manifold))$, if $n \geq K' d^4 L^5$ then one has on an event of probability at least $1 -
  e^{-cd}$
  \begin{equation*}
    \norm{\ntk - \ntk^{\mathrm{NTK}}}_{L^\infty(\manifold \times \manifold)}
    \leq Cn/L ,%
  \end{equation*}
  where $c, C, K, K' > 0$ are absolute constants. 
\end{theorem}

\begin{lemma}[{\cite[Lemma D.11]{Buchanan2020-er}}] \label{lem:zeta_approx_bound}
  There are absolute constants $K, K'>0$ such that if $d \geq K \log(n\adim \len(\manifold))$ and
  $n \geq K' d^4 L$, then
  \begin{equation*}
    \mathbb{P}\left[\left\Vert f_{\vtheta_0}\right\Vert_{L^\infty}\leq\sqrt{d}\right]\geq1-e^{-cd},
  \end{equation*}
  \begin{equation*}
    \mathbb{P}\left[\left\Vert \zeta_0\right\Vert_{L^\infty}\leq\sqrt{d}\right]\geq1-e^{-cd}.
  \end{equation*}

  Define
  \begin{equation*}
    \zeta(\boldsymbol{x})=-f_{\star}(\boldsymbol{x})+\underset{\mathcal{M}}{\int}f_{\vtheta_0}(\boldsymbol{x}')\mathrm{d}\mu(\boldsymbol{x}').
  \end{equation*}
  Then under the same assumptions
  \begin{equation*}
    \mathbb{P}\left[\left\Vert \zeta_0-\zeta\right\Vert_{L^\infty}\leq\sqrt{\frac{d}{L^{2}}+d^{5/2}\sqrt{\frac{L}{n}}}\right]\geq1-e^{-cd}
  \end{equation*}
  for some numerical constant $c$.
\end{lemma}

\subsection{Basic Estimates for the Infinite-Width Neural Tangent Kernel}

\begin{lemma}[{\cite[Lemma E.5]{Buchanan2020-er}}]
  \label{lem:aevo_properties}
  One has
  \begin{enumerate}
    \item $\varphi \in C^\infty(0, \pi)$, and $\dot{\varphi}$, $\ddot{\varphi}$, and
      $\dddot{\varphi}$ extend to continuous functions on $[0, \pi]$;
    \item $\varphi(0) = 0$ and $\varphi(\pi) = \pi/2$; $\dot{\varphi}(0) = 1$,
      $\ddot{\varphi}(0) = -2/(3\pi)$, and $\dddot{\varphi}(0) = -1/(3\pi^2)$; and
      $\dot{\varphi}(\pi) = \ddot{\varphi}(\pi) = 0$;
    \item $\varphi$ is concave and strictly increasing on $[0, \pi]$ (strictly concave
      in the interior);
    \item $\ddot{\varphi} < -c < 0$ for an absolute constant $c > 0$ on $[0, \pi/2]$;
    \item $0 < \dot{\varphi} < 1$  and $0 > \ddot{\varphi} \geq -C$ on $(0, \pi)$ for some
      absolute constant $C > 0$;
    \item $\nu(1 - C_1 \nu) \leq \varphi(\nu) \leq \nu(1 - c_1 \nu)$ on $[0, \pi]$ for some
      absolute constants $C_1, c_1 > 0$.
  \end{enumerate}
\end{lemma}
\begin{proof}
  Combine the results in {\cite[Lemma E.5]{Buchanan2020-er}} with \Cref{phi_third_der_bound} to
  obtain the conclusion.
\end{proof}

\begin{lemma}[Corollaries of {\Cref{lem:aevo_properties}, stated in \cite[Lemma
  C.10]{Buchanan2020-er}}]\label{lem:phi_concave}%
  One has:
  \begin{enumerate}
    \item The function $\varphi$ is smooth on $(0, \pi)$, and (at least) $C^3$ on $[0, \pi]$. 
    \item For each $\ell = 0, 1, \cdots, L$, the functions $\varphi^{[\ell]}$ are nonnegative,
      strictly increasing, and concave (positive and strictly concave on $(0, \pi)$). 
    \item If $0 \le \ell < L$, the functions $\xi_{\ell}$ are nonnegative, strictly decreasing,
      and convex (positive and strictly convex on  $(0, \pi)$). 
    \item The function $\psi$ is smooth on $(0, \pi)$, $C^3$ on $[0, \pi]$, and is nonnegative,
      strictly decreasing, and convex.
  \end{enumerate}
\end{lemma}

\begin{lemma}[{\cite[Lemma C.13]{Buchanan2020-er}}]%
\label{lem:phi_l_lb}
    If $\ell \in \N_0$, the iterated angle evolution function satisfies the estimate
    \begin{equation*}
        \varphi^{[\ell]}(t) \ge \frac{t}{1 + \ell t/\pi},
    \end{equation*}
\end{lemma}

\begin{lemma}[{\cite[Lemma C.17]{Buchanan2020-er}}]%
\label{lem:iterphi_partic_vals}
    One has for every $\ell \in \{0, 1, \cdots, L\}$
    \begin{align*}
      \varphi^{[\ell]}(0) = 0; 
      \quad \dot{\varphi}^{[\ell]}(0) = 1; 
      \quad \ddot{\varphi}^{[\ell]}(0) = -\frac{2\ell}{3\pi}, 
    \end{align*}
    and for $\ell \in [L]$, 
    \begin{align*}
      \dot{\varphi}^{[\ell]}(\pi) = \ddot{\varphi}^{[\ell]}(\pi) = 0.
    \end{align*}
    Finally, we have $\dot{\varphi}^{[0]}(\pi) = 1$ and $\ddot{\varphi}^{[0]}(\pi) = 0$. 
\end{lemma}

\begin{lemma}[{\cite[Lemma C.18]{Buchanan2020-er}}]
  \label{lem:skeleton_iphi_ixi_derivatives}
  For first and second derivatives of $\ixi{\ell}$, one has
  \begin{equation*}
    \ixid{\ell} = -\pi\inv \sum_{\ell' = \ell}^{L-1} \ivphid{\ell'} 
    \prod^{L-1}_{
      \substack{
        \ell'' = \ell \\
        \ell'' \neq \ell'
      }
    } \bigl( 1 - \pi\inv \ivphi{\ell''} \bigr),
  \end{equation*}
  and
  \begin{align*} & 
    \ixidd{\ell} \\ & 
    = \frac{-1}{\pi} \sum_{\ell' = \ell}^{L-1}\left[
      \ivphidd{\ell'} \prod^{L-1}_{
      \substack{
        \ell'' = \ell \\
        \ell'' \neq \ell'
      }
    } 
    \bigl( 1
      - \pi\inv \ivphi{\ell''} \bigr)
      - \pi\inv \ivphid{\ell'} 
      \sum_{\substack{\ell'' = \ell \\ \ell'' \neq \ell'}}^{L-1} \ivphid{\ell''} \prod^{L-1}_{
      \substack{
        \ell''' = \ell \\
        \ell''' \neq \ell', \ell''' \neq \ell''
      }
    } \bigl( 1 - \pi\inv \ivphi{\ell'''} 
      \bigr)
    \right],
  \end{align*}
  where empty sums are interpreted as zero, and empty products as $1$.  In particular, one
  calculates
  \begin{equation*}
    \ixi{\ell}(0) = 1; \qquad
    \ixid{\ell}(0) = -\frac{L - \ell}{\pi}; \qquad
    \ixidd{\ell}(0) 
    = \frac{(L - \ell)(L - \ell - 1)}{\pi^{2}} +\frac{L(L-1) - \ell(\ell - 1)}{3\pi^2},
  \end{equation*}
  and
  \begin{equation*}
    \ixi{0}(\pi) = 0;
    \qquad \ixid{\ell}(\pi) = -\frac{1}{\pi} \ixi{1}(\pi) \indicator{\ell = 0};
    \qquad \ixidd{\ell}(\pi) = 0.
  \end{equation*}
\end{lemma}

\begin{lemma}[{\cite[Lemma C.20]{Buchanan2020-er}}]\label{lem:xi_l_old_ub}
    For all $\ell \in \{0, 1, \dots, L-1\}$, one has
    \begin{equation*}
        \xi_{\ell}(t) \le \frac{1 + \ell t/\pi}{1 + Lt/\pi}
    \end{equation*}
\end{lemma}

\begin{lemma}[{\cite[Lemma C.21]{Buchanan2020-er}}]\label{lem:xi_dot_l_old_ub}%
  One has
  \begin{equation*}
    \abs{\dot{\xi}_{\ell}(t)} \le 3 \frac{L -\ell}{1 + Lt/\pi} .
  \end{equation*}
\end{lemma}

\begin{lemma}[{\cite[Lemma C.23]{Buchanan2020-er}}]
  \label{lem:xiddot_ub}
  There are absolute constants $c, C > 0$ such that for all $\ell \in \set{0, \dots, L-1}$, one
  has
  \begin{equation*}
    \abs*{
      \ixidd{\ell}
    }
    \leq
    C \frac{L(L-\ell)(1 + \ell \nu / \pi)}{(1 + c L \nu)^2}
    + C \frac{(L - \ell)^2}{(1 + cL\nu)(1 + c\ell\nu)}.
  \end{equation*}
\end{lemma}

\end{document}